\documentclass{article} %

\usepackage{etoolbox}
\newcommand{\arxiv}[1]{\iftoggle{colt}{}{#1}}
\newcommand{\colt}[1]{\iftoggle{colt}{#1}{}}
\newtoggle{colt}
\global\togglefalse{colt}

\newcommand{\loose}{\looseness=-1}

\colt{
  
  \usepackage{times}
}

\arxiv{
	\usepackage{natbib}
	\bibliographystyle{plainnat}
	\bibpunct{(}{)}{;}{a}{,}{,}
}

\usepackage[utf8]{inputenc} %
\usepackage[T1]{fontenc}    %

\arxiv{
	\usepackage[letterpaper, left=1in, right=1in, top=1in,
        bottom=1in]{geometry}
        \usepackage{hyperref}
        \hypersetup{colorlinks,citecolor=blue!70!black,linkcolor =
          blue!70!black}       %
        \usepackage{parskip}
      }
\usepackage{url}            %
\usepackage{booktabs}       %
\usepackage{amsfonts}       %
\usepackage{nicefrac}       %
\usepackage{microtype}      %
\arxiv{
  \usepackage[dvipsnames]{xcolor}         %
}
\colt{

  \PassOptionsToPackage{dvipsnames}{xcolor}
  }

\usepackage{xspace}

\usepackage{makecell}
\usepackage{enumitem}
\usepackage{etoolbox}
\usepackage{comment}
\usepackage{wrapfig}
\usepackage{colortbl}
\usepackage{setspace}
\usepackage{transparent}
\usepackage{inconsolata}
\usepackage[scaled=.90]{helvet}
\usepackage{xspace}
\usepackage{verbatim}
\usepackage{mathtools}

\usepackage{algorithm}
\usepackage[noend]{algpseudocode}
\newcommand{\multiline}[1]{\parbox[t]{\dimexpr\linewidth-\algorithmicindent}{#1}}

\usepackage{tablefootnote}

\usepackage{pifont}
\usepackage{amsthm}
\usepackage{mathtools}
\usepackage{amsmath}
\usepackage{bm}
\usepackage{bbm}
\usepackage{amsfonts}
\usepackage{amssymb}
\usepackage{MnSymbol} %
\usepackage[nameinlink,capitalize]{cleveref}

\makeatletter
\newcommand{\neutralize}[1]{\expandafter\let\csname c@#1\endcsname\count@}
\makeatother

\usepackage{thmtools}
\declaretheorem[name=Theorem,parent=section]{theorem}
\declaretheorem[name=Lemma,parent=section]{lemma}
\declaretheorem[name=Corollary,parent=section]{corollary}

\declaretheorem[name=Assumption, parent=section]{assumption}
\declaretheorem[name=Condition, parent=section]{condition}

\declaretheorem[name=Remark,parent=section]{remark}
\declaretheorem[name=Proposition, parent=section]{proposition}

\usepackage{crossreftools}
\newcommand{\creftitle}[1]{\crtcref{#1}}

\makeatletter
\renewenvironment{proof}[1][Proof]%
{%
	\par\noindent{\bfseries\upshape {#1.}\ }%
}%
{\qed\newline}
\makeatother

\theoremstyle{plain}
\newtheorem{definition}[theorem]{Definition}

\usepackage{xpatch}
\xpatchcmd{\proof}{\itshape}{\normalfont\proofnameformat}{}{}
\newcommand{\proofnameformat}{\bfseries}

\newcommand{\pref}[1]{\cref{#1}}
\newcommand{\pfref}[1]{Proof of \pref{#1}}

\renewcommand{\eqref}[1]{\texorpdfstring{\hyperref[#1]{(\ref*{#1})}}{(\ref*{#1})}}
\crefformat{equation}{#2Eq.\,(#1)#3}
\Crefformat{equation}{#2Eq.\,(#1)#3}

\Crefformat{figure}{#2Figure~#1#3}
\Crefformat{assumption}{#2Assumption~#1#3}

\Crefname{assumption}{Assumption}{Assumptions}

\def\ddefloop#1{\ifx\ddefloop#1\else\ddef{#1}\expandafter\ddefloop\fi}
\def\ddef#1{\expandafter\def\csname bb#1\endcsname{\ensuremath{\mathbb{#1}}}}
\ddefloop ABCDEFGHIJKLMNOPQRSTUVWXYZ\ddefloop
\def\ddefloop#1{\ifx\ddefloop#1\else\ddef{#1}\expandafter\ddefloop\fi}
\def\ddef#1{\expandafter\def\csname b#1\endcsname{\ensuremath{\mathbf{#1}}}}
\ddefloop ABCDEFGHIJKLMNOPQRSTUVWXYZ\ddefloop
\def\ddef#1{\expandafter\def\csname sf#1\endcsname{\ensuremath{\mathsf{#1}}}}
\ddefloop ABCDEFGHIJKLMNOPQRSTUVWXYZ\ddefloop
\def\ddef#1{\expandafter\def\csname c#1\endcsname{\ensuremath{\mathcal{#1}}}}
\ddefloop ABCDEFGHIJKLMNOPQRSTUVWXYZ\ddefloop
\def\ddef#1{\expandafter\def\csname h#1\endcsname{\ensuremath{\widehat{#1}}}}
\ddefloop ABCDEFGHIJKLMNOPQRSTUVWXYZ\ddefloop
\def\ddef#1{\expandafter\def\csname hc#1\endcsname{\ensuremath{\widehat{\mathcal{#1}}}}}
\ddefloop ABCDEFGHIJKLMNOPQRSTUVWXYZ\ddefloop
\def\ddef#1{\expandafter\def\csname t#1\endcsname{\ensuremath{\widetilde{#1}}}}
\ddefloop ABCDEFGHIJKLMNOPQRSTUVWXYZ\ddefloop
\def\ddef#1{\expandafter\def\csname tc#1\endcsname{\ensuremath{\widetilde{\mathcal{#1}}}}}
\ddefloop ABCDEFGHIJKLMNOPQRSTUVWXYZ\ddefloop
\def\ddefloop#1{\ifx\ddefloop#1\else\ddef{#1}\expandafter\ddefloop\fi}
\def\ddef#1{\expandafter\def\csname scr#1\endcsname{\ensuremath{\mathscr{#1}}}}
\ddefloop ABCDEFGHIJKLMNOPQRSTUVWXYZ\ddefloop

\newcommand{\ls}{\ell}
\newcommand{\indic}{\mathbbm{I}}    %

\newcommand{\eps}{\epsilon}
\newcommand{\veps}{\varepsilon}

\DeclareMathOperator*{\argmin}{arg\,min} %
\DeclareMathOperator*{\argmax}{arg\,max}

\def\ddef#1{\expandafter\def\csname b#1\endcsname{\ensuremath{\mb{#1}}}}
\ddefloop abcdeghijklmnopqrstuvwxyz\ddefloop

\newcommand{\sss}[1]{{\scriptscriptstyle#1}}
\newcommand{\ind}[1]{^{\sss{(#1)}}}

\DeclarePairedDelimiter{\abs}{\lvert}{\rvert} %
\DeclarePairedDelimiter{\brk}{[}{]}
\DeclarePairedDelimiter{\crl}{\{}{\}}
\DeclarePairedDelimiter{\prn}{(}{)}

\DeclarePairedDelimiter{\ceil}{\lceil}{\rceil}
\DeclarePairedDelimiter{\floor}{\lfloor}{\rfloor}

\let\P\undefined
\DeclareMathOperator{\En}{\mathbb{E}}

\DeclareMathOperator{\P}{P}

\newcommand{\mb}[1]{\boldsymbol{#1}}
\renewcommand{\bm}[1]{\boldsymbol{#1}}
\newcommand{\wt}[1]{\widetilde{#1}}
\newcommand{\wh}[1]{\widehat{#1}}
\newcommand{\wb}[1]{\widebar{#1}}

\let\underbar\undefined

\makeatletter
\let\save@mathaccent\mathaccent
\newcommand*\if@single[3]{%
	\setbox0\hbox{${\mathaccent"0362{#1}}^H$}%
	\setbox2\hbox{${\mathaccent"0362{\kern0pt#1}}^H$}%
	\ifdim\ht0=\ht2 #3\else #2\fi
}
\newcommand*\rel@kern[1]{\kern#1\dimexpr\macc@kerna}
\newcommand*\widebar[1]{\@ifnextchar^{{\wide@bar{#1}{0}}}{\wide@bar{#1}{1}}}
\newcommand*\underbar[1]{\@ifnextchar_{{\under@bar{#1}{0}}}{\under@bar{#1}{1}}}
\newcommand*\wide@bar[2]{\if@single{#1}{\wide@bar@{#1}{#2}{1}}{\wide@bar@{#1}{#2}{2}}}
\newcommand*\under@bar[2]{\if@single{#1}{\under@bar@{#1}{#2}{1}}{\under@bar@{#1}{#2}{2}}}
\newcommand*\wide@bar@[3]{%
	\begingroup
	\def\mathaccent##1##2{%
		\let\mathaccent\save@mathaccent
		\if#32 \let\macc@nucleus\first@char \fi
		\setbox\z@\hbox{$\macc@style{\macc@nucleus}_{}$}%
		\setbox\tw@\hbox{$\macc@style{\macc@nucleus}{}_{}$}%
		\dimen@\wd\tw@
		\advance\dimen@-\wd\z@
		\divide\dimen@ 3
		\@tempdima\wd\tw@
		\advance\@tempdima-\scriptspace
		\divide\@tempdima 10
		\advance\dimen@-\@tempdima
		\ifdim\dimen@>\z@ \dimen@0pt\fi
		\rel@kern{0.6}\kern-\dimen@
		\if#31
		\overline{\rel@kern{-0.6}\kern\dimen@\macc@nucleus\rel@kern{0.4}\kern\dimen@}%
		\advance\dimen@0.4\dimexpr\macc@kerna
		\let\final@kern#2%
		\ifdim\dimen@<\z@ \let\final@kern1\fi
		\if\final@kern1 \kern-\dimen@\fi
		\else
		\overline{\rel@kern{-0.6}\kern\dimen@#1}%
		\fi
	}%
	\macc@depth\@ne
	\let\math@bgroup\@empty \let\math@egroup\macc@set@skewchar
	\mathsurround\z@ \frozen@everymath{\mathgroup\macc@group\relax}%
	\macc@set@skewchar\relax
	\let\mathaccentV\macc@nested@a
	\if#31
	\macc@nested@a\relax111{#1}%
	\else
	\def\gobble@till@marker##1\endmarker{}%
	\futurelet\first@char\gobble@till@marker#1\endmarker
	\ifcat\noexpand\first@char A\else
	\def\first@char{}%
	\fi
	\macc@nested@a\relax111{\first@char}%
	\fi
	\endgroup
}
\newcommand*\under@bar@[3]{%
	\begingroup
	\def\mathaccent##1##2{%
		\let\mathaccent\save@mathaccent
		\if#32 \let\macc@nucleus\first@char \fi
		\setbox\z@\hbox{$\macc@style{\macc@nucleus}_{}$}%
		\setbox\tw@\hbox{$\macc@style{\macc@nucleus}{}_{}$}%
		\dimen@\wd\tw@
		\advance\dimen@-\wd\z@
		\divide\dimen@ 3
		\@tempdima\wd\tw@
		\advance\@tempdima-\scriptspace
		\divide\@tempdima 10
		\advance\dimen@-\@tempdima
		\ifdim\dimen@>\z@ \dimen@0pt\fi
		\rel@kern{0.6}\kern-\dimen@
		\if#31
		\underline{\rel@kern{-0.6}\kern\dimen@\macc@nucleus\rel@kern{0.4}\kern\dimen@}%
		\advance\dimen@0.4\dimexpr\macc@kerna
		\let\final@kern#2%
		\ifdim\dimen@<\z@ \let\final@kern1\fi
		\if\final@kern1 \kern-\dimen@\fi
		\else
		\underline{\rel@kern{-0.6}\kern\dimen@#1}%
		\fi
	}%
	\macc@depth\@ne
	\let\math@bgroup\@empty \let\math@egroup\macc@set@skewchar
	\mathsurround\z@ \frozen@everymath{\mathgroup\macc@group\relax}%
	\macc@set@skewchar\relax
	\let\mathaccentV\macc@nested@a
	\if#31
	\macc@nested@a\relax111{#1}%
	\else
	\def\gobble@till@marker##1\endmarker{}%
	\futurelet\first@char\gobble@till@marker#1\endmarker
	\ifcat\noexpand\first@char A\else
	\def\first@char{}%
	\fi
	\macc@nested@a\relax111{\first@char}%
	\fi
	\endgroup
}
\makeatother

\usepackage{accents}

\newcommand{\w}{\bm{w}}

\newcommand{\forwardexo}{$\learnlevel^{\exo}$.\texttt{bc}\xspace}
\newcommand{\forwardexoeq}{\learnlevel^{\exo}.\texttt{bc}\xspace}

\newcommand{\mainalge}{$\learnlevel^{\exo}$\xspace}

\newcommand{\alghyperref}[1]{\hyperref[#1]{Alg.~\ref*{#1}}}

\newcommand{\setupi}{\textbf{Setup \I}}
\newcommand{\setupii}{\textbf{Setup \II}}

\newcommand{\golf}{\texttt{GOLF}\xspace}

\newcommand{\bzeta}{\bm{\zeta}}

\newcommand{\Ccor}{C_{\texttt{exo}}}
\newcommand{\Ccov}{C_{\texttt{cov}}}
\newcommand{\betastat}{\beta_{\texttt{stat}}}
\newcommand{\Cexo}{\Ccor}
\newcommand{\Tendo}{T^{\texttt{endo}}}
\newcommand{\Texo}{T^{\texttt{exo}}}
\newcommand{\gobs}{g^{\texttt{obs}}}

\newcommand{\phistar}{\phi^{\star}}

\newcommand{\pibar}{\bar{\pi}}

\newcommand{\ldef}{\vcentcolon=}

\renewcommand{\ln}{\log}        %

\newcommand{\tfrak}{\mathfrak{t}}

\newcommand{\pihat}{\wh{\pi}}
\newcommand{\pihatb}{\bm{\pihat}}
\newcommand{\pistar}{\pi^{\star}}

\newcommand{\pibb}{\tilde{\bm{\pi}}}
\let\oldparagraph\paragraph
\renewcommand{\paragraph}[1]{\oldparagraph{#1.}}
\renewcommand{\colon}{:}        %

\newcommand{\reg}{\texttt{reg}}

\newcommand{\reals}{\mathbb{R}}

\renewcommand{\P}{\mathbb{P}}

\newcommand{\E}{\mathbb{E}}

\newcommand{\nn}{\nonumber} 
\newcommand{\ldotst}{%
	\mathinner{{\ldotp}{\ldotp}}%
}

\newcommand{\unifa}{\pi_\texttt{unif}}
\newcommand{\piunif}{\pi_\texttt{unif}}
\newcommand{\unif}{\texttt{unif}}
\newcommand{\stat}{\texttt{stat}}

\renewcommand{\a}{\bm{a}}
\newcommand{\bxi}{\bm{\xi}}
\newcommand{\x}{\bm{x}}

\newcommand{\y}{\bm{y}}

\newcommand{\wtilde}[1]{\widetilde{#1}}

\newcommand{\fhat}{\wh{f}}

\newcommand{\mainalg}{\spanrl}

\newcommand{\Pim}{\Pi_{\texttt{S}}}
\newcommand{\PiS}{\Pi_{\texttt{S}}}

\newcommand{\bpi}{\bm{\pi}}

\renewcommand{\emptyset}{\varnothing}

\newcommand{\algcommentlight}[1]{\textcolor{blue!70!black}{\transparent{0.5}\footnotesize{\texttt{\textbf{//\hspace{2pt}#1}}}}}

\newcommand{\algcommentbiglight}[1]{\textcolor{blue!70!black}{\transparent{0.5}\footnotesize{\texttt{\textbf{/* #1~*/}}}}}

\newcommand{\approxleq}{\lesssim}

\newcommand{\bigoh}{O}
\newcommand{\bigoht}{\wt{O}}

\newcommand{\poly}{\mathrm{poly}}

\newcommand{\tv}[2]{D_{\mathrm{tv}}(#1,#2)}

\newcommand{\mathand}{\quad\text{and}\quad}

\newcommand{\iid}{i.i.d.\xspace}

\newcommand{\filt}{\mathscr{F}}

\newcommand{\test}{\texttt{test}}

\newcommand{\pibell}{\tilde\pi}

\newcommand{\Vhat}{\widehat{V}}
\newcommand{\pib}{\tilde{\pi}}

\newcommand{\forward}{\texttt{BehaviorCloning}}
\newcommand{\imit}{\texttt{bc}}
\newcommand{\exe}{\texttt{bc}}
\newcommand{\rvflF}{\texttt{RVFS.bc}}
\newcommand{\simu}{\texttt{sim}}
\newcommand{\Phat}{\widehat{\bm{\cP}}}
\newcommand{\cVhat}{\widehat{\cV}}

\newcommand{\Pims}{\Pi_\texttt{S}}
\newcommand{\bQtilde}{\wtilde{\bm{Q}}}
\newcommand{\bQhat}{\widehat{\bm{Q}}}

\newcommand{\exo}{\texttt{exo}}
\newcommand{\Nest}{N_\texttt{est}}

\newcommand{\framework}{RLLS\xspace}
\newcommand{\ggolf}{\texttt{SimGolf}\xspace}
\renewcommand{\mainalg}{\texttt{RVFS}\xspace}
\newcommand{\Cpush}{C_{\texttt{push}}}
\newcommand{\Qstar}{Q^\star}
\newcommand{\Vstar}{V^\star}
\newcommand{\Qhat}{\widehat{Q}}

	\newcommand{\learnlevel}{\texttt{RVFS}\xspace}
\newcommand{\learnlevelh}[1][h]{$\texttt{RVFS}_{#1}$\xspace}

\newcommand{\boost}{\texttt{boost}}
\newcommand{\rel}{\texttt{real}}

\Crefformat{line}{#2Line #1#3}

\makeatletter
\let\OldStatex\Statex
\renewcommand{\Statex}[1][3]{%
  \setlength\@tempdima{\algorithmicindent}%
  \OldStatex\hskip\dimexpr#1\@tempdima\relax}
\makeatother

\newcommand{\paragraphi}[1]{\par\noindent\emph{#1.}}

\newcommand{\fakepar}[1]{\paragraph{#1}}

\usepackage[suppress]{color-edits}

\addauthor{df}{ForestGreen}
\addauthor{zm}{red}
\addauthor{sr}{magenta}
\usepackage{autonum}

\colt{
    \title[The Power of Resets in Online Reinforcement Learning]{The Power of Resets in Online Reinforcement
    Learning}

  \coltauthor{%
 \Name{Author Name1} \Email{abc@sample.com}\\
 \addr Address 1
 \AND
 \Name{Author Name2} \Email{xyz@sample.com}\\
 \addr Address 2%
}
  }

\arxiv{
  \title{The Power of Resets in Online Reinforcement Learning}
  \author{Zakaria Mhammedi\\{\small \texttt{mhammedi@google.com}} \and    Dylan J. Foster\\{\small \texttt{dylanfoster@microsoft.com}} \and Alexander Rakhlin\\{\small \texttt{rakhlin@mit.edu}}
}
\date{}
}

  \addtocontents{toc}{\protect\setcounter{tocdepth}{0}}

\begin{document}
	
	\maketitle
	
	\begin{abstract}
Simulators are a pervasive tool in reinforcement learning, but most existing algorithms cannot efficiently exploit simulator access---particularly in high-dimensional domains that require general function approximation. We explore the power of simulators through online reinforcement learning with \emph{local simulator access} (or, local planning), an RL protocol where the agent is allowed to \emph{reset} to previously observed states and follow their dynamics during training. We use local simulator access to unlock new statistical guarantees that were previously out of reach:\loose
\colt{\begin{enumerate}[leftmargin=3.5em,rightmargin=3em,topsep=2pt,itemsep=2pt]}
  \arxiv{\begin{enumerate}}
\item We show that MDPs with low
  \emph{coverability} \citep{xie2023role}---a general structural condition that subsumes
  Block MDPs and Low-Rank MDPs---can be learned in a sample-efficient
  fashion with \emph{only $Q^{\star}$-realizability} (realizability
  of the optimal state-value function); existing online RL algorithms
  require significantly stronger representation conditions.
\item As a consequence, we show that the notorious \emph{Exogenous Block MDP} problem \citep{efroni2022provably} is tractable under local simulator access.
\end{enumerate}
The results above are achieved through a computationally inefficient
algorithm. We complement them with a more computationally efficient algorithm, \mainalg (\emph{Recursive Value Function Search}), which achieves provable sample complexity guarantees under a strengthened statistical assumption known as \emph{pushforward coverability}. \mainalg can be viewed as a principled, provable counterpart to a successful empirical paradigm that combines recursive search (e.g., MCTS) with value function approximation.
 	\end{abstract}

        \colt{
          \begin{keywords}%
            Reinforcement learning, learning theory, generative model,
            simulator, coverability.%
          \end{keywords}
        }

	\section{Introduction}
	\label{sec:intro}

Simulators are a widely used tool in reinforcement learning. Many of the most
well-known benchmarks for reinforcement learning research make use of
simulators (Atari \citep{bellemare2012investigating}, MuJoCo
\citep{todorov2012mujoco}, OpenAI Gym \citep{brockman2016openai}, 
DeepMind Control Suite \citep{tassa2018deepmind}), and high-quality
simulators are available for a wide range of real-world control tasks,
including robotic control
\citep{qassem2010modeling,akkaya2019solving}, autonomous vehicles
\citep{bojarski2016end,aradi2020survey}, and game playing \citep{silver2016mastering,silver2018general}. Simulators also provide a useful abstraction for \emph{planning} with a known or learned model, an important building block for
many RL techniques
\citep{schrittwieser2020mastering}. Yet, in spite of the ubiquity of simulators, almost all existing research into algorithm
design---empirical and theoretical---has focused on the \emph{online reinforcement learning} (where
only trajectory-based feedback is available), and does not take advantage of the extra
information available through the simulator. Relatively
little is known about the full power of RL with
simulator access, either in terms of algorithmic principles or
fundamental limits.\loose

We explore the power of simulators through online reinforcement learning with \emph{local
  simulator access} (\framework for short), also known as \emph{local
  planning}
\citep{weisz2021query,li2021sample,amortila2022few,weisz2022confident,yin2022efficient,yin2023sample}. Here,
the agent learns by repeatedly executing policies
and observing the resulting trajectories (as in online RL), but is
allowed to \emph{reset} to previously observed states and
follow their dynamics during training.

Empirically,
algorithms based on local simulators have received limited
investigation, but with promising results. Notably, the Go-Explore
algorithm \citep{ecoffet2019go,ecoffet2021first} uses local simulator access to achieve state-of-the-art performance for Montezuma's
Revenge (a difficult Atari game that requires systematic
exploration), beating the performance of the best agents
trained with online RL
\citep{badia2020agent57,guo2022byol} by a significant margin that has
yet to be closed. The successful line of research on AlphaGo and
successors
\citep{silver2016mastering,silver2018general,schrittwieser2020mastering}
also uses local simulator access, albeit at test time in addition to
training time.

              These results suggest that developing improved algorithm design
              principles for RL with local simulator access could have
              significant practical implications, but current theoretical understanding
              of local simulators is limited. Recent work has
      shown that local simulator access has provable benefits for
      reinforcement learning with various types of linear function
      approximation
      \citep{weisz2021query,li2021sample,amortila2022few,yin2022efficient,weisz2022confident},
      but essentially nothing is known for RL problems in large state
      spaces that demand general, potential neural function
      approximation. This leads us to ask: \colt{                  \emph{Can we develop algorithms for reinforcement
                    learning with general function approximation that
                    provably benefit from local simulator access?}}
      \arxiv{              \begin{quote}
                \begin{center}
                  \emph{Can we develop algorithms for reinforcement
                    learning with general function approximation that
                    provably benefit from local simulator access?}
                \end{center}
              \end{quote}}

              From an algorithm design perspective, perhaps the greatest challenge in using local simulators to speed up learning is to understand which
              states are ``informative'' in the sense that we should
              prioritize revisiting them.
              Here, we are faced with a chicken-and-egg problem: to
              understand which states to prioritize, we must
              explore and gather information, but it is unclear how to
              do so efficiently unless we already have a way to
              understand which states are informative. It is
              natural to let function approximation\arxiv{ (valued-based,
              model-based, or otherwise)} guide us; to this end, recent research
              \citep{li2021sample,yin2022efficient,weisz2022confident}
              on linearly-parameterized RL with local simulators
              makes use of \emph{core-sets}: small, adaptively chosen
              sets of informative state-action pairs designed to cover the
              feature space and enable efficient value function
              learning. Core-sets facilitate sample complexity
              guarantees for linear models that are not possible
              without local simulator access (e.g., \citet{li2021sample}). Yet, for general function classes---particularly
              rich models like neural networks that do not readily
              support extrapolation---defining a suitable notion of
              core-set is challenging. Consequently, existing
              techniques have yet to meaningfully leverage local simulator access beyond the linear regime.

\subsection{Contributions}

We show that local simulator access unlocks new guarantees for online
\arxiv{reinforcement learning}\colt{RL} with general value function approximation---statistical and computational---that were
previously out of reach.\loose

\fakepar{Sample-efficient learning} We show that MDPs with low
  \emph{coverability} \citep{xie2023role}---a general structural condition that subsumes
  Block MDPs and Low-Rank MDPs---can be learned in a sample-efficient
  fashion with \emph{only $Q^{\star}$-realizability} (that is, realizability
  for the optimal state-action value function). This is achieved through a
  new algorithm, \ggolf, that augments the principle of global
  optimism with local simulator access, and improves
  upon the best existing guarantees for the fully online RL setting,
  which require significantly stronger representation conditions. As a consequence, we show for the first time that the notoriously
  challenging \emph{Exogenous Block MDP} (ExBMDP) problem
  \citep{efroni2022provably,efroni2022sample} is tractable in its most
  general form under
  local simulator access.\loose

\fakepar{Practical, computationally efficient learning} Our results above are achieved
  through a computationally inefficient algorithm. We complement them
  with a practical and computationally efficient algorithm, \mainalg (``Recursive Value
  Function Search''), which achieves sample-efficient learning
  guarantees with general value function approximation under a
  strengthened, yet novel, statistical assumption known as
  \emph{pushforward coverability}
  \arxiv{\citep{xie2021batch,amortila2024scalable}}\colt{\citep{xie2021batch}}. Assuming either i)
  realizability of the optimal state-value function $V^{\star}$ and a state-action gap or ii)
  realizability of $V^{\pi}$ for all $\pi$, \mainalg achieves
  polynomial sample complexity in a computationally efficient fashion,
  and leads to guarantees for a new class of
  Exogenous Block MDPs with \emph{weakly correlated} exogenous noise.   \mainalg explores by building core-sets with a novel
  value function-guided scheme, and can be viewed as a principled
  counterpart to algorithms including MCTS and AlphaZero
  \citep{silver2016mastering,silver2018general,ecoffet2019go,ecoffet2021first,yin2023sample},
  that combine recursive search with value function
  approximation. Compared to these approaches, \mainalg is designed to
  provably address stochastic environments and distribution shift.\loose

\fakepar{Paper organization}
\cref{sec:setting} \arxiv{formally }introduces\arxiv{ our problem setting,
  including} the local simulator framework. \cref{sec:sample}
presents our main sample complexity guarantees, and
               \cref{sec:computation} gives computationally
               efficient algorithms. We close with discussion in
               \cref{sec:discussion}. All proofs are deferred to the appendix.

\colt{\vspace{-14pt}}

	\section{Setup: Reinforcement Learning with Local Simulator Access}
	\label{sec:setting}

We consider an episodic reinforcement learning setting. A Markov
Decision Process (MDP) is a tuple $\cM = ( \cX, \cA, T, R, H)$, where
$\cX$ is a (large/potentially infinite) state space, $\cA$ is the
action space (we abbreviate $A=\abs{\cA}$), $H \in \bbN$ is the horizon, $R = \{R_h\}_{h=1}^H$ is
the reward function (where $R_h : \cX \times \cA \rightarrow [0,1]$)
and $T = \{T_h\}_{h=0}^{H}$ is the transition distribution (where $T_h
: \cX \times \cA \rightarrow \Delta(\cX)$), with the convention that
$T_0(\cdot\mid\emptyset)$ is the initial state distribution. A
policy is a sequence of functions $\pi = \{ \pi_h: \cX
\rightarrow \Delta(\cA)\}_{h=1}^H$; we use $\Pim$ to denote the set of
all such functions. When a policy is executed, it
generates a trajectory $(\bx_1,\ba_1,\br_1), \dots, (\bx_H, \ba_H,
\br_h)$ via the process $\ba_h \sim \pi_h(\bx_h), \br_h \sim
R_h(\bx_h,\ba_h), \bx_{h+1} \sim T_h(\cdot\mid\bx_h,\ba_h)$,
initialized from $\bx_1 \sim T_0(\cdot\mid\emptyset)$ (we use
$\bx_{H+1}$ to denote a terminal state with zero
reward). We write $\bbP^{\pi}\brk*{\cdot}$ and $\En^{\pi}\brk*{\cdot}$
to denote the law and expectation under this process. \loose

For a policy \(\pi\), $J(\pi) \coloneqq \bbE^\pi\brk[\big]{
  \sum_{h=1}^H \br_h}$ denotes expected reward, and the value functions are given by \colt{$V_h^\pi(x) \coloneqq \bbE^\pi\brk[\big]{\sum_{h'=h}^H \br_{h'} \mid \bx_h=x}$, and $Q_h^\pi(x,a) \coloneqq \bbE^\pi\brk[\big]{\sum_{h'=h}^H \br_{h'} \mid \bx_h=x,\ba_h=a}$.}
\arxiv{
\[
V_h^\pi(x) \coloneqq \bbE^\pi\brk*{\sum_{h'=h}^H \br_{h'} \mid \br_h=x}, \quad \text{ and } \quad Q_h^\pi(x,a) \coloneqq \bbE^\pi\brk*{\sum_{h'=h}^H \br_{h'} \mid \bx_h=x,\ba_h=a}.
\]}
We denote by $\pi^\star\arxiv{ = \{\pi^\star_h\}_{h=1}^H}$ the optimal
deterministic policy that maximizes $Q^{\pistar}$\arxiv{ at all states}, and
write $\Qstar\ldef{}Q^{\pistar}$ and $\Vstar\ldef{}V^{\pistar}$.

\arxiv{\subsection{Online Reinforcement Learning with Local
    Simulator Access}}
\colt{\fakepar{Online reinforcement learning with a local simulator}}
In the standard online reinforcement learning framework, the learner
repeatedly interacts with an (unknown) MDP by executing a policy and
observing the resulting trajectory, with the goal of maximizing the
total reward. Formally, for each episode $\tau\in\brk{N_\texttt{episodes}}$, the learner
selects a policy $\pi\ind{\tau} = \{\pi_h\ind{\tau}\}_{h=1}^H$, executes it
in the underlying MDP $\cM^\star$ and observes the trajectory
$\{(\bx_h\ind{\tau},\ba_h\ind{\tau},\br_h\ind{\tau})\}_{h=1}^H$. After all $N_\texttt{episodes}$
episodes conclude, the learner produces a policy $\pihat\in\Pim$ with
the goal of minimizing the risk given by $\En\brk*{J(\pistar) -  J(\pihat)}$.
In online RL with local simulator access, or \framework,  \citep{weisz2021query,li2021sample,yin2022efficient,weisz2022confident,yin2023sample},
we augment the online RL protocol as follows:\footnote{We use the term
  ``local simulator'' instead of ``local planning'' because we find it
  to be slightly more descriptive.} At each episode
$\tau\in\brk{N}$, instead of starting from a random initial state $\bx_1\sim{}T_0(\cdot\mid{}\emptyset)$, the
agent can \emph{reset} the MDP to any layer $h\in\brk{H}$ and any state
$\bx_h$ previously encountered, and proceed with a new episode starting
from this point. As in the online RL protocol, the goal is to produce a policy $\pihat\in\Pim$ such that \colt{$\En\brk*{J(\pistar) -  J(\pihat)} \leq \veps$}
\arxiv{\[
\En\brk*{J(\pistar) -  J(\pihat)} \leq \veps
\]}
with as few episodes of interaction as possible; our main results take
$N_\texttt{episodes}=\poly(C,\veps^{-1})$ for a suitable problem parameter $C$.
\colt{\framework acts as a happy medium between online RL, which is
  realistic but often intractable, and global simulator access (where the
  agent can query transitions for \emph{arbitrary}
state-action pairs) \citep{kearns1998finite,kakade2003sample,du2019good,yang2020reinforcement,lattimore2020learning}, which is
powerful, yet unrealistic.
}

\fakepar{Executable versus non-executable policies}
We focus on learning policies that can be executed without access to a local simulator (in other words, the
            local simulator used at train time, but not test
            time). Some recent work using local simulators for
            RL with linear function approximation
            \citep{weisz2021query} considers a more permissive setting
            where the final policy $\pi$ produced by the learner can be
            \emph{non-executable}; our function approximation requirements can be
            slightly relaxed in this case.\loose
            \begin{definition}[Non-executable policy]
  \label{def:executable}
We refer to a policy $\pi$ for
which computing $\pi(x)\in\Delta(\cA)$ for any $x\in\cX$ requires $n$ local
simulator queries as a \emph{non-executable policy with
  sample complexity $n$}.\loose
\end{definition}

\arxiv{
  \paragraph{Implications for planning}
    RL with local simulator access is a convenient abstraction for the
problem of \emph{planning}: Given a known (e.g., learned) model,
compute an optimal policy. Planning with a learned model is an important task in theory
\citep{foster2021statistical,liu2023optimistic} and practice (e.g., \citet{schrittwieser2020mastering}).
Since the model is known, computing an
optimal policy is a purely computational problem, not a statistical
problem. Nonetheless, for planning problems in large state spaces, where
enumerating over all states is undesirable, algorithms for online RL
with local simulator access can be directly applied, using the model
to simulate the environment the agent interacts with. Here, any
computationally efficient algorithm in our framework immediately yields an efficient
algorithm for planning. See \cref{sec:additional_related} for more discussion.\loose
}

          \arxiv{          \subsection{Additional Notation}}
          \colt{\fakepar{Additional notation}}
                    For any $m,n \in\mathbb{N}$, we denote by $[m\ldotst{}n]$
 the integer interval $\{m,\dots, n\}$. We also let
 $[n]\coloneqq [1\ldotst{}n]$. We refer to a scalar $c>0$ as an \emph{absolute constant} to indicate
 that it is independent of all problem parameters and use
 $\bigoht(\cdot)$ to denote a bound up to factors polylogarithmic in
 parameters appearing in the expression. We define $\unifa\in\Pim$ as the random policy that
 selects actions in $\cA$ uniformly\arxiv{ at random at each layer}. We define
 the occupancy measure for policy $\pi$ via $d_h^\pi(x,a) \coloneqq
 \bbP^\pi[\bx_h = x, \ba_h=a]$. For functions $g:\cX\times \cA\to\bbR$
 and $f:\cX\to\bbR$, we define Bellman backup operators by
 $\cT_h\brk{g}(x,a)=\En\brk*{\br_h+ \max_{a'\in
     \cA}g(\x_{h+1},a')\mid{}\x_h=x,\a_h=a}$ and
 $\cP_h\brk{f}(x,a)=\En\brk*{\br_h + f
   (\x_{h+1})\mid{}\x_h=x,\a_h=a}$. \dfedit{For a stochastic policy
   $\pi\in\Pim$, we will occasionally use the bold notation
   $\bm{\pi}_h(x)$ as shorthand for the random variable $\a_h\sim{}\pi_h(x)\in \Delta(\cA)$. 
   } For a function $f\colon \cA \rightarrow \reals$, we write $a'\in \argmax_{a\in \cA} f(a)$ to denote the action that maximizes $f$. If there are ties, we break them by picking the action with the smallest index; we assume without loss of generality that actions in $\cA$ are index from $1,\dots,|\cA|$.

        \section{New Sample-Efficient Learning Guarantees via Local
          Simulators}
        \label{sec:sample}

        This section presents our most powerful results for
        \framework. We
        present a new algorithm for learning with local simulator
        access, \ggolf{} (\cref{sec:golf}), and show that it enables sample-efficient RL for MDPs with low
        \emph{coverability} \citep{xie2023role} using only
        $\Qstar$-realizability (\cref{sec:mainresultgolf}). We then
        give implications for the Exogenous Block MDP problem (\cref{sec:exbmdp}).\loose

  \fakepar{Function approximation setup and coverability} To achieve
  sample complexity guarantees for online reinforcement learning that are
  suitable for large, high-dimensional state spaces, we appeal to \emph{value function
    approximation}.
  We assume access to a function class $\cQ
  \subset(\cX\times \cA\times\brk{H}\to\brk{0,H})$ that contains the optimal
  state-action value function $Q^{\star}$; we define
  $\cQ_h=\crl*{Q_h\mid{}Q\in\cQ}$.
  
\begin{assumption}[$Q^{\star}$-realizability]
	\label{ass:realgolf}
	For all $h\in [H]$, we have $Q^{\star}_h\in \cQ_h$.
      \end{assumption}
$\Qstar$-realizability is widely viewed as a minimal representation condition for
online RL \citep{wen2017efficient,du2019good,du2019provably,lattimore2020learning,weisz2021exponential,wang2021exponential}.
The class $\cQ$ encodes the learner's prior knowledge about the
MDP, and can be parameterized by rich function approximators like neural networks. We assume for simplicity of exposition that $\cQ$ and $\Pi$ are finite, and aim for sample complexity guarantees scaling with
$\log\abs{\cQ}$ and $\log\abs{\Pi}$; extending our results to infinite
classes via standard uniform convergence arguments is straightforward.\loose

\fakepar{Coverability}
Beyond representation conditions like realizability, online
\arxiv{reinforcement learning}\colt{RL} algorithms require \emph{structural conditions}
that limit the extent to which deliberately designed algorithms can be
surprised by substantially new state distributions. We focus on a
structural condition known as \emph{coverability}
\citep{xie2023role}, which is inspired by
connections between online and offline \arxiv{reinforcement learning}\colt{RL}.
\colt{
\begin{assumption}
	\label{ass:cover}
	The \emph{coverability} coefficient is 
        $C_{\texttt{cov}} \coloneqq \max_{h\in\brk{H}}\inf_{\mu_h\in \Delta(\cX \times \cA)} \sup_{\pi \in \Pims} \left\|\frac{d_h^\pi}{\mu_h}\right\|_{\infty}$.
      \end{assumption}
      }
\arxiv{
\begin{assumption}[Coverability \citep{xie2023role}]
	\label{ass:cover}
	The coverability coefficient $C_{\texttt{cov}}>0$ is given by 
	\begin{align}
          \label{eq:cover}
		C_{\texttt{cov}} \coloneqq \max\limits_{h\in\brk{H}}\inf\limits_{\mu_h\in \Delta(\cX \times \cA)} \sup\limits_{\pi \in \Pims} \left\|\frac{d_h^\pi}{\mu_h}\right\|_{\infty}.
	\end{align}
      \end{assumption}
      }
Coverability is an intrinsic strutural property of the underlying
MDP. Examples of MDP families with low coverability include
(Exogenous) Block MDPs, which have $\Ccov\leq{}\abs*{\cS}\abs*{\cA}$,
where $\cS$ is the \emph{latent state space} \citep{xie2023role}, and Low-Rank MDPs, which
have $\Ccov\leq{}d\abs*{\cA}$, where $d$ is the feature dimension
\citep{huang2023reinforcement}; importantly, these settings
exhibit high-dimensional state spaces and require nonlinear
function approximation. \arxiv{To the best of our knowledge,
coverability is the only structural parameter that is known to be
small for the Exogenous Block MDP problem, as other well-known
parameters like Bellman rank and Bellman-Eluder dimension can be
arbitrarily large \citep{xie2023role}. }\arxiv{We emphasize that as}\colt{As} in prior work\arxiv{
\citep{xie2023role,amortila2023harnessing}}, our algorithms require
\emph{no prior knowledge} of the distribution $\mu_h$ \arxiv{that
  minimizes \cref{eq:cover}.}\colt{that achieves the minimum in \cref{ass:cover}.}

        \subsection{Algorithm}
\label{sec:golf}

Our main algorithm, \ggolf, is displayed in
\cref{alg:generative_golf}. The algorithm is a variant of the \golf
method of \citet{jin2021bellman,xie2023role} with novel adaptations to
exploit the availability of a local simulator. Like \golf, \ggolf
explores using the principle of \emph{global optimisim}: At each
iteration $t\in\brk{N}$, it maintains
a confidence set (or, version space) $\cQ\ind{t}\subset\cQ$ of
candidate value functions with low squared Bellman error under the
data collected so far, and chooses a new exploration policy $\pi\ind{t}$ by picking the most ``optimistic''
value function in this set. As the algorithm gathers more data, the confidence set shrinks, leaving only near-optimal policies.

The main novelty in \ggolf arises in the data
collection strategy and design of confidence sets. Like \golf,
\ggolf algorithm constructs the confidence set
$\cQ\ind{t}\subset\cQ$ such that all value functions $g\in\cQ\ind{t}$
have small squared Bellman error:\colt{\footnote{\citet{xie2023role} show
  that squared (versus average) Bellman error is essential to
  derive guarantees for coverability.}}
\begin{align}
  \label{eq:golf_conf}
  \sum_{i<t}\En^{\pi\ind{i}}\brk*{\prn*{g_h(\x_h, \a_h) - \cT_{h}[g_{h+1}](\x_h,\a_h)}^2}
  \approxleq \log\abs{\cQ}\quad\forall{}h\in\brk{H}.
\end{align}
Due to the presence of the Bellman backup $\cT_h\brk{g_{h+1}}$ in
\cref{eq:golf_conf}, naively estimating squared Bellman error leads to the
notorious \emph{double sampling} problem. To avoid this, the approach
taken with \golf and related work \citep{zanette2020learning,jin2021bellman} is to adapt a certain de-biasing
technique to remove double sampling bias, but this requires access to
a value function class that satisfies \emph{Bellman completeness}, a representation
significantly more restrictive than  realizability (e.g., \citet{foster2022offline}).\loose

The idea behind \ggolf is to use local simulator access to
directly produce high-quality estimates for the Bellman backup
function $\cT_h\brk{g_{h+1}}$ in \cref{eq:golf_conf}. In particular, for a given state-action pair
$(x,a)\in\cX\times\cA$, we can estimate the Bellman backup
$\cT_{h}[g_{h+1}](x,a)$ for all functions $g\in\cQ$ simultaneously by
collecting $K$ next-state transitions $\wt{\x}_{h+1}\ind{1},\dots,
 \wt{\x}_{h+1}\ind{K}\stackrel{\text{i.i.d.}}{\sim}T_h(\cdot \mid
 x,a)$
 and
 $K$ rewards $\wt{\br}_{h}\ind{1},\dots,
 \wt{\br}_{h}\ind{K}\stackrel{\text{i.i.d.}}{\sim}R_h(x,a)$, then taking
 the empirical mean: $\cT_{h}[g_{h+1}](x,a)\approx\frac{1}{K}\sum_{k=1}^K
 (\wt{\br}_{h}\ind{k}+\max_{a'\in \cA}g_{h+1}(\wt{\x}_{h+1}\ind{k}
 ,a'))$. \cref{line:simu} of \ggolf uses this technique to directly
 estimate the Bellman residual backup under a trajectory gathered with
 $\pi\ind{t}$, sidestepping the double sampling problem and removing the need for Bellman
completeness. \colt{We suspect this technique (estimation
      with respect to squared Bellman error using local
      simulator access) may find broader use.}\arxiv{We suspect that the idea of performing estimation
      with respect to squared Bellman error directly using local
      simulator access may find broader use.}
      \loose

      \arxiv{
  \begin{remark}[Squared Bellman error versus average Bellman error]
   A general approach to removing the need for Bellman completeness is
   to estimate \emph{average Bellman error} instead of squared Bellman
   error (e.g., \citet{jiang2017contextual}). However, \citet{xie2023role} show that this is insufficient
   to obtain sample complexity guarantees under coverability.
 \end{remark}
 }

\begin{algorithm}[tp]
  \caption{\ggolf: Global Optimism via Local Simulator Access}
	\label{alg:generative_golf}
	\begin{algorithmic}[1]
          \State {\bfseries input:} Value function class
	$\cQ$, coverability $\Ccov>0$, suboptimality \arxiv{parameter
        }$\veps>0$, and confidence \arxiv{parameter }$\delta>0$. 
          \State Set $N\gets\wt{\Theta}(H^2\Ccov\beta/\veps^2)$, $\betastat \gets 16\log(2HN\abs*{\cQ}\delta^{-1})$, $\beta
          \gets 2 \beta_\stat$, and $K\gets\frac{8N}{\beta_\stat}$. 
          \State {\bfseries initialize:} $\cQ\ind{1} \leftarrow \cQ$.%
		\For{iteration $t = 1,2,\dotsc,N$}
		\State Select $g\ind{t} = \argmax_{g \in \cQ\ind{t}}
                \sum_{s < t} \max_{a\in \cA}g_1(\x\ind{s}_1, a)$. 
		\State For each $h\in[H]$ and $x\in \cX$, define $\pi\ind{t}_h(x) \in \argmax_{a\in \cA} g_h\ind{t}(x,a)$. \label{line:localsim} 
		\State Execute $\pi\ind{t}$ for an episode and observe
                $\bm{\tau}\ind{t} \ldef (\x_{1}\ind{t},\a_{1}\ind{t}),\dots,(\x_{H}\ind{t},\a_{H}\ind{t})$. \label{line:tau}
		\For{$h\in\brk{H}$} \label{label:condexpt}
		\State Draw $K$ independent samples $\x\ind{t,k}_{h+1}\sim
		T_h(\cdot\mid{}\x_h\ind{t},\a_h\ind{t})$,
		$\br\ind{t,k}_h\sim{}R_h(\x_h\ind{t}, \a_h\ind{t})$.\label{line:simu} %
		\EndFor
		\State Compute confidence set: \label{line:cond}
                \begin{small}
                  \begin{align}
                    \label{eq:def_vspace_rf}
                    \cQ\ind{t+1} \leftarrow \left\{ g \in \cQ:
                    \sum_{h\leq{}t}\prn[\bigg]{
                    g_h(\x_h\ind{t}, \a_h\ind{t})
                    - \frac{1}{K}\sum_{k=1}^{K}\prn*{\br_h\ind{t,k}+\max_{a\in \cA} g_{h+1}(\x_{h+1}\ind{t,k},a)}
                    }^2
                    \leq \beta, ~\forall h \in [H] \right\}.
                  \end{align}
                \end{small}
		\EndFor
		\State \textbf{return:} $\pihat=\unif(\pi\ind{1},\ldots,\pi\ind{T})$.
	\end{algorithmic}
\end{algorithm}

\subsection{Main Result}
\label{sec:mainresultgolf}
We now state the main guarantee for \ggolf{} and discuss some of its implications.
    \begin{theorem}[Main guarantee for \ggolf]
   	\label{thm:golf}
   	Let $\veps, \delta \in(0,1)$ be given and suppose 
        \cref{ass:realgolf} ($Q^\star$-realizability) and \cref{ass:cover}
        (coverability) hold with $\Ccov>0$. Then the policy $\pihat$
        produced by $\ggolf(\cQ, \Ccov,\veps, \delta)$ (\cref{alg:generative_golf}) has
        $
   	J(\pistar)-\En\brk*{J(\pihat)} \leq \veps$
        with probability at least $1-\delta$. The total sample
        complexity in the \framework framework is bounded by \colt{$       \bigoht\prn*{\frac{H^5\Ccov^2\log(\abs{\cQ}/\delta)}{\veps^4}}$.}
        \arxiv{
        \begin{align}
        \bigoht\prn*{\frac{H^5\Ccov^2\log(\abs{\cQ}/\delta)}{\veps^4}}.
        \end{align}
        }

   \end{theorem}
        This result (whose proof is in \cref{proof:golf}) shows that under only $Q^\star$-realizability and
        coverability, \ggolf{} learns an $\veps$-optimal policy with
        polynomial sample complexity, significantly relaxing the
        representation assumptions (Bellman completeness, weight
        function realizability) required by prior algorithms for
        coverability \citep{xie2023role,amortila2023harnessing}. This is the first instance we are aware of where
        local simulator access unlocks sample complexity guarantees for
        reinforcement learning with \emph{nonlinear} function
        approximation that were previously out of reach; perhaps
        the most important technical idea here is our approach to combining global
        optimism with local simulator access, in
        contrast to greedy layer-by-layer schemes used in prior work
        on local simulators (with the exception of
        \citet{weisz2021query}). In
      particular, we suspect that the idea of performing estimation
      with respect to squared Bellman error directly using local
      simulator access may find broader use beyond coverability.
        Improving the
        polynomial dependence on problem parameters is an
        interesting question for future work.

        \fakepar{A conjecture}
        By analogy to results in offline
        reinforcement learning, where $\Qstar$-realizability and
        concentrability (the offline counterpart to coverability)
        alone are
        known to be insufficient for sample-efficient learning
        \citep{chen2019information,foster2022offline}, we conjecture
        that $Q^{\star}$-realizability and coverability alone
        are not sufficient for polynomial sample complexity in vanilla online RL. If true, this would imply a new
        separation between online RL with and without
        local simulators.\loose

      \arxiv{
\paragraph{Proof sketch}
  As described in \cref{sec:golf}, the main difference between
  \ggolf{} and \golf lies in the construction of the confidence
  sets. The most important new step in the proof of
  \cref{thm:golf} is to show that the local simulator-based confidence
  set construction in \cref{line:cond} is valid in the sense that the
  property \cref{eq:golf_conf} holds with high probability. From here,
  the sample complexity bound follows by adapting the
  change-of-measure argument based on coverability from
  \citet{xie2023role}.
  }

        \subsection{Implications for Exogenous Block MDPs}
\label{sec:exbmdp}
We now apply \ggolf and \cref{thm:golf} to the \emph{Exogenous Block
  MDP} (ExBMDP) problem
\citep{efroni2022provably,efroni2022sample,lamb2023guaranteed,islam2023agent},
a challenging rich-observation reinforcement learning setting in which
the observed states $\bx_h$ are high-dimensional, while the underlying
dynamics of the system are low-dimensional, yet confounded by
temporally correlated exogenous noise.

     Formally, an Exogenous Block
        MDP $\cM=(\cX,\cS,\Xi,\cA,H,T,R,g)$ is defined by a \emph{latent
        	state space} and an \emph{observation space}. We begin \arxiv{by
        describing }\colt{with }the latent state space. Starting from an initial
        \emph{endogenous state} $\bs_1\in\cS$ and \emph{exogenous state}
        $\bxi_1\in\Xi$, the latent state $\bz_h=(\bs_h,\bxi_h)$ evolves for
        $h\in\brk{H}$ via\loose
        \begin{align}
        	\bs_{h+1}\sim{}\Tendo_h(\cdot\mid{}\bs_h,\ba_h),\mathand\bxi_{h+1}\sim \Texo_h(\cdot\mid{}\xi_h),
        \end{align}
        where $\ba_h\in\cA$ is the agent's action at layer $h$; we adopt the convention
        that $\bs_1\sim{}\Tendo_0(\cdot\mid{}\emptyset)$ and
        $\bxi_1\sim{}\Texo_0(\cdot\mid\emptyset)$. Note that only the
        endogenous state is causally influenced by the action. The latent
        state is not observed; instead, at each step $h$, the agent receives
        an \emph{observation} $\bx_h\in\cX$ generated via\footnote{        	A more standard formulation\arxiv{ of the observation process}
        	\citep{efroni2022provably,efroni2022sample,lamb2023guaranteed,islam2023agent}
        	assumes that observations are generated via
        	$\bx_h\sim{}q_h(\bs_h,\bxi_h)$, where $q_h(\cdot,\cdot)$ is a
        	conditional distribution with the decodability property. This is
        	equivalent to \cref{eq:exbmdp_observation} \arxiv{under
        	mild measure-theoretic conditions}, as randomness in the
                emission process can be included in the exogenous
                state \arxiv{without loss of generality}\colt{w.l.o.g.}\arxiv{, but the formulation in
        	\cref{eq:exbmdp_observation} makes our proofs slightly more compact.}}
        \begin{align}
        	\label{eq:exbmdp_observation}
        	\x_h = \gobs_h(\bs_h,\bxi_h),	
        \end{align}
        where $\gobs_h:\cS \times \Xi  \rightarrow \cX$ is the
        \emph{emission function}. We assume the endogenous latent space $\cS$ and action space $\cA$ are
        finite, and define $S\ldef{}\abs{\cS}$ and
        $A\ldef{}\abs{\cA}$. However, the exogenous state space $\Xi$ and
        observation space $\cX$ may be arbitrarily large or infinite, with
        $\abs*{\Xi},\abs{\cX}\gg\abs{\cS}$.\footnote{To simplify
          presentation, we assume that $\Xi$ and $\cX$ are
        		countable; our results trivially extend to the case where the
        		corresponding variables are continuous with an
                        appropriate measure-theoretic treatment.}

                      The final property of the ExBMDP
        model is \emph{decodability},
        which asserts the existence of a \emph{decoder}
        such that $\phi_\star\colon \cX \rightarrow \cS$ such that
        $\phi_\star(\bx_h)=\bs_h$ a.s. for all $h\in\brk{H}$ under the
        process in \cref{eq:exbmdp_observation}. Informally,
        decodability ensures the existence of an (unknown to the learner)
        mapping that allows one to perfectly recover the endogenous latent
        state from observations. In addition to decodability, we
        assume the rewards in the ExBMDP are \emph{endogenous}; that
        is, the reward distribution $R_h(\bx_h,\ba_h)$ only depends on the observations $(\x_h)$
        through the corresponding latent states
        $(\phistar(\x_h)=\bs_h)$.
        To enable sample-efficient learning, we assume access to a
        \emph{decoder class} $\Phi$ that contains $\phistar$, as in
        prior work.\loose

        \begin{assumption}[Decoder realizability]
        \label{ass:phistar}
        We have access to a decoder class $\Phi$ such that $\phistar \in \Phi$.
        \end{assumption}

        \fakepar{Applying \ggolf and \cref{thm:golf}}
        To apply \cref{thm:golf} to the ExBMDP problem, we need to
        verify that $Q^\star$-realizability and coverability
        hold. \colt{Realizability is a straightforward consequence of
        decodability (\cref{lem:realex} in \cref{sec:analysis_golf} of
        the appendix).}
        For coverability, \citet{xie2023role} show that ExBMDPs
        have $\Ccov\leq{}SA$ under decodability, in spite of the time-correlated exogenous noise process $(\bxi_h)$
        and potentially infinite observation space
        $\cX$ (interestingly, coverability is essentially the only useful
        structural property that ExBMDPs are known to satisfy, which
        is our primary motivation for studying it). %
        \arxiv{Realizability is also a straightforward consequence of
        the decodability assumption. %
        \begin{lemma}[\cite{efroni2022sample}]
          \label{lem:realex}
          For the ExBMDP setting, under \cref{ass:phistar}, (i) the function class $\cQ_h\coloneqq
        \{x\mapsto g(\phi(x),a) : g\in [0,H]^{SA}, \phi\in \Phi\}$
        satisfies \cref{ass:realgolf}, and (ii) the
        policy class $\Pi_h=\crl*{x\mapsto{}\pi(\phi(x)) :
          \pi\in\cA^{S}, \phi\in\Phi}$ satisfies
        \cref{ass:pireal}. Both classes have 
        $\log\abs{\cQ_h}=\log\abs{\Pi_h}=\bigoht(S+\log\abs{\Phi})$.\footnote{Formally,
          this requires a standard covering number argument; we omit
          the details.}
        \end{lemma}}        %
      This leads to the following corollary of \cref{thm:golf}.      
      \begin{corollary}[\ggolf for ExBMDPs]
        \label{cor:exbmdp}
        Consider the ExBMDP setting. Suppose that
        \cref{ass:phistar} holds, and let $\cQ$ be
        constructed as in \cref{lem:realex} of \cref{sec:analysis_golf}. Then for any $\veps, \delta
        \in (0,1)$, the policy $\pihat = \ggolf(\cQ, SA, \veps,
        \delta)$ has $J(\pistar)- J(\pihat)\leq \veps$ with
        probability at least $1-\delta$. The total sample
        complexity in the \framework framework is \colt{$N = \bigoht\prn*{\frac{H^5S^3A^3\log\abs{\Phi}}{\veps^4}}$.}
        \arxiv{bounded by
        \begin{align}
          \bigoht\prn*{\frac{H^5S^3A^3\log\abs{\Phi}}{\veps^4}}.
        \end{align}
        }
        \end{corollary}
        This shows for the first time that general ExBMDPs are
        learnable with local simulator access. Prior to this work, online RL algorithms for ExBMDPs
           required either (i) deterministic latent dynamics
           \citep{efroni2022provably}, or (ii) factored
           emission structure \citep{efroni2022sample}.
           \citet{xie2023role} observed that ExBMDPs admit low
           coverability, but their algorithm requires Bellman
           completeness, which is not satisfied by ExBMDPs (see
           \citet{islam2023agent}). See
           \cref{sec:additional_related} for more discussion.\loose

        \section{Computationally Efficient Learning with Local
          Simulators}
        \label{sec:computation}

\ifdefined\vepsllnum
\else
\newcommand{\vepsllnum}{\veps H^{-1}/48}
\fi

\ifdefined\vepsll
\else
\newcommand{\vepsll}{\veps_\learnlevel}
\fi

\ifdefined\Mnum
\renewcommand{\Mnum}{\ceil{8  \veps^{-1} \Cpush H}}
\else
\newcommand{\Mnum}{\ceil{8  \veps^{-1} \Cpush H}}
\fi 

\ifdefined\Ntestnum
\renewcommand{\Ntestnum}{2^8  M^2 H \veps^{-1}\log(8 M^6 H^8 \veps^{-2} \delta^{-1})}
\else
\newcommand{\Ntestnum}{2^8  M^2 H \veps^{-1}\log(8 M^6 H^8 \veps^{-2} \delta^{-1})}
\fi

\ifdefined\Nregnum
\renewcommand{\Nregnum}{2^8 M^2  \veps^{-1}\log(8|\cV|^2 HM^2 \delta^{-1})}
\else
\newcommand{\Nregnum}{2^8 M^2 \veps^{-1}\log(8|\cV|^2 HM^2 \delta^{-1})}
\fi

\ifdefined\bbeta
\renewcommand{\bbeta}{\frac{9 M H^2\log(8M^2H|\cV|^2/\delta)}{N_\reg}  +  \frac{34 MH^3\log(8M^6 N^2_\test  H^8/\delta)}{N_\test} }
\else
\newcommand{\bbeta}{\frac{9 M H^2\log(8M^2H|\cV|^2/\delta)}{N_\reg}  +  \frac{34 MH^3\log(8M^6 N^2_\test  H^8/\delta)}{N_\test} }
\fi

\ifdefined\Nsimu
\renewcommand{\Nsimu}{2N_\reg^2 \log(8 A N_\reg H M^3/\delta)}
\else
\newcommand{\Nsimu}{2N_\reg^2 \log(8 A N_\reg H M^3/\delta)}
\fi

\ifdefined\deltaprime
\renewcommand{\deltaprime}{\delta/(8M^7N_\test^2 H^8|\cV|)}
\else
\newcommand{\deltaprime}{\delta/(8M^7N_\test^2 H^8|\cV|)}
\fi

\ifdefined\testbound
\renewcommand{\testbound}{\frac{4 \log(8M^6 N^2_\test  H^8/\delta)}{N_\test}}
\else
\newcommand{\testbound}{\frac{4 \log(8M^6 N^2_\test  H^8/\delta)}{N_\test}}
\fi

\ifdefined\deltap
\renewcommand{\deltap}{\frac{\delta}{|\cV|M H}}
\else
\newcommand{\deltap}{\frac{\delta}{|\cV|M H}}
\fi 

\ifdefined\pibell
\renewcommand{\pibell}{\pihat}
\else
\newcommand{\pibell}{\pihat}
\fi

\ifdefined\numepisodes
\renewcommand{\numepisodes}{\mathrm{poly}(S,A,...)}
\else
\newcommand{\numepisodes}{\mathrm{poly}(S,A,...)}
\fi

Our result in \cref{sec:sample} show that local simulator access
facilitates sample-efficient learning in MDPs with low coverability, a
challenging setting that was previously out of reach. However, our
algorithm \ggolf is computationally-inefficient because it relies on
global optimism, a drawback found in most prior work on RL with
general function approximation
\citep{jiang2017contextual,jin2021bellman,du2021bilinear}. It
remains an open question whether any form of global optimism can be
implemented efficiently, and some variants have provable barriers to
efficient implementation \citep{dann2018oracle}.

To address this drawback, in this section we present a new algorithm,
\learnlevel{} (Recursive Value Function Search; \cref{alg:learnlevel3}), which requires stronger versions of the coverability and realizability
assumptions in \cref{sec:sample}, but is computationally efficient in
the sense that it reduces to convex optimization over the state-value
function class $\cV$.
\learnlevel makes use of a sophisticated recursive exploration
scheme based on core-sets, sidestepping the need for global optimism\colt{.}\arxiv{,
and can be applied to Exogenous Block MDPs with \emph{weakly
  correlated} exogenous noise.\loose}

\subsection{Function Approximation and Statistical Assumptions}
\label{sec:rvfs_assumptions}
      \newcommand{\I}{$\textbf{I}$}%
      \newcommand{\II}{$\textbf{II}$}%
To begin, we require the following strengthening of the coverability assumption in \cref{ass:cover}. 
       \begin{assumption}[Pushforward coverability]
       	\label{ass:pushforward}
       	The pushforward coverability coefficient $\Cpush>0$ is given
        by \colt{$       		\Cpush =   \max_{h\in[H]} \inf_{\mu_h\in\Delta(\cX)}\sup_{(x_{h-1} ,a_{h-1} ,x_h)\in \cX_{h-1}\times \cA\times \cX}\frac{T_{h-1}(x_h\mid{}x_{h-1},a_{h-1})}{\mu_h(x_h)}$.}
       	\arxiv{\begin{align}
       		\label{eq:pushforward}
       		\Cpush =   \max_{h\in[H]} \inf_{\mu_h\in\Delta(\cX)}\sup_{(x_{h-1} ,a_{h-1} ,x_h)\in \cX_{h-1}\times \cA\times \cX}\frac{T_{h-1}(x_h\mid{}x_{h-1},a_{h-1})}{\mu_h(x_h)}.
               \end{align}
               }
      \end{assumption}
      Pushforward coverability is inspired by the \emph{pushforward
        concentrability} condition used in offline RL by
      \cite{xie2021batch,foster2022offline}. Concrete examples
      include, (i) Block MDPs with latent space $\cS$, which admit
      $\Cpush\leq\abs{\cS}$, (ii) Low-Rank MDPs in dimension $d$,
      which admit $\Cpush\leq{}d$ \citep{xie2021batch}, and (iii)
      Exogenous Block MDPs for which the exogenous noise process
      satisfies a \emph{weak correlation condition} that we introduce
      in \cref{sec:rvfs_exbmdp}. Note that $\Ccov\leq\Cpush\abs{\cA}$,
      but the converse is not true in general.

Instead of state-action value function approximation as in \ggolf, in
this section we make use of a state value function class $\cV
\subset(\cX\times\brk{H}\to\brk{0,H})$, but require somewhat stronger
representation conditions than in \cref{sec:sample}. We
      consider two complementary setups\arxiv{, which can be summarized
      briefly as follows}:
        \begin{itemize}
        \item \textbf{Setup \I:} \cref{ass:real,ass:pireal}
          ($\Vstar/\pistar$-realizability) and \cref{ass:gap} ($\Delta$-gap) hold.
        \item \textbf{Setup \II:} \cref{ass:realpi}
          ($V^{\pi}$-realizability) and \cref{ass:allpireal}
          ($\pi$-realizability) hold.
        \end{itemize}
We describe these assumptions in more detail below.

\fakepar{Function approximation setup \I} First, instead of
$Q^\star$-realizability, we consider the weaker
$V^\star$-realizability \citep{jiang2017contextual,weisz2021query,amortila2022few}.

\begin{assumption}[$V^{\star}$-realizability]
	\label{ass:real}
	For all $h\in [H]$, we have $V^{\star}_h\in \cV_h$.
      \end{assumption}

Under $V^\star$-realizability, our algorithm learns a near-optimal
policy, but the policy is \emph{non-executable}\colt{ (cf. \cref{def:executable})}\arxiv{ in the sense of
\cref{def:executable}}; this property is shared by prior work on local
simulator access with value function realizability
\citep{weisz2021query}. To produce executable policies,
we additionally require access to a policy class $\Pi\subset\Pim$
containing $\pistar$; we define $\Pi_h=\crl*{\pi_h\mid{}\pi\in\Pi}$.\colt{\vspace{-15pt}}
\begin{assumption}[$\pistar$-realizability]
	\label{ass:pireal}
	The policy class $\Pi$ contains the optimal policy $\pistar$.
\end{assumption} 

\arxiv{Note that }$V^{\star}$-realizability (\cref{ass:real}) and
$\pistar$-realizability (\cref{ass:pireal}) are both implied by $\Qstar$-realizability, and hence are
weaker. However,
we also assume the
optimal $Q$-function admits constant \arxiv{suboptimality }gap (this
makes the representation conditions for \setupi{} incomparable to \cref{ass:realgolf}).
\loose

       \begin{assumption}[$\Delta$-Gap]
         \label{ass:gap}
         The optimal action $\pistar_h(x)$ is unique, and
         there exists $\Delta>0$ such that for all $h\in[H]$, $x\in
         \cX$, and $a \in \cA \setminus \{\pi^\star_h(x)\}$, \colt{$          \Qstar_h(x,\pistar_h(x))>\Qstar_h(x,a) + \Delta$. }
         \arxiv{\begin{align}
       		\label{eq:gap}
          \Qstar_h(x,\pistar_h(x))>\Qstar_h(x,a) + \Delta. 
       	\end{align}}
      \end{assumption}
      This condition has been used in a many prior works on
      computationally efficient RL
with function approximation \citep{du2019provably,du2019good,foster2020instance,wang2021exponential}.

      \fakepar{Function approximation setup \II}
      \arxiv{As an alternative to assuming constant suboptimality gap,
        we }\colt{We }
      also provide guarantees under the assumption that the value
      function class $\cV$ satisfies \emph{all-policy realizability}
      \citep{weisz2022confident,yin2022efficient,weisz2023online} in
      the sense that it contains the value functions $V^{\pi}$ for all $\pi \in \Pims$.\loose%
              \begin{assumption}[$V^{\pi}$-realizability]
       	\label{ass:realpi}
       	The \arxiv{value function }class $\cV=\cV_{1:H}$ has $V^{\pi}_h \in \cV_h$ for all $\pi \in \Pims$ and $h\in[H]$.\loose
      \end{assumption}      
This assumption will be sufficient to learn a non-executable policy,
but to learn executable policies we require an analogous strengthening
of \cref{ass:realpi}.
          \begin{assumption}[$\pi$-realizability]
       	\label{ass:allpireal}
       	\arxiv{The policy class $\Pi$ is such that for}\colt{For} all
          $\pi\in\Pim$, we have that $x\mapsto\argmax_{a\in \cA} \cP_h[V^{\pi}_{h+1}](\cdot, a)\in \Pi$.  
      \end{assumption}
      This assumption, also known as \emph{policy completeness} has
      been used by a number of prior works on computationally
      efficient RL \citep{bagnell2003policy,misra2020kinematic}. \cref{ass:realpi,ass:allpireal} are both implied by the
      slightly simpler-to-state assumption of
      \emph{$Q^{\pi}$-realizability} \citep{weisz2022confident,yin2022efficient,weisz2023online}, which asserts access to a class
      $\cQ$ that contains $Q^{\pi}$ for all $\pi\in\Pim$.\loose

      \arxiv{
      \begin{remark}
        \cref{ass:realpi,ass:allpireal} can both be weakened to only
        require realizability for near-optimal policies; i.e..,~we only
        need to assume that $V^\pi_h \in \cV$ for all $h\in[H]$ and
        near-optimal policies $\pi$ (instead of all policies). We will
        only use this weaker assumption in the analysis.
      \end{remark}
      }

        \subsection{Algorithm}
        \label{sec:rvfs}

\begin{algorithm}[t]
	\caption{\learnlevelh: Recursive Value Function Search (Informal version of \cref{alg:learnlevel3})}
	\label{alg:learnlevel3_simp}
	\begin{algorithmic}[1]
		\State 	{\bfseries parameters:} Value function class
		$\cV$, suboptimality $\veps\in(0,1)$, confidence
		$\delta\in(0,1)$.
		\State \colt{\multiline{{\bfseries input:} Level
				$h\in[0\ldotst H]$, value functions $\Vhat_{h+1:H}$, confidence sets $\cVhat_{h+1:H}$, core-sets $\cC_{h:H}$.}\vspace{5pt}}
		\arxiv{
			{\bfseries input:}
			\begin{itemize}
				\item Level $h\in\crl{0,\ldots,H}$.
				\item Value function estimates $\Vhat_{h+1:H}$, confidence sets $\cVhat_{h+1:H}$, state-action collections $\cC_{h:H}$.
			\end{itemize}
		}
		\State Initialize parameters $M$, $N_\test$, $N_\reg$,
                $\veps_\reg^2$, and $\beta$ (see \cref{alg:learnlevel3} for
                parameter settings).
                \label{line:paramsVpi_simp} 
		\Statex[0] \algcommentbiglight{Test the fit for the
			estimated value functions $\Vhat_{h+1:H}$ at future layers.}
		\For{$(x_{h-1},a_{h-1})\in\cC_h$ and $\ell = H,\dots,h+1$}\label{line:begin3_simp}
		\For{$n=1,\ldots,N_\test$} \label{line:third_simp}
		\State \multiline{Draw $\x_h\sim{}T_{h-1}(\cdot\mid{}x_{h-1},a_{h-1})$, then draw
			$\x_{\ell-1}$ by rolling out with $\pibell_{h:H}$,
			where\footnote{We use the convention that $\Vhat_{H+1}\equiv 0$.}
			\label{line:draw_simp}
			\begin{align}\label{eq:Qhat3_simp}
				\forall \tau\in[H], \quad \pibell_\tau(\cdot)\in \argmax_{a\in \cA}  \cP_{\tau} [\Vhat_{\tau+1}](\cdot,a).  
			\end{align}
		}
		\For{$a_{\ell-1}\in \cA$} \label{line:fourth_simp}\colt{\algcommentlight{Test fit; if test fails, re-fit value functions
			$\Vhat_{h+1:\ell}$ up to layer $\ell$.}}
		\arxiv{\Statex[3]\algcommentbiglight{Test fit; if test fails, re-fit value functions
			$\Vhat_{h+1:\ell}$ up to layer $\ell$.}}
		\If{$\sup_{f\in \cVhat_{\ell}} |(\cP_{\ell-1}[\Vhat_{\ell}]- \cP_{\ell-1}[f_\ell])( \x_{\ell-1},a_{\ell-1})| >  \veps  + \veps \cdot \beta$}\label{line:test3_simp} \hfill %
		\State $\cC_{\ell}\gets\cC_{\ell}\cup\crl{(\x_{\ell-1},a_{\ell-1})}$. \label{line:added_simp}
		\For{$\tau= \ell,\dots, h+1$} \label{line:forloop_simp}
		\State $(\Vhat_{\tau:H},\cVhat_{h:H}, \cC_{\tau:H})\gets{}\learnlevel_{\tau}(\Vhat_{\tau+1:H},\cVhat_{h+1:H}, \cC_{\tau:H};\cV,\veps,\delta)$. 
		\EndFor
		\State \textbf{go to line \ref*{line:begin3_simp}.} \label{line:goto3_simp}
		\EndIf
		\EndFor
		\EndFor
		\EndFor
		\If{$h=0$} \textbf{return:} $(\Vhat_{1:H},\cdot,\cdot,\cdot,\cdot)$.
		\EndIf
		\Statex[0] \algcommentbiglight{Re-fit $\Vhat_h$ and build a new confidence set.}
		\For{$(x_{h-1},a_{h-1})\in\cC_h$} \hfill \algcommentlight{$\E^{\pibell_{h+1:H}}[\sum_{\ell=h}^H \bm{r}_\ell \mid
			\x_h,a] $ can be estimated using local simulator\arxiv{ and roll-outs}.}
		\State \multiline{Set $\cD_h(x_{h-1},a_{h-1})\gets\emptyset$. For
                $i=1,\dots, N_\reg$, sample
		$\x_h\sim{}T_{h-1}(\cdot\mid{}x_{h-1},a_{h-1})$ and
 update $\cD_h(x_{h-1},a_{h-1})\gets \cD_h(x_{h-1},a_{h-1}) \cup \{(\x_h, \E^{\pibell_{h:H}}[\sum_{\ell=h}^H \bm{r}_\ell \mid
                  \x_h]  )\}$. } \label{line:data3_simp}
		\EndFor
		\State Let $\Vhat_h\ldef\argmin_{f\in\cVhat}\sum_{(x_{h-1},a_{h-1})\in\cC_h} \sum_{(x_h,v_h)\in\cD_{h}(x_{h-1},a_{h-1})}(f(x_h)-v_h)^2$. \label{line:updateQ3_simp}
		\State Compute value function confidence set:  
		\begin{align}
			\cVhat_{h} \coloneqq \left\{ f\in \cV \left| \  \sum\colt{\nolimits}_{(x_{h-1},a_{h-1})\in\cC_h} \frac{1}{N_\reg} \sum\colt{\nolimits}_{(x_h,\text{-}) \in \cD_h(x_{h-1},a_{h-1})}\left(\Vhat_h(x_{h}) -f(x_h)\right)^2 \leq   \veps_\reg^2  \right.     \right\}.  \label{eq:confidence3_simp} 
		\end{align}
		\State \textbf{return} $(\Vhat_{h:H},\cVhat_{h:H},  \cC_{h:H})$.
	\end{algorithmic}
\end{algorithm}

        For ease of exposition, we defer the full version of our
        algorithm, \mainalg (\cref{alg:learnlevel3}), to
        \cref{sec:omitted} and present a simplified version here
        (\cref{alg:learnlevel3_simp}). The algorithms are nearly
        identical, except that the simplified version assumes that
        certain quantities of interest (e.g., Bellman backups) can
        be computed exactly, while the full version (provably)
        approximates them from samples.

\mainalg maintains a value function estimator
$\Vhat=\Vhat_{1:H}$ that aims to approximate the optimal value function
$\Vstar_{1:H}$, as well as \emph{core sets} $\cC_{1},\ldots,\cC_{H}$
of state-action pairs that are used to perform estimation and guide
exploration. At a high level, \learnlevel{} alternates between (i)
fitting the value function $\Vhat_h$ for a given layer $h\in\brk{H}$ based on Monte-Carlo rollouts,
and (ii) using the core-sets to test whether the current value
function estimates $\Vhat_{h+1:H}$  remain accurate as the
roll-in policy induced by $\Vhat_h$ changes.

In more detail, \mainalg is based on recursion across the layers
$h\in\brk{H}$. When invoked for layer $h$ with value function
estimates $\Vhat_{h+1:H}$ and core-sets $\cC_h,\ldots,\cC_H$,
$\mainalg_h$ performs two steps:
\colt{\begin{enumerate}[leftmargin=*]}
  \arxiv{\begin{enumerate}}
\item \label{item:two} For each state-action pair
  $(x_{h-1},a_{h-1})\in\cC_h$,\footnote{Informally, $\cC_h$ represents a
    collection of state-action pairs $(x_{h-1},a_{h-1})$ at layer $h-1$ for which we
    want
    $\En\brk*{\abs*{\Vhat_h(\x_h)-\Vstar_h(\x_h)}\mid{}\x_{h-1}=x_{h-1},\a_{h-1}=a_{h-1}}\leq\veps$
    for some small $\veps>0$.}
    the algorithm gathers $N_\test$ trajectories by
    rolling out from $(x_{h-1},a_{h-1})$ with the greedy policy
    $\pihat_\ell(x)\in \argmax_{a\in\cA}\cP_\ell\brk{\Vhat_{\ell+1}}(x,a)$ that
    optimizes the estimated value function; in the full version of
    \learnlevel{} (see \cref{alg:learnlevel3}), we estimate the
    bellman backup $\cP_\ell\brk{\Vhat_{\ell+1}}(x,a)$ using the local
    simulator. For all states
    $x_{\ell-1}\in\crl{\x_{h},\ldots,\x_{H-1}}$ encountered during this process, the
    algorithm checks whether
    $
    \abs*{\En\brk*{\Vhat_\ell(\x_{\ell})-\Vstar_\ell(\x_\ell)\mid{}\x_{\ell-1}=x_{\ell-1},\a_{\ell-1}=a_{\ell-1}}}\approxleq\veps$
    for all $a_{\ell-1}\in\cA$
    using a test based on (implicitly maintained) confidence sets. If
    the test fails, this indicates that distribution shift has
    occurred, and the algorithm adds the pair
    $(x_{\ell-1},a_{\ell-1})$ to the core-set $\cC_{\ell}$ and
    recurses on layer $\ell$ via $\mainalg_{\ell}$.
  \item\label{item:one}
If all tests above pass, this \arxiv{indicates}\colt{means} that
$\Vhat_{h+1},\ldots,\Vhat_{H}$ are accurate, and no distribution shift
has occurred. In this case, the algorithm fits $\Vhat_h$ by collecting
Monte-Carlo rollouts from all state-action pairs in the core-set $\cC_{h}$ with 
$\pihat_\ell(x)\in \argmax_{a\in\cA}\cP_\ell\brk{\Vhat_{\ell+1}}(x,a)$
(cf. \cref{line:updateQ3_simp}), and returns.\loose
  	\end{enumerate}
When the tests in \cref{item:two} succeed for all\arxiv{ layers} $h\in\brk{H}$, the
algorithm\arxiv{ terminates and} returns the estimated value functions
$\Vhat_{1:H}$; in this case, the greedy policy
$\pihat_\ell(x)\in\argmax_{a\in\cA}\cP_\ell\brk{\Vhat_{\ell+1}}(x,a)$
is guaranteed to be near optimal. The full version of \learnlevel{} in \cref{alg:learnlevel3} uses local simulator access to estimate the Bellman backups $\cP_h[\Vhat_{h+1}](x,a)$ for different state-action pairs $(x,a)$\arxiv{ (see \cref{eq:Qhat3} of \cref{alg:learnlevel3})}. These backups are used to (i) compute actions
of the greedy policy that maximizes $\Vhat_{1:H}$ via (e.g.,
\cref{eq:Qhat3_simp}); (ii) generate trajectories by rolling out from
state-action pairs in the core-sets (\cref{line:draw_simp}); and (iii) perform the test in
\cref{item:two} (\cref{line:test3_simp}).\loose

\learnlevel{} is inspired by the \texttt{DMQ} algorithm of
\citet{du2019provably,wang2021exponential}, which was originally
introduced in the context of online reinforcement learning with
linearly realizable $Q^\star$. \learnlevel incorporates local
simulator access (most critically, via core-set construction) to allow for more general \emph{nonlinear} function
approximation without restrictive statistical assumptions. Prior
algorithms for RLLS have used core-sets of state-action pairs in a similar
fashion \citep{li2021sample,yin2022efficient,weisz2022confident}, but
in a way that is tailored to linear function approximation.

In what follows, we discuss various features of the algorithm in
greater detail.

      \paragraph{Bellman backup policies} Since \mainalg works with
      state value functions instead of state-action value functions,
      we need a way to extract policies from the former. The most
      natural way to extract a policy from estimated value
      functions $\Vhat_{1:H}\in \cV$ is as follows: for all $h\in
      [H]$, define $\pihat_h(x)\in \argmax_{a\in \cA}
      \cP_h[\Vhat_{h+1}](x,a)$. In reality, we do not have access to
      $\cP_h[\Vhat_{h+1}](x,a)$ directly, so the full version of
      \mainalg (\cref{alg:learnlevel3}) estimates this quantity on the
      fly using the local
      simulator using the following scheme (\cref{alg:Phat} in \cref{sec:omitted}): Given a state $x$, for each $a$, we sample $K$ rewards $\br_h \sim
      R_h(x,a)$ and next-state transitions $\x_{h+1}\sim T_h(\cdot
      \mid x,a)$, then approximate $\cP_h[\Vhat_{h+1}](x,a)$ by the
      empirical mean. We remark that the use of these Bellman backup
      policies is actually crucial in the analysis for \learnlevel;
      even if we were to work with estimated state-action value
      functions $\Qhat_{1:H}$ instead, our analysis would require executing
      the Bellman backup policies $\pihat_h(x)\in \argmax_{a\in \cA}
      \cT_h[\Qhat_{h+1}](x,a)$ (instead of naively using $\pihat_h(x)\in \argmax_{a\in \cA} \Qhat_h(x,a)$).

   \fakepar{Invoking the algorithm}
The base invocation of \mainalg takes the form
\begin{align}
\Vhat_{1:H} \gets \mainalg_0(\Vhat_{1:H}=\mathsf{arbitrary}, \cVhat_{1:H}=\crl*{\cV_h}_{h=1}^{H},
  \cC_{0:H}=\crl*{\emptyset}_{h=0}^{H},
  ;\cV, \veps, \delta).\label{eq:rvfs_base}
\end{align}
Whenever this call returns, the greedy policy induced by $\Vhat_{1:H}$
is guaranteed to be near-optimal. Naively, the approximate Bellman backup policy induced by
$\Vhat_{1:H}$ (described above) is non-executable, and must be computed by invoking the
local simulator. To provide an end-to-end
guarantee to learn an executable policy, we give an outer-level
algorithm, \rvflF{} (\cref{alg:forward_vpi}, deferred to
\cref{sec:omitted} for space), which invokes
$\mainalg_0$, then extracts an executable policy from $\Vhat_{1:H}$
using behavior cloning. Subsequent recursive calls to \mainalg take the
form \colt{$(\Vhat_{h:H},\cVhat_{h:H},
		\cC_{h:H})\gets{}\learnlevel_{h}(\Vhat_{h+1:H},\cVhat_{h+1:H},\cC_{h:H};\cV,\veps,\delta)$.}
\arxiv{\begin{align}
  \label{eq:rvfs_rec}
(\Vhat_{h:H},\cVhat_{h:H},
		\cC_{h:H})\gets{}\learnlevel_{h}(\Vhat_{h+1:H},\cVhat_{h+1:H},\cC_{h:H};\cV,\veps,\delta).
\end{align}}
\arxiv{For such a call, the arguments}\colt{The arguments here} are: \colt{(i) $\Vhat_{h+1:H}$: Value function estimates for subsequent layers;
(iii) $\cVhat_{h+1:H}$: Value function confidence sets
  $\cVhat_{h+1:H}\subset\cV_{h+1:H}$, which are used in
  the test on \cref{line:test3_simp} to quantify
  uncertainty on new state-action pairs and decide
  whether to expand the core-sets; and (iii) $\cC_{h:H}$: Core-sets for current and subsequent layers.
}%
\arxiv{\begin{itemize}
\item $\Vhat_{h+1:H}$: Value function estimates for subsequent layers.
\item $\cVhat_{h+1:H}$: Value function confidence sets
  $\cVhat_{h+1:H}\subset\cV_{h+1:H}$, which are used in
  the test on \cref{line:test3} to quantify
  uncertainty on new state-action pairs and decide
  whether to expand the core-sets.
\item $\cC_{h:H}$: Core-sets for current and subsequent layers.
\end{itemize}}
Importantly, the confidence sets $\cVhat_{h+1:H}$ do not need
to be explicitly maintained, and can be used
implicitly whenever a \emph{regression oracle} for the
value function class is available (discussed below). 

\fakepar{Oracle-efficiency}
\learnlevel{} is \emph{computationally efficient}\arxiv{ in the sense
  that}\colt{:} it reduces to convex optimization over the value
function class $\cV$.
In particular,
the only computationally intensive steps in the algorithm are (i)
the regression step in \cref{line:updateQ3_simp},
and (ii) the test in
\cref{line:test3_simp} involving the confidence set $\cVhat_\ell$. For the
latter, we do not explicitly need to maintain $\cVhat_\ell$, as
the optimization problem over this set in \cref{line:test3_simp} (for
the full version of \mainalg in \cref{alg:learnlevel3}) reduces to solving 
$\argmax_{V\in\cV}\crl*{\pm\sum_{i=1}^{n}V(\wt{x}\ind{i})\mid{}\sum_{i=1}^{n}(V(x\ind{i})-y\ind{i})^2\leq\beta^2}$
for a dataset
$\crl*{(x\ind{i},\wt{x}\ind{i},y\ind{i})}_{i=1}^{n}$. This is convex
optimization problem in function space, and in particular can be implemented in a
provably efficient fashion whenever $\cV$ is linearly
parameterized. We expect that the problem can also be reduced to
a square loss regression by adapting the techniques in
\citet{krishnamurthy2017active,foster2018practical}, but we do not
pursue this here.

\subsection{Main Result}
We present the main guarantee for \learnlevel{} under the
function approximation assumptions in \cref{sec:rvfs_assumptions}.
        
   \begin{theorem}[Main guarantee for \mainalg]
   	\label{thm:vpiforward_main}
   	Let $\veps,\delta\in(0,1)$ be given, and suppose that
        \cref{ass:pushforward} (pushforward coverability) holds with
        $\Cpush>0$. \colt{Further, suppose that either \textbf{Setup \I} (\cref{ass:real,ass:pireal}
          ($\Vstar/\pistar$-realizability) and \cref{ass:gap}
          ($\Delta$-gap)) with $ \veps \leq 8H^{2} \Delta$ or \textbf{\setupii} (\cref{ass:realpi}
          ($V^{\pi}$-realizability) and \cref{ass:allpireal}
          ($\pi$-realizability)) holds.}
        \arxiv{ Further, suppose that one the following holds: 
          \begin{itemize}
                  \item \emph{\textbf{Setup \I:}} \cref{ass:real,ass:pireal}
          ($\Vstar/\pistar$-realizability) and \cref{ass:gap}
          ($\Delta$-gap) hold, and $\veps\leq  6 H\cdot\Delta$. 
        \item \emph{\textbf{Setup \II:}} \cref{ass:realpi}
          ($V^{\pi}$-realizability) and \cref{ass:allpireal}
          ($\pi$-realizability) hold.
   	\end{itemize}
      }
   	Then, $\rvflF(\Pi, \cV, \veps, \delta)$ (\cref{alg:forward_vpi}) returns a policy
        $\pihat_{1:H}$ such that $J(\pistar)-
        J(\pihat_{1:H}) \leq  2\veps$ with probability at least
        $1-\delta$, and has total sample complexity bounded
        by
        \colt{$          \wtilde{O}\left(\Cpush^4 H^{23}A\cdot{}\veps^{-7}\right)$.\footnote{For \textbf{Setup \I}, note that the sample complexity
          implicitly depends on $\Delta^{-1}$ through the constraint
          that $\veps\leq\bigoh(H^2\Delta)$.}}
   	\arxiv{\begin{align}
          \wtilde{O}\left(\Cpush^8 H^{23}A\cdot{}\veps^{-13}\right).
               \end{align}
             }
   \end{theorem}
  \cref{thm:vpiforward_main} shows
  for the first time that sample- and computationally-efficient RL
  with local simulator access
  is possible under pushforward coverability. In particular,
  \mainalg is the first computationally efficient algorithm for RL with
  local simulator access that supports nonlinear function
  approximation. The assumptions in
  \cref{thm:vpiforward_main}, while stronger than those in
  \cref{sec:sample}, are not known to enable sample-efficient RL
  without simulator access. Nonetheless, understanding whether \mainalg can be
  strengthened to support general coverability or weaker function
  approximation is an important open problem. \colt{See
    \cref{sec:rvfs_overview} for an overview of the analysis; we remark (\cref{sec:setupii}) that
    the result is actually proven under slightly weaker 
    assumptions than those in \setupi/\setupii.\loose
  }

\arxiv{  We mention in passing that \learnlevel{} can be slightly modified to
  recover other existing sample complexity guarantees for RL with
  linear function approximation and local simulator access (which do not
  require pushforward coverability), including
  linear-$Q^{\star}$ realizability with gap \citep{li2021sample} and
  $Q^{\pi}$-realizability \citep{yin2022efficient}; we leave a more
  general treatment for future work. } 

\fakepar{Connection to empirical algorithms}

\learnlevel{} bears some similarity to Monte-Carlo Tree Search (MCTS)
\citep{coulom2006efficient,kocsis2006bandit} and AlphaZero
\citep{silver2018general}, which perform planning with local simulator. Informally, MCTS can be viewed as a form of
breadth-first search over the state space (where each node represents a
state at a given layer), and AlphaZero is a particular instantiation of a MCTS
that leverages \colt{$V-$}value function approximation\arxiv{ (through a class $\cV$ that
aims to approximate $\Vstar$)} to \arxiv{accommodate stochastic
environments and }allow for generalization across states.
Compared to \mainalg, MCTS and AlphaZero perform exploration via
simple bandit-style heuristics, and are not explicitly designed to handle
\emph{distribution shifts} that arise in settings where actions have long-term
downstream effects. \learnlevel{} may be viewed as a provable
counterpart that uses function approximation to address distribution
shift in a principled fashion\arxiv{ (in particular, through the use of confidence
sets and the test in \cref{line:test3})}.\footnote{We note in passing that in
the context of tree search, the pushforward coverability assumption (\cref{ass:pushforward}) may be viewed as the stochastic analogue of branching factor.}

\colt{
\learnlevel{} also has a resemblance to the Go-Explore algorithm
of \citet{ecoffet2019go,ecoffet2021first}, which also uses core-sets of informative
state-action pairs to guide exploration (albeit in an ad-hoc,
domain-specific fashion) and uses imitation learning
to extract an executable policy after exploration completes. In light of these connections, we are optimistic that the techniques
in \learnlevel{}---in particular, using value function approximation
to guide systematic exploration---can help to inform the design of better practical algorithms for
learning and planning with local simulators.
}

\arxiv{\learnlevel{} also has some resemblance to the Go-Explore algorithm
of \citet{ecoffet2019go,ecoffet2021first} Like \mainalg, Go-Explore makes use of core-sets of informative
state-action pairs to guide exploration, and uses imitation learning
to extract an executable policy after the exploration phase completes. However, Go-Explore uses an
ad-hoc, domain specific approach to building the core set, and does
not use function approximation in the exploration phase; such
a strategy is unlikely to succeed in more challenging environments
where the effective horizon is longer or less domain-specific
information is available a-priori.

In light of these connections, we are optimistic that the techniques
in \learnlevel{}---in particular, using value function approximation
to guide systematic exploration---can help to inform the design of practical algorithms for
learning and planning with local simulators.
}

        \colt{
          \fakepar{Applying \mainalg to Exogenous Block MDPs}
          ExBMDPs satisfy coverability (\cref{ass:cover}),
          but do not satisfy the pushforward coverability assumption
          (\cref{ass:pushforward}) \arxiv{required by \mainalg }in
          general. However, it turns out that ExBMDPs \emph{do}
          satisfy pushforward coverability when the exogenous noise
          process is weakly correlated across time, a new statistical
          assumption we refer to the \emph{weak correlation
            condition}. In \cref{sec:rvfs_exbmdp} (\cref{thm:exbmdpforward_main}), we give a variant
          of \mainalg for ExBMDPs that succeeds under (i) weak
          correlation, and (ii) decoder realizability,
          sidestepping the need for the $\Delta$-gap or
          $V^{\pi}$-realizability.\loose %
          }
        
        \arxiv{
          \subsection{Applying \mainalg to Exogenous Block MDPs}
          \label{sec:rvfs_exbmdp}
\arxiv{We now apply \mainalg to the Exogenous Block MDP (ExBMDP) model
introduced in
\cref{sec:exbmdp}.}
\colt{In this section, we apply \mainalg to the Exogenous Block MDP (ExBMDP) model
introduced in
\cref{sec:exbmdp}.}
ExBMDPs satisfy coverability (\cref{ass:cover}),
but do not satisfy the pushforward coverability assumption
(\cref{ass:pushforward}) required by \mainalg in general. However, it turns out that ExBMDPs \emph{do} satisfy pushforward coverability when the exogenous noise process is weakly correlated across time; we refer to this new statistical assumption as the \emph{weak correlation condition}. 
\begin{assumption}[Weak correlation condition]
	\label{ass:great} 
For the underlying ExBMDP $\cM$, there is a constant $\Ccor \geq 1$ such that for all $h\in[H-1]$ and $(\xi,\xi')\in \Xi_{h-1}\times \Xi_{h}$, we have\footnote{Throughout this paper, when
		considering the law for the exogenous variables
		$\bxi_1,\ldots,\bxi_H$, we write $\bbP\brk*{\cdot}$ instead of
		$\bbP^{\pi}\brk*{\cdot}$ to emphasize that the law is independent of
		the agent's policy.}
	\begin{align}
		\label{eq:corr}
		\P[\bxi_h =\xi,\bxi_{h+1}=\xi'] \leq \Ccor\cdot  \P[\bxi_h =\xi] \cdot \P[\bxi_{h+1}=\xi'].
	\end{align}
\end{assumption}
The weak correlation property asserts that the joint law for the
exogenous noise variables $\bxi_h$ and $\bxi_{h+1}$ is at most a
multiplicative factor $\Ccor\geq1$ larger than the corresponding product distribution obtained by sampling $\bxi_h$ and $\bxi_{h+1}$ independently from their marginals. This setting strictly generalizes the (non-exogenous) Block MDP model
\citep{krishnamurthy2016pac,du2019latent,misra2019kinematic,zhang2022efficient,mhammedi2023representation}, by allowing for arbitrary stochastic dynamics for the endogenous state and an arbitrary emission process, but requires that temporal correlations in the exogenous noise decay over time.

\arxiv{We show that under}\colt{Under} \cref{ass:great}, pushforward
coverability is satisfied with $\Cpush\leq \Ccor \cdot{}S A$
(\cref{lem:pushforward} in \cref{sec:analysis_exbmdp}). In addition,
$V^{\star}$-realizability is implied by decoder realizability
(\cref{lem:realex}). Thus, by
applying \cref{thm:vpiforward_main} (\setupi), we conclude that \mainalg
efficiently learns a near-optimal policy for any weakly correlated ExBMDP for
which the optimal value function has $\Delta$-gap.\loose

\paragraph{An improved algorithm for ExBMDPs: \mainalge}
At first glance, removing the gap assumption for \mainalg in ExBMDPs seems
difficult: The $V^{\pi}$-realizability assumption required to invoke
\cref{thm:vpiforward_main} (\setupii) is not satisfied by ExBMDPs, as
decoder realizability only implies $V^{\pi}$ realizability for
\emph{endogenous} policies $\pi$.\footnote{We say that a policy $\pi$ is \emph{endogenous}
  if it does not depend on exogenous noise, in the sense that
  $\pi(\x_h)$ is a measurable function of $\phistar(\x_h)$.}
In spite of this, we now show that with a slight modification,
\mainalg can efficiently learn any weakly correlated ExBMDP under decoder
realizability alone (without gap or $V^\pi$-realizability).

Our new variant of \mainalg, $\learnlevel^{\exo}$, is presented in
\cref{alg:learnlevel2}\arxiv{ (deferred to \cref{sec:omitted} for space)}. The
algorithm is almost identical to \learnlevel{}
(\cref{alg:learnlevel3}), with the main difference being that we use
an additional \emph{randomized rounding} step to compute the policies
$\pihat_{1:H}$ from the learned value functions $\Vhat_{1:H}$. In particular, instead
of directly defining the policies $\pihat_{1:H}$ based on the bellman
backups $\cP_{h}[\Vhat_{h+1}]$ as in \cref{eq:Qhat3},
$\learnlevel^\exo$ targets a ``rounded'' version of the backup given by
\begin{align}
	\veps \cdot \ceil{\cP_{h}[\Vhat_{h+1}](x,a)/\veps + \bzeta_h}, \label{eq:know}
	\end{align}
where $\veps\in(0,1)$ is a rounding parameter and $\bzeta_1, \dots, \bzeta_H$
are i.i.d.~random variables sampled uniformly at random from the
interval $[0,1/2]$ (at the beginning of the algorithm's
execution). Concretely, $\learnlevel^\exo$ estimates the bellman
backup $\cP_{h}[\Vhat_{h+1}](x,a)$ in \cref{eq:know} using the local
simulator (as in \cref{eq:Qhat3} of \cref{alg:learnlevel3}), and defines its policies via
\begin{align}
	\pihat_{h}(\cdot) \in \argmax_{a\in \cA} \ceil{\cP_{h}[\Vhat_{h+1}](\cdot, a)/\veps + \bzeta_h}. \label{eq:pihate}
\end{align}

This rounding scheme, which quantizes the Bellman backup
into $\veps^{-1}$ bins with a random offset, is designed to emulate certain properties
implied by the $\Delta$-gap assumption (\cref{ass:gap}). Specifically,
we show that with constant probability over the draw of
$\bzeta_{1:H}$,\zmdelete{and with high probability over the randomness in
$\Phat$,} the policy $\pihat$ in \eqref{eq:pihate} ``snaps'' on to an
\emph{endogenous} policy $\pi$. This means that for
$\learnlevel^{\exo}$ to succeed (with constant probability), it
suffices to pass it a class $\cV$ that realizes the value functions
$(V^{\pi}_h)$ for endogenous policies $\pi\in \Pim$. Fortunately, such a function
class can be constructed explicitly under decoder realizability (\cref{ass:phistar}).%
   \begin{lemma}[\cite{efroni2022sample}]
	\label{lem:realex2}
	For the ExBMDP setting, under \cref{ass:phistar}, the function class $\cV_h\coloneqq
	\{x\mapsto f(\phi(x)) : f\in [0,H]^S, \phi\in \Phi\}$
	is such that $V_h^{\pi}\in \cV_h$ for all endogenous policies $\pi$. Furthermore, the policy class $\Pi_h \coloneqq \left\{ \pi(\cdot)\in \argmax_{a\in \cA} f(\phi(\cdot),a):f\in [0,H]^{S\times A }, \phi\in \Phi \right\}$ contains all endogenous policies.
\end{lemma} 

A small technical challenge with the scheme above is that it is only
guaranteed to succeed with constant probability over the draw of the
rounding parameters $\bzeta_1,\ldots,\bzeta_H$. To address this, we
provide an outer-level algorithm, \forwardexo
(\cref{alg:forward_exbmdp}\arxiv{, deferred to \cref{sec:omitted} for space}), which performs
confidence boosting by invoking \mainalge multiple times
independently, and extracts a high-quality executable policy using
behavior cloning.\loose

\paragraph{Main result}

We now state the main guarantee for $\learnlevel^\exo$ (the proof is in \cref{sec:exbmdp_app}).
  \begin{theorem}[Main guarantee of \mainalge for EXBMDPs]
  \label{thm:exbmdpforward_main}
  Consider the ExBMDP setting. Suppose the decoder class $\Phi$
  satisfies \cref{ass:phistar}, and that \cref{ass:great} holds with $\Ccor>0$.
  Let $\veps,\delta\in(0,1)$ be given, and let $\cV_{h}$ and
        $\Pi_h$ be as in \cref{lem:realex2}.
        Then $\forwardexoeq(\Pi, \cV_{1:H}, \veps, \zeta_{1:H},\delta)$
        (\cref{alg:forward_exbmdp}) produces a policy $\pihat_{1:H}$
        such that $J(\pistar)- J(\pihat_{1:H}) \leq  \veps$, and has
        total sample complexity 
	\begin{align}
		\wtilde{O}\left(\Cexo^8 S^8H^{36}A^9\cdot \veps^{-26}\right).
	\end{align}
\end{theorem}
This result shows for the first time that sample- and
computationally-efficient learning is possible for ExBMDPs beyond
deterministic or factored settings \citep{efroni2022provably,efroni2022sample}.

We mention in passing that our use of randomized rounding
to emulate certain consequences of the $\Delta$-gap
assumption leverages the fact that ExBMDPs have a finite number of
(endogenous) latent states. It is unclear if this technique can be used when the (latent) state space is large or infinite.

           }

\colt{\vspace{-10pt}}

        \section{Conclusion and Open Problems}
        \label{sec:discussion}
Our results show that online RL with local simulator access can enable powerful sample complexity guarantees for learning
with general value function approximation. Interesting open problems
and directions for future research include:
\colt{(i) Can we show that value function realizability and coverability
  are \emph{not sufficient} for sample-efficient learning in the
  online RL framework, thereby separating this framework from \framework?
  (ii) Can we obtain computationally-efficient algorithms under the
  same assumptions as \ggolf? In addition, we are excited to explore the empirical performance of
exploration schemes inspired by \mainalg in a large-scale evaluation.
}
\arxiv{
\begin{itemize}
\item Is it true that value function realizability and coverability
  are \emph{not sufficient} for sample-efficient learning in the
  online RL model? Combined with our results, this would imply a new
  separation between \framework and online RL. More broadly, it would be interesting to develop a unified
  understanding for precisely when local simulator access can lead to
  statistical benefits over fully online RL, and to characterize the
  statistical complexity for this framework.

\item The \mainalg algorithm is computationally efficient, but
  requires stronger statistical assumptions than our inefficient
  algorithms. Closing this gap with an efficient algorithm is an
  important but challenging open problem.
\end{itemize}
In addition, we are excited to explore the empirical performance of
exploration schemes inspired by \mainalg.%
}

	\newpage

  \newpage
  \subsection*{Acknowledgements}
  Part of this work was done while ZM was at MIT. ZM and AR acknowledge support from the ONR through awards N00014-20-1-2336 and N00014-20-1-2394, and ARO through award W911NF-21-1-0328.

	\bibliography{refs.bib}
	
	\newpage

	\appendix
		\renewcommand{\contentsname}{Contents of Appendix}
		\addtocontents{toc}{\protect\setcounter{tocdepth}{2}}
		{
			\hypersetup{hidelinks}
			\tableofcontents
		}

	\clearpage

	\section{Additional Related Work}
        \label{sec:additional_related}
\paragraph{Local simulators: Theoretical research}
RL with local simulators has received extensive interest in the context
of linear function approximation. Most notably, \citet{weisz2021query} show that
reinforcement learning with linear $V^{\star}$ is tractable with local
simulator access, and \citet{li2021sample} show that RL with
linear $Q^{\star}$ and a state-action gap is tractable; online RL is
known to be intractable under the same assumptions
\citep{weisz2021query,wang2021exponential}. \citet{amortila2022few}
show that the gap assumption can be removed if a small number of
expert queries are available. Also of note are the
works of \citet{yin2022efficient,weisz2022confident}, which give
computationally efficient algorithms under linear
$Q^{\pi}$-realizability for all $\pi$; this setting is known to be tractable in the
online RL model \citep{weisz2023online}, but computationally efficient algorithms are
currently only known for \framework.

\emph{Global simulators}---in which the agent can query arbitrary
state-action pairs and observe next state transitions---have also
received theoretical investigation, but like local simulators, results
are largely restricted to tabular reinforcement learning and linear
models
\citep{kearns1998finite,kakade2003sample,sidford2018near,du2019good,yang2019sample,lattimore2020learning}.

\paragraph{Local simulators: Empirical research}
The Go-Explore
algorithm \citep{ecoffet2019go,ecoffet2021first} uses local simulator
access to achieve state-of-the-art performance for the Atari games
Montezuma's Revenge and Pitfall---both notoriously difficult games that require systematic
exploration. To the best of our knowledge, the performance of
Go-Explore on these tasks has yet to be matched by online
reinforcement learning; the performing agents
\citep{badia2020agent57,guo2022byol} are roughly a factor of four
worse in terms of cumulative reward. Interestingly,
like \mainalg, Go-Explore makes use of core sets of informative
state-action pairs to guide exploration. However, Go-Explore uses an
ad-hoc, domain specific approach to designing the core set, and does
not use function approximation to drive exploration.

Recent work of \citet{yin2023sample} provides an empirical framework
for online RL with local planning that can take advantage of deep
neural function approximation, and is inspired by the theoretical
works in
\citet{weisz2021query,li2021sample,yin2022efficient,weisz2022confident}. This
approach does not have provable guarantees, but achieves super-human
performance at Montezuma's Revenge.

Other notable empirical works that incorporate local simulator access, as
highlighted by \citet{yin2023sample}, include \citet{schulman2015trust,salimans2018learning,tavakoli2018exploring}.

\paragraph{Planning}
RL with local simulator access is a convenient abstraction for the
problem of \emph{planning}: Given a known (e.g., learned) model,
compute an optimal policy. Planning with a learned model is an important task in theory
\citep{foster2021statistical,liu2023optimistic} and practice (e.g., MuZero \citep{schrittwieser2020mastering}).
Since the model is known, computing an
optimal policy is a purely computational problem, not a statistical
problem. Nonetheless, for planning problems in large state spaces, where
enumerating over all states is undesirable, algorithms for online RL
with local simulator access can be directly applied, treating the
model as if it were the environment the agent is interacting
with. Here, any computationally efficient \framework algorithm immediately yields an efficient
algorithm for planning.

Empirically, Monte-Carlo Tree Search \citep{coulom2006efficient,kocsis2006bandit} is a successful paradigm for
planning, acting as a key component in AlphaGo
\citep{silver2016mastering} and AlphaZero
\citep{silver2018general}.\footnote{Compare to our work, a small
  difference is that these works are not concerned with producing
  executable policies, c.f. \cref{def:executable}.} Viewed as a
planning algorithm, a potential advantage of \mainalg is that it is
well suited to stochastic environments, and provides a principled way
to use estimated (neural) value function estimates to guide exploration.

\paragraph{Coverability}
\citet{xie2023role} introduced coverability as a structural parameter
for online reinforcement leanring, inspired by connections between
online and offline RL. Existing guarantees for the online RL framework
based on coverability require either Bellman completeness \citep{xie2023role}, model-based
  realizability \citep{amortila2024scalable}, or weight
  function realizability
  \citep{amortila2023harnessing,amortila2024scalable}), and it is not
  currently known whether value function realizability is sufficient
  in this framework.

  \paragraph{Exogenous Block MDPs}
  Our results in \cref{sec:exbmdp} (\cref{cor:exbmdp}) show that
  general Exogenous Block MDPs are learnable with local simulator
  access. Prior work, on learning EXBMDPs in the online RL model requires additional assumptions:
       \begin{itemize}
       	\item \emph{Deterministic ExBMDP}
          \citep{efroni2022provably}. In this setting, the latent transition
          distribution $\Tendo$ is assumed to be deterministic. In
          this case, it suffices to learn \emph{open-loop} policies
          (i.e., policies that play a deterministic sequence of actions). This avoids compounding errors due to learning imperfect decoders that depend on the exogenous noise, making this setting much less challenging than the general ExBMDP setting.
       	\item \emph{Factored ExMDP} \citep{efroni2022sample}. This is
          an ExBMDP setting with a restrictive structure in which the
          observation is a $d$-dimensional vector and the latent state
          is a $k$-dimensional subset of the observed
          coordinates. This structure prevents the setting from
          subsuming the basic (non-exogenous) Block MDP framework, and makes it
          possible to learn decoders that act only on the endogenous
          state, preventing compounding errors.
       	\item \emph{Bellman completeness}.            \citet{xie2023role} observed that ExBMDPs admit low
           coverability, but their algorithm requires Bellman
           completeness, which is not satisfied by ExBMDPs (see
           \citet{efroni2022provably,islam2023agent}).

       	\end{itemize}

\clearpage

\colt{
  \part{Additional Results}

  This section of the appendix contains additional results omitted
  from the main body due to space constraints.
  
  \section{Applying \mainalg to Exogenous Block MDPs}
          \label{sec:rvfs_exbmdp}
\arxiv{We now apply \mainalg to the Exogenous Block MDP (ExBMDP) model
introduced in
\cref{sec:exbmdp}.}
\colt{In this section, we apply \mainalg to the Exogenous Block MDP (ExBMDP) model
introduced in
\cref{sec:exbmdp}.}
ExBMDPs satisfy coverability (\cref{ass:cover}),
but do not satisfy the pushforward coverability assumption
(\cref{ass:pushforward}) required by \mainalg in general. However, it turns out that ExBMDPs \emph{do} satisfy pushforward coverability when the exogenous noise process is weakly correlated across time; we refer to this new statistical assumption as the \emph{weak correlation condition}. 
\begin{assumption}[Weak correlation condition]
	\label{ass:great} 
For the underlying ExBMDP $\cM$, there is a constant $\Ccor \geq 1$ such that for all $h\in[H-1]$ and $(\xi,\xi')\in \Xi_{h-1}\times \Xi_{h}$, we have\footnote{Throughout this paper, when
		considering the law for the exogenous variables
		$\bxi_1,\ldots,\bxi_H$, we write $\bbP\brk*{\cdot}$ instead of
		$\bbP^{\pi}\brk*{\cdot}$ to emphasize that the law is independent of
		the agent's policy.}
	\begin{align}
		\label{eq:corr}
		\P[\bxi_h =\xi,\bxi_{h+1}=\xi'] \leq \Ccor\cdot  \P[\bxi_h =\xi] \cdot \P[\bxi_{h+1}=\xi'].
	\end{align}
\end{assumption}
The weak correlation property asserts that the joint law for the
exogenous noise variables $\bxi_h$ and $\bxi_{h+1}$ is at most a
multiplicative factor $\Ccor\geq1$ larger than the corresponding product distribution obtained by sampling $\bxi_h$ and $\bxi_{h+1}$ independently from their marginals. This setting strictly generalizes the (non-exogenous) Block MDP model
\citep{krishnamurthy2016pac,du2019latent,misra2019kinematic,zhang2022efficient,mhammedi2023representation}, by allowing for arbitrary stochastic dynamics for the endogenous state and an arbitrary emission process, but requires that temporal correlations in the exogenous noise decay over time.

\arxiv{We show that under}\colt{Under} \cref{ass:great}, pushforward
coverability is satisfied with $\Cpush\leq \Ccor \cdot{}S A$
(\cref{lem:pushforward} in \cref{sec:analysis_exbmdp}). In addition,
$V^{\star}$-realizability is implied by decoder realizability
(\cref{lem:realex}). Thus, by
applying \cref{thm:vpiforward_main} (\setupi), we conclude that \mainalg
efficiently learns a near-optimal policy for any weakly correlated ExBMDP for
which the optimal value function has $\Delta$-gap.\loose

\paragraph{An improved algorithm for ExBMDPs: \mainalge}
At first glance, removing the gap assumption for \mainalg in ExBMDPs seems
difficult: The $V^{\pi}$-realizability assumption required to invoke
\cref{thm:vpiforward_main} (\setupii) is not satisfied by ExBMDPs, as
decoder realizability only implies $V^{\pi}$ realizability for
\emph{endogenous} policies $\pi$.\footnote{We say that a policy $\pi$ is \emph{endogenous}
  if it does not depend on exogenous noise, in the sense that
  $\pi(\x_h)$ is a measurable function of $\phistar(\x_h)$.}
In spite of this, we now show that with a slight modification,
\mainalg can efficiently learn any weakly correlated ExBMDP under decoder
realizability alone (without gap or $V^\pi$-realizability).

Our new variant of \mainalg, $\learnlevel^{\exo}$, is presented in
\cref{alg:learnlevel2}\arxiv{ (deferred to \cref{sec:omitted} for space)}. The
algorithm is almost identical to \learnlevel{}
(\cref{alg:learnlevel3}), with the main difference being that we use
an additional \emph{randomized rounding} step to compute the policies
$\pihat_{1:H}$ from the learned value functions $\Vhat_{1:H}$. In particular, instead
of directly defining the policies $\pihat_{1:H}$ based on the bellman
backups $\cP_{h}[\Vhat_{h+1}]$ as in \cref{eq:Qhat3},
$\learnlevel^\exo$ targets a ``rounded'' version of the backup given by
\begin{align}
	\veps \cdot \ceil{\cP_{h}[\Vhat_{h+1}](x,a)/\veps + \bzeta_h}, \label{eq:know}
	\end{align}
where $\veps\in(0,1)$ is a rounding parameter and $\bzeta_1, \dots, \bzeta_H$
are i.i.d.~random variables sampled uniformly at random from the
interval $[0,1/2]$ (at the beginning of the algorithm's
execution). Concretely, $\learnlevel^\exo$ estimates the bellman
backup $\cP_{h}[\Vhat_{h+1}](x,a)$ in \cref{eq:know} using the local
simulator (as in \cref{eq:Qhat3} of \cref{alg:learnlevel3}), and defines its policies via
\begin{align}
	\pihat_{h}(\cdot) \in \argmax_{a\in \cA} \ceil{\cP_{h}[\Vhat_{h+1}](\cdot, a)/\veps + \bzeta_h}. \label{eq:pihate}
\end{align}

This rounding scheme, which quantizes the Bellman backup
into $\veps^{-1}$ bins with a random offset, is designed to emulate certain properties
implied by the $\Delta$-gap assumption (\cref{ass:gap}). Specifically,
we show that with constant probability over the draw of
$\bzeta_{1:H}$,\zmdelete{and with high probability over the randomness in
$\Phat$,} the policy $\pihat$ in \eqref{eq:pihate} ``snaps'' on to an
\emph{endogenous} policy $\pi$. This means that for
$\learnlevel^{\exo}$ to succeed (with constant probability), it
suffices to pass it a class $\cV$ that realizes the value functions
$(V^{\pi}_h)$ for endogenous policies $\pi\in \Pim$. Fortunately, such a function
class can be constructed explicitly under decoder realizability (\cref{ass:phistar}).%
   \begin{lemma}[\cite{efroni2022sample}]
	\label{lem:realex2}
	For the ExBMDP setting, under \cref{ass:phistar}, the function class $\cV_h\coloneqq
	\{x\mapsto f(\phi(x)) : f\in [0,H]^S, \phi\in \Phi\}$
	is such that $V_h^{\pi}\in \cV_h$ for all endogenous policies $\pi$. Furthermore, the policy class $\Pi_h \coloneqq \left\{ \pi(\cdot)\in \argmax_{a\in \cA} f(\phi(\cdot),a):f\in [0,H]^{S\times A }, \phi\in \Phi \right\}$ contains all endogenous policies.
\end{lemma} 

A small technical challenge with the scheme above is that it is only
guaranteed to succeed with constant probability over the draw of the
rounding parameters $\bzeta_1,\ldots,\bzeta_H$. To address this, we
provide an outer-level algorithm, \forwardexo
(\cref{alg:forward_exbmdp}\arxiv{, deferred to \cref{sec:omitted} for space}), which performs
confidence boosting by invoking \mainalge multiple times
independently, and extracts a high-quality executable policy using
behavior cloning.\loose

\paragraph{Main result}

We now state the main guarantee for $\learnlevel^\exo$ (the proof is in \cref{sec:exbmdp_app}).
  \begin{theorem}[Main guarantee of \mainalge for EXBMDPs]
  \label{thm:exbmdpforward_main}
  Consider the ExBMDP setting. Suppose the decoder class $\Phi$
  satisfies \cref{ass:phistar}, and that \cref{ass:great} holds with $\Ccor>0$.
  Let $\veps,\delta\in(0,1)$ be given, and let $\cV_{h}$ and
        $\Pi_h$ be as in \cref{lem:realex2}.
        Then $\forwardexoeq(\Pi, \cV_{1:H}, \veps, \zeta_{1:H},\delta)$
        (\cref{alg:forward_exbmdp}) produces a policy $\pihat_{1:H}$
        such that $J(\pistar)- J(\pihat_{1:H}) \leq  \veps$, and has
        total sample complexity 
	\begin{align}
		\wtilde{O}\left(\Cexo^8 S^8H^{36}A^9\cdot \veps^{-26}\right).
	\end{align}
\end{theorem}
This result shows for the first time that sample- and
computationally-efficient learning is possible for ExBMDPs beyond
deterministic or factored settings \citep{efroni2022provably,efroni2022sample}.

We mention in passing that our use of randomized rounding
to emulate certain consequences of the $\Delta$-gap
assumption leverages the fact that ExBMDPs have a finite number of
(endogenous) latent states. It is unclear if this technique can be used when the (latent) state space is large or infinite.

\ifdefined\Mnum
\renewcommand{\Mnum}{\ceil{8  \veps^{-2} \Ccor S A H}}
\else
\newcommand{\Mnum}{\ceil{8  \veps^{-2} \Ccor S A H}}
\fi 

\ifdefined\Ntestnum
\renewcommand{\Ntestnum}{2^8  M^2 H \veps^{-2}\log(8 M^6 H^8 \veps^{-2} \delta^{-1})}
\else
\newcommand{\Ntestnum}{2^8  M^2 H \veps^{-2}\log(8 M^6 H^8 \veps^{-2} \delta^{-1})}
\fi

\ifdefined\Nregnum
\renewcommand{\Nregnum}{2^8 M^2 \veps^{-2}\log(8|\cV| HM^2 \delta^{-1})}
\else
\newcommand{\Nregnum}{2^8 M^2 \veps^{-2}\log(8|\cV| HM^2 \delta^{-1})}
\fi

\ifdefined\bbeta
\renewcommand{\bbeta}{\frac{9 MH^2\log(8M^2H|\cV|/\delta)}{N_\reg}  +  \frac{34 MH^3\log(8M^6 N^2_\test  H^8/\delta)}{N_\test} }
\else
\newcommand{\bbeta}{\frac{9 MH^2\log(8M^2H|\cV|/\delta)}{N_\reg}  +  \frac{34 MH^3\log(8M^6 N^2_\test  H^8/\delta)}{N_\test} }
\fi

\ifdefined\Nsimu
\renewcommand{\Nsimu}{2N_\reg^2 \log(8  N_\reg H M^3/\delta)}
\else
\newcommand{\Nsimu}{2N_\reg^2 \log(8  N_\reg H M^3/\delta)}
\fi

\ifdefined\deltaprime
\renewcommand{\deltaprime}{\delta/(8M^7N_\test^2 H^8|\cV|)}
\else
\newcommand{\deltaprime}{\delta/(8M^7N_\test^2 H^8|\cV|)}
\fi

\ifdefined\bbetap
\renewcommand{\bbetap}{\frac{2H\log(4\abs{\cV}  H/\delta)}{N_\reg} + 8H^3 A |\cC_h| \cdot \frac{\log(4H|\cC_h|/\delta)}{N_\test}}
\else
\newcommand{\bbetap}{\frac{2H\log(4\abs{\cV}  H/\delta)}{N_\reg} + 8H^3 A |\cC_h| \cdot \frac{\log(4H|\cC_h|/\delta)}{N_\test}}
\fi

\ifdefined\vepsllnum
\else
\newcommand{\vepsllnum}{\veps H^{-1}/48}
\fi

\begin{algorithm}[H]
	\caption{$\learnlevel^\exo_h$: Recursive Value Function Search
        for Exogenous Block MDPs}
	\label{alg:learnlevel2}
	\begin{algorithmic}[1]
          \State 	\multiline{{\bfseries parameters:} Value function class $\cV$, suboptimality $\veps\in(0,1)$, seeds $\zeta_{1:H}\in (0,1)$, confidence $\delta\in(0,1)$.}
          \State \multiline{{\bfseries input:}
	 Level $h\in[0 \ldotst H]$, value function estimates $\Vhat_{h+1:H}$, confidence sets $\cVhat_{h+1:H}$, state-action collections $\cC_{h:H}$, and buffers $\cB_{h:H}$, and counters $t_{h:H}$.}\vspace{5pt}
		\Statex[0] \algcommentbiglight{Initialize parameters.}
		\State Set $M \gets \Mnum$.\label{line:paramsExBMDPM}
		\State Set $N_\test\gets  \Ntestnum$, $N_\reg\gets  \Nregnum$. 
		\State Set $\Nest(k) \gets 2N_\reg^2 \log(8 A N_\reg H k^3/\delta)$ and $\delta'\gets \delta/(4M^7N_\test^2 H^8|\cV|)$.  \label{line:paramsExBMDP}
		\State Set $\veps_\reg^2 \gets \bbeta$. \label{line:beta_ex}
		\State Set $\beta(t) \gets \sqrt{\log_{1/\delta'}(8 M A|\cV|t^2/\delta)}$.
		\Statex[0] 		  \algcommentbiglight{Test the fit for the
			estimated value functions $\Vhat_{h+1:H}$ at future layers.}
		\For{$(x_{h-1},a_{h-1})\in\cC_h$}\label{line:begin2}
		\For{layer $\ell = H,\dots,h+1$}
		\For{$n=1,\ldots,N_\test$}
		\State \multiline{Draw $\x_h\sim{}T_{h-1}(\cdot\mid{}x_{h-1},a_{h-1})$, then draw
		$\x_{\ell-1}$ by rolling out with $\pibell_{h+1:H}$, where  \label{line:draw2}
		\begin{align}
		\label{eq:Qhat2}
		\forall \tau\in[H], \ \  \pibell_\tau(\cdot)\in  \argmax_{a\in \cA}  \ceil*{\Phat_{\tau,\veps^2, \delta'}[\Vhat_{\tau+1}](\cdot,a)\cdot \veps^{-1} + \zeta_\tau}, \quad \text{with}\quad  \Vhat_{H+1}\equiv 0.
		\end{align}
	}
		\For{$a_{\ell-1}\in \cA$} 
		\State Update $t_\ell \gets t_\ell +1$. %
			\Statex[4]\algcommentbiglight{Test fit; if test fails, re-fit value functions
			$\Vhat_{h+1:\ell}$ up to layer $\ell$.}
		\If{$\sup_{f\in \cVhat_{\ell}} |(\Phat_{\ell-1,\veps^2, \delta'}[\Vhat_{\ell}]- \Phat_{\ell-1,\veps^2, \delta'}[f_\ell])( \x_{\ell-1},a_{\ell-1})|> \veps^2 + \veps^2 \cdot \beta(t_\ell)$}\label{line:test2}%
		\State $\cC_{\ell}\gets\cC_{\ell}\cup\crl{(\x_{\ell-1},a_{\ell-1})}$ and $\cB_\ell \gets \cB_\ell \cup\{ (\x_{\ell-1},a_{\ell-1}, \Vhat_\ell ,\cVhat_\ell, t_\ell)\}$. \label{line:added2}
		\For{$\tau= \ell,\dots, h+1$} \label{line:forloop2}
		\State $(\Vhat_{\tau:H},\cVhat_{\tau:H},
		\cC_{\tau:H},\cB_{\tau:H}, t_{\tau:H})\gets{}\learnlevel^\exo_{\tau}(\Vhat_{\tau+1:H},\cVhat_{\tau+1:H},\cC_{\tau:H},\cB_{\tau:H}, t_{\tau:H};\cV,\veps, \zeta_{1:H},\delta)$.
		\EndFor
		\State \textbf{go to line \ref*{line:begin3}}.
                \label{line:goto2}
		\EndIf
		\EndFor
		\EndFor
		\EndFor
		\EndFor
		\If{$h=0$} \textbf{return} $(\Vhat_{1:H},\cdot,\cdot,\cdot, \cdot)$.
		\EndIf
		\Statex[0] \algcommentbiglight{Re-fit $\Vhat_h$ and build a new confidence set.}
		\For{$(x_{h-1},a_{h-1})\in\cC_h$}
		\State Set $\cD_h(x_{h-1},a_{h-1})\gets\emptyset$.
		\For{$i=1,\dots, N_\reg$}
		\State Sample
		$\x_h\sim{}T_{h-1}(\cdot\mid{}x_{h-1},a_{h-1})$. 
		\State \multiline{For each $a\in \cA$, let $\Vhat_h(\x_h)$
			be a Monte-Carlo estimate for $\E^{\pibell_{h:H}}[\sum_{\ell=h}^H \bm{r}_\ell \mid
			\x_h]$ computed by collecting $\Nest(|\cC_h|)$ trajectories starting from $\x_h$ and rolling out with
			$\pibell_{h:H}$.}%
		\State Update $\cD(x_{h-1},a_{h-1})\gets \cD(x_{h-1},a_{h-1}) \cup \{(\x_h, \Vhat_h(\x_h))\}$.
		\EndFor
		\EndFor
		\State Let $\Vhat_h\ldef\argmin_{f\in\cV_h}\sum_{(x_{h-1},a_{h-1})\in\cC_h} \sum_{(x_h,v_h)\in\cD_{h}(x_{h-1},a_{h-1})}(f(x_h)-v_h)^2$. \label{line:updateQ2}
		\State Compute value function confidence set 
		\begin{align}
			\cVhat_{h} \coloneqq \left\{ f\in \cV_h \left| \  \sum_{(x_{h-1},a_{h-1})\in\cC_h} \frac{1}{N_\reg} \sum_{(x_h,\text{-}) \in \cD_h(x_{h-1},a_{h-1})}\left(\Vhat_h(x_{h}) -f(x_h)\right)^2 \leq   \veps_\reg^2  \right.     \right\}.  \label{eq:confidence2} 
		\end{align}
		\State \textbf{return} $(\Vhat_{h:H},\cVhat_{h:H}, \cC_{h:H}, \cB_{h:H}, t_{h:H})$.
	\end{algorithmic}
\end{algorithm}

\begin{algorithm}[H]
\caption{\forwardexo: Learn an executable policy with
  $\learnlevel^\exo$ via imitation learning.}
\label{alg:forward_exbmdp}

\begin{algorithmic}[1]
  \State {\bfseries input:} Decoder class $\Phi$, suboptimality $\veps \in(0,1)$, confidence $\delta\in(0,1)$. 
	\Statex[0] \algcommentbiglight{Set parameters for \learnlevel{} and define the value function and policy classes.}
	\State Set $\vepsll\gets \veps H^{-1}/48$.
	\State Set $\cV=\cV_{1:H}$, where $\cV_h=
	\{x\mapsto f(\phi(x)) : f\in [0,H]^S, \phi\in \Phi\}$, $\forall h\in[H]$.
	\State Set $\Pi = \Pi_{1:H}$, where $\Pi_h = \left\{
          \pi(\cdot)\in \argmax_{a\in \cA} f(\phi(\cdot),a):f\in
          [0,H]^{S\times A }, \phi\in \Phi \right\}$, $\forall
        h\in[H]$.
        \Statex[0] \algcommentbiglight{Set parameters for \forward.}
	\State  \mbox{Set $N_\imit \gets 8H^2 \log (4H|\Pi|/\delta)/\veps$, 	$N_{\boost} \gets \log(1/\delta)/\log(24 SA H \veps)$, $N_{\texttt{eval}} \gets 16^2 \veps^{-2}\log(2 N_{\boost}/\delta)$. }
	\State  \mbox{Set $M \gets \ceil{8 \veps_\learnlevel^{-1} S A \Ccov H}$, $N_\test \gets 2^{8} M^2 H \veps_\learnlevel^{-1} \log(80 M^6 H^8N_\boost \veps_\learnlevel^{-2} \delta^{-1})$, and $\delta' = \frac{\delta}{40 M^7 N^2 H^8|\cV|N_\boost}$.}
	\State Set $N_\reg \gets 2^{8} M^2 \veps_\learnlevel^{-1} \log(80|\Phi|^2 H M^2 N_\boost \delta^{-1})$.
	\State Set $\Vhat_{1:H}\gets\mathsf{arbitrary}$,
        $\cVhat_{1:H}\gets \cV$, $\cC_{0:H}\gets\emptyset$, $\cB_{0:H}
        \gets\emptyset$, $i_\texttt{opt}=1$, and $J_{\max}=0$.
        \Statex[0] \algcommentbiglight{Repeatedly invoke $\learnlevel^\exo$ and
          extract policy with \forward{} to boost confidence.}
	\For{$i = 1,\dots, N_\boost$}
	\Statex[1] \algcommentbiglight{Invoke $\learnlevel^\exo$.}
	\State $(\Vhat\ind{i}_{1:H},\cdot,\cdot,\cdot,
        \cdot)\gets{}\learnlevel^\exo_{0}(\Vhat_{1:H},\cVhat
        _{1:H},\cC_{0:H},\cB_{0:H};\cV
        ,N_\reg,N_\test,\vepsll,\delta/(10N_\boost))$.
        \Statex[1] \algcommentbiglight{Imitation learning with \forward.}
      \State Define $\mb{\pihat}^{\learnlevel}_h(\cdot) \in \argmax_{a\in \cA} \Phat_{h,\veps_\learnlevel, \delta'}[\Vhat\ind{i}_{h+1}](\cdot,a)$.
      \State Compute $\pihat\ind{i}_{1:H}\gets \forward(\Pi, \veps, \pihat_{1:H}^\learnlevel, \delta/(2N_\boost))$.
	\Statex[1] \algcommentbiglight{Evaluate current policy.}
	\State $v = 0$.
	\For{$ =1, \dots, N_\texttt{eval}$} 
	\State Sample trajectory $(\x_1,\a_1,\bm{r}_1,\dots, \x_H,\a_H,\bm{r}_H)$ by executing $\pihat\ind{i}_{1:H}$.
	\State Set $v\gets v+ \sum_{h=1}^H \bm{r}_h$. 
	\EndFor
	\State Set $\widehat{J}(\pihat\ind{i}_{1:H})  \gets v/N_\texttt{eval}$.
	\If{$\widehat{J}(\pihat\ind{i}_{1:H}) >J_{\max}$}
	\State Set $i_\texttt{opt}=i$.
	\State Set $J_{\max} =\widehat{J}(\pihat\ind{i}_{1:H}) $.
	\EndIf
	\EndFor
	\State \textbf{return:} $\pihat_{1:H}=\pihat\ind{i_\texttt{opt}}_{1:H}$.
\end{algorithmic}
\end{algorithm}

 }

		\section{Helper Lemmas}
		This section of the appendix contains supporting lemmas used within
the proofs of our main results. 
		
 	\label{sec:helper}
\ifdefined\Lhat
\else
\newcommand{\Lhat}{\widehat{L}}
\fi

\subsection{Concentration and Probability}

\begin{lemma}
	\label{lem:unionbound}
	Let $\delta \in(0,1)$ and $H\geq 1$ be given. If a sequence of events $\cE_1,\ldots,\cE_H$ satisfies $\bbP[\cE_h\mid{}\cE_1,\ldots,\cE_{h-1}]\geq{}1-\delta/H$ for all $h\in[H]$, then \[\bbP[\cE_{1:H}]\geq{}1-\delta.\]
\end{lemma}
\begin{proof}[\pfref{lem:unionbound}]
	By the chain rule, we have
	\begin{align}
		\bbP[\cE_{1:H}] = \prod_{h\in[H]} \bbP[\cE_h\mid{}\cE_1,\ldots,\cE_{h-1}] \geq  \prod_{h\in[H]} (1-\delta/H) =(1-\delta/H)^H \geq 1-\delta.
	\end{align}
\end{proof}

We make use of the following version of Freedman's inequality, due to \citet[Lemma 9]{agarwal2014taming}:
\begin{lemma}
	\label{lem:freed}
	Let $R>0$ be given and let $\w_1,\dots \w_n$ be a sequence of
        real-valued random variables adapted to filtration
        $\cH_1,\cdots, \cH_n$. Assume that for all $t\in[n]$, $\w_i \leq R$ and $\E[\w_i\mid \cH_{i-1}]=0$. Define $\bm{S}_n\coloneqq \sum_{t=1}^n\w_i$ and $V_n\coloneqq \sum_{t=1}^n \E[\w_i^2\mid \cH_{i-1}].$ Then, for any $\delta \in(0,1)$ and $\lambda \in[0,1/R]$, with probability at least $1-\delta$, 
	\begin{align}
		\bm{S}_n \leq \lambda V_n + \ln (1/\delta)/\lambda.
	\end{align}
      \end{lemma}

We will also use the following lemma, which is a standard consequence
of Freedman's inequality.      
\begin{lemma}[e.g., \citet{foster2021statistical}]
	\label{lem:multiplicative_freedman}
	Let $(\w_t)_{t\leq{T}}$ be a sequence of random
	variables adapted to a filtration $\prn{\cH_{t}}_{t\leq{}T}$. If
	$0\leq{}\w_t\leq{}R$ almost surely, then with probability at least
	$1-\delta$,
	\begin{align}
		&\sum_{t=1}^{T}\w_t \leq{}
		\frac{3}{2}\sum_{t=1}^{T}\En_{t-1}\brk*{\w_t} +
		4R\log(2\delta^{-1}),
		\intertext{and}
		&\sum_{t=1}^{T}\En_{t-1}\brk*{\w_t} \leq{} 2\sum_{t=1}^{T}\w_t + 8R\log(2\delta^{-1}).
	\end{align}
\end{lemma}

\subsection{Regression}
      
Using \cref{lem:freed,lem:multiplicative_freedman}, we obtain the following concentration lemma, which will be
used to prove guarantees for square loss regression within our algorithms.
\begin{lemma}	
	\label{lem:corbern}
	Let $B>0$ and $n\in \mathbb{N}$ be given, and let $\cY$ be an
        abstract set. Further, let $\cQ \subseteq \{g:\cY \rightarrow
        [0,B]\}$ be a finite function class and $\y_1,\dots, \y_n$ be
        a sequence of random variables in $\cY$ adapted to filtration
        a $\cH_1, \cdots, \cH_n$. Then, for any $\delta \in(0,1)$, with probability at least $1-\delta$, we have 
	\begin{align}
		\forall g\in \cQ, \quad \frac{1}{2} \|g\|^2 - 2 B^2\log(2|\cQ|/\delta) \leq \|g\|_n^2 \leq 2 \|g\|^2 + 2B^2 \log(2|\cQ|/\delta),
		\end{align}
		where $\|g\|^2 \coloneqq \sum_{i\in[n]}\E[g(\y_i)^2\mid \cH_{i-1}]$ and $\|g\|_n^2 \coloneqq \sum_{i=1}^n g(\y_i)^2$.
\end{lemma}
\begin{proof}[\pfref{lem:corbern}]
Fix $g\in\cQ$. Applying \cref{lem:freed} with $\w_i = g(\y_i)^2- \E[g(\y_i)^2\mid \cH_{i-1}]$, for all $i\in[n]$, and $(R,\lambda) =(B^2,1/B^2)$, we get that with probability at least $1-\delta/(2|\cQ|)$:
\begin{align}
	\|g\|_n^2 - \|g\|^2 &\leq \lambda B^2 \|g\|^2 + \log(2|\cQ|/\delta)/\lambda.
\end{align}
By substituting $\lambda = B^{-2}$ and rearranging, we get 
\begin{align}
	\|g\|_n^2 \leq 2 \|g\|^2 + B^2 \log (2|\cQ|/\delta). \label{eq:firts}
	\end{align}
Similarly, applying \cref{lem:freed} with $\w_i = \E[g(\y_i)^2\mid \cH_{i-1}]-g(\y_i)^2$, for all $i\in[n]$, and $(R,\lambda) =(B^2, 1/(2B^2))$, we get that with probability at least $1-\delta/(2|\cQ|)$:
\begin{align}
	\|g\|^2 - \|g\|_n^2  &\leq \lambda B^2 \|g\|^2 + \log(2|\cQ|/\delta)/\lambda.
\end{align}
By substituting $\lambda = 2^{-1}B^{-2}$ and rearranging, we get 
\begin{align}
	\|g\|_n^2 \geq  \frac{1}{2} \|g\|^2 -2B^2 \log (2|\cQ|/\delta). 
\end{align}
Combining this with \eqref{eq:firts} and the union bound, we get the desired result.
\end{proof}

With this lemma, we now prove the following key result for square loss regression.
\begin{lemma}[Generic regression guarantee]
		\label{lem:reg}
		Let $B>0$ and $n\in \mathbb{N}$ be given and $\cY$ be an abstract set. Further, let $\cF \subseteq \{f:\cY \rightarrow [0,B]\}$ be a finite function class, and suppose that there is a function $f_\star \in \cF$ and a sequence of random variables $(\y_1,\x_1),\dots, (\y_n,\x_n)\in \cY\times \reals$ such that for all $i\in[n]$:
		\begin{itemize}
			\item $\x_i = f_\star(\y_i) + \bm{\veps}_i + \bm{b}_i$; 
			\item$|\bm{b}_i|\leq \xi$;
			\item $\bm{\veps}_i\in[-B,B]$; and
			\item$\E[\bm{\veps}_i \mid
                          \mathfrak{F}_i]=0$, where $\mathfrak{F}_i \coloneqq \sigma(\y_{1:i}, \bm{\veps}_{1:i-1}, \x_{1:i-1}, \bm{b}_{1:i-1})$.
			\end{itemize}
Then, for $\fhat \in \argmin_{f\in \cF}\sum_{i=1}^n(f(\y_i)- \x_i)^2$ and any $\delta \in(0,1)$, with probability at least $1-\delta/2$,
		\begin{align}
		\|\fhat - f_\star\|_n^2  \leq 4B^2 \log(2|\cF|/\delta)+4 B \sum_{i=1}^n |\bm{b}_i|,
		\end{align}
	where $\|\fhat- f_\star\|^2_n  \coloneqq \sum_{i=1}^n (\fhat(\y_i)- f^\star(\y_i))^2$.
      \end{lemma} 
	\begin{proof}[\pfref{lem:reg}]
		Fix $\delta \in (0,1)$ and let $\Lhat_n(f) \coloneqq \sum_{i=1}^n (f(\y_i)- \x_i)^2$, for $f\in \cF$, and note that since $\fhat \in \argmin_{f\in \cF}\Lhat_n(f)$, we have
		\begin{align}
		0 \geq \Lhat_n(\fhat)- \Lhat_n(f_\star) = \nabla \Lhat_n(f_\star)[\fhat - f_\star] + \|\fhat - f_\star\|^2_n,
		\end{align}
  where $\nabla$ denotes directional derivative. Rearranging, we get that 
	\begin{align}
	\|\fhat - f_\star\|^2_n 
	&\leq - 2 \nabla \Lhat_n(f_\star)[\fhat - f_\star] - \|\fhat - f_\star\|_n^2, \nn \\
	& = 4 \sum_{i=1}^n (\x_i - f_\star(\y_i)) (\fhat(\y_i)- f_\star(\y_i))- \|\fhat - f_\star\|^2_n, \nn \\
	& \leq 4 \sum_{i=1}^n (\bm{\veps}_i + \bm{b}_i) (\fhat(\y_i)- f_\star(\y_i))- \|\fhat - f_\star\|^2_n, \nn \\
	& \leq \underbrace{4\sum_{i=1}^n \bm{\veps}_i \cdot (\fhat(\y_i)- f_\star(\y_i))- \|\fhat - f_\star\|^2_n}_{\text{I}}  +\underbrace{4 \sum_{i=1}^n  \bm{b}_i\cdot  (\fhat(\y_i)- f_\star(\y_i))}_{\text{II}}. \label{eq:twoterms}
	\end{align}
	\paragraphi{Bounding Term I} To bound Term I, we apply \cref{lem:freed} with $\w_i = \bm{\veps}_i \cdot (\fhat(\y_i)- f_\star(\y_i))$, $R = B^2$, $\lambda = 1/(8B^2)$, and $\cH_{i}=\mathfrak{F}_{i+1}^{-}$, and use 
	\begin{enumerate}
		\item the union bound over $f\in \cF$; and
		\item the facts that $\E[\y_i\mid \mathfrak{F}_{i}^-]= \y_i$ and $\E[\bm{\veps_i}\mid \mathfrak{F}_{i}^-]= 0$,
		\end{enumerate}
		to get that with probability at least $1-\delta/2$,
	\begin{align}
		 4\sum_{i=1}^n \bm{\veps}_i \cdot (\fhat(\y_i)- f_\star(\y_i)) \leq \|\fhat -f_\star\|^2_n + 4 B^2 \log(2|\cF|/\delta).
		\end{align}
		By rearranging, we get that with probability at least $1-\delta/2$,  
		\begin{align}
			\text{Term I} \leq 4 B^2 \log(2 |\cF|/\delta). \label{eq:term1}
			\end{align}
			\paragraphi{Bounding Term II} We now bound the
                        second term in \eqref{eq:twoterms}. For this, note that since $\|\fhat - f_\star\|_{\infty}\leq B$, we have 
			\begin{align}
				\text{Term II} \leq 4 B\sum_{i=1}^n |\bm{b}_i|.
				\end{align}
				This completes the proof. 
	\end{proof}

\subsection{Reinforcement Learning}
\begin{lemma}[Performance Difference Lemma (e.g., \citet{kakade2003sample})]
	\label{lem:perform}
	For any two policies $\pihat, \pi \in \Pims$ and $t \in[H]$, we have
	\begin{align}
		\E^{\pi}\left[ V^{\pi}_t(\x_t)- V^{\pihat}_t(\x_t)\right] = \E^{\pi}\left[\sum_{h=t}^H Q_h^{\pihat}(\x_h,\bm\pi_{h}(\x_h)) - Q_h^{\pihat}(\x_h, \pihatb_h(\x_h)) \right]. \label{eq:bla}
	\end{align}
	In particular, applying this for $t = 1$ gives
	\begin{align}
			J({\pi})-J(\pihat) = \E^{\pi}\left[\sum_{h=1}^H Q_t^{\pihat}(\x_h,\bm{\pi}_{h}(\x_h)) - Q_h^{\pihat}(\x_h, \pihatb_h(\x_h)) \right].
        \end{align}
\end{lemma}

\begin{lemma}[Potential lemma \citep{xie2023role}]
	\label{lem:potential}
	Fix $h\in\brk{H}$. Suppose we have a sequence of functions $g\ind{1},\ldots,g\ind{T}\in\brk{0,B}$
	and policies $\pi\ind{1},\dots,
	\pi\ind{T}$ such that 
	\begin{align}
		\forall t \in[T], \quad  \sum_{i<t}\E^{\pi\ind{i}}\brk*{(g\ind{t}(\x_h))^2}
		\leq \beta^2
	\end{align}
        for some $\beta\geq{}0$.
	Then under \cref{ass:pushforward}, we have
            \begin{align}
    \sum_{t=1}^{T}\En^{\pi\ind{t}}\brk*{g\ind{t}(\x_h)}
              \leq 2\sqrt{\beta^2\Cpush{}T\log(2T)} + 2B\Cpush,
            \end{align}
            and consequently
	\begin{align}
		\min_{t\in[T]}\En^{\pi\ind{t}}\brk*{g\ind{t}(\x_h)}
		\leq 2\sqrt{\frac{\beta^2 \Cpush  \log (2T)}{T} } +\frac{ 2B\Cpush}{T}. \label{eq:firstone}
	\end{align}
      \end{lemma}
      \begin{proof}[\pfref{lem:pushforward_potential}]%
       See proof of \cite[Theorem 1]{xie2023role}.
	  \end{proof}

The following result is a variant of the coverability-based potential
argument given in \citet{xie2023role}.

	\begin{lemma}[Pushforward coverability potential lemma]
	\label{lem:pushforward_potential}
	Fix $h\in\brk{H}$. Suppose we have a sequence of functions $g\ind{1},\ldots,g\ind{T}\in\brk{0,B}$
	and state-action pairs $(x\ind{1},a\ind{1}),\ldots,
	(x\ind{T},a\ind{T})$ such that 
	\begin{align}
		\forall t \in[T], \quad  \sum_{i<t}\E\brk*{(g\ind{t}(\x_h))^2\mid{}\x_{h-1}=x\ind{i},\a_{h-1}=a\ind{i}}
		\leq \beta^2
	\end{align}
        for some $\beta\geq{}0$.
	Then under \cref{ass:pushforward}, we have
            \begin{align}
    \sum_{t=1}^{T}\En\brk*{g\ind{t}(\x_h)\mid{}\x_{h-1}=x\ind{t},\a_{h-1}=a\ind{t}}
              \leq 2\sqrt{\beta^2\Cpush{}T\log(2T)} + 2B\Cpush,
            \end{align}
            and consequently
	\begin{align}
		\min_{t\in[T]}\En\brk*{g\ind{t}(\x_h)\mid{}\x_{h-1}=x\ind{t},\a_{h-1}=a\ind{t}}
		\leq 2\sqrt{\frac{\beta^2 \Cpush  \log (2T)}{T} } +\frac{ 2B\Cpush}{T}. \label{eq:firstone}
	\end{align}
      \end{lemma}
      \begin{proof}[\pfref{lem:pushforward_potential}]%
        \newcommand{\dtil}{\wt{d}}%
  Define
  $d_h\ind{t}(x)\ldef\bbP[\x_h=x\mid{}\x_{h-1}=x\ind{t},\a_{h-1}=a\ind{t}]$,
  and let $\dtil_h\ind{t}\ldef\sum_{i<t}d\ind{i}$. Let
  \[
\tau_h(x)\ldef{}\min\crl*{t\mid{}\dtil_h\ind{t}(x)\geq\Cpush\cdot\mu_h(x)}.
\]
We have
\begin{align}
&   \sum_{t=1}^{T}\En\brk*{g\ind{t}(\x_h)\mid{}\x_{h-1}=x\ind{t},\a_{h-1}=a\ind{t}}
  \nn \\ &= \sum_{t=1}^{T}\sum_{x\in\cX}d\ind{t}_h(x)g\ind{t}(x) \\
  &\leq
  \sum_{t=1}^{T}\sum_{x\in\cX}d\ind{t}_h(x)g\ind{t}(x)\indic\crl{t\geq{}\tau_h(x)}
  + B\sum_{t=1}^{T}\sum_{x\in\cX}d\ind{t}_h(x)\indic\crl{t<\tau_h(x)}.
\end{align}
From the definition of pushforward coverability, we can bound
\begin{align}
  \sum_{t=1}^{T}\sum_{x\in\cX}d\ind{t}_h(x)\indic\crl{t<\tau_h(x)}
  =   \sum_{x\in\cX}\dtil\ind{\tau_h(x)}_h(x)
  \leq 2\Cpush \sum_{x\in\cX}\mu_h(x)=2\Cpush.
\end{align}
For the other term, we bound
\begin{align}
  &\sum_{t=1}^{T}\sum_{x\in\cX}d\ind{t}_h(x)g\ind{t}(x)\indic\crl{t\geq{}\tau_h(x)} \\
  & \leq \prn*{
    \sum_{t=1}^{T}\sum_{x\in\cX}\frac{(d\ind{t}_h(x))^2}{\dtil\ind{t}_h(x)}\indic\crl{t\geq{}\tau_h(x)}}^{1/2}\cdot
    \prn*{
    \sum_{t=1}^{T}\sum_{x\in\cX}\dtil_h\ind{t}(x)(g_h\ind{t}(x))^2
    }^{1/2}\\
    & = \prn*{
    \sum_{t=1}^{T}\sum_{x\in\cX}\frac{(d\ind{t}_h(x))^2}{\dtil\ind{t}_h(x)}\indic\crl{t\geq{}\tau_h(x)}}^{1/2}\cdot
    \prn*{
    \sum_{t=1}^{T}    \sum_{i<t}\En\brk*{(g\ind{t}(\x_h))^2\mid{}\x_{h-1}=x\ind{i},\a_{h-1}=a\ind{i}}
      }^{1/2} \\
  &\leq     \prn*{\sum_{t=1}^{T}\sum_{x\in\cX}\frac{(d\ind{t}_h(x))^2}{\dtil\ind{t}_h(x)}\indic\crl{t\geq{}\tau_h(x)}}^{1/2}\cdot\sqrt{\beta^2T}.
\end{align}
Finally, we have
\begin{align}
  \sum_{t=1}^{T}\sum_{x\in\cX}\frac{(d\ind{t}_h(x))^2}{\dtil\ind{t}_h(x)}\indic\crl{t\geq{}\tau_h(x)}
  &  \leq 2
    \sum_{t=1}^{T}\sum_{x\in\cX}\frac{(d\ind{t}_h(x))^2}{\dtil\ind{t}_h(x)+\Cpush\mu_h(x)}\\
  &  \leq 2\Cpush
    \sum_{t=1}^{T}\sum_{x\in\cX}\mu_h(x)\frac{d\ind{t}_h(x)}{\dtil\ind{t}_h(x)+\Cpush\mu_h(x)}\\
  &  = 2\Cpush
    \sum_{x\in\cX}\mu_h(x)
    \sum_{t=1}^{T}\frac{d\ind{t}_h(x)}{\dtil\ind{t}_h(x)+\Cpush\mu_h(x)}\\
    &  = 4\Cpush\log(T+1),
\end{align}
where the last line uses Lemma 4 of \citet{xie2023role}.
\end{proof}

\clearpage
 
	\part{Proofs for \ggolf (\creftitle{sec:sample})}

        \label{sec:analysis_golf}
\renewcommand{\filt}{\mathcal{H}}
\newcommand{\ellbar}{\wb{\ell}}

  As described in \cref{sec:golf}, the main difference between
  \ggolf{} and \golf lies in the construction of the confidence
  sets. The most important new step in the proof of
  \cref{thm:golf} is to show that the local simulator-based confidence
  set construction in \cref{line:cond} is valid in the sense that the
  property \cref{eq:golf_conf} holds with high probability. From here,
  the sample complexity bound follows by adapting the
  change-of-measure argument based on coverability from
  \citet{xie2023role}.

  To this end, this part of the appendix is organized as follows. We first state and prove technical lemmas
  concerning realizability (\cref{lem:realex}) and the confidence set
  construction 
  (\cref{lem:ggolf_concentration} and \cref{lem:ggolf_confidence}) in
  \cref{sec:golf_prelim}. Then, in \cref{proof:golf}, we prove
  \cref{thm:golf} as a consequence.

\section{Preliminary Lemmas for Proof of \creftitle{thm:golf}}
\label{sec:golf_prelim}

  For this section, we define
\begin{align}
	\ell_h\ind{t}(g)
	\coloneqq \left(
		g_h(\x_h\ind{t}, \a_h\ind{t})
		- \frac{1}{K}\sum_{k=1}^{K} \left( \br_h\ind{t,k}+ \max_{a\in \cA}g_{h+1}(\x_{h+1}\ind{t,k},a) \right) 
	\right)^2
\end{align}
and
\begin{align}
	\ellbar_h\ind{t}(g)
	\ldef  \En^{\pi\ind{t}}\brk*{\prn*{g_h(\x_h, \a_h)-\cT_h[g_{h+1}](\x_h,\a_h)}^2},
\end{align}
where $(\x_h\ind{t}, \a_h\ind{t},\br_h\ind{t,k},\x_{h+1}\ind{t,k})$ are as in \cref{alg:generative_golf}.
\colt{
\begin{lemma}[\cite{efroni2022sample}]
  \label{lem:realex}
  For the ExBMDP setting, under \cref{ass:phistar}, the function class $\cQ_h\coloneqq
  \{(x,a)\mapsto g(\phi(x),a) : g\in [0,H]^{S A}, \phi\in \Phi\}$
  satisfies \cref{ass:realgolf} and has
  $\log\abs{\cQ_h}=\log\abs{\Pi_h}=\bigoht(SA+\log\abs{\Phi})$.\footnote{Formally,
    this requires a standard covering number argument; we omit
    the details.}
\end{lemma} 
}

\begin{lemma}
	\label{lem:ggolf_concentration}
	With probability at least $1-\delta$, for all $h\in\brk{H}$, 
	$t\in\brk{N}$, and $g\in\cQ$,
	\begin{align}
		&\sum_{i<t}\ls_h\ind{i}(g) \leq{}
		3\sum_{i<t}\En^{\pi\ind{i}}\brk*{\prn*{
				g_h(\x_h, \a_h) - \cT_h[g_{h+1}](\x_h,\a_h)}^2} +
		\frac{8N}{K}
		+ 
		16\log(2HN\abs*{\cQ}\delta^{-1}),
		\intertext{and}
		&\sum_{i<t}\En^{\pi\ind{i}}\brk*{\prn*{
				g_h(\x_h,\a_h) - \cT_h[g_{h+1}](\x_h,\a_h)}^2} \leq{}
		4\sum_{i<t}\ls_h\ind{i}(g) + \frac{8N}{K} +64\log(2HN\abs*{\cQ}\delta^{-1}).
		\label{eq:ggolf_concentration}
	\end{align}
\end{lemma}

\begin{proof}[\pfref{lem:ggolf_concentration}]
	Let $t\in\brk{N}$ and $h\in\brk{H}$ be fixed. Let us denote
	$\bz_h\ind{t}=\crl*{(\br_h\ind{t,k},\x_{h+1}\ind{t,k})}_{k\in\brk{K}}$. Define a filtration 
	\begin{align}
		\filt\ind{t}
		= \sigma(\bm{\tau}\ind{1},\bz\ind{1}_1,\ldots,\bz\ind{1}_H,\ldots,\bm{\tau}\ind{t},\bz\ind{t}_1,\ldots,\bz\ind{t}_H),
	\end{align}
	where $\bm{\tau}\ind{i}$ is the trajectory generated in the
        $i$th iteration of \cref{alg:generative_golf} (see \cref{line:tau}). Fix $g\in\cQ$. Observe that $\ls_h\ind{i}(g)\in\brk*{0,4}$, so
	\cref{lem:multiplicative_freedman} ensures that with probability at
	least $1-\delta$,
	\begin{align}
		&\sum_{i<t}\ls_h\ind{i}(g) \leq{}
		\frac{3}{2}\sum_{i<t}\En\brk*{\ls_h\ind{i}(g)\mid{}\filt\ind{i-1}} +
		16\log(2\delta^{-1}),
		\intertext{and}
		&\sum_{i<t}\En\brk*{\ls_h\ind{i}(g)\mid{}\filt\ind{i-1}} \leq{}
		2\sum_{i<t}\ls_h\ind{i}(g) + 32\log(2\delta^{-1}).
		\label{eq:ggolf_conc0}
	\end{align}
	By the AM-GM inequality, for all $i<t$, we can bound
		\begin{align}
		&\En\brk*{\ls_h\ind{i}(g)\mid{}\filt\ind{i-1}} \\
				&=\En\brk*{\prn*{
				g_h(\x_h\ind{i}, \a_h\ind{i})
				- \frac{1}{K}\sum_{k=1}^{K}\prn*{\br_h\ind{i,k}+\max_{a\in \cA}g_{h+1}(\x_{h+1}\ind{i,k}, a)}
			}^2\mid\filt\ind{i-1}} \\
				&\leq
		2  \En\brk*{\prn*{
				g_h(\x_h\ind{i}, \a_h\ind{i}) - \cT_h[g_{h+1}](\x_h\ind{i},\a_h\ind{i})}^2\mid\filt\ind{i-1}} \nn \\
			& \quad 
			+ 2\En\brk*{\prn*{
					\cT_h[g_{h+1}](\x_h\ind{i},\a_h\ind{i})
						-\frac{1}{K}\sum_{k=1}^{K}\prn*{\br_h\ind{i,k}+\max_{a\in \cA} g_{h+1}(\x_{h+1}\ind{i,k},a)}
					}^2\mid\filt\ind{i-1}}.
		\end{align}
		and
		\begin{align}
			&\En\brk*{\ls_h\ind{i}(g)\mid{}\filt\ind{i-1}} \\
			&\geq
			\frac{1}{2}\En\brk*{\prn*{
					g_h(\x_h\ind{i}, \a_h\ind{i}) - \cT_h[g_{h+1}](\x_h\ind{i},\a_h\ind{i})}^2\mid\filt\ind{i-1}} \nn \\
				& \quad 
			-\En\brk*{\prn*{
					{\cT_h[g_{h+1}]}(\x_h\ind{i},\a_h\ind{i})
						-\frac{1}{K}\sum_{k=1}^{K}\prn*{\br_h\ind{i,k}+\max_{a\in \cA} g_{h+1}(\x_{h+1}\ind{i,k}, a)}
					}^2\mid\filt\ind{i-1}}.
			\end{align}
			We have
			\begin{align}
				\En\brk*{\prn*{
						g_h(\x_h\ind{i}, \a_h\ind{i}) -
						\cT_h[g_{h+1}](\x_h\ind{i},\a_h\ind{i})}^2\mid\filt\ind{i-1}}
				= \En^{\pi\ind{i}}\brk*{\prn*{
						g_h(\x_h, \a_h) - \cT_h[g_{h+1}](\x_h,\a_h)}^2}
			\end{align}
			and %
			\begin{align}
				& \En\brk*{\prn*{
							\cT_h[g_{h+1}](\x_h\ind{i},\a_h\ind{i})
							-\frac{1}{K}\sum_{k=1}^{K}\prn*{\br_h\ind{i,k}+\max_{a\in\cA}g_{h+1}(\x_{h+1}\ind{i,k}, a)}
						}^2\mid\filt\ind{i-1}} \\
					&=\En\brk*{\En\brk*{\prn*{
								\cT_h[g_{h+1}](\x_h,\a_h)
									-\frac{1}{K}\sum_{k=1}^{K}\prn*{\br_h\ind{i,k}+\max_{a\in \cA}g_{h+1}(\x_{h+1}\ind{i,k}, a)}
								}^2\mid\x_h=\x_h\ind{i},\a_h =\a_h\ind{i}}\mid\filt\ind{i-1}}.
						\end{align}
						Since \[\En\brk*{\br_h\ind{i,k}+\max_{a\in \cA}g_{h+1}(\x_{h+1}\ind{i,k}, a)\mid{}\x_h = \x_h\ind{i},\a_h  = \a_h\ind{i}}=\cT_h[g_{h+1}](\x_h\ind{i},\a_h\ind{i})\]
						and $\crl*{(\br_h\ind{i,k},\x_{h+1}\ind{i,k})}_{k\in\brk{K}}$ are
						\iid conditioned on $(\x_h,\a_h)=(\x_h\ind{i},\a_h\ind{i})$, we have, 
						\begin{align}
							&\En\brk*{\prn*{
										\cT_h[g_{h+1}](\x_h,\a_h)
										-\frac{1}{K}\sum_{k=1}^{K}\prn*{\br_h\ind{i,k}+\max_{a\in \cA}g_{h+1}(\x_{h+1}\ind{i,k}, a)}
									}^2\mid\x_h = \x_h\ind{i},\a_h = \a_h\ind{i}}\\
								&= \frac{1}{K}\En\brk*{\prn*{
											\cT_h[g_{h+1}](\x_h,\a_h)
											-\prn*{\br_h\ind{i,k}+ \max_{a\in \cA}g_{h+1}(\x_h\ind{i,k}, a)}}
										^2\mid\x_h = \x_h\ind{i},\a_h = \a_h\ind{i}} 
									\\
									&\leq \frac{4}{K},
								\end{align}
								so that
								\begin{align}
									\En\brk*{\prn*{
												\cT_h[g_{h+1}](\x_h\ind{i},\a_h\ind{i})
												-\frac{1}{K}\sum_{k=1}^{K}\prn*{\br_h\ind{i,k}+\max_{a\in \cA}g_{h+1}(\x_{h+1}\ind{i,k},a)}
											}^2\mid\filt\ind{i-1}}  \leq \frac{4}{K}.
									\end{align}
									Combining
                                                                        these
                                                                        bounds
                                                                        with
                                                                        \eqref{eq:ggolf_conc0}
                                                                        and
                                                                        rearranging thus
									gives
									\begin{align}
										&\sum_{i<t}\ls_h\ind{i}(g) \leq{}
										3\sum_{i<t}\En^{\pi\ind{i}}\brk*{\prn*{
												g_h(\x_h,\a_h) - \cT_h[g_{h+1}](\x_h,\a_h)}^2} +
										\frac{8N}{K}
										+ 
										16\log(2\delta^{-1}),
										\intertext{and}
										&\sum_{i<t}\En^{\pi\ind{i}}\brk*{\prn*{
												g_h(\x_h,\a_h) - \cT_h[g_{h+1}](\x_h,\a_h)}^2} \leq{}
										4\sum_{i<t}\ls_h\ind{i}(g) + \frac{8N}{K} + 64\log(2\delta^{-1}).
										\label{eq:ggolf_conc1}
									\end{align}
									Taking a union bound yields the result.
								\end{proof}
								
								\begin{lemma}
									\label{lem:ggolf_confidence}
									Define $\betastat=16\log(2HN\abs*{\cQ}\delta^{-1})$.
									Suppose we set $K\geq\frac{8N}{\betastat}$ and $\beta\geq{}2\betastat$. Then with
									probability at least $1-\delta$, for all $t\in\brk{N}$ and
									$h\in\cH$:
									\begin{itemize}
										\item $\Qstar\in\cQ\ind{t}$.
										\item All $g\in\cQ\ind{t}$ satisfy
										\begin{align}
											\label{eq:ggolf_confidence}
											\sum_{i<t}\En^{\pi\ind{i}}\brk*{\prn*{g_h(\x_h,\a_h)-\cT_h[g_{h+1}](\x_h,\a_h)}^2}
											\leq 9\beta.
										\end{align}
									\end{itemize}
								\end{lemma}
								\begin{proof}[\pfref{lem:ggolf_confidence}]
									Condition on the event in \cref{lem:ggolf_concentration}. For any
									fixed $t\in\brk{N}$ and $h\in\brk{H}$, we have that
									\begin{align}
										\sum_{i<t}\ls_h\ind{i}(\Qstar) &\leq{}
										3\sum_{i<t}\En^{\pi\ind{i}}\brk*{\prn*{
												\Qstar_h(\x_h,\a_h) - \cT_h[\Qstar_{h+1}](\x_h,\a_h)}^2} +
										\frac{8N}{K}
										+ 
										16\log(2HN\abs*{\cQ}\delta^{-1})\\
										&\leq{}
										\frac{8N}{K}
										+ 
										16\log(2HN\abs*{\cQ}\delta^{-1})
										\leq2\betastat,
									\end{align}
									where the first inequality uses that
									$\Qstar_h=\cT_h[\Qstar_{h+1}]$ and the second inequality uses
									our choice for $K$. It follows that $\Qstar\in\cQ\ind{t}$ as long as
									$\beta\geq2\betastat$.
									
									To prove the second claim, we note that for all $g\in\cQ\ind{t}$, by construction,
									\begin{align}
										\sum_{i<t}\En^{\pi\ind{i}}\brk*{\prn*{
												g_h(\x_h) - \cT_h[g_{h+1}](\x_h,\a_h)}^2} &\leq{}
										4\sum_{i<t}\ls_h\ind{i}(g) + \frac{8N}{K}
										+64\log(2HN\abs*{\cQ}\delta^{-1})\\
										&\leq      4\sum_{i<t}\ls_h\ind{i}(g) + 5\betastat \leq 9\beta.
									\end{align}
									
								\end{proof}
								
								\section{Proof of \creftitle{thm:golf}}
								\label{proof:golf}

								\begin{proof}[\pfref{thm:golf}]
									From \cref{lem:ggolf_confidence}, the parameter setting in the
									theorem
                                                                        statement
                                                                        ensures
                                                                        that
                                                                        with
                                                                        probability
                                                                        at
                                                                        least
                                                                        $1-\delta$, 
                                                                        for
                                                                        all $t\in [2\ldotst N]$,
									$\Qstar\in\cQ\ind{t}$, and all $g\in\cQ\ind{t}$ satisfy
									\begin{align}\sum_{i<t}\En^{\pi\ind{i}}\brk*{\prn*{g_h(\x_h)-\cT_h[g_{h+1}](\x_h,\a_h)}^2}
									\leq 9\beta. 
									\label{eq:before}
									\end{align}
									for
                                                                        all
                                                                        $h$.
                                                                        Let
                                                                        us
                                                                        condition
                                                                        on
                                                                        this
                                                                        event
                                                                        going
                                                                        forward. First,
                                                                        note
                                                                        that
                                                                        since
                                                                        $\Qstar\in
                                                                        \cQ\ind{t}$
                                                                        for
                                                                        all
                                                                        $t\in [2 \ldotst N]$,
                                                                        we
                                                                        have that
									\begin{align}
										J(\pistar) \leq \E\left[\max_{a\in \cA}\Qstar_1(\x_1,a)\right]\leq \sup_{g\in \cQ\ind{t}}\E\left[\max_{a\in \cA}g_1(\x_1, a)\right]. \label{eq:thisoneee}
									\end{align}
									On
                                                                        the
                                                                        other
                                                                        hand,
                                                                        we
                                                                        have
                                                                        $g\ind{t}\in
                                                                        \argmax_{g\in
                                                                          \cQ\ind{t}}
                                                                        \sum_{s<
                                                                          t}
                                                                        \max_{a\in
                                                                          \cA}g_1(\x\ind{s}_1,
                                                                        a)$,
                                                                        and
                                                                        so
                                                                        since
                                                                        $\x_1\ind{1},\x_1\ind{2},
                                                                        \dots$
                                                                        are
                                                                        i.i.d.~and
                                                                        any
                                                                        $g\in
                                                                        \cQ\ind{t}$
                                                                        take
                                                                        values
                                                                        in
                                                                        $[0,H]$,
                                                                        we
                                                                        have
                                                                        that
                                                                        by
                                                                        Hoeffding's inequality, there is an event $\cE$ of probability at least $1- \delta$ under which  
									\begin{align}
										\forall t \in[2 \ldotst N], \forall g\in \cQ,  \quad \left|\E\left[\max_{a\in \cA}g_1(\x_1,a )\right]- \frac{1}{t-1}\sum_{s < t} \max_{a\in \cA}g_1(\x_1\ind{s}, a)\right| \leq  \sqrt{(t-1)^{-1}\log(2N|\cQ|/\delta)}. \label{eq:hoefffing}
									\end{align}
									This implies that under $\cE$, we have 
									\begin{align}
										\forall t \in [2\ldotst N], \quad \sup_{g\in \cQ\ind{t}}\E\left[\max_{a\in \cA}g_1(\x_1,a)\right] & \leq \sup_{g\in \cQ\ind{t}} \frac{1}{t-1}\sum_{s< t} \max_{a\in \cA}g_1(\x_1\ind{s},a)+ \sqrt{(t-1)^{-1}\log(2N|\cQ|/\delta)}, \nn \\
										& = \frac{1}{t-1}\sum_{s< t} \max_{a\in \cA}g_1\ind{t}(\x_1\ind{s},a)+ \sqrt{(t-1)^{-1}\log(2N|\cQ|/\delta)}, \nn \\
										& \leq \E\left[\max_{a\in \cA}g_1\ind{t}(\x_1,a)\right] + 2\sqrt{(t-1)^{-1}\log(2N|\cQ|/\delta)}, \label{eq:sum}
									\end{align}
									where
                                                                        in
                                                                        the
                                                                        last
                                                                        inequality
                                                                        we
                                                                        have
                                                                        used
                                                                        \eqref{eq:hoefffing} with $f= g\ind{t}$. Thus, summing \eqref{eq:sum} for $t=2,\dots N$ and using \eqref{eq:thisoneee} gives that under $\cE$:
									\begin{align}
									\sum_{t=2}^N J(\pistar) \leq \sum_{t=2}^N \E[g_1\ind{t}(\x_1,\a_1)] + 4\sqrt{N\log(2N|\cQ|/\delta)},
									\shortintertext{and so since $J(\pistar)\leq H$,}
									\sum_{t=1}^N J(\pistar) \leq \sum_{t=1}^N \E[g_1\ind{t}(\x_1,\a_1)] + 4\sqrt{N\log(2N|\cQ|/\delta)} + H.	\label{eq:thiseee}
									\end{align}			
									On the other hand, using that $g\ind{t}_{H+1}\equiv 0$, we get
									\begin{align}
										&\sum_{t=1}^N(\E\left[g_1\ind{t}(\x_1, \a_1)\right]  - J(\pi\ind{t}))\\ & \leq \sum_{t=1}^N\sum_{h=1}^H \E^{\pi\ind{t}}\left[g\ind{t}_h (\x_h, \a_h) - \br_h - \max_{a\in \cA}g\ind{t}_{h+1}(\x_{h+1}, a)\right], \nn \\
										& = \sum_{t=1}^N\sum_{h=1}^H \E^{\pi\ind{t}}\left[ g\ind{t}_h (\x_h, \a_h) - \E\left[\br_h + \max_{a\in \cA}g\ind{t}_{h+1}(\x_{h+1}, a)\mid \x_h,\a_h\right]\right], \quad \text{(law of total expectation)}\nn \\
										& = \sum_{t=1}^N\sum_{h=1}^H \E^{\pi\ind{t}}\left[ g\ind{t}_h (\x_h,\a_h) - \cT_h[g\ind{t}_{h+1}](\x_h,\a_h)\right]. \label{eq:thissub}
										\shortintertext{and so, by the potential lemma (\cref{lem:potential}) and \eqref{eq:before}, we have}
										& \leq  6H \sqrt{\Ccov\beta  N \log(2N)}+ 2 H^2 \Ccov.
									\end{align}
									Combining
                                                                        this
                                                                        with
                                                                        \eqref{eq:thiseee},
                                                                        we
                                                                        obtain
                                                                        that
                                                                        with
                                                                        probability
                                                                        at
                                                                        least $1-2\delta$,
									\begin{align}
									\sum_{t=1}^N( J(\pistar) - J(\pi\ind{t}))&  \leq 6H \sqrt{\Ccov\beta  N  \log(2N)} +4\sqrt{N\log(2N|\cQ|/\delta)}+ 3 H^2 \Ccov.
									\end{align}
 It follows that if  $N=\bigoht(H^2\Ccov\beta/\veps^2)$, then the policy \[\pihat \in \unif(\pi\ind{1},\dots, \pi\ind{N})\] returned by \ggolf{} satisfies, with probability at least $1-\delta$:
									\begin{align}
									 J(\pistar) - \E[J(\pihat )] \leq \veps.
									\end{align}
									\paragraph{Sample complexity}
									We
                                                                        now
                                                                        bound
                                                                        the
                                                                        number
                                                                        of
                                                                        episodes. Note
                                                                        that
                                                                        that
                                                                        within
                                                                        an
                                                                        iteration
                                                                        $t$
                                                                        of
                                                                        \ggolf{},
                                                                        the
                                                                        local
                                                                        simulator
                                                                        is
                                                                        called
                                                                        $KH$
                                                                        times
                                                                        to
                                                                        update
                                                                        the
                                                                        confidence
                                                                        set,
                                                                        where
                                                                        $K
                                                                        \leq
                                                                        N/\log(2HN|\cQ|/\delta))$. Consequently,
                                                                        the
                                                                        total sample complexity is bounded by 
									\begin{align}
									H	N K \leq    \wtilde{O}(H^5 \Ccov^2 \log(|\cQ|/\delta)/\veps^{4}).
										\end{align}
								\end{proof}

\clearpage
	
	\part{Proofs for \learnlevel (\creftitle{sec:computation})}
	\label{part:learnlevel}

        \section{Full Version of \mainalg}
        \label{sec:omitted}
\cref{alg:learnlevel3} displays the full version of \mainalg. \cref{alg:forward_vpi} contains an ``outer-level'' wrapper for
\mainalg, \rvflF, which invokes
\mainalg and extracts an executable policy with imitation learning,
and \cref{alg:Phat} contains the subroutine used within
\cref{alg:learnlevel3} to approximate Bellman backups for value
functions using local simulator access. Additionally, we display the
variant of \mainalg for Exogenous Block MDPs, described in
\cref{sec:rvfs_exbmdp}, in \cref{alg:learnlevel2,alg:forward_exbmdp}.
Before diving into the proof, we first describe how the full version
of the algorithm differs from the informal version presented in the
main body in greater detail.

\paragraph{Differences between full version (\cref{alg:learnlevel3})
  and informal version (\cref{alg:learnlevel3_simp}) of \mainalg} The main difference between \cref{alg:learnlevel3_simp} and its full version in \cref{alg:learnlevel3} is that in the former we simply assume access to quantities involving conditional expectations such as:
  \begin{itemize}
  	\item The bellman backups $\cP_h[\Vhat_{h+1}]$, which are required to evaluate the actions of \learnlevel's policies (see \eqref{eq:Qhat3_simp}), and to perform the tests in \cref{line:test3_simp}; and 
  	\item The value functions $\E^{\pihat_{h+1:H}}[\sum_{\ell=h}^H \br_\tau\mid \x_h =x,\a_h =a]$ in \cref{line:data3_simp}, which are needed in the regression problem in \cref{line:updateQ3_simp}.
  	\end{itemize}
 These quantities are not available to the algorithm directly, but they can be estimated using the local simulator. This is reflected in the full version of \learnlevel{} in \cref{alg:learnlevel3}.

	\paragraph{Extracting policies from value functions}
	
	Let us briefly comment in more detail on how
	\cref{alg:learnlevel3} extracts the policy $\pi\ind{t}$ from the
	optimistic value function $f\ind{t}\in\cV$ at iteration $t$. From the
	Bellman equation, the ideal
	choice would be to set
	$\pi_h\ind{t}(x)=\argmax_{a\in\cA}\cP_h\brk*{f_{h+1}\ind{t}}(x,a)$,
	but this requires knowledge of the transition distribution. Instead,
	given parameters $\veps,\delta\in(0,1)$, \ggolf invokes
	\cref{alg:Phat} via
        $\pi\ind{t}_h(x)\in \argmax_{a\in \cA}  \Phat_{h,\veps,\delta}
        [f\ind{t}_{h+1}](x,a)$.
	The operator $\Phat_{h,\veps,\delta}[f]$
	(\cref{alg:Phat}), when given input $(x,a)\in \cX\times \cA$
	and $f_{h+1}:\cX\to\bbR$, uses the local simulator to generate $N_\simu\geq 1$ next
	states $\x_{h+1}\ind{1}, \dots,
	\x_{h+1}\ind{N_\simu}\stackrel{\text{i.i.d.}}{\sim} T_h(\cdot
	\mid x,a)$ to estimate the bellman back-up
	$\cP_{h}[f_{h+1}]$ via $\frac{1}{N_\simu}
	\sum_{i=1}^{N_\simu} (\br\ind{i}_h+
	f_{h+1}(\x_{h+1}\ind{i}))$, where $\br_h\ind{1},\dots,
	\br_h\ind{N_\simu}\stackrel{\text{i.i.d.}}{\sim}R_h
	(x,a)$. The number of samples $N_\simu$ in \cref{alg:Phat} is
	set as a function of $(\veps,\delta)$ such that with probability at least $1-\delta$, $
	|\Phat_{h,\veps,\delta}[f_{h+1}](x,a)-
	\cP_{h}[f_{h+1}](x,a)|\leq \veps$. 

\paragraph{Invoking the algorithm}
The base invocation of \mainalg takes the form
\begin{align}
	\Vhat_{1:H} \gets \mainalg_0(\Vhat_{1:H}=\mathsf{arbitrary}, \cVhat_{1:H}=\crl*{\cV}_{h=1}^{H},
	\cC_{0:H}=\crl*{\emptyset}_{h=0}^{H},
	\cB_{0:H}=\crl*{\emptyset}_{h=0}^{H},t_{1:H}=\crl*{1}_{h=1}^{H};\cdots).\label{eq:rvfs_base}
\end{align}
Whenever this call returns, the greedy policy induced by $\Vhat_{1:H}$
is guaranteed to be near-optimal. Naively, the policy induced by
$\Vhat_{1:H}$ is non-executable, and must be computed by invoking the
local simulator through \cref{line:test3}. To provide an end-to-end
guarantee to learn an executable policy, the outer-level
algorithm, \rvflF{} (\cref{alg:forward_vpi}, invokes
$\mainalg_0$, then extracts an executable policy from $\Vhat_{1:H}$
using imitation learning.\loose

Subsequent recursive calls take the
form
\begin{align}
	\label{eq:rvfs_rec}
	(\Vhat_{h:H},\cVhat_{h:H},
	\cC_{h:H},\cB_{h:H}, t _{h:H})\gets{}\learnlevel_{h}(\Vhat_{h+1:H},\cVhat_{h+1:H},\cC_{h:H},\cB_{h:H}, t_{h:H};\cV,\veps,\delta).
\end{align}
For such a call, the arguments above are: \colt{(i) $\Vhat_{h+1:H}$: Value function estimates for subsequent layers;
	(ii) $\cC_{h:H}$: Core-sets for current and subsequent layers;
	(iii) $t_{h:H}$: Counters that track the number of times
	\cref{alg:Phat} is called in
	the test on \cref{line:test3}, which facilitate
	tuning of confidence parameters;
	(iv) $\cVhat_{h+1:H}$: Value function confidence sets
	$\cVhat_{h+1:H}\subset\cV_{h+1:H}$, which are used in
	the test on \cref{line:test3} to quantify
	uncertainty on new state-action pairs and decide
	whether to expand the core-sets; and (v)
	$\cB_{h:H}$: Buffers of tuples
	$(x_{h-1},a_{h-1}, \Vhat_h ,\cVhat_h, t_h)$, which
	record relevant features of the algorithm's state whenever the test on \cref{line:test3}
	fails and a recursive call is performed.}%
\arxiv{\begin{itemize}
		\item $\Vhat_{h+1:H}$: Value function estimates for
                  subsequent layers.
                  		\item $\cVhat_{h+1:H}$: Value function confidence sets
		$\cVhat_{h+1:H}\subset\cV_{h+1:H}$, which are used in
		the test on \cref{line:test3} to quantify
		uncertainty on new state-action pairs and decide
		whether to expand the core-sets.
		\item $\cC_{h:H}$: Core-sets for current and subsequent layers.

		\item $\cB_{h:H}$: Buffers of tuples
		$(x_{h-1},a_{h-1}, \Vhat_h ,\cVhat_h, t_h)$, which
		record relevant features of the algorithm's state whenever the test on \cref{line:test3}
		fails and a recursive call is performed.

                		\item $t_{h:H}$: Counters that track the number of times
		\cref{alg:Phat} is called in
		the test on \cref{line:test3}, which facilitate
		tuning of confidence parameters.
\end{itemize}}
Importantly, the confidence sets $\cVhat_{h+1:H}$ do not need
to be explicitly maintained, and can be invoked
implicitly whenever a \emph{regression oracle} for the
value function class is available (cf. discussion in \cref{sec:computation}). Likewise, the buffers
$\cB_{h:H}$ are only used in our analysis, and do not
need to be explicitly maintained.

\subsection{\learnlevel Pseudocode}

\ifdefined\vepsllnum
\else
\newcommand{\vepsllnum}{\veps H^{-1}/48}
\fi

\ifdefined\vepsll
\else
\newcommand{\vepsll}{\veps_\learnlevel}
\fi

\ifdefined\Mnum
\renewcommand{\Mnum}{\ceil{8  \veps^{-1} \Cpush H}}
\else
\newcommand{\Mnum}{\ceil{8  \veps^{-1} \Cpush H}}
\fi 

\ifdefined\Ntestnum
\renewcommand{\Ntestnum}{2^8  M^2 H \veps^{-1}\log(8 M^6 H^8 \veps^{-2} \delta^{-1})}
\else
\newcommand{\Ntestnum}{2^8  M^2 H \veps^{-1}\log(8 M^6 H^8 \veps^{-2} \delta^{-1})}
\fi

\ifdefined\Nregnum
\renewcommand{\Nregnum}{2^8 M^2  \veps^{-1}\log(8|\cV|^2 HM^2 \delta^{-1})}
\else
\newcommand{\Nregnum}{2^8 M^2 \veps^{-1}\log(8|\cV|^2 HM^2 \delta^{-1})}
\fi

\ifdefined\bbeta
\renewcommand{\bbeta}{\frac{9 M H^2\log(8M^2H|\cV|^2/\delta)}{N_\reg}  +  \frac{34 MH^3\log(8M^6 N^2_\test  H^8/\delta)}{N_\test} }
\else
\newcommand{\bbeta}{\frac{9 M H^2\log(8M^2H|\cV|^2/\delta)}{N_\reg}  +  \frac{34 MH^3\log(8M^6 N^2_\test  H^8/\delta)}{N_\test} }
\fi

\ifdefined\Nsimu
\renewcommand{\Nsimu}{2N_\reg^2 \log(8 A N_\reg H M^3/\delta)}
\else
\newcommand{\Nsimu}{2N_\reg^2 \log(8 A N_\reg H M^3/\delta)}
\fi

\ifdefined\deltaprime
\renewcommand{\deltaprime}{\delta/(8M^7N_\test^2 H^8|\cV|)}
\else
\newcommand{\deltaprime}{\delta/(8M^7N_\test^2 H^8|\cV|)}
\fi

\ifdefined\testbound
\renewcommand{\testbound}{\frac{4 \log(8M^6 N^2_\test  H^8/\delta)}{N_\test}}
\else
\newcommand{\testbound}{\frac{4 \log(8M^6 N^2_\test  H^8/\delta)}{N_\test}}
\fi

\ifdefined\deltap
\renewcommand{\deltap}{\frac{\delta}{|\cV|M H}}
\else
\newcommand{\deltap}{\frac{\delta}{|\cV|M H}}
\fi 

\ifdefined\pibell
\renewcommand{\pibell}{\pihat}
\else
\newcommand{\pibell}{\pihat}
\fi

\ifdefined\numepisodes
\renewcommand{\numepisodes}{\mathrm{poly}(S,A,...)}
\else
\newcommand{\numepisodes}{\mathrm{poly}(S,A,...)}
\fi

 \begin{algorithm}[H]
        	\caption{\learnlevelh: Recursive Value Function Search}
	\label{alg:learnlevel3}

	\begin{algorithmic}[1]
          \State 	{\bfseries parameters:} Value function class
          $\cV$, suboptimality $\veps\in(0,1)$, confidence
          $\delta\in(0,1)$.
          \State \colt{\multiline{{\bfseries input:} Level
          $h\in\crl{0,\ldots,H}$, value function estimates $\Vhat_{h+1:H}$, confidence sets $\cVhat_{h+1:H}$, state-action collections $\cC_{h:H}$, buffers $\cB_{h:H}$, and counters $t_{h:H}$.}\vspace{5pt}}
        \arxiv{
          {\bfseries input:}
	\begin{itemize}
		\item Level $h\in\crl{0,\ldots,H}$.
		\item Value function estimates $\Vhat_{h+1:H}$, confidence sets $\cVhat_{h+1:H}$, state-action collections $\cC_{h:H}$, buffers $\cB_{h:H}$, and counters $t_{h:H}$.
                \end{itemize}
                }
		\Statex[0] \algcommentbiglight{Initialize parameters.}
		\State Set $M \gets \Mnum$.  \label{line:paramsVpiM}
		\State Set $N_\test \gets  \Ntestnum$, $N_\reg\gets  \Nregnum$, 
		\State Set $\Nest(k) \gets 2N_\reg^2 \log(8 A N_\reg H k^3/\delta)$ and $\delta'\gets \deltaprime$. \label{line:paramsVpi}
		\State Set $\veps_\reg^2 \gets \bbeta$. \label{line:beta}  
		\State Set $\beta(t) \gets \sqrt{2\log_{1/\delta'}(8AM|\cV|t^2/\delta)}$.
		\Statex[0] \algcommentbiglight{Test the fit for the
                  estimated value functions $\Vhat_{h+1:H}$ at future layers.}
		\For{$(x_{h-1},a_{h-1})\in\cC_h$}\label{line:begin3}
		\For{layer $\ell = H,\dots,h+1$} \label{line:second}
		\For{$n=1,\ldots,N_\test$} \label{line:third}
		\State \multiline{Draw $\x_h\sim{}T_{h-1}(\cdot\mid{}x_{h-1},a_{h-1})$, then draw
		$\x_{\ell-1}$ by rolling out with $\pibell_{h:H}$,
                where\footnote{We use the convention that $\Vhat_{H+1}\equiv 0$.}
                \label{line:draw}
		\begin{align}\label{eq:Qhat3}
			\forall \tau\in[H], \quad \bm{\pibell}_\tau(\cdot)\in \argmax_{a\in \cA}  \Phat_{\tau,\veps,\delta'} [\Vhat_{\tau+1}](\cdot,a), \quad \text{with} \quad \Vhat_{H+1} \equiv 0. %
		\end{align}
	}
        \For{$a_{\ell-1}\in \cA$} \label{line:fourth}
        \Statex[4]                 \algcommentbiglight{Number of times
                  $\Phat_{\ell-1,\veps, \delta'}$ (\cref{alg:Phat}) is called in
                  the test on \cref{line:test3}.}
		\State Update $t_\ell \gets
                t_{\ell}+1$. \label{line:counter3}
                \Statex[4]\algcommentbiglight{Test fit; if test fails, re-fit value functions
                  $\Vhat_{h+1:\ell}$ up to layer $\ell$.}
		\If{$\sup_{f\in \cVhat_{\ell}} |(\Phat_{\ell-1,\veps, \delta'}[\Vhat_{\ell}]- \Phat_{\ell-1,\veps, \delta'}[f_\ell])( \x_{\ell-1},a_{\ell-1})| >  \veps  + \veps \cdot \beta(t_\ell)$}\label{line:test3} \hfill %
		\State $\cC_{\ell}\gets\cC_{\ell}\cup\crl{(\x_{\ell-1},a_{\ell-1})}$ and $\cB_\ell \gets \cB_\ell \cup\{ (\x_{\ell-1},a_{\ell-1}, \Vhat_\ell ,\cVhat_\ell, t_\ell)\}$. \label{line:added}
		\For{$\tau= \ell,\dots, h+1$} \label{line:forloop}
		\State $(\Vhat_{\tau:H},\cVhat_{\tau:H},
		\cC_{\tau:H},\cB_{\tau:H}, t
                _{\tau:H})\gets{}\learnlevel_{\tau}(\Vhat_{\tau+1:H},\cVhat_{\tau+1:H},\cC_{\tau:H},\cB_{\tau:H},
                t_{\tau:H};\cV,\veps,\delta)$. 
		\EndFor
		\State \textbf{go to line \ref*{line:begin3}.} \label{line:goto3}
		\EndIf
		\EndFor
		\EndFor
		\EndFor
		\EndFor
		\If{$h=0$} \textbf{return:} $(\Vhat_{1:H},\cdot,\cdot,\cdot,\cdot)$.
		\EndIf
		\Statex[0] \algcommentbiglight{Re-fit $\Vhat_h$ and build a new confidence set.}
		\For{$(x_{h-1},a_{h-1})\in\cC_h$} \label{eq:gatherdata}
		\State Set $\cD_h(x_{h-1},a_{h-1})\gets\emptyset$.
		\For{$i=1,\dots, N_\reg$}
		\State Sample
		$\x_h\sim{}T_{h-1}(\cdot\mid{}x_{h-1},a_{h-1})$. 
	\State \multiline{Let $\Vhat_h(\x_h)$
          be a Monte-Carlo estimate for $\E^{\pibell_{h:H}}[\sum_{\ell=h}^H \bm{r}_\ell \mid
        \x_h]$ computed by collecting $\Nest(|\cC_h|)$ trajectories starting from $\x_h$ and rolling out with
        $\pibell_{h:H}$.}
		\State Update $\cD_h(x_{h-1},a_{h-1})\gets \cD_h(x_{h-1},a_{h-1}) \cup \{(\x_h, \Vhat_h(\x_h))\}$. 
		\EndFor
		\EndFor
		\State Let $\Vhat_h\ldef\argmin_{f\in\cVhat}\sum_{(x_{h-1},a_{h-1})\in\cC_h} \sum_{(x_h,v_h)\in\cD_{h}(x_{h-1},a_{h-1})}(f(x_h)-v_h)^2$. \label{line:updateQ3}
		\State Compute value function confidence set  
		\begin{align}
			\cVhat_{h} \coloneqq \left\{ f\in \cV \left| \  \sum_{(x_{h-1},a_{h-1})\in\cC_h} \frac{1}{N_\reg} \sum_{(x_h,\text{-}) \in \cD_h(x_{h-1},a_{h-1})}\left(\Vhat_h(x_{h}) -f(x_h)\right)^2 \leq   \veps_\reg^2  \right.     \right\}.  \label{eq:confidence3} 
		\end{align}
		\State \textbf{return} $(\Vhat_{h:H},\cVhat_{h:H}, \cC_{h:H}, \cB_{h:H}, t_{h:H})$.
	\end{algorithmic}
\end{algorithm}

\ifdefined\vepsllnum
\else
\newcommand{\vepsllnum}{\veps H^{-1}/48}
\fi

\ifdefined\vepsll
\else
\newcommand{\vepsll}{\veps_\learnlevel}
\fi

\ifdefined\Mnum
\renewcommand{\Mnum}{\ceil{8  \veps^{-1} \Cpush H}}
\else
\newcommand{\Mnum}{\ceil{8  \veps^{-1} \Cpush H}}
\fi 

\ifdefined\Ntestnum
\renewcommand{\Ntestnum}{2^8  M^2 H \veps^{-1}\log(8 M^6 H^8 \veps^{-2} \delta^{-1})}
\else
\newcommand{\Ntestnum}{2^8  M^2 H \veps^{-1}\log(8 M^6 H^8 \veps^{-2} \delta^{-1})}
\fi

\ifdefined\Nregnum
\renewcommand{\Nregnum}{2^8 M^2  \veps^{-1}\log(8|\cV|^2 HM^2 \delta^{-1})}
\else
\newcommand{\Nregnum}{2^8 M^2 \veps^{-1}\log(8|\cV|^2 HM^2 \delta^{-1})}
\fi

\ifdefined\bbeta
\renewcommand{\bbeta}{\frac{9 M H^2\log(8M^2H|\cV|^2/\delta)}{N_\reg}  +  \frac{34 MH^3\log(8M^6 N^2_\test  H^8/\delta)}{N_\test} }
\else
\newcommand{\bbeta}{\frac{9 M H^2\log(8M^2H|\cV|^2/\delta)}{N_\reg}  +  \frac{34 MH^3\log(8M^6 N^2_\test  H^8/\delta)}{N_\test} }
\fi

\ifdefined\Nsimu
\renewcommand{\Nsimu}{2N_\reg^2 \log(8 A N_\reg H M^3/\delta)}
\else
\newcommand{\Nsimu}{2N_\reg^2 \log(8 A N_\reg H M^3/\delta)}
\fi

\ifdefined\deltaprime
\renewcommand{\deltaprime}{\delta/(8M^7N_\test^2 H^8|\cV|)}
\else
\newcommand{\deltaprime}{\delta/(8M^7N_\test^2 H^8|\cV|)}
\fi

\ifdefined\testbound
\renewcommand{\testbound}{\frac{4 \log(8M^6 N^2_\test  H^8/\delta)}{N_\test}}
\else
\newcommand{\testbound}{\frac{4 \log(8M^6 N^2_\test  H^8/\delta)}{N_\test}}
\fi

\ifdefined\deltap
\renewcommand{\deltap}{\frac{\delta}{|\cV|M H}}
\else
\newcommand{\deltap}{\frac{\delta}{|\cV|M H}}
\fi 

\ifdefined\pibell
\renewcommand{\pibell}{\pihat}
\else
\newcommand{\pibell}{\pihat}
\fi

\ifdefined\numepisodes
\renewcommand{\numepisodes}{\mathrm{poly}(S,A,...)}
\else
\newcommand{\numepisodes}{\mathrm{poly}(S,A,...)}
\fi

\begin{algorithm}[H]
	\caption{\rvflF: Learn an executable policy with $\learnlevel$
          via behavior cloning.}
	\label{alg:forward_vpi}

	\begin{algorithmic}[1]
          \State {\bfseries input:} Value function class $\cV$, policy class $\Pi$, suboptimality  $\veps \in(0,1)$, confidence $\delta\in(0,1)$.
        \Statex[0] \algcommentbiglight{Set parameters for \learnlevel.}  
		\State Set $\vepsll\gets\vepsllnum$.
		\State Set $\Vhat_{1:H}\gets\mathsf{arbitrary}$, $\cVhat_{1:H}\gets \cV$, $\cC_{0:H}\gets\emptyset$, $\cB_{0:H} \gets\emptyset$, and $t_{i}\gets 0$, for all $i\in[0\ldotst H]$.
		\Statex[0] \algcommentbiglight{Set parameters for \forward.}  
		\State Set $M \gets \ceil{8 \veps_\learnlevel^{-1} \Cpush H}$ and $N_\test \gets 2^{8} M^2 H \veps_\learnlevel^{-1} \log(80 M^6 H^8 \veps_\learnlevel^{-2} \delta^{-1})$.
		\State $N_\reg \gets 2^{8} M^2 \veps_\learnlevel^{-1} \log(80 |\cV|^2 HM^2  \delta^{-1})$ and $\delta' = \frac{\delta}{40 M^7 N^2 H^8|\cV|}$. 
		\Statex[0] \algcommentbiglight{Get value functions from \learnlevel{}}
		\State $(\Vhat_{1:H},\cdot,\cdot,\cdot,
                \cdot)\gets{}\learnlevel_{0}(\Vhat_{1:H},\cVhat_{1:H},\cC_{0:H},\cB_{0:H}, t_{0:H};\cV
                ,N_\reg,N_\test,\vepsll,\delta/10)$.
                \Statex[0] \algcommentbiglight{Extract executable
                  policy via \forward{} algorithm for imitation learning.}
                  \State Define $\mb{\pihat}^{\learnlevel}_h(\cdot) \in \argmax_{a\in
                  	\cA}
                  \Phat_{h,\veps_\learnlevel,\delta'}[\Vhat_{h+1}](\cdot,a)$. 
                 \State Compute $\pihat_{1:H} \gets \forward(\Pi, \veps, \pihat^\learnlevel_{1:H}, \delta/2)$
		\State \textbf{return:} $\pihat_{1:H}$.
	\end{algorithmic}
\end{algorithm}

\begin{algorithm}[H]
	\caption{$\Phat_{h,\veps, \delta}[f]$: Estimate conditional expectation $\E[\br_h+f(\x_{h+1})\mid \x_h = \cdot, \a_h = \cdot]$.}
	\label{alg:Phat}

	\begin{algorithmic}[1]
          \State {\bfseries parameters:} Layer $h$, suboptimality $\veps \in(0,1)$, confidence $\delta\in(0,1)$, target function $f$.
          \State {\bfseries input:} $(x,a)\in \cX \times \cA$.
		\State  Set $N_\simu \ldef{} 2\log (1/\delta)/\veps^{2}$. 
		\State Set $\cD \gets \emptyset$
		\For{$i=1,\dots, N_\simu$}
		\State Sample $\br_h \sim R_h(x,a)$ and $\x_{h+1} \sim T_h(\cdot \mid x,
                a)$. \hfill \algcommentlight{Uses local simulator access.}
		\State Update $\cD \gets \cD \cup \{(\br_h, \x_{h+1})\}$.
		\EndFor
		\State \textbf{return:} $N_\simu^{-1}\cdot
                \sum_{(r,x)\in \cD} (r+ f(x))$. 
	\end{algorithmic}
      \end{algorithm}

        \arxiv{\clearpage
          \subsection{$\learnlevel^\exo$ Pseudocode}
\ifdefined\Mnum
\renewcommand{\Mnum}{\ceil{8  \veps^{-2} \Ccor S A H}}
\else
\newcommand{\Mnum}{\ceil{8  \veps^{-2} \Ccor S A H}}
\fi 

\ifdefined\Ntestnum
\renewcommand{\Ntestnum}{2^8  M^2 H \veps^{-2}\log(8 M^6 H^8 \veps^{-2} \delta^{-1})}
\else
\newcommand{\Ntestnum}{2^8  M^2 H \veps^{-2}\log(8 M^6 H^8 \veps^{-2} \delta^{-1})}
\fi

\ifdefined\Nregnum
\renewcommand{\Nregnum}{2^8 M^2 \veps^{-2}\log(8|\cV| HM^2 \delta^{-1})}
\else
\newcommand{\Nregnum}{2^8 M^2 \veps^{-2}\log(8|\cV| HM^2 \delta^{-1})}
\fi

\ifdefined\bbeta
\renewcommand{\bbeta}{\frac{9 MH^2\log(8M^2H|\cV|/\delta)}{N_\reg}  +  \frac{34 MH^3\log(8M^6 N^2_\test  H^8/\delta)}{N_\test} }
\else
\newcommand{\bbeta}{\frac{9 MH^2\log(8M^2H|\cV|/\delta)}{N_\reg}  +  \frac{34 MH^3\log(8M^6 N^2_\test  H^8/\delta)}{N_\test} }
\fi

\ifdefined\Nsimu
\renewcommand{\Nsimu}{2N_\reg^2 \log(8  N_\reg H M^3/\delta)}
\else
\newcommand{\Nsimu}{2N_\reg^2 \log(8  N_\reg H M^3/\delta)}
\fi

\ifdefined\deltaprime
\renewcommand{\deltaprime}{\delta/(8M^7N_\test^2 H^8|\cV|)}
\else
\newcommand{\deltaprime}{\delta/(8M^7N_\test^2 H^8|\cV|)}
\fi

\ifdefined\bbetap
\renewcommand{\bbetap}{\frac{2H\log(4\abs{\cV}  H/\delta)}{N_\reg} + 8H^3 A |\cC_h| \cdot \frac{\log(4H|\cC_h|/\delta)}{N_\test}}
\else
\newcommand{\bbetap}{\frac{2H\log(4\abs{\cV}  H/\delta)}{N_\reg} + 8H^3 A |\cC_h| \cdot \frac{\log(4H|\cC_h|/\delta)}{N_\test}}
\fi

\ifdefined\vepsllnum
\else
\newcommand{\vepsllnum}{\veps H^{-1}/48}
\fi

\begin{algorithm}[H]
	\caption{$\learnlevel^\exo_h$: Recursive Value Function Search
        for Exogenous Block MDPs}
	\label{alg:learnlevel2}
	\begin{algorithmic}[1]
          \State 	\multiline{{\bfseries parameters:} Value function class $\cV$, suboptimality $\veps\in(0,1)$, seeds $\zeta_{1:H}\in (0,1)$, confidence $\delta\in(0,1)$.}
          \State \multiline{{\bfseries input:}
	 Level $h\in[0 \ldotst H]$, value function estimates $\Vhat_{h+1:H}$, confidence sets $\cVhat_{h+1:H}$, state-action collections $\cC_{h:H}$, and buffers $\cB_{h:H}$, and counters $t_{h:H}$.}\vspace{5pt}
		\Statex[0] \algcommentbiglight{Initialize parameters.}
		\State Set $M \gets \Mnum$.\label{line:paramsExBMDPM}
		\State Set $N_\test\gets  \Ntestnum$, $N_\reg\gets  \Nregnum$. 
		\State Set $\Nest(k) \gets 2N_\reg^2 \log(8 A N_\reg H k^3/\delta)$ and $\delta'\gets \delta/(4M^7N_\test^2 H^8|\cV|)$.  \label{line:paramsExBMDP}
		\State Set $\veps_\reg^2 \gets \bbeta$. \label{line:beta_ex}
		\State Set $\beta(t) \gets \sqrt{\log_{1/\delta'}(8 M A|\cV|t^2/\delta)}$.
		\Statex[0] 		  \algcommentbiglight{Test the fit for the
			estimated value functions $\Vhat_{h+1:H}$ at future layers.}
		\For{$(x_{h-1},a_{h-1})\in\cC_h$}\label{line:begin2}
		\For{layer $\ell = H,\dots,h+1$}
		\For{$n=1,\ldots,N_\test$}
		\State \multiline{Draw $\x_h\sim{}T_{h-1}(\cdot\mid{}x_{h-1},a_{h-1})$, then draw
		$\x_{\ell-1}$ by rolling out with $\pibell_{h+1:H}$, where  \label{line:draw2}
		\begin{align}
		\label{eq:Qhat2}
		\forall \tau\in[H], \ \  \pibell_\tau(\cdot)\in  \argmax_{a\in \cA}  \ceil*{\Phat_{\tau,\veps^2, \delta'}[\Vhat_{\tau+1}](\cdot,a)\cdot \veps^{-1} + \zeta_\tau}, \quad \text{with}\quad  \Vhat_{H+1}\equiv 0.
		\end{align}
	}
		\For{$a_{\ell-1}\in \cA$} 
		\State Update $t_\ell \gets t_\ell +1$. %
			\Statex[4]\algcommentbiglight{Test fit; if test fails, re-fit value functions
			$\Vhat_{h+1:\ell}$ up to layer $\ell$.}
		\If{$\sup_{f\in \cVhat_{\ell}} |(\Phat_{\ell-1,\veps^2, \delta'}[\Vhat_{\ell}]- \Phat_{\ell-1,\veps^2, \delta'}[f_\ell])( \x_{\ell-1},a_{\ell-1})|> \veps^2 + \veps^2 \cdot \beta(t_\ell)$}\label{line:test2}%
		\State $\cC_{\ell}\gets\cC_{\ell}\cup\crl{(\x_{\ell-1},a_{\ell-1})}$ and $\cB_\ell \gets \cB_\ell \cup\{ (\x_{\ell-1},a_{\ell-1}, \Vhat_\ell ,\cVhat_\ell, t_\ell)\}$. \label{line:added2}
		\For{$\tau= \ell,\dots, h+1$} \label{line:forloop2}
		\State $(\Vhat_{\tau:H},\cVhat_{\tau:H},
		\cC_{\tau:H},\cB_{\tau:H}, t_{\tau:H})\gets{}\learnlevel^\exo_{\tau}(\Vhat_{\tau+1:H},\cVhat_{\tau+1:H},\cC_{\tau:H},\cB_{\tau:H}, t_{\tau:H};\cV,\veps, \zeta_{1:H},\delta)$.
		\EndFor
		\State \textbf{go to line \ref*{line:begin3}}.
                \label{line:goto2}
		\EndIf
		\EndFor
		\EndFor
		\EndFor
		\EndFor
		\If{$h=0$} \textbf{return} $(\Vhat_{1:H},\cdot,\cdot,\cdot, \cdot)$.
		\EndIf
		\Statex[0] \algcommentbiglight{Re-fit $\Vhat_h$ and build a new confidence set.}
		\For{$(x_{h-1},a_{h-1})\in\cC_h$}
		\State Set $\cD_h(x_{h-1},a_{h-1})\gets\emptyset$.
		\For{$i=1,\dots, N_\reg$}
		\State Sample
		$\x_h\sim{}T_{h-1}(\cdot\mid{}x_{h-1},a_{h-1})$. 
		\State \multiline{For each $a\in \cA$, let $\Vhat_h(\x_h)$
			be a Monte-Carlo estimate for $\E^{\pibell_{h:H}}[\sum_{\ell=h}^H \bm{r}_\ell \mid
			\x_h]$ computed by collecting $\Nest(|\cC_h|)$ trajectories starting from $\x_h$ and rolling out with
			$\pibell_{h:H}$.}%
		\State Update $\cD(x_{h-1},a_{h-1})\gets \cD(x_{h-1},a_{h-1}) \cup \{(\x_h, \Vhat_h(\x_h))\}$.
		\EndFor
		\EndFor
		\State Let $\Vhat_h\ldef\argmin_{f\in\cV_h}\sum_{(x_{h-1},a_{h-1})\in\cC_h} \sum_{(x_h,v_h)\in\cD_{h}(x_{h-1},a_{h-1})}(f(x_h)-v_h)^2$. \label{line:updateQ2}
		\State Compute value function confidence set 
		\begin{align}
			\cVhat_{h} \coloneqq \left\{ f\in \cV_h \left| \  \sum_{(x_{h-1},a_{h-1})\in\cC_h} \frac{1}{N_\reg} \sum_{(x_h,\text{-}) \in \cD_h(x_{h-1},a_{h-1})}\left(\Vhat_h(x_{h}) -f(x_h)\right)^2 \leq   \veps_\reg^2  \right.     \right\}.  \label{eq:confidence2} 
		\end{align}
		\State \textbf{return} $(\Vhat_{h:H},\cVhat_{h:H}, \cC_{h:H}, \cB_{h:H}, t_{h:H})$.
	\end{algorithmic}
\end{algorithm}

\begin{algorithm}[H]
\caption{\forwardexo: Learn an executable policy with
  $\learnlevel^\exo$ via imitation learning.}
\label{alg:forward_exbmdp}

\begin{algorithmic}[1]
  \State {\bfseries input:} Decoder class $\Phi$, suboptimality $\veps \in(0,1)$, confidence $\delta\in(0,1)$. 
	\Statex[0] \algcommentbiglight{Set parameters for \learnlevel{} and define the value function and policy classes.}
	\State Set $\vepsll\gets \veps H^{-1}/48$.
	\State Set $\cV=\cV_{1:H}$, where $\cV_h=
	\{x\mapsto f(\phi(x)) : f\in [0,H]^S, \phi\in \Phi\}$, $\forall h\in[H]$.
	\State Set $\Pi = \Pi_{1:H}$, where $\Pi_h = \left\{
          \pi(\cdot)\in \argmax_{a\in \cA} f(\phi(\cdot),a):f\in
          [0,H]^{S\times A }, \phi\in \Phi \right\}$, $\forall
        h\in[H]$.
        \Statex[0] \algcommentbiglight{Set parameters for \forward.}
	\State  \mbox{Set $N_\imit \gets 8H^2 \log (4H|\Pi|/\delta)/\veps$, 	$N_{\boost} \gets \log(1/\delta)/\log(24 SA H \veps)$, $N_{\texttt{eval}} \gets 16^2 \veps^{-2}\log(2 N_{\boost}/\delta)$. }
	\State  \mbox{Set $M \gets \ceil{8 \veps_\learnlevel^{-1} S A \Ccov H}$, $N_\test \gets 2^{8} M^2 H \veps_\learnlevel^{-1} \log(80 M^6 H^8N_\boost \veps_\learnlevel^{-2} \delta^{-1})$, and $\delta' = \frac{\delta}{40 M^7 N^2 H^8|\cV|N_\boost}$.}
	\State Set $N_\reg \gets 2^{8} M^2 \veps_\learnlevel^{-1} \log(80|\Phi|^2 H M^2 N_\boost \delta^{-1})$.
	\State Set $\Vhat_{1:H}\gets\mathsf{arbitrary}$,
        $\cVhat_{1:H}\gets \cV$, $\cC_{0:H}\gets\emptyset$, $\cB_{0:H}
        \gets\emptyset$, $i_\texttt{opt}=1$, and $J_{\max}=0$.
        \Statex[0] \algcommentbiglight{Repeatedly invoke $\learnlevel^\exo$ and
          extract policy with \forward{} to boost confidence.}
	\For{$i = 1,\dots, N_\boost$}
	\Statex[1] \algcommentbiglight{Invoke $\learnlevel^\exo$.}
	\State $(\Vhat\ind{i}_{1:H},\cdot,\cdot,\cdot,
        \cdot)\gets{}\learnlevel^\exo_{0}(\Vhat_{1:H},\cVhat
        _{1:H},\cC_{0:H},\cB_{0:H};\cV
        ,N_\reg,N_\test,\vepsll,\delta/(10N_\boost))$.
        \Statex[1] \algcommentbiglight{Imitation learning with \forward.}
      \State Define $\mb{\pihat}^{\learnlevel}_h(\cdot) \in \argmax_{a\in \cA} \Phat_{h,\veps_\learnlevel, \delta'}[\Vhat\ind{i}_{h+1}](\cdot,a)$.
      \State Compute $\pihat\ind{i}_{1:H}\gets \forward(\Pi, \veps, \pihat_{1:H}^\learnlevel, \delta/(2N_\boost))$.
	\Statex[1] \algcommentbiglight{Evaluate current policy.}
	\State $v = 0$.
	\For{$ =1, \dots, N_\texttt{eval}$} 
	\State Sample trajectory $(\x_1,\a_1,\bm{r}_1,\dots, \x_H,\a_H,\bm{r}_H)$ by executing $\pihat\ind{i}_{1:H}$.
	\State Set $v\gets v+ \sum_{h=1}^H \bm{r}_h$. 
	\EndFor
	\State Set $\widehat{J}(\pihat\ind{i}_{1:H})  \gets v/N_\texttt{eval}$.
	\If{$\widehat{J}(\pihat\ind{i}_{1:H}) >J_{\max}$}
	\State Set $i_\texttt{opt}=i$.
	\State Set $J_{\max} =\widehat{J}(\pihat\ind{i}_{1:H}) $.
	\EndIf
	\EndFor
	\State \textbf{return:} $\pihat_{1:H}=\pihat\ind{i_\texttt{opt}}_{1:H}$.
\end{algorithmic}
\end{algorithm}

}
        
\clearpage
        
	\arxiv{\section{Organization}}
        \colt{\section{Organization}}
        \label{sec:rvfs_organization}

\ifdefined\vepsllnum
\else
\newcommand{\vepsllnum}{\veps H^{-1}/48}
\fi

\ifdefined\vepsll
\else
\newcommand{\vepsll}{\veps_\learnlevel}
\fi

\ifdefined\Mnum
\renewcommand{\Mnum}{\ceil{8  \veps^{-1} \Cpush H}}
\else
\newcommand{\Mnum}{\ceil{8  \veps^{-1} \Cpush H}}
\fi 

\ifdefined\Ntestnum
\renewcommand{\Ntestnum}{2^8  M^2 H \veps^{-1}\log(8 M^6 H^8 \veps^{-2} \delta^{-1})}
\else
\newcommand{\Ntestnum}{2^8  M^2 H \veps^{-1}\log(8 M^6 H^8 \veps^{-2} \delta^{-1})}
\fi

\ifdefined\Nregnum
\renewcommand{\Nregnum}{2^8 M^2  \veps^{-1}\log(8|\cV|^2 HM^2 \delta^{-1})}
\else
\newcommand{\Nregnum}{2^8 M^2 \veps^{-1}\log(8|\cV|^2 HM^2 \delta^{-1})}
\fi

\ifdefined\bbeta
\renewcommand{\bbeta}{\frac{9 M H^2\log(8M^2H|\cV|^2/\delta)}{N_\reg}  +  \frac{34 MH^3\log(8M^6 N^2_\test  H^8/\delta)}{N_\test} }
\else
\newcommand{\bbeta}{\frac{9 M H^2\log(8M^2H|\cV|^2/\delta)}{N_\reg}  +  \frac{34 MH^3\log(8M^6 N^2_\test  H^8/\delta)}{N_\test} }
\fi

\ifdefined\Nsimu
\renewcommand{\Nsimu}{2N_\reg^2 \log(8 A N_\reg H M^3/\delta)}
\else
\newcommand{\Nsimu}{2N_\reg^2 \log(8 A N_\reg H M^3/\delta)}
\fi

\ifdefined\deltaprime
\renewcommand{\deltaprime}{\delta/(8M^7N_\test^2 H^8|\cV|)}
\else
\newcommand{\deltaprime}{\delta/(8M^7N_\test^2 H^8|\cV|)}
\fi

\ifdefined\testbound
\renewcommand{\testbound}{\frac{4 \log(8M^6 N^2_\test  H^8/\delta)}{N_\test}}
\else
\newcommand{\testbound}{\frac{4 \log(8M^6 N^2_\test  H^8/\delta)}{N_\test}}
\fi

\ifdefined\deltap
\renewcommand{\deltap}{\frac{\delta}{|\cV|M H}}
\else
\newcommand{\deltap}{\frac{\delta}{|\cV|M H}}
\fi 

\ifdefined\pibell
\renewcommand{\pibell}{\pihat}
\else
\newcommand{\pibell}{\pihat}
\fi

\ifdefined\numepisodes
\renewcommand{\numepisodes}{\mathrm{poly}(S,A,...)}
\else
\newcommand{\numepisodes}{\mathrm{poly}(S,A,...)}
\fi

This remainder of \cref{part:learnlevel} of the appendix contains the proofs for the main
results concerning
the \mainalg algorithm (\cref{thm:vpiforward_main} and \cref{thm:exbmdpforward_main}).
\begin{itemize}
	\item First, in \cref{sec:prelims} we give a brief overview of
          the analysis and introduce a
	restricted set of \emph{benchmark policies} which
	will be used throughout the proofs for
	\cref{thm:vpiforward_main} and
	\cref{thm:exbmdpforward_main}. The benchmark policy
	class is central to the regret decomposition for
	\mainalg, and facilitates an analysis that does not
	require optimism.
	\item In \cref{sec:vpi}, we prove
	\cref{thm:vpiforward_main} under \setupii{}
	($V^{\pi}$-realizability). This constitutes the main
	technical development for
	\cref{thm:vpiforward_main}. The central technical
	results proven here are \cref{thm:lbc3},
	\cref{thm:vpiforward} and which generalize \cref{thm:vpiforward_main}.
	\item In \cref{sec:vstar}, we prove
	\cref{thm:vpiforward_main} under \setupi{}
	($V^{\star}$-realizability), as a straightforward
	consequence of the tools developed in
	\cref{sec:vpi} (\cref{thm:lbc3} and
	\cref{thm:vpiforward}).
	\item Finally, in \cref{sec:exbmdp_app}, we prove
	\cref{thm:exbmdpforward_main} (analysis of \mainalge
	for the ExBMDP problem). This analysis has a
	similar structure to the proof of
	\cref{thm:vpiforward_main} under \setupii{} in
	\cref{sec:vpi}, and builds on the same analysis
	techniques, but requires specialized arguments due
	to extra technical challenges in the ExBMDP setting.
	\item \cref{sec:forward} gives a self-contained presentation of
the \forward{} algorithm for imitation learning, which is used within
\rvflF{} and \forwardexo{}.
\end{itemize}

	\section{Overview of Analysis and Preliminaries}
	\label{sec:prelims}
In this section, we present some notation, technical tools, and
preliminary results we require for the analysis of \learnlevel{} in
the settings we described in \cref{sec:computation}. We start by
defining a set of restricted \emph{benchmark policies} used in the
regret decomposition for \learnlevel{}.

\colt{
	\subsection{Overview of Analysis}
	\label{sec:rvfs_overview}
	In this section we give a brief overview of the
	analysis techniques behind
	\cref{thm:vpiforward_main} and
	\cref{thm:exbmdpforward_main}. We focus on
	\cref{thm:vpiforward_main} to begin.

	Recall that \learnlevel{} is recursive in the sense that whenever the
	test in \cref{line:test3} fails for a layer $h\in [H]$, an recursive
	call $\learnlevel_h$ is initiated. Throughout the recursion, via the
	steps in \cref{item:two} and \cref{item:one}, \learnlevel{} maintains
	the following invariant: whenever a call to $\learnlevel_h$ (an
	instance of $\learnlevel$ initiated at layer $h$) terminates, the
	confidence sets $\cVhat_{h+1:H}$ that it outputs satisfy, with high probability:
	\begin{align}
		\forall \ell \in[h+1 \ldotst H],\quad 	V^{\star}_\ell \in \cVhat_\ell. \label{eq:invar}\tag{Inv1}
	\end{align}
	
	In addition, $\learnlevel_h$ can only return if the value function
	tests in \cref{line:test3} (which involve the confidence sets
	$\cVhat_{h+1:H}$) all succeed. From the invariant in \eqref{eq:invar}, it can be shown that the tests only succeed if the estimated value functions $\Vhat_{h+1:H}$ satisfy:
	\begin{align}
		\forall \ell \in [h+1\ldotst H],\quad 	\P^{\pihat}\left[\sup_{a\in \cA}| (\cP_\ell[\Vhat_{\ell+1}]- \cP_\ell[V^{\star}_{\ell+1}])(\x_h,a)| \geq 3 \veps  \right]\leq \veps_\test, \label{eq:invar2}\tag{Inv2}
	\end{align}
	where $\veps_\test>0$ is a parameter set by the algorithm. We
        use pushforward coverability to show that \mainalg can only
        expand the core-sets $\cC_{1:H}$ a polynomial number of times
        before the algorithm terminates and \eqref{eq:invar2} is satisfied.
	
	The invariant in \eqref{eq:invar2} is useful because it ensures that
	the greedy policies
	\[\pihat_h(x)\approx\argmax_{a\in\cA}\cP_h\brk{\Vhat_{h+1}}(x,a)\]
	induced by the learned value functions $\Vhat_{1:H}$ are
	near-optimal. To make this precise, recall that given parameters $\veps,\delta\in(0,1)$, the action $\pihat_h(x)$ of \learnlevel's policy at layer $h$ and state $x\in \cX$, is given by 
	\begin{align}
		\pibell_h(x)\in \argmax_{a\in \cA}  \Phat_{h,\veps,\delta} [\Vhat_{h+1}](x,a),\label{eq:defphat}
	\end{align}
	where
	$\Vhat_{h+1}$ is the estimated value function at layer
	$h+1$.
	The operator $\Phat_{h,\veps,\delta}[\Vhat_{h+1}]$
	(\cref{alg:Phat}), when given input $(x,a)\in \cX\times \cA$,
	ensures that probability at least $1-\delta$, $
	|\Phat_{h,\veps,\delta}[\Vhat_{h+1}](x,a)-
	\cP_{h}[\Vhat_{h+1}](x,a)|\leq \veps$. Combining this with the
	invariant in \eqref{eq:invar2} and the fact that $\Vstar_h
	\equiv \cP_h[\Vstar_{h+1}]$, one can see that with high
	probability (over the randomness in $\x_h\sim \P^{\pihat}$ and
	$\Phat$), we have: \begin{align}\max_{a\in \cA}|\Phat_{h,\veps,\delta}[\Vhat_{h+1}](\x_h,a) - V^\star_{h}(\x_h,a) | \leq 4 \veps. \label{eq:remmm}
                           \end{align}
\paragraph{Analysis under \setupi}
	For \cref{thm:vpiforward_main} (\setupi), \cref{eq:remmm} together with the definition of $\pibell_h$ in \cref{eq:defphat} and the gap assumption (\cref{ass:gap}) implies that if $\veps\leq \Delta/8$, we have that with high probability (over the randomness in $\x_h\sim \P^{\pihat}$ and $\Phat$),
	\begin{align}
		\pibell_h(\x_h)  \in \argmax_{a\in \cA}\Phat_{h,\veps,\delta}[\Vhat_{h+1}](\x_h,a)  =  \argmax_{a\in \cA}V_h^\star(x,a)= \pistar_h(\x_h).  \label{eq:high}
	\end{align}
	This suffices to show that $\pihat$ is near-optimal, since by the performance lemma \citep{kakade2003sample}, the suboptimality of $\pihat$ can be bounded as 
	\begin{align}
          \label{eq:regret_gap}
		J(\pistar)- J(\pibell) \leq \sum_{h=1}^H \P^{\pihat}[	\pibell_h(\x_h)  \neq 	\pistar_h(\x_h)].
	\end{align}
	This suffices to prove the performance bound in \cref{thm:vpiforward_main} under \setupi.

\paragraph{Analysis under \setupii}        
	For \cref{thm:vpiforward_main} (\setupii), an immediate
        challenge in applying a similar analysis to \setupi{} is the
        lack of suboptimality gap $\Delta$, which makes it impossible
        to directly bound the probability that $\pihat\neq\pistar$ in
        \cref{eq:regret_gap}. To address this, we introduce a
        \emph{restricted benchmark policy class} $\Pi_\veps\subset\Pim$ in
        \cref{sec:benchmark} below. The class
        $\Pi_\veps\subset\Pim$ is constructed such that that there exists a policy
        $\pi\in\Pi_\veps$ that (i) is $\bigoh(\veps)$-suboptimal, and
        (ii) emulates certain properties of a gap. Together, these properties
        facilitate analysis similar to \setupi. Overall, this argument is similar to the ``virtual policy iteration''
        analysis in \citet{yin2022efficient}.

\paragraph{Analysis of \mainalge}        
	The analysis of \mainalge for ExBMDPs
        (\cref{thm:exbmdpforward_main}) uses the same idea as the analysis for
        \setupii, except that we can only realize $V^{\pi}$ for
        \emph{endogenous policies} that act on $\phistar(\x_h)$. To
        address this, we use the randomized rounding scheme in
        \mainalge, and the crux of the proof is to show that with high
        probability, the rounded policies in \mainalge ``snap'' onto
        endogenous policies, facilitating an argument similar to \setupii.

\paragraph{Generalizing the analysis}	
  We mention in passing that \learnlevel{} can be slightly modified to
		recover other existing sample complexity guarantees for RL with
		linear function approximation and local simulator access that do not
		require pushforward coverability, including
		linear-$Q^{\star}$ realizability with gap \citep{li2021sample} and
		$Q^{\pi}$-realizability \citep{yin2022efficient}; we leave a more
		general treatment for future work.	

}

\subsection{Benchmark Policy Class and Randomized Policies}
\label{sec:benchmark}

As described above, central to our analysis is a set of
$O(\veps)$-suboptimal policies against which we benchmark the policies
returned \learnlevel{}, which emulate certain consequences of the
$\Delta$-gap assumption (\cref{ass:gap}). Before introducing this
concept formally, we first define the notion of a \emph{randomized policy}.

\paragraph{Induced stochastic policies}

Given an arbitrary collection of independent random variables $\bQtilde = (\bQtilde_{h}(x,a))_{(h,x,a)\in [H]\times \cX\times \cA}$, we say that a policy $\pi$ is \emph{induced} by $\bQtilde$ if $\pi$ satisfies 
\begin{align}
\forall h \in[H],\forall x \in \cX, \quad  \bpi_h(x) \in \argmax_{a'\in \cA} \bQtilde_h(x,a'), \label{eq:induced}
\end{align}
where we use the bold notation
   $\bm{\pi}_h(x)$ as shorthand for the random variable $\a_h\sim{}\pi_h(x)\in \Delta(\cA)$; in other words, for each $x\in\cX$, $\pi_h(x)\in\Delta(\cA)$ is the distribution induced by sampling $\bQ_h(x,\cdot)$ and playing the action $\bpi_h(x) \in \argmax_{a'\in \cA} \bQtilde_h(x,a')$. If there are ties in \eqref{eq:induced}, we break them by picking the action with the smallest index; we assume without loss of generality that actions in $\cA$ are index from $1,\dots,|\cA|$.

\paragraph{Benchmark policy class}
  
  We now define the benchmark policy class as follows.

\begin{definition}[Benchmark policy class]
  For $\veps\in (0,1)$, let $\Pi_\veps\subseteq \PiS$ be the set of stochastic
  policies such that $\pi \in \Pi_\veps$ if and only if there exists a
  collection of \emph{independent} random variables
  $\bQtilde = (\bQtilde_{h}(x,a))_{(h,x,a)\in [H]\times \cX\times
    \cA}$ in $[0,H]$ such
  that:%
  \begin{itemize}
  \item $\pi$ is induced by $\bQtilde$ (i.e.~\cref{eq:induced} is
    satisfied); and
  \item For all $(h, x,a)\in [H]\times\cX\times \cA$, we have
    $|(\bQtilde_h- Q^{\pi}_{h})(x,a)| \leq \veps$, almost
    surely {under the draw of $\bQtilde$}.
  \end{itemize}
\end{definition}
Intuitively, the set $\Pi_\veps$ contains the set of all (stochastic)
policies corresponding to (randomized) state-action value functions
that are point-wise $O(\veps)$ close to $Q^\star$. We formalize this
claim in the next lemma.
\begin{lemma}[Suboptimality of benchmark policies]
	\label{lem:policysubopt3}
	Let $\veps \in (0,1)$ be given. Let $\pib \in
        \Pi_\veps$ {be a stochastic policy} induced by a collection of (independent) random state-action value functions
        $(\bQtilde_{h}(x,a))_{(h,x,a)\in[H]\times \cX\times
          \cA}\subset [0,H]$ such that for all $h\in[H]$ and all $(x,a)\in \cX \times \cA$:
	\begin{align}
		|\bQtilde_h(x,a) - Q^{\pib}_{h}(x,a)|\leq  \veps\ \   \text{ almost surely, and} \quad  \tilde{\bm{\pi}}_h(x) \in \argmax_{a'\in \cA} \bQtilde_h(x,a').
	\end{align}
        Then, for all $h\in\brk{H}$,
	\begin{align}
		\forall x\in \cX, \quad    \Vstar_h(x)   & \leq  V^{\pib}_h(x) +   {3 H \veps}.
		\label{eq:toinductone}
	\end{align}
      \end{lemma}
	  \begin{proof}[\pfref{lem:policysubopt3}]
		Using backward induction over $\ell=H,\dots,1$, we start by showing that all $\ell$: 
		  \begin{align}
			  \forall (x,a)\in \cX\times \cA, \quad  
			\Qstar_\ell(x,a)  &\leq     \bQtilde_\ell(x,a)  + 2 \veps \cdot  (H-\ell+1).
		 \label{eq:induction}
		  \end{align}	
		  almost surely. We then instantiate this with $\ell=1$ and use the fact that $\|\bQtilde_h- Q^{\pib}_h\|_\infty\leq \veps$ to get the desired result.
		  
		  \paragraph{Base case $[\ell=H]$} By definition of the state-action value function, we have, for all $\pi \in \Pims$, $Q^\star_H \equiv Q^\pi_H$. Thus, since $ \sup_{(x,a)\in \cX\times \cA}|(\bQtilde_H- Q^{\pib}_{H})(x,a)|  \leq \veps$ almost surely (by definition of $\bQtilde_{1:H}$), we get that 
		  \begin{align}
			  \forall (x,a)\in \cX \times \cA, \quad |\bQtilde_H(x,a) - Q^\star_{H}(x,a)| \leq \veps,
		  \end{align} 
		  almost surely. This implies \eqref{eq:toinductone} for $\ell=H$.
		  
		  \paragraph{General case $[\ell <h]$}	Fix $h\in[H-1]$ and suppose that \eqref{eq:induction} holds for all $\ell\in [h+1,\dots, H]$ almost surely. We show that it holds for $\ell=h$ almost surely. Fix $(x,a)\in \cX\times \cA$. We have 
		  \begin{align}
			  Q^\star_h(x,a)- \bQtilde_h(x,a) &  =\cT_{h} [Q^\star_{h+1}](x,a)  - \cT_{h} [\bQtilde_{h+1}](x,a)+ \cT_{h} [\bQtilde_{h+1}](x,a) - \bQtilde_h(x,a),\nn \\
			  & \leq {2\veps \cdot (H-h)}+ \cT_{h} [\bQtilde_{h+1}](x,a) - \bQtilde_h(x,a), \label{eq:twist3}
		  \end{align}
		  almost surely, where the last step follows by the induction hypothesis. We now bound $\cT_{h} [\bQtilde_{h+1}](x,a) - \bQtilde_h(x,a)$. We have, almost surely, that 
		  \begin{align}
			  \cT_{h} [\bQtilde_{h+1}](x,a) - \bQtilde_h(x,a) & = 	\cT_{h} [\bQtilde_{h+1}](x,a) - \cP_{h} [V^{\pib}_{h+1}](x,a) + \cP_{h} [V^{\pib}_{h+1}](x,a) - \bQtilde_h(x,a), \nn \\
			  & = \cT_{h}[\bQtilde_{h+1}](x,a) - \cP_{h} [V^{\pib}_{h+1}](x,a) + Q^{\pib}_{h}(x,a) - \bQtilde_h(x,a), \nn \\
			  & = \cT_{h}[\bQtilde_{h+1}](x,a) - \cP_{h} [V^{\pib}_{h+1}](x,a) + \veps, \quad \text{(by the assumption on $\bQtilde_h$)} \nn \\
			  & = \E\left[\max_{a'\in \cA}\bQtilde_{h+1}(\x_{h+1},a')-  Q^{\pib}_{h+1}(\x_{h+1},\pibb_{h+1}(\x_{h+1})) \mid \x_h =x, \a_h =a \right] + \veps, \nn \\
			  & = \E\left[\bQtilde_{h+1}(\x_{h+1},\pibb_{h+1}(\x_{h+1}))-  Q^{\pib}_{h+1}(\x_{h+1},\pibb_{h+1}(\x_{h+1})) \mid \x_h =x, \a_h =a\right] + \veps,  \nn \\
			  & \leq   2\veps, \label{eq:inde}
		  \end{align}
		  where the penultimate inequality follows by the definition of $\pibb_{h+1}$, and the last inequality follows by the fact that $\|\bQtilde_{h+1}- Q^{\pib}_{h+1}\|_{\infty} \leq \veps$ almost surely, by assumption. Plugging \eqref{eq:inde} into \eqref{eq:twist3} completes the induction, and so we have that 
		  \begin{align}
		\forall (x,a)\in \cX\times \cA,\quad 	Q^\star_1(x,a) \leq \bQtilde_1(x,a) \leq 2 H \veps.
		  \end{align}
	In particular, taking the max over $a$ on both sides and using the definition of $\pib$, we get that 
	\begin{align}
		\forall x\in \cX,\quad 	V^\star_1(x) \leq \bQtilde_1(x,\pibb_1(x)) \leq 2 H \veps,
	\end{align}	  
	almost surely. Combining this with the fact that $\bQtilde_1(x,\pibb_1(x)) \leq Q^{\pib}_1(x,\pibb_1(x)) + \veps$, almost surely (since $\|\bQtilde_1 - Q^{\pib}_1\|_{\infty}\leq \veps$ almost surely by assumption) implies that
		\begin{align}
			V^\star_1(x) & \leq Q^{\pi}_1(x,\pibb_1(x)) + 2 H \veps + \veps.
		\end{align}
		Taking expectation over $\pibb_1(x)$ and bounding $2 H \veps + \veps$ by $3H \veps$ leads to the desired result.

	  \end{proof}

\subsection{Additional Preliminaries}      
The following lemma gives a guarantee for the Bellman backup
approximation algorithm $\Phat$ 
(\cref{alg:Phat}) that is tailored to the analysis of \learnlevel{}.
\begin{lemma}
	\label{lem:phat}
	Let $\veps, \delta,\delta' \in(0,1)$, $B>0$, and $h\in[H]$, be given and let $\cV$ be a finite function class. For any sequence $(x_i)_{i\geq 1}\subset \cX$ of state action pairs, the outputs $(\Phat_{h, \veps, \delta'}[f](x_i,a))_{i\geq 1, a\in \cA}$ of \cref{alg:Phat} satisfy, with probability at least $1-\delta$, 
	\begin{align} 
		\forall i\geq 1,  \forall f \in \cV, \forall a\in \cA, \quad 	|\Phat_{h, \veps, \delta'}[f](x_i,a)- \cP_h[f](x_i,a)|\leq \veps \cdot \sqrt{2\log_{1/\delta'}(2 A i^2|\cV|/\delta)}. \label{eq:past}
	\end{align}
\end{lemma}
\begin{proof}[\pfref{lem:phat}]
	By Hoeffding's inequality \citep{hoeffding1963} and the
        union bound over $a\in \cA$ and $f\in \cV$, we have that for any $i\geq 1$, with probability at least $1 - \delta/(2i^2)$,
		\begin{align}
			\forall f \in \cV, \forall a\in \cA, \quad 	|\Phat_{h, \veps, \delta'}[f](x_i,a)- \cP_h[f](x_i,a)|\leq \veps \cdot \sqrt{2\log_{1/\delta'}(2A i^2|\cV|/\delta)}.
		\end{align}
		The desired result follows by the union bound over $i\geq 1$ and the fact that $\sum_{i\geq 1}1/i^2 = \pi^2/6 \leq 2$.
      \end{proof}

	\section{Guarantee under $V^\pi$-Realizability (Proof of
          \creftitle{thm:vpiforward_main}, \setupii)}
	\label{sec:vpi}

\ifdefined\vepsllnum
\else
\newcommand{\vepsllnum}{\veps H^{-1}/48}
\fi

\ifdefined\vepsll
\else
\newcommand{\vepsll}{\veps_\learnlevel}
\fi

\ifdefined\Mnum
\renewcommand{\Mnum}{\ceil{8  \veps^{-1} \Cpush H}}
\else
\newcommand{\Mnum}{\ceil{8  \veps^{-1} \Cpush H}}
\fi 

\ifdefined\Ntestnum
\renewcommand{\Ntestnum}{2^8  M^2 H \veps^{-1}\log(8 M^6 H^8 \veps^{-2} \delta^{-1})}
\else
\newcommand{\Ntestnum}{2^8  M^2 H \veps^{-1}\log(8 M^6 H^8 \veps^{-2} \delta^{-1})}
\fi

\ifdefined\Nregnum
\renewcommand{\Nregnum}{2^8 M^2  \veps^{-1}\log(8|\cV|^2 HM^2 \delta^{-1})}
\else
\newcommand{\Nregnum}{2^8 M^2 \veps^{-1}\log(8|\cV|^2 HM^2 \delta^{-1})}
\fi

\ifdefined\bbeta
\renewcommand{\bbeta}{\frac{9 M H^2\log(8M^2H|\cV|^2/\delta)}{N_\reg}  +  \frac{34 MH^3\log(8M^6 N^2_\test  H^8/\delta)}{N_\test} }
\else
\newcommand{\bbeta}{\frac{9 M H^2\log(8M^2H|\cV|^2/\delta)}{N_\reg}  +  \frac{34 MH^3\log(8M^6 N^2_\test  H^8/\delta)}{N_\test} }
\fi

\ifdefined\Nsimu
\renewcommand{\Nsimu}{2N_\reg^2 \log(8 A N_\reg H M^3/\delta)}
\else
\newcommand{\Nsimu}{2N_\reg^2 \log(8 A N_\reg H M^3/\delta)}
\fi

\ifdefined\deltaprime
\renewcommand{\deltaprime}{\delta/(8M^7N_\test^2 H^8|\cV|)}
\else
\newcommand{\deltaprime}{\delta/(8M^7N_\test^2 H^8|\cV|)}
\fi

\ifdefined\testbound
\renewcommand{\testbound}{\frac{4 \log(8M^6 N^2_\test  H^8/\delta)}{N_\test}}
\else
\newcommand{\testbound}{\frac{4 \log(8M^6 N^2_\test  H^8/\delta)}{N_\test}}
\fi

\ifdefined\deltap
\renewcommand{\deltap}{\frac{\delta}{|\cV|M H}}
\else
\newcommand{\deltap}{\frac{\delta}{|\cV|M H}}
\fi 

\ifdefined\pibell
\renewcommand{\pibell}{\pihat}
\else
\newcommand{\pibell}{\pihat}
\fi

\ifdefined\numepisodes
\renewcommand{\numepisodes}{\mathrm{poly}(S,A,...)}
\else
\newcommand{\numepisodes}{\mathrm{poly}(S,A,...)}
\fi

In this section, we prove \cref{thm:vpiforward_main} under \setupii. First, in \cref{sec:setupii} we state a number of supporting technical lemmas, then use them to prove a more general version of \cref{thm:vpiforward_main}, \cref{thm:vpiforward}, which holds under a weaker realizability assumption (informally, $V^{\pi}$-realizability only for \emph{near-optimal} policies $\pi$); \cref{thm:vpiforward_main} follows as an immediate consequence. The remainder of the section (\cref{proof:testfailures} through to \cref{proof:lbc3}) contains the proofs for the intermediate results.

\subsection{Analysis: Proof of \creftitle{thm:vpiforward_main} (\setupii)}
\label{sec:setupii}
We analyze \learnlevel{} in the setting of \cref{thm:vpiforward_main} (\setupii), where we have a function class $\cV$ satisfying $V^\pi$-realizability (\cref{ass:realpi}). We will actually show that the conclusion of \cref{thm:vpiforward_main} holds under a weaker function approximation setup we refer to as \emph{relaxed $V^{\pi}$-realizability}: instead of requiring $V^\pi$-realizability for all $\pi\in \Pim$, we only require it for policies $\pi$ in the set of near-optimal policies corresponding to the \emph{benchmark policy class} $\Pi_{\veps_\rel}$ for some $\veps_\rel>0$ ($\Pi_\veps$ is defined in \cref{sec:prelims}).
\begin{assumption}[Relaxed $V^{\pi}$-realizability]
	\label{ass:relaxreal}
	For $\veps_\rel>0$ and all $\pi \in \Pi_{\veps_\rel}$ and $h\in [H]$, we have $V^{\pi}_h(x)\in \cV \subseteq \{f: \cX \rightarrow [0,H]\}$.
      \end{assumption}

We will analyze \learnlevel{} under \cref{ass:relaxreal} and \cref{ass:pushforward}. However, it turns out that all of the main results for \mainalg can be derived under this assumption: As we will see in \cref{sec:vstar} in the sequel, when the $\Delta$-gap assumption (\cref{ass:gap}) is satisfied, then $\Pi_{\veps_\rel} =\{\pistar\}$ for all $\veps_\rel<\Delta$, allowing us to prove \cref{thm:vpiforward_main} under \setupi{} as a special case of relaxed $V^{\pi}$-realizability. Our analysis for the ExBMDP setting in \cref{sec:exbmdp_app} requires more work, but uses that for ExBMDPs, \cref{ass:relaxreal} is satisfied for a subset of $\Pi_{\veps_\rel}$ corresponding to endogenous policies. 

We begin with our analysis under \setupii{} by bounding the number of times the test in \cref{line:test3} fails. Since the sizes of the core sets $\cC_{1:H}$ in \learnlevel{} are directly related to the number of test failures, the next result, which bounds $|\cC_h|$ for $h\in[H]$, allows us to show that \learnlevel{} terminates in polynomial iterations with high probability. The proof is in \cref{proof:testfailures}.

\begin{lemma}[Bounding the number of test failures]
	\label{lem:testfailures}
	Let $\delta,\veps\in(0,1)$ be given, and suppose that \cref{ass:pushforward} (pushforward coverability) holds with parameter $\Cpush>0$.
	Further, let $f \in \cV$, be given, where $\cV$ is an arbitrary function class. Then there is an event $\cE$ of probability at least $1-\delta$ under which any call  $\learnlevel_{0}(f,\cV_{1:H},\emptyset,\emptyset,0;\cV,\veps,\delta)$ (\cref{alg:learnlevel3}) terminates, and throughout the execution of $\learnlevel_{0}$, we have 
	\begin{align}
		\forall h \in [H], \quad |\cC_h|\leq \Mnum.\label{eq:one3}
	\end{align}
\end{lemma}
In particular, \cref{lem:testfailures} ensures that with high probability, every call to \learnlevelh{} (made recursively via the call to $\learnlevel_0$) terminates in polynomial iterations. When \learnlevelh{} terminates, all the tests in \cref{line:test3} must have passed for all $\ell>h$. Using this and a standard concentration argument, we get the following guarantee for the estimated value functions and confidence sets returned by each call to \learnlevelh{}. The proof is in \cref{proof:confidence}. 
\begin{lemma}[Consequence of passing the tests]
	\label{lem:confidence}
	Let $h\in [0 \ldotst H]$ and $\veps,\delta\in(0,1)$ be given and consider a call to $\learnlevel_0$ in the setting of \cref{lem:testfailures}. Further, let $\cE$ be the event of \cref{lem:testfailures}. There exists an event $\cE'_h$ of probability at least $1 -\delta/H$ such that under $\cE \cap \cE'_h$, if a call to $\learnlevel_h$ within the execution of $\learnlevel_0$ terminates and returns $(\Vhat_{h:H}, \cVhat_{h:H}, \cC_{h:H}, \cB_{h:H}, t_{h:H})$, then for any $(x_{h-1},a_{h-1})\in \cC_h$ and $\ell\in [h+1\ldotst H+1]$: 
	\begin{align}
		\label{eq:test_passed_lem}
		&\bbP^{\pibell}\brk*{\sup_{f\in \cVhat _\ell} \max_{a\in \cA} \left| \cP_{\ell-1}[\Vhat_{\ell}]- \cP_{\ell-1}[f_\ell](\x_{\ell-1},a) \right|  > 3\veps \mid  \x_{h-1}=x_{h-1},\a_{h-1}=a_{h-1}}  \leq  \testbound, 
	\end{align}
	where $(\pibell_\tau)_{\tau\geq h}$ {is the stochastic policy induced by $\Vhat_{h:H}$} and $M$ and $N_\test$ are defined as in \cref{alg:learnlevel3}. 
\end{lemma}

We now give a guarantee for the estimated value functions $\Vhat_{1:H}$ computed within $\learnlevel$ in \cref{line:updateQ3} (the proof is in \cref{proof:confidencesets1}).
\begin{lemma}[Value function regression guarantee]
	\label{lem:confidencesets1}
	Let $h\in [0 \ldotst H]$ and $\delta, \veps'\in (0,1)$ be given, and consider a call to $\learnlevel_0$ in the setting of \cref{lem:testfailures}. Further, let $\Pi'\subseteq \Pim$ be a finite policy class such that the class $\cV$ realizes the value functions $V^{\pi}$ for $\pi \in \Pi'$ (i.e.~$\cV$ satisfies \cref{ass:relaxreal} with $\Pi_{\veps_\rel}$ replaced by $\Pi'$). Then, there is an event $\cE_h''$ of probability at least $1 -\delta/H$ under which for all $k\geq 1$, if 
	\begin{enumerate}
		\item $\learnlevel_h$ gets called for the $k$th time during the execution of $\learnlevel_0$; and
		\item this $k$th call terminates and returns $(\Vhat_{h:H}, \cVhat_{h:H}, \cC_{h:H}, \cB_{h:H}, t_{h:H})$,
	\end{enumerate} 
	then if $(\pibell_\tau)_{\tau\geq h}$ is the policy induced by $\Vhat_{h:H}$ and $N_\reg$ is set as in \cref{alg:learnlevel3}, we have that for all $\pi \in \Pi'$, 
	\begin{align}
		& \sum_{(x_{h-1},a_{h-1})\in \cC_h}\frac{1}{N_\reg}\sum_{(x_h,-)\in \cD_h(x_{h-1}, a_{h-1})} \left( \Vhat_{h}(x_h)-V^{\pi}_{h}(x_{h})\right)^2\nn \\
		& \leq \frac{9{k} H^2\log(8k^2H|\Pi'||\cV|/\delta)}{N_\reg}  +{8 H^2} \sum_{(x_{h-1},a_{h-1})\in \cC_h}\sum_{\tau=h}^H \E^{\pibell}\left[\tv{\pibell_\tau(\x_\tau)}{\pi_\tau(\x_\tau)}  \mid  \x_{h-1}=x_{h-1},\a_{h-1}=a_{h-1}\right],
	\end{align}
	where the datasets $\{\cD_h(x, a): (x,a)\in \cC_h\}$ are as in the definition of $\cVhat_h$ in \eqref{eq:confidence3}.
\end{lemma}

Next, we use this result to show that the confidence sets $\cVhat_{1:H}$ returned by $\learnlevel_0$ are ``valid'' in the sense that they contain a value function $(V^{\pib}_h)$ corresponding to a near-optimal stochastic policy $\pib$ in the benchmark class $\Pi_{{4\veps}}$. In the sequel, we use this fact to substitute $V^{\pib}_\ell$ for $f_\ell$ in \cref{eq:test_passed_lem} and bound the suboptimality of the learned policy $\pibell$. 
\begin{lemma}[Confidence sets]
	\label{lem:confidencesets}
	Let $\veps,\delta\in(0,1)$ be given and suppose that \cref{ass:pushforward} (pushforward coverability) holds with parameter $\Cpush>0$ .
	Let $f \in \cV$ be arbitrary, and suppose that $\cV$ satisfies \cref{ass:relaxreal} with $\veps_\rel = 4 \veps$. Then, there is an event $\cE'''$ of probability at least $1-3\delta$ under which a call to $\learnlevel_{0}(f,\cV,\emptyset,\emptyset,0;\cV,\veps,\delta)$ (\cref{alg:learnlevel3}) terminates and returns tuple $(\Vhat_{1:H}, \cVhat_{1:H}, \cC_{1:H}, \cB_{1:H}, t_{1:H})$ such that 
	\begin{align}
		\forall h \in [H], \quad V^{\pib}_h \in \cVhat_h, \label{eq:induct}
	\end{align}
	where $\pib_{1:H}\in \Pims$ is the stochastic policy defined recursively via
	\colt{
			\begin{align}
			\forall x \in \cX,\ \ 	\pibb_\tau(x) \in \argmax_{a\in \cA} \left\{ \begin{array}{ll}  \bQhat_\tau(x,a),  & \text{if }  \| \bQhat_\tau(x,\cdot) -\cP_\tau [ V^{\pib}_{\tau+1}](x,\cdot)\|_\infty \leq 4\veps, \\   
				\cP_\tau[V^{\pib}_{\tau+1}](x,a), & \text{otherwise}, \end{array}\right.  
			\label{eq:pib}
		\end{align} 
	for $\tau = H,\dots,1$, and $\bQhat_\tau(x,a) \coloneqq  \Phat_{\tau,\veps, \delta'} [\Vhat_{\tau+1}](x,a)$ is a realization of the stochastic output of the $\Phat$ operator in \cref{alg:Phat} given input $(x,a)$, with $\delta'$ is as in \cref{alg:learnlevel3}.
	}
	\arxiv{
	\begin{align}
		\forall x \in \cX,\ \ 	\pibb_\tau(x) \in \argmax_{a\in \cA} \left\{ \begin{array}{ll}  \bQhat_\tau(x,a),  & \text{if }  \| \bQhat_\tau(x,\cdot) -\cP_\tau [ V^{\pib}_{\tau+1}](x,\cdot)\|_\infty \leq 4\veps, \\   
			\cP_\tau[V^{\pib}_{\tau+1}](x,a), & \text{otherwise}, \end{array}\right.  \text{ for $\tau = H,\dots,1$,} 
		\label{eq:pib}
	\end{align} 
	where $\bQhat_\tau(x,a) \coloneqq  \Phat_{\tau,\veps, \delta'} [\Vhat_{\tau+1}](x,a)$ is a realization of the stochastic output of the $\Phat$ operator in \cref{alg:Phat} given input $(x,a)$, and $\delta'$ is as in \cref{alg:learnlevel3}. Furtherore, we have $\pib\in \Pi_{4\veps}$.
}
\end{lemma}

The proof of the lemma is in \cref{proof:confidencesets}.

Equipped with the preceding lemmas, we now state the main technical result of this section, \cref{thm:lbc3}, a generalization of \cref{thm:vpiforward_main} which holds under relaxed $V^{\pi}$-realizability (\cref{ass:relaxreal}). The proof is in \cref{proof:lbc3}.
\begin{theorem}[Guarantee for \learnlevel under relaxed $V^{\pi}$-realizability]
	\label{thm:lbc3}
	Let $\delta,\veps\in(0,1)$ be given, and suppose that \cref{ass:pushforward} (pushforward coverability) holds with parameter $\Cpush>0$.
	Let $f \in \cV$ be arbitrary, and assume that $\cV$ that satisfies \cref{ass:relaxreal} with $\veps_\rel = 4 \veps$. Then, with probability at least $1-5\delta$, $\learnlevel_{0}(f,\cV,\emptyset,\emptyset,0;\cV,\veps,\delta)$ (\cref{alg:learnlevel3}) terminates and returns value functions $\Vhat_{1:H}$ that satisfy
	\begin{align}
		\forall h \in [H], \quad \bbE^{\pibell}\brk*{ \tv{\pibell_{h}(\x_h)}{\pib_h(\x_{h})}
		} \leq \frac{\veps}{4H^3 \Cpush}, \label{eq:returned3}
	\end{align}
	where $\mb{\pibell}_h(x) \in \argmax_{a\in \cA} \Phat_{h,\veps,\delta'}[\Vhat_{h+1}](x,a)$ for all $h\in\brk{H}$, with $\pib \in \Pi_{4\veps}$ defined as in \cref{lem:confidencesets} and $\delta'$ defined as in \cref{alg:learnlevel3}. Furthermore, the number of episodes is bounded by \begin{align}
	\wtilde{O}(\Cpush^8 H^{10}A\cdot{}\veps^{-13}).
	\end{align}
\end{theorem}

Next, we state a guarantee for the outer-level algorithm, \rvflF, under relaxed $V^{\pi}$-realizability. Recall that \rvflF{} invokes $\mainalg_0$, then extracts an executable policy by applying the \forward{} algorithm (see \cref{sec:forward}), with the ``expert'' policy set to be the output of  \learnlevel{}. 
\begin{theorem}[Main guarantee of \rvflF]
	\label{thm:vpiforward}
	Let $\delta, \veps\in(0,1)$ be given, and define $\veps_{\learnlevel}= \vepsllnum $. Suppose that \begin{itemize}
		\item \cref{ass:pushforward} (pushforward coverability) holds with parameter $\Cpush>0$;
		\item  the function class $\cV$ satisfies \cref{ass:relaxreal} with $\veps_\rel=1$ (i.e.~all $\pi$-realizability); and
		\item the policy class $\Pi$ satisfies \cref{ass:pireal}.
	\end{itemize}
	Then, with probability at least $1-\delta$, $\pihat_{1:H}= \rvflF(\Pi, \cV, \veps, \delta)$ (\cref{alg:forward_vpi}) satisfies
	\begin{align}
		J(\pistar)- J(\pihat_{1:H}) \leq  \veps. \label{eq:final}
	\end{align}
	Furthermore, the total number sample complexity in the \framework framework is bounded by
	\begin{align}
	\wtilde{O}\left(\Cpush^8 H^{23}  A \epsilon^{-13}\right).
	\end{align}
      \end{theorem}
      The proof is in \cref{proof:vpiforward}. Note that \cref{thm:vpiforward} is a restatement of \cref{thm:vpiforward_main} in \setupii{} (restated for convenience). As a result, \cref{thm:vpiforward_main} is an immediate corollary.
\begin{proof}[Proof of \cref{thm:vpiforward_main}]
	The result follows from \cref{thm:vpiforward}, since \cref{ass:realpi} is stronger than \cref{ass:relaxreal}.
	\end{proof}

\subsection{Proof of \creftitle{lem:testfailures} (Number of Test Failures)}
\label{proof:testfailures}
\begin{proof}[\pfref{lem:testfailures}]
	Fix $h\in[H]$. We note that the size of $\cC_h$ corresponds to the number of times the test in \cref{line:test3} fails for $\ell=h$ throughout the execution of $\learnlevel_{0}(f,\cV,\emptyset,\emptyset;\cV,\veps,\delta)$. 
	
	Let $M \coloneqq  \Mnum$ denote the desired upper bound on $|\cC_h|$. Suppose that the test in \cref{line:test3} fails at least twice for $\ell=h$ (if the test fails at most twice, then $|\cC_h|\leq 2$ and so \eqref{eq:one3} holds for $\ell=h$ trivially), %
	and let \[(x\ind{1}_{h-1}, a\ind{1}_{h-1},\Vhat\ind{1}_{h}, \cVhat \ind{1}_h, t_h\ind{1}), (x\ind{2}_{h-1}, a\ind{2}_{h-1},\Vhat\ind{2}_h, \cVhat \ind{2}_h, t_h\ind{2}), \dots\] denote the elements of the set $\cB_h$ in the order at which they are added to the latter in \cref{line:added} of \cref{alg:learnlevel3}. Note that $|\cB_h|=|\cC_h|$. Note also that $t_h\ind{i}$ represents the number of times the subroutine $\Phat_{h-1,\veps,\delta'}$ has been called in the test of \cref{line:test3} throughout the execution of $\learnlevel_0$ and up to the time the test failed for $(x\ind{i}_{h-1}, a\ind{i}_{h-1})$. We will use this fact in a concentration argument in the sequel. %
	
	By definition of $(\cVhat _h\ind{i})$ and \cref{lem:corbern} (Freedman's inequality) instantiated with 
	\begin{itemize}
		\item $\cQ = \{ \Vhat\ind{i}_h - f_h  : f\in \cVhat\ind{i}_h \}$;
		\item $\y_h = \x_h$;
		\item $B = H$; and 
		\item $n = N_\reg \cdot i$;
	\end{itemize} 
	and the union bound over $i\in [M \wedge |\cC _h|]$, we get that there is an event $\cE_h$ of probability at least $1-\delta/(2H)$ under which 
	\colt{
	\begin{align}
		\forall i\in\brk{M \wedge |\cC_h|},\forall f\in \cVhat_h\ind{i}, \quad 	&\sum_{j<i} \E\brk[\big]{  (\Vhat\ind{i}_{h}(\x_h)-f_h(\x_h))^2 \mid{}x\ind{j}_{h-1},a\ind{j}_{h-1}}\nn \\ &  \leq\tilde\veps_\reg^2(i) \coloneqq 2\veps_\reg^2 + \frac{4 H^2 \log (4 MH|\cV|/\delta)}{N_\reg}. \label{eq:leee3}
	\end{align}
}
\arxiv{
	\begin{align}
	\forall i\in\brk{M \wedge |\cC_h|},\forall f\in \cVhat_h\ind{i}, \quad & \sum_{j<i} \E\brk[\big]{  (\Vhat\ind{i}_{h}(\x_h)-f_h(\x_h))^2 \mid{}\x_{h-1}= x\ind{j}_{h-1}, \a_{h-1}=a\ind{j}_{h-1}}  & \nn \\
	& \leq\tilde\veps_\reg^2 \coloneqq 2\veps_\reg^2 + \frac{4 H^2 \log (4 MH|\cV|/\delta)}{N_\reg}. \label{eq:leee3}
\end{align}
}
	Now, define $f\ind{i}_h \in \argmax_{f\in \cVhat _h\ind{i}} \big|\En\big[{ \Vhat\ind{i}_{h}(\x_h)-f_{h}(\x_h) \mid{}\x_{h-1}= x\ind{i}_{h-1},\a_{h-1}=a\ind{i}_{h-1}  }\big]\big|$. From \eqref{eq:leee3}, we have that under $\cE_h$:  
	\begin{align}
		\forall i\in\brk{M\wedge |\cC_h|}, \quad 		\sum_{j<i} \E\brk[\big]{(\Vhat\ind{i}_{h}(\x_h)-f\ind{i}_h(\x_h))^2 \mid{}\x_{h-1}=x\ind{j}_{h-1},\a_{h-1}=a\ind{j}_{h-1}}\leq\tilde\veps_\reg^2. \label{eq:firstpart3}
	\end{align}
	We now use this to bound the number of times the test in \cref{line:test3} fails for $\ell=h$. Suppose for the sake of contradiction that the test fails at least $N$ times for some $N\geq M$ (i.e.~$|\cC_h|=N\geq M$). Conditioned on $\cE_h$, we have by \cref{lem:pushforward_potential} and \cref{eq:firstpart3}, 
		\begin{align}
			&	\min_{i\in[M]} \sup_{f\in\cVhat \ind{i}_h} \left|\En\left[{ \Vhat\ind{i}_{h}(\x_h)-f_{h}(\x_h) \mid{}\x_{h-1}=x\ind{i}_{h-1},\a_{h-1}=a\ind{i}_{h-1}  }\right]\right|  \nn \\
			&	= \min_{i\in[M]}  \left|\En\left[{ \Vhat\ind{i}_{h}(\x_h)-f\ind{i}_{h}(\x_h) \mid{} \x_{h-1}=x\ind{i}_{h-1},\a_{h-1}=a\ind{i}_{h-1} }\right]\right|,  \nn \\
			&\leq 2 \left( \frac{\Cpush }{M^2}M \tilde\veps_\reg^2 \log(2M)\right)^{1/2}+ \frac{ 2\Cpush H}{M}.\nn \\
			\intertext{Now, substituting the expression of $\tilde\veps_\reg^2$ in \eqref{eq:leee3} and using the definition of $\veps_\reg^2$ in \cref{line:beta} of \cref{alg:learnlevel3}, we get }
			& = 2\left(\frac{\Cpush }{M} \cdot \left(2\veps_\reg^2 + \frac{4 M H^2 \log (4MH|\cV|/\delta)}{N_\reg}\right)\right)^{1/2}  + \frac{2\Cpush H}{M},  \nn  \\
			& \leq  2\left(\frac{\Cpush }{M} \cdot \left(\frac{22 M H^2\log(8M^2H|\cV|^2/\delta)}{N_\reg}  +  \frac{68 MH^3\log(8M^6 N^2_\test  H^8/\delta)}{N_\test}\right)\right)^{1/2}  + \frac{2\Cpush H}{M},  \nn  \\
			& \leq \veps, \label{eq:contradiction3}
		\end{align} %
	where the last inequality uses that $M =  \Mnum$ and \[ N_\reg=  \Nregnum  \  \text{ and }\  N_\test=   \Ntestnum;\] 
	see \cref{line:paramsVpi} of \cref{alg:learnlevel3}. 
	
	On the other hand, by \cref{lem:phat}, there is an event $\cE_h'$ of probability at least $1-\delta/(2MH)$ under which for all $f \in \cV$, all $i \in [M]$, and $\delta'$ as in \cref{alg:learnlevel3}: 
		\colt{
		\begin{align}
                  \abs*{\Phat_{h-1,\veps,\delta'}[\Vhat_h\ind{i}](x_{h-1}\ind{i}, a_{h-1}\ind{i})-	\Phat_{h-1,\veps,\delta'}[f_h](x_{h-1}\ind{i}, a_{h-1}\ind{i})}& \leq  \abs*{(\cP_{h-1}[\Vhat_h\ind{i}]-\cP_{h-1}[f_h])(x_{h-1}\ind{i}, a_{h-1}\ind{i})}\nn \\ & \quad  + \veps \cdot \sqrt{2\log_{1/\delta'} (8MH|\cV|(t\ind{i}_h)^2/\delta)}. \label{eq:conc}
		\end{align}
	}
	\arxiv{
          \begin{align}
            &\abs*{\Phat_{h-1,\veps,\delta'}[\Vhat_h\ind{i}](x_{h-1}\ind{i}, a_{h-1}\ind{i})-	\Phat_{h-1,\veps,\delta'}[f_h](x_{h-1}\ind{i}, a_{h-1}\ind{i})}\\
            &\leq  \abs*{\cP_{h-1}[\Vhat_h\ind{i}-f_h](x_{h-1}\ind{i}, a_{h-1}\ind{i})} + \veps \cdot \sqrt{2\log_{1/\delta'} (4M AH|\cV|(t\ind{i}_h)^2/\delta)}, \nn\\
			& =  \abs*{\cP_{h-1}[\Vhat_h\ind{i}-f_h](x_{h-1}\ind{i}, a_{h-1}\ind{i})}  + \veps \cdot \beta(t_h\ind{i}), \label{eq:conc}
	\end{align}
	where $\beta(t_h\ind{i})$ is as in \cref{alg:learnlevel3}.} Thus, under $\cE'_h$, the test in \cref{line:test3} fails for $\ell=h$ at least $M$ times only if 
	\begin{align}
		\forall i \in[M], \quad  \veps  %
		& <  \sup_{f\in\cVhat \ind{i}_h} \left|(\Phat_{h-1, \veps,\delta'}[\Vhat\ind{i}_{h}]-\Phat_{h-1, \veps,\delta'}[f_{h}])(x_{h-1}\ind{i}, a_{h-1}\ind{i}) \right| \nn \\ \quad &\quad - \veps \cdot  \beta(t_h\ind{i}), \nn \\ & \leq \sup_{f\in\cVhat\ind{i}_h} \left| \E\left[\Vhat\ind{i}_{h}(\x_h)-f_{h}(\x_{h}) \mid{}\x_{h-1}= x\ind{i}_{h-1},\a_{h-1}= a\ind{i}_{h-1}  \right]\right|   \quad \text{(by \eqref{eq:conc})},\nn \\
		& <  \sup_{f\in\cVhat \ind{i}_h} \left|\En\left[\Vhat\ind{i}_{h}(\x_h)-f_{h}(\x_{h}) \mid{}\x_{h-1} =x\ind{i}_{h-1},\a_{h-1}=a\ind{i}_{h-1}  \right]\right|.
	\end{align}
        Unless $N<M$, this is a contradiction to \cref{eq:contradiction3}. We conclude that under the event $\cE_h\cap \cE_h'$, the test in \cref{line:test3} fails at most $N < M= \Mnum$ times for $\ell=h$, and so under $\cE_1\cap \cE_1'\cap \dots \cap \cE_H\cap \cE_H$, we have
	\begin{align}
		\forall h \in[H], \quad	|\cC_h| \leq\Mnum. \label{eq:boundonCh3}
	\end{align}
	By the union bound, we have $\P[\cE_1\cap \cE_1'\cap \dots \cap \cE_H\cap \cE_H'] \geq 1 -\delta$, which completes the proof.
\end{proof}

\subsection{Proof of \creftitle{lem:confidence} (Consequence of Passing the Tests)}
\label{proof:confidence}
\begin{proof}[\pfref{lem:confidence}]
  Let $h\in\brk{H}$ be given. Fix $\ell\in [h+1 \ldotst H]$ and let $\x\ind{1}_{\ell-1},\x\ind{2}_{\ell-1},\dots$ denote the sequence of states used in the tests of \cref{line:test3} throughout the execution of $\learnlevel_0$; we assume that the sequence is \emph{ordered} in the sense that if $i<j$, then $\x\ind{i}_{\ell-1}$ is used in the test of \cref{line:test3} before $\x\ind{j}_{\ell-1}$. Let $\bm{T}_\ell \in \mathbb{N}\cup\{+\infty\}$ be the random variable representing the total number of times the operator $\Phat_{\ell-1,\veps,\delta'}$ is invoked in \cref{line:test3} throughout the execution of $\learnlevel_0$ ($\bm{T}_\ell$ is also the random length of the sequence $\x\ind{1}_{\ell-1},\x\ind{2}_{\ell-1},\dots$; if $\learnlevel_0$ terminates, then $\bm{T}_\ell$ is finite. The first step of the proof will be to show that under the event $\cE$ of \cref{lem:testfailures}, $\bm{T}_\ell$ is no larger than $M^3 N_\test H^3$ at any point during the execution of $\learnlevel_0$. This will help us establish key concentration results, leading to the desired inequality \eqref{eq:test_passed_lem}.

	\paragraph{Bounding $\bm{T}_\ell$ under $\cE$}
	First, note that under the event $\cE$ of \cref{lem:testfailures}, we have that for any $\tau \in [H]$, \begin{align} |\cC_\tau|\leq M\coloneqq\Mnum,
		\label{eq:bread}\end{align} and so $\learnlevel_\tau$ gets called at most $M$ times throughout the execution of $\learnlevel_0$. For the rest of this paragraph, we condition on $\cE$ and fix $\tau\in[0 \ldotst H]$. Within any given call to $\learnlevel_\tau$ (throughout the execution of $\learnlevel_0$), the operator $\Phat_{\ell-1,\veps,\delta'}$ is invoked at most \[
	\underbrace{|\cC_\tau|  N_\test H}_{\text{Due to the for-loops in \cref{line:begin3}, \cref{line:second}, \& \cref{line:third}}} \times  \underbrace{HM}_{\text{Number of times $\learnlevel_\tau$ returns to \cref{line:begin3} (see below)}} \leq M^2 N_\test  H^2
	\]
	times. This is because the for-loop in \cref{line:begin3} of $\learnlevel_\tau$ resumes whenever a test in \cref{line:test3} fails for one of the layers $\tau+1,\dots,H$ (see \cref{line:goto3}) once the recursive calls return, and the total number of test failures across all these layers is bounded by $H M$ (by \eqref{eq:bread}). Now, since $\learnlevel_\tau$ gets called at most $M$ times throughout the execution of $\learnlevel_0$ (as argued in the prequel), the total number of times the operator $\Phat_{\ell-1,\veps,\delta'}$ is invoked in \cref{line:test3} within $\learnlevel_\tau$ is at most 
	\begin{align}
		M^3 N_\test H^2.
	\end{align} 
	Finally, the total number of times the operator $\Phat_{\ell-1,\veps,\delta'}$ is called in \cref{line:test3} throughout the execution of $\learnlevel_0$ is at most $H$ times larger (accounting for the contributions from $\learnlevel_\tau$ for all $\tau \in [H]$); that is, it is at most
        $M^3 N_\test H^3$. We conclude that the random variable $\bm{T}_\ell$ satisfies  
	\begin{align} \bm{T}_\ell \leq M^3 N_\test  H^3 \label{eq:thetbound}
	\end{align}
        under $\cE$.
	\paragraph{Specifying $\cE'_h$} In this paragraph, we no longer condition on $\cE$. We will specify the event $\cE'_h$ in the lemma statement. 	Let $\delta'$ be defined as in \cref{alg:learnlevel3}. By \cref{lem:phat}, we have that there is an event $\cE'_{h,\ell}$ of probability at least $1-\delta/(2 H^2)$ under which:
	\begin{align}
          	\forall i \in[\bm{T}_\ell], \forall a_{\ell-1} \in \cA\;:\;  & 	\sup_{f\in \cVhat_{\ell}} |(\Phat_{\ell-1,\veps, \delta'}[\Vhat_{\ell}]- \Phat_{\ell-1,\veps, \delta'}[f_\ell])( \x\ind{i}_{\ell-1},a_{\ell-1})|  - \veps -  \veps \cdot \beta(\bm{T}_\ell)\nn \\
		& = \sup_{f\in \cVhat_{\ell}} |(\Phat_{\ell-1,\veps, \delta'}[\Vhat_{\ell}]- \Phat_{\ell-1,\veps, \delta'}[f_\ell])( \x\ind{i}_{\ell-1},a_{\ell-1})|  - \veps -  \veps \cdot \sqrt{2\log_{1/\delta'}(8AH^2M|\cV|\bm{T}_\ell^2/\delta)},\nn \\
		&\geq  \sup_{f\in \cVhat_{\ell}} |(\cP_{\ell-1}[\Vhat_{\ell}]- \cP_{\ell-1}[f_\ell])( \x\ind{i}_{\ell-1},a_{\ell-1})|  - \veps - \veps \cdot \sqrt{2\log_{1/\delta'}(8AH^2M|\cV|\bm{T}_\ell^2/\delta)} \nn \\
		& \qquad - \veps \cdot \sqrt{2\log_{1/\delta'}(4 A H^2M|\cV| i^2/\delta)}, \quad \text{(\cref{lem:phat})} \nn \\  
		&  \geq  \sup_{f\in \cVhat_{\ell}} |(\cP_{\ell-1}[\Vhat_{\ell}]- \cP_{\ell-1}[f_\ell])( \x\ind{i}_{\ell-1},	a_{\ell-1})|  - \veps - 2\veps \cdot \sqrt{2\log_{1/\delta'}(8 AH^2M|\cV|\bm{T}_\ell^2/\delta)}.   \label{eq:toindex}
	\end{align}
	On the other hand, for $k\in [\bm{T}_\ell-N_\test +1]$, we have by \cref{lem:corbern} (Freedman's inequality) instantiated with 
	\begin{itemize}
		\item $n = N_\test$ and $\y_i = \mathbb{I}\left\{\sup_{f\in \cVhat _\ell} \max_{a\in \cA} \left| (\cP_{\ell-1}[\Vhat_{\ell}]- \cP_{\ell-1})[f_\ell](\x\ind{k+i}_{\ell-1},a) \right|> 3 \veps \right\}$, for all $i\in[N_\test]$;
		\item $\cQ = \{ \mathrm{id}\}$;
		\item $B=1$; and
		\item $\lambda=1$;
                \end{itemize}
	that there is an event $\cE''_{h,\ell,k}$ of probability at least $1-\delta/(4 k^2H^2)$ under which 
	\begin{align}
		&\sum_{0\leq i<N_\test}\bbP\brk*{\sup_{f\in \cVhat _\ell} \max_{a\in \cA} \left| (\cP_{\ell-1}[\Vhat_{\ell}]- \cP_{\ell-1}[f_\ell])(\x\ind{k+i}_{\ell-1},a) \right|  > 3\veps} \nn \\
		&   \leq  4 \log(8H^2\bm{T}^2_\ell/\delta) + \sum_{0\leq i<N_\test} \mathbb{I}\left\{\sup_{f\in \cVhat_{\ell},a \in \cA} |(\cP_{\ell-1}[\Vhat_{\ell}]-\cP_{\ell-1}[ f_\ell])(\x\ind{k +i}_{\ell-1},a)| >3\eps\right\}.
	\end{align}
	Now, let $\cE''_{h,\ell} \coloneqq \bigcap_{k\in  [\bm{T}_\ell - N_\test +1]} \cE''_{h,\ell,k}$. By the union bound and the fact that $\sum_{k\geq 1}1/k^2 = \pi^2/6 \leq 2$, we have that $\P[\cE''_{h,\ell}] \geq 1 - \delta/(2 H^2)$. Furthermore, under $\cE''_{h,\ell}$, we have 
	\begin{align}
		\label{eq:test_passed_lem3}
	& 	\forall k \in [\bm{T}_\ell- N_\test +1],  \nn \\ 
		&\sum_{0\leq i<N_\test}\bbP\brk*{\sup_{f\in \cVhat _\ell} \max_{a\in \cA} \left| (\cP_{\ell-1}[\Vhat_{\ell}]- \cP_{\ell-1}[f_\ell])(\x\ind{k+i}_{\ell-1},a) \right|  > 3\veps} \nn \\
		&   \leq  4 \log(8H^2\bm{T}^2_\ell/\delta) + \sum_{0\leq i<N_\test} \mathbb{I}\left\{\sup_{f\in \cVhat_{\ell},a \in \cA} |(\cP_{\ell-1}[\Vhat_{\ell}]-\cP_{\ell-1}[ f_\ell])(\x\ind{k +i}_{\ell-1},a)| >3\eps\right\}. \label{eq:freed}
	\end{align}
	We define $\cE'_h \coloneqq \cE'_{h,1}\cap \cE_{h,1}''\cap \dots \cap \cE_{h,H}'\cap \cE_{h,H}''$. Note that by the union bound, we have $\P[\cE'_h] \geq 1 - \frac{\delta}{H}$ as desired.
        
	\paragraph{Termination of $\learnlevel_h$ under $\cE\cap \cE'_h$} We now show that under $\cE\cap\cE'_h$, if $\learnlevel_h$ terminates, its output satisfies \eqref{eq:test_passed_lem}. For the rest of the proof, we condition on $\cE \cap \cE'_h$. Suppose that $\learnlevel_h$ terminates and returns $(\Vhat_{h:H}, \cVhat_{h:H}, \cC_{h:H}, \cB_{h:H}, t_{h:H})$. In this case, the value function $\Vhat_{\ell}$ must have passed the tests in \cref{line:test3} for all $(x_{h-1},a_{h-1})\in \cC_h$, $n\in N_\test$, and $a_{\ell-1}\in \cA$. Fix $(x_{h-1}, a_{h-1})\in \cC_h$ and let $k \in [\bm{T}_\ell - N_\test \cdot A+1]$ be such that $(\x^{k + j}_{\ell-1})_{j \in [0 \ldotst N_\test -1]}$ represents a subsequence of states that pass the tests in \cref{line:test3} at layer $\ell$ for $(x_{h-1}, a_{h-1})$ within the call to $\learnlevel_h$. The fact that the sequence $(\x\ind{i}_{\ell-1})_{i\geq 1}$ is ordered (see definition in the first paragraph of this proof) and that $(\x^{k + j}_{\ell-1})_{j \in [0 \ldotst N_\test -1]}$ pass the tests imply that
	\begin{enumerate}
		\item \label{item:3}The states $(\x\ind{k+i}_{\ell-1})_{i\in[0\ldotst N_\test-1]}$ at layer $\ell-1$ are i.i.d., and are obtained by rolling out with $\pihat_{h:H}$ starting from $(x_{h-1},a_{h-1})$; and 
		\item \label{item:4} The test in \cref{line:test3} succeeds for all $(\x\ind{k+j}_{\ell-1})_{j\in [0 \ldotst N_\test -1]}$; that is 
		\begin{align}
			\forall j \in [0 \ldotst N_\test-1], \forall a_{\ell-1} \in \cA,\quad 	& \sup_{f\in \cVhat_{\ell}}  |(\Phat_{\ell-1,\veps, \delta'}[\Vhat_{\ell}]- \Phat_{\ell-1,\veps, \delta'}[f_\ell])( \x\ind{k+j}_{\ell-1},a_{\ell-1})| \nn \\
			&\leq   \veps  + \veps \cdot \beta(k+j), \nn \\	
			&\leq   \veps  + \veps \cdot \sqrt{2\log_{1/\delta'}(8AM|\cV|(k+j)^2/\delta)}, \nn \\
			& \leq \veps  + \veps \cdot \sqrt{2\log_{1/\delta'}(8AM|\cV|\bm{T}_\ell^2/\delta)}.
		\end{align}
	\end{enumerate}
	This implies that
	\colt{
	\begin{align}
	&	\forall i \in [0\ldotst N_\test-1], 	\forall a_{\ell-1} \in \cA\colon \nn  \\   &\ \ \sup_{f\in \cVhat_{\ell}} |(\cP_{\ell-1}[\Vhat_{\ell}]- [f_\ell])( \x\ind{k+ i}_{\ell-1},a_{\ell-1})| - 3 \veps\nn \\
		&\ \  \leq  \sup_{f\in \cVhat_{\ell}} |(\cP_{\ell-1}[\Vhat_{\ell}]- \cP_{\ell-1}[f_\ell])( \x\ind{k+i}_{\ell-1},a_{\ell-1})|  - \veps - 2\veps \cdot \sqrt{2\log_{1/\delta'}(8MH^2|\cV|\bm{T}_\ell^2\/\delta)},   \label{eq:tojust} \\
		&\ \  \leq  \sup_{f\in \cVhat_{\ell}} |(\Phat_{\ell-1,\veps, \delta'}[\Vhat_{\ell}]- \Phat_{\ell-1,\veps, \delta'}[f_\ell])( \x\ind{k+i}_{\ell-1},a_{\ell-1})|  - \veps -  \veps \cdot \sqrt{2\log_{1/\delta'}(8MH^2|\cV|\bm{T}_\ell^2/\delta)}, \quad  \text{(by \eqref{eq:toindex})}  \nn \\
		&\ \ \leq  0.\quad \text{(by \cref{item:4})} \label{eq:won}
	\end{align}
}
	\arxiv{
	\begin{align}
        &  \forall i \in [0\ldotst N_\test-1], 	\forall a_{\ell-1} \in \cA\;:\;\  \nn  \\   &  \sup_{f\in \cVhat_{\ell}} |\cP_{\ell-1}[\Vhat_{\ell}- f_\ell]( \x\ind{k+ i}_{\ell-1},a_{\ell-1})| - 3 \veps\nn \\
		& \leq  \sup_{f\in \cVhat_{\ell}} |\cP_{\ell-1}[\Vhat_{\ell}- f_\ell]( \x\ind{k+i}_{\ell-1},a_{\ell-1})|  - \veps - 2\veps \cdot \sqrt{2\log_{1/\delta'}(4AMH^2|\cV|\bm{T}_\ell^2\/\delta)},   \label{eq:tojust} \\
                                                                                 & \stackrel{\eqref{eq:toindex}}{\leq} 
\sup_{f\in \cVhat_{\ell}} |(\Phat_{\ell-1,\veps, \delta'}[\Vhat_{\ell}]- \Phat_{\ell-1,\veps, \delta'}[f_\ell])( \x\ind{k+i}_{\ell-1},a_{\ell-1})|
- \veps -  \veps \cdot \sqrt{2\log_{1/\delta'}(4AMH^2|\cV|\bm{T}_\ell^2/\delta)},  \nn \\
		& \leq  0.\quad \text{(by \cref{item:4})} \label{eq:won}
	\end{align}
}
	where \eqref{eq:tojust} follows by \eqref{eq:thetbound} and the choice of $\delta'$ in \cref{alg:learnlevel3}.

        Now, by \cref{item:3}, we have that $\x_{\ell-1}\ind{k+i}$ has probability law $\P^{\pihat_{h:H}}[\cdot \mid \x_{h-1}=x_{h-1},\a_{h-1}=a_{h-1}]$ for all $i \in [0 \ldotst N_\test -1]$, and so by \eqref{eq:freed}, we have: 
	\begin{align}
		\label{eq:test_passed_lem2}
		&\bbP^{\pibell}\brk*{\sup_{f\in \cVhat _\ell} \max_{a\in \cA} \left| (\cP_{\ell-1}[\Vhat_{\ell}]- \cP_{\ell-1}[f_\ell])(\x_{\ell-1},a) \right|  > 3\veps \mid  \x_{h-1}=x_{h-1},\a_{h-1}=a_{h-1}} \nn \\
		&   \leq   \frac{4 \log(8H^2\bm{T}^2_\ell/\delta)}{N_\test} + \frac{1}{N_\test} \sum_{0\leq i<N_\test} \mathbb{I}\left\{ \sup_{f\in \cVhat_{\ell},a \in \cA} |\cP_{\ell-1}[\Vhat_{\ell}]- \cP_{\ell-1}[f_\ell])(\x\ind{k+ i }_{\ell-1},a)| >3\eps\right\}, \nn \\
		& \leq  \frac{4 \log(8M^6 N^2_\test  H^8/\delta)}{N_\test} \quad (\text{using \eqref{eq:thetbound} and the fact that all the tests pass, i.e. \eqref{eq:won}}).
	\end{align}

	\paragraph{Concluding} We have established that under $\cE\cap \cE'_h$, we have for all $\ell\in[h+1\ldotst H]$ and all $(x_{h-1},a_{h-1})\in \cC_h$:
	\begin{align}
		&\bbP^{\pibell}\brk*{\sup_{f\in \cVhat _\ell} \max_{a\in \cA} \left|( \cP_{\ell-1}[\Vhat_{\ell}]- \cP_{\ell-1}[f_\ell])(\x_{\ell-1},a) \right|  > 3\veps \mid  \x_{h-1}=x_{h-1},\a_{h-1}=a_{h-1}}    \leq   \testbound,
	\end{align}
	Furthermore, we have $\P[\cE'_h] \geq 1 - \delta/H$. This completes the proof.
\end{proof}

\subsection{Proof of \creftitle{lem:confidencesets1} (Value Function Regression Guarantee)}
\label{proof:confidencesets1}
\begin{proof}[\pfref{lem:confidencesets1}]
	Fix $\pi \in \Pi'\subseteq \Pim$ and $k\geq 1$, and consider the $k$th call to $\learnlevel_h$ as per the lemma statement, and let $\cS_{k}$ be the state of $\learnlevel_0$ during the $k^{\text{th}}$ call to $\learnlevel_h$ and immediately before \cref{eq:gatherdata}, i.e.~immediately before gathering data for the regression step in $\learnlevel_h$.

	\paragraph{Relating the regression targets to $V^{\pihat}_h$} Observe that $\Vhat_h$ is the least-squares solution of the objective in \cref{line:updateQ3}, where the targets are empirical estimates of $V^{\pihat}_{h}$. In particular, if we let $\{\cD_h(x, a): (x,a)\in \cC_h\}$ be the datasets in the definition of $\cVhat_h$ in \eqref{eq:confidence3}, then for any $(x_{h-1},a_{h-1})\in \cC_h$ and $(x_h,v_h)\in \cD_h(x_{h-1},a_{h-1})$, the target $v_h$ satisfies 
	\begin{align}
		v_h =  \Vhat_h(x_h),
	\end{align}
	where $\Vhat_h(x_h)$ is an empirical estimate of $V^{\pihat}_h(x_h)$ obtained by sampling $\Nest(|\cC_h|)=\Nest(k)$ episodes (for $\Nest(\cdot)$ defined as in \cref{alg:learnlevel3}) by rolling out with $\pihat$ after starting from $x_{h}$ and playing action $a$ at layer $h$; note that $|\cC_h|= k$ because we are considering the $k$th call to $\learnlevel_h$. Thus, by Hoeffding's inequality and the union bound over $(x_{h-1},a_{h-1})\in \cC_h$ and $(x_h,-)\in \cD_h(x_{h-1},a_{h-1})$, there is an event $\cE''_{h,k}(\cS_k)$ of probability at least $1-\delta/(8 k^2 H)$ under which 
	\begin{align}
          & \forall (x_{h-1},a_{h-1})\in \cC_h,	\forall (x_h,-) \in \cD_h(x_{h-1},a_{h-1}): \\
          &|V^{\pihat}_h(x_h) - \Vhat_h(x_h)| \leq H\sqrt{\frac{2 \log (8 |\cC_h| N_\reg H k^2 /\delta)}{\Nest(k)}}
                                                         \leq H\sqrt{\frac{2 \log (8  N_\reg H k^3 /\delta)}{\Nest(k)}}, \label{eq:hoeff}
	\end{align}
	where $N_\reg$ is as in \cref{line:paramsVpi}, and the last inequality follows by $|\cC_h|\leq k$ since we are considering the $k$th call to $\learnlevel_h$. Thus, under $\cE''_{h,k}$, we have 
	\begin{align}
          &\forall (x_{h-1},a_{h-1})\in \cC_h,	\forall (x_h,v_h) \in \cD_h(x_{h-1},a_{h-1}): \\ & |V^{\pihat}_h(x_h) - v_h|
           = \left|V^{\pihat}_h(x_h) -\Vhat_h(x_h)\right|  \leq H\sqrt{\frac{2 \log (8  N_\reg H k^3 /\delta)}{\Nest(k)}}%
                                                                                                  \leq  \frac{H}{N_\reg}  \label{eq:cheer} ,  
	\end{align}
	where the second-to-last inequality is by \eqref{eq:hoeff} and the last inequality follows by the choice of $\Nest$ in \cref{alg:learnlevel3}. 
        	\paragraph{Bounding the discrepancy $V^{\pihat}_h - V_h^{\pi}$}
	On the other hand, by the performance difference lemma, the value function $V^{\pihat}_h$ satisfies:
	\begin{align}
		\forall x \in \cX, \quad 	|V^{\pibell}_{h}(x) - V^{\pi}_h(x)| &  \leq  \sum_{\tau=h}^H \E^{\pibell}[\abs*{Q^{\pi}_\tau(\x_\tau, \bm{\pi}_\tau(\x_\tau)) -  Q_\tau^{\pi}(\x_\tau, \bm{\pibell}_\tau(\x_\tau))}\mid \x_h = x],  \nn \\
		& \leq H \sum_{\tau=h}^H \E^{\pibell}[ \tv{\pibell_\tau(\x_\tau)}{\pi_\tau(\x_\tau)} \mid \x_h = x].   \label{eq:sequel}%
	\end{align} 
	Now, let $(x_{h-1}\ind{1},a_{h-1}\ind{1}), (x_{h-1}\ind{2},a_{h-1}\ind{2}), \dots$ denote the elements of $\cC_h$ in the order in which they are added to the latter in \cref{line:added}. By \cref{lem:freed} (Freedman's inequality) instantiated with
	\begin{itemize}
		\item $n = N_\reg \cdot k$.
		\item $\w_i = |V_h^\pi(\x\ind{i}_h)- V^{\pihat}_h(\x\ind{i}_h)|- \E[|V_h^\pi(\x_h)- V^{\pihat}_h(\x_h)|\mid \x_{h-1}=x_{h-1}\ind{j},\a_{h-1}= a_{h-1}\ind{j}]$, for all $i\in[n]$ and $j = \floor{i/N_\reg}+1$, where $\x_h\ind{N_\reg\cdot  j},\dots, \x_h\ind{N_\reg \cdot j + N_\reg -1}\stackrel{\text{i.i.d.}}{\sim}T_h(\cdot \mid \x_{h-1}=x_{h-1}\ind{j},\a_{h-1}= a_{h-1}\ind{j})$;
		\item $\cH_i = \sigma(\x_{h}\ind{1}, \dots \x_{h}\ind{i-1})$, for all $i\in[n]$;
		\item $R = H$; and 
		\item $\lambda = 1/H$;
	\end{itemize} we get that there is an event $\wtilde\cE''_{h,k,\pi}(\cS_k)$ of probability at least $1-\delta/(8 k^2 H|\Pi'|)$ under which:
	\begin{align}
		&	\sum_{(x_{h-1},a_{h-1})\in \cD_h}\sum_{(x_h,-)\in \cD_{h}(x_{h-1},a_{h-1})} |V_h^\pi(x_h)- V^{\pihat}_h(x_h)|  \nn \\
		& = 2	N_\reg \sum_{(x_{h-1},a_{h-1})\in \cD_h}\E\left[|V_h^\pi(\x_h)- V^{\pihat}_h(\x_h)|\mid \x_{h-1}=x_{h-1}, \a_{h-1}=a_{h-1}\right] + H \log(8 k^2 |\Pi'|H/\delta),\nn \\
		& \leq 2H	N_\reg \sum_{(x_{h-1},a_{h-1})\in \cD_h}\sum_{\tau=h}^H \E^{\pibell}\left[ \tv{\pibell_\tau(\x_\tau)}{\pi_\tau(\x_\tau)} \mid  \x_{h-1}=x_{h-1}, \a_{h-1}=a_{h-1}\right] + H \log(8 k^2|\Pi'| H/\delta), \label{eq:midnightsky}
	\end{align}
	where the last inequality follows by \eqref{eq:sequel} and the law of total expectation. 
	\paragraph{Regression guarantee}
	Since $\pi \in \Pi' \subseteq \Pim$ and \cref{ass:relaxreal} holds, \cref{lem:reg} (regression guarantee) instantiated with
	\begin{itemize}
		\item $f_\star(x) = V^{\pi}_h(x)$;
		\item $B = H$;
		\item $\bm{b}_i=\bm{v}_h-V^{\pi}_h(\x_h)$ (where $\bm{v}_h \coloneqq \max_{a\in \cA} \Qhat_h(\x_h,a)$); and
		\item $\xi = H$;
	\end{itemize}
	implies that there is an event $\breve\cE''_{h,k,\pi}(\cS_k)$ of probability at least $1-\delta/(4k^2 H|\Pi'|)$ under which we have:
	\begin{align}
		&  \sum_{(x_{h-1},a_{h-1})\in \cC_h}\frac{1}{N_\reg}\sum_{(x_h,-)\in \cD_h(x_{h-1}, a_{h-1})} \left( \Vhat_{h}(x_h)-V^{\pi}_{h}(x_{h})\right)^2\nn \\
		&\leq \frac{4 {k} H^2\log(4k^2H|\Pi'||\cV|/\delta)}{N_\reg}  +\frac{4 H}{N_\reg}  \sum_{(x_{h-1},a_{h-1})\in \cC_h}\sum_{(x_h, v_h)\in \cD_h(x_{h-1},a_{h-1})} |V_h^\pi(x_h)- v_h|, \nn \\
		&\leq \frac{4 {k} H^2\log(4k^2H|\Pi'||\cV|/\delta)}{N_\reg} +\frac{4 H}{N_\reg}  \sum_{(x_{h-1},a_{h-1})\in \cC_h}\sum_{(x_h, v_h)\in \cD_h(x_{h-1},a_{h-1})} |V_h^{\pihat}(x_h)- v_h| \nn \\ & \quad +\frac{4 H}{N_\reg}  \sum_{(x_{h-1},a_{h-1})\in \cC_h}\sum_{(x_h, v_h)\in \cD_h(x_{h-1},a_{h-1})}  |V_h^\pi(x_h)- V^{\pihat}_h(x_h)|, \label{eq:triangle}
	\end{align}
	where the last step follows by the triangle inequality. Thus, by plugging \eqref{eq:midnightsky} and \eqref{eq:cheer} into \eqref{eq:triangle}, we get that under $\cE''_{h,k}(\cS_k) \cap \wtilde\cE''_{h,k,\pi}(\cS_k)\cap \breve\cE''_{h,k,\pi}(\cS_k)$:
	\begin{align}
		&  \sum_{(x_{h-1},a_{h-1})\in \cC_h}\frac{1}{N_\reg}\sum_{(x_h,-)\in \cD_h(x_{h-1}, a_{h-1})} \left( \Vhat_{h}(x_h)-V^{\pi}_{h}(x_{h})\right)^2\nn \\
		&\leq \frac{9  {k} H^2\log(8k^2H|\Pi'||\cV|/\delta)}{N_\reg}  +{8 H^2} \sum_{(x_{h-1},a_{h-1})\in \cC_h}\sum_{\tau=h}^H \E^{\pibell}\left[\tv{\pibell_\tau(\x_\tau)}{\pi_\tau(\x_\tau)}  \mid \x_{h-1}=x_{h-1},\a_{h-1}=a_{h-1}\right]. \label{eq:robust}
	\end{align}

	\paragraph{Applying the union bound to conclude}
	Let $\bm\cS_{k}$ be the random state of $\learnlevel_0$ during the $k^{\text{th}}$ call to $\learnlevel_h$ and immediately before \cref{eq:gatherdata}, i.e.~immediately before gathering data for the regression step in $\learnlevel_h$. Further, let $\bm\cS_{k}^+$ be the random state of $\learnlevel_0$ during the $k^{\text{th}}$ call to $\learnlevel_h$ and immediately before \cref{line:updateQ3}, i.e.~immediately before the regression step in $\learnlevel_h$. If $\learnlevel_0$ terminates before the $k^{\text{}th}$ call to $\learnlevel_h$, we use the convention that $\bm{\cS}_{k}= \bm{\cS}_k^+= \mathfrak{t}$, where $\mathfrak{t}$ denotes a terminal state, and define $\cE_{h,k}''(\tfrak)=\wtilde\cE''_{h,k,\pi}(\tfrak)= \breve\cE''_{h,k, \pi}(\tfrak) = \{\tfrak\}$. Further, we define 
\begin{align}
	\cE''_h \coloneqq \left\{ \prod_{k\in\mathbb{N}, \pi\in \Pi'}\mathbb{I}\{\bm{\cS}^+_k\in  \cE''_{h,k}(\bm{\cS}_{k}) \cap \wtilde\cE''_{h,k,\pi}(\bm{\cS}_{k})\cap \breve\cE''_{h,k, \pi}(\bm{\cS}_{k})\} =1 \right\}.
\end{align} 
Note that by the argument in the sequel and the union bound, we have that
\begin{align}
\forall k \geq 1, \forall \cS_{k},\quad 	\P[\bm{\cS}^{+}_k\in \cE''_{h,k}({\cS}_{k}) \cap \wtilde\cE''_{h,k,\pi}({\cS}_{k})\cap \breve\cE''_{h,k, \pi}({\cS}_{k}) ] \geq 1 - \frac{\delta}{2k^2H}, \label{eq:midpoint}
\end{align}
where $\cS_k$ denotes the state of $\learnlevel_0$ during the $k^{\text{th}}$ call to $\learnlevel_h$ and immediately before \cref{eq:gatherdata}. By letting $\bm{\cS}_1',\bm{\cS}_2',\dots$ denote an identical, independent copy of the sequence $\bm{\cS}_1,\bm{\cS}_2,\dots$, we have by the chain rule: 
\begin{align}
	\P[\cE''_h] & = \E_{\bm{\cS}'_{1},\bm{\cS}_2',\dots}\left[ \prod_{k\geq 1}\P[\bm{\cS}^+_k\in \cE''_{h,k}(\bm{\cS}_{k}) \cap \wtilde\cE''_{h,k,\pi}(\bm{\cS}_{k})\cap \breve\cE''_{h,k, \pi}(\bm{\cS}_{k}) \mid \bm{\cS}_{k}=\bm{\cS}_{k}']\right], \nn \\
	& \geq  \prod_{k\geq 1} \left(1-\frac{\delta}{2k^2H}\right), \quad \text{(by \eqref{eq:midpoint})}\nn \\
	& \geq 1- \frac{\delta}{H}, \label{eq:finalfinal}
\end{align}
where the last inequality follows from the fact that for any sequence $x_1,x_2,\dots \in(0,1)$, we have $\prod_{k\geq 1}(1-x_k) \geq 1 - \sum_{k\geq 1}x_k$. Combining \eqref{eq:finalfinal} with \eqref{eq:robust} implies that $\cE''_h$ gives the desired result. 
\end{proof}

\subsection{Proof of \creftitle{lem:confidencesets} (Guarantee for Confidence Sets)}
\label{proof:confidencesets}

To prove \cref{lem:confidencesets}, we need one additional result pertaining to the order in which the instances $(\learnlevel_{h})_{h\in[H]}$ are called.

\begin{lemma}
	\label{lem:recurse2}
	Let $h\in [0\ldotst H]$ be given, and consider the setting of \cref{lem:confidencesets}. Further, consider a call to $\learnlevel_0(f,\cV,\emptyset,\emptyset;\cV,\veps,\delta)$ that terminates, and let $h\in [H]$ be any layer such that $\learnlevel_h$ is called during the execution of $\learnlevel_0$. Then, after the last call to $\learnlevel_h$ terminates, no instance of $\learnlevel$ in $(\learnlevel_{\tau})_{\tau >h}$ is called before $\learnlevel_0$ terminates.
\end{lemma} 
\begin{proof}[\pfref{lem:recurse2}]
	Suppose there is an instance of  $\learnlevel$ in $(\learnlevel_{\tau})_{\tau >h}$ that is called after the last call to $\learnlevel_h$ terminates. Let $\tau>h$ be the lowest layer where $\learnlevel_\tau$ is called after the last call to $\learnlevel_h$ terminates. Let $\learnlevel_\tau^{\texttt{last}}$ denote the corresponding instance of $\learnlevel_\tau$. Further, let $\ell <\tau$ be such that $\learnlevel_\ell$ is the parent instance of $\learnlevel_\tau^{\texttt{last}}$ (i.e.~the instance that called $\learnlevel_\tau^{\texttt{last}}$). Note that we cannot have $\ell=h$ as this would imply that an instance of $\learnlevel_h$ terminates after $\learnlevel_\tau^{\texttt{last}}$, and we have assumed that $\learnlevel_\tau^{\texttt{last}}$ terminates after the last call $\learnlevel_h$. It is also not possible to have $\ell>h$ as this would imply that $\tau$ is not the lowest layer where $\learnlevel_\tau$ is called after the last call to $\learnlevel_h$ terminates. Thus, we must have that $\ell< h$. Now, the for-loop in \cref{line:forloop} ensures that that there is an instance of $\learnlevel_h$ that is called after $\learnlevel_\tau^{\texttt{last}}$ terminates and before $\learnlevel_\ell$ does. This contradicts the assumption that $\learnlevel_\tau^{\texttt{last}}$ is called after the last call to $\learnlevel_h$.
\end{proof}

\begin{proof}[Proof of \cref{lem:confidencesets}]
We start by showing that $\pib\in \Pi_{4\veps}$ by constructing the corresponding collection of random state-action value functions $(\bQtilde_{h}(x,a))_{(h,x,a)\in[H]\times \cX\times \cA}\subset [0,H]$ in the definition of $\Pi_{4\veps}$. In particular, for $(h,x,a)\in[H]\times \cX\times \cA$, we define
	\begin{align}
		\bQtilde_h(x,a) = \left\{ \begin{array}{ll}   \bQhat_h(x,a),  & \text{if }  \| \bQhat_h(x,\cdot)- \cP_h[V^{\pib}_{h+1}](x,\cdot)\|_\infty \leq 4\veps, \\   
			\cP_h[V^{\pib}_{h+1}](x,a), & \text{otherwise}, \end{array}\right.  \quad \text{for $h = H,\dots,1$.}
	\end{align}
	where $\bQhat_\tau(x,a)\coloneqq  \Phat_{\tau,\veps, \delta'} [\Vhat_{\tau+1}](x,a)$. Note that $\bQtilde_h(x,a)$ only depends on the randomness of $\Phat_{h,\veps, \delta'} [\Vhat_{h+1}](x,a)$, and so $(\bQtilde_h(x,a))_{(h,x,a)\in [H]\times \cX \times \cA}$ are independent random variables. Furthermore, since $\cP_h[V^{\pib}_{h+1}]\equiv Q^{\pib}_{h}$, we have that  
	\begin{align}
	\forall (x,a)\in \cX \times \cA	,\quad  \|\bQtilde_h(x,a)- Q^{\pib}_h(x,a)\| \leq 4\veps.
	\end{align}
	Finally, since $\pibb_h(\cdot) \in \argmax_{a\in \cA} \bQtilde_h(\cdot,a)$, we have that $\pib\in \Pi_{4\veps}$.

	\paragraph{We show $V^{\pib}_h \in \cVhat_h$} We prove that $V^{\pib}_h \in \cVhat_h$, for all $h\in[H]$, under the event $\cE'''\coloneqq \cE \cap \cE'_{1} \cap \cE_1'' \cap \dots \cap \cE_H' \cap \cE_H''$, where $\cE$, $(\cE_h')$, and $(\cE_h'')$ are the events defined in \cref{lem:testfailures}, \cref{lem:confidence}, and \cref{lem:confidencesets1}, respectively. Throughout, we condition on $\cE'''$. First, note that by \cref{lem:testfailures}, $\learnlevel_0$ terminates. Let $(\Vhat_{1:H},\cVhat _{1:H},
	\cC_{1:H},\cB_{1:H}, t_{1:H})$ be the tuple it returns.
	
	We will show via backwards induction over $\ell=H+1,\dots, 1$, that
	\begin{align}
		V^{\pib}_\ell \in \cVhat_\ell, \label{eq:newthirt}
	\end{align} 
	where $\pib_{1:H}$ is the stochastic policy defined recursively via
		\colt{
		\begin{align}
			\pibb_\tau(x) \in \argmax_{a\in \cA} \left\{ \begin{array}{ll}  \bQhat_\tau(x,a) \coloneqq  \Phat_{\tau,\veps, \delta'} [\Vhat_{\tau+1}](x,a),  & \text{if }  \| \bQhat_\tau(x,\cdot) -\cP_\tau [ V^{\pib}_{\tau+1}](x,\cdot)\|_\infty \leq 4\veps, \\   
				\cP_\tau[V^{\pib}_{\tau+1}](x,a), & \text{otherwise}, \end{array}\right.
			\label{eq:pibnew}
		\end{align}
		for $\tau = H,\dots,1$,	where $\bQhat_\tau(x,a) \coloneqq  \Phat_{\tau,\veps, \delta'} [\Vhat_{\tau+1}](x,a)$.
	}
	\arxiv{
	\begin{align}
		\pibb_\tau(x) \in \argmax_{a\in \cA} \left\{ \begin{array}{ll}  \bQhat_\tau(x,a) \coloneqq  \Phat_{\tau,\veps, \delta'} [\Vhat_{\tau+1}](x,a),  & \text{if }  \| \bQhat_\tau(x,\cdot) -\cP_\tau [ V^{\pib}_{\tau+1}](x,\cdot)\|_\infty \leq 4\veps, \\   
			\cP_\tau[V^{\pib}_{\tau+1}](x,a), & \text{otherwise}, \end{array}\right.  \ \text{for } \tau = H,\dots,1,
		\label{eq:pibnew}
	\end{align} 	
		where $\bQhat_\tau(x,a) \coloneqq  \Phat_{\tau,\veps, \delta'} [\Vhat_{\tau+1}](x,a)$.
}

	\paragraph{Base case [$\ell=H+1$]} This holds trivially because $V^{\pi}_{H+1}\equiv 0$ for any $\pi \in \Pims$ by convention. 
	
	\paragraph{General case [$\ell\leq H$]} Fix $h\in[H]$ and suppose that \eqref{eq:newthirt} holds for all $\ell\in[h+1\ldotst H+1]$. We show as a consequence that \eqref{eq:newthirt} holds for $\ell=h$. First, note that if $\learnlevel_h$ is never called during the execution of $\learnlevel_0$, then $\cVhat_h = \cV$, and so \eqref{eq:newthirt} trivially holds for $\ell=h$ under \cref{ass:relaxreal} with $\veps_\rel = 4 \veps$. 
	
	Now, suppose that $\learnlevel_h$ is called at least once, and let $(\Vhat^+_{h:H},\cVhat^+_{h:H},
	\cC^+_{h:H},\cB^+_{h:H},t^+_{h:H})$ be the output of the last call to $\learnlevel_h$ during the execution of $\learnlevel_0$. We claim that \begin{align}(\Vhat^+_{h:H},\cVhat^+_{h:H},
		\cC^+_{h:H}) = (\Vhat_{h:H},\cVhat_{h:H},
		\cC_{h:H}).  \label{eq:same}
	\end{align}
	To see this, first note that the for-loop in \cref{line:forloop} ensures that no instance of $(\learnlevel_\tau)_{\tau>h}$ can be called after the last call to $\learnlevel_h$ (by \cref{lem:recurse2}). Thus, the estimated value functions, confidence sets, and core sets for layers $h+1, \dots, H$ remain unchanged after the last call to $\learnlevel_h$; that is, \eqref{eq:same} holds. Thus, by \cref{lem:confidence}, and since we are conditioning on $\cE_{h+1:H}'$, we have that for all $(x_{h-1},a_{h-1}) \in \cC_h$ and $\ell \in [h+1 \ldotst H+1]$:
	\begin{align}
		\bbP^{\pibell}\brk*{\sup_{f\in \cVhat _\ell} \max_{a\in \cA} \left| (\cP_{\ell-1}[\Vhat_{\ell}]- \cP_{\ell-1}[f_\ell])(\x_{\ell-1},a) \right|  > 3\veps \mid  \x_{h-1}=x_{h-1},\a_{h-1}=a_{h-1}}  \leq  \testbound, \label{eq:whole}
	\end{align}
	where $M = \Mnum$. Now, by the induction hypothesis, we have $V^{\pib}_\ell\in \cVhat_\ell$, and so substituting $V^{\pib}_\ell$ for $f_\ell$ in \eqref{eq:whole}, we get that for all $(x_{h-1},a_{h-1}) \in \cC_h$ and $\ell \in [h+1 \ldotst H+1]$:
	\begin{align}
		\bbP^{\pibell}\brk*{ \max_{a\in \cA} \left|( \cP_{\ell-1}[\Vhat_{\ell}]- \cP_{\ell-1}[V^{\pib}_\ell])(\x_{\ell-1},a) \right|  > 3\veps \mid \x_{h-1}= x_{h-1},\a_{h-1}=a_{h-1}}  \leq  \testbound.
	\end{align}
	Therefore, by \cref{lem:tvdistance} (instantiated with $\mu[\cdot] = \bbP^{\pibell}[\cdot \mid \x_{h-1}=x_{h-1},\a_{h-1}=a_{h-1}]$, $\tau = \ell-1$, and $V_{\tau+1} = V^{\pib}_{\ell}$), we have that $(x_{h-1},a_{h-1}) \in \cC_h$ and $\ell \in [h+1 \ldotst H+1]$: 
	\begin{align}
		\En^{\pibell}\brk*{ \tv{\pibell_{\ell-1}(\x_{\ell-1})}{\pib_{\ell-1}(\x_{\ell-1})}  \mid \x_{h-1}=x_{h-1},\a_{h-1}=a_{h-1}}\leq \testbound + \delta', \label{eq:above}
	\end{align}
	where $\delta'$ is as in \cref{alg:learnlevel3}. 
	\paragraph{Applying the regression guarantee to conclude the induction} Note that $\pib \in \Pi'$, where $\Pi'\subset \Pim$ is the set of stochastic policies such that $\pi \in \Pi'$ if and only if there exists $V_{1:H} \in \cV$ such that $\pi$ is defined recursively as 
	\begin{align}
	\bm{\pi}_\tau(x) \in \argmax_{a\in \cA} \left\{ \begin{array}{ll}  \bm{Q}_\tau(x,a) \coloneqq  \Phat_{\tau,\veps, \delta'} [V_{\tau+1}](x,a),  & \text{if }  \| \bm{Q}_\tau(x,\cdot) -\cP_\tau [ V^{\pi}_{\tau+1}](x,\cdot)\|_\infty \leq 4\veps, \\   
	\cP_\tau[V^{\pi}_{\tau+1}](x,a), & \text{otherwise}, \end{array}\right.
	\label{eq:pibnew}
	\end{align}
	for $\tau = H,\dots,1$,	where $\bm{Q}_\tau(x,a) \coloneqq  \Phat_{\tau,\veps, \delta'} [V_{\tau+1}](x,a)$. The policy class $\Pi'$ is finite and $|\Pi'|\leq |\cV|$. Furthermore, we have $\Pi' \subseteq \Pi_{4\veps}$ as shown at the beginning of this proof. Therefore, if we let $\{\cD_h(x, a): (x,a)\in \cC_h\}$ be the datasets in the definition of $\cVhat_h$ in \eqref{eq:confidence3}, we have by \cref{lem:confidencesets1} (under \cref{ass:relaxreal} with $\veps_\rel = 4 \veps$) and the conditioning on $\cE_{h+1:H}''$ and $\cE$:
	\begin{align}
		& \sum_{(x_{h-1},a_{h-1})\in \cC_h}\frac{1}{N_\reg}\sum_{(x_h,-)\in \cD_h(x_{h-1}, a_{h-1})} \left( \Vhat_{h}(x_h)-V^{\pib}_{h}(x_{h})\right)^2\nn \\
		& \leq \frac{9|\cC_h| H^2\log(8|\cC_h|^2H|\cV|^2/\delta)}{N_\reg}  +{8 H^2} \sum_{(x_{h-1},a_{h-1})\in \cC_h}\sum_{\tau=h}^H \E^{\pibell}\left[\tv{\pibell_\tau(\x_\tau)}{{\pib}_\tau(\x_\tau)}  \mid \x_{h-1}=x_{h-1},\a_{h-1} a_{h-1}\right], \nn \\
		& \leq \frac{9 {M} H^2\log(8M^2H|\cV|^2/\delta)}{N_\reg}  +8 H^2  \sum_{(x_{h-1},a_{h-1})\in \cC_h}\sum_{\tau=h}^H \E^{\pibell}\left[\tv{\pibell_\tau(\x_\tau)}{{\pib}_\tau(\x_\tau)}  \mid \x_{h-1}= x_{h-1}, \a_{h-1}=a_{h-1}\right], \label{eq:happ}
	\end{align}
	where the last inequality follows by the fact that $|\cC_h| \leq M$ under $\cE$. Combining \eqref{eq:happ} with \eqref{eq:above}, implies that  
	\begin{align}
		& \sum_{(x_{h-1},a_{h-1})\in \cC_h}\frac{1}{N_\reg}\sum_{(x_h,-)\in \cD_h(x_{h-1}, a_{h-1})} \left( \Vhat_{h}(x_h)-V^{\pib}_{h}(x_{h})\right)^2\nn\\
		& \leq \frac{9 {M} H^2\log(8M^2H|\cV|^2/\delta)}{N_\reg}  + 8 M H^3\cdot \testbound  + 8M {H^3}\delta', \nn \\
		& = \frac{9 {M} H^2\log(8M^2H|\cV|^2/\delta)}{N_\reg}  + 8 M H^3\cdot \testbound  + 8M {H^3}  \frac{\delta}{4M^7N_\test^2 H^8|\cV|},\nn \\
		& \leq    \veps_\reg^2, \label{eq:los}
	\end{align}
  where the last inequality follows by the fact that $\delta \in(0,1)$ and the definition of $\veps_\reg^2$ in \cref{alg:learnlevel3}.
	By the definition of $\cVhat_h$ in \eqref{eq:confidence3}, \eqref{eq:los} implies that $V^{\pib}_h \in \cVhat_h$, which completes the induction.
\end{proof}

\subsection{Proof of \creftitle{thm:lbc3} (Main Guarantee of \learnlevel)}
\label{proof:lbc3}
\begin{proof}[\pfref{thm:lbc3}] We condition on the event $\wtilde{\cE}\coloneqq \cE \cap \cE'''\cap \cE'_1\cap \dots \cap \cE_H'$, where $\cE$, $\cE'''$, and $(\cE_h')$ are the events in \cref{lem:testfailures}, \cref{lem:confidencesets}, and \cref{lem:confidence}, respectively. Note that by the union bound, we have $\P[\wtilde{\cE}] \geq  1 - 5 \delta$. By \cref{lem:confidence}, we have that 
	\begin{align}
		\label{eq:test_passed33}
		\forall h \in[H],\quad 	\bbP^{\pibell}\brk*{\sup_{f\in \cVhat _h} \max_{a\in \cA}\left|(\cP_{h-1}[\Vhat_{h}]-\cP_{h-1}[f_{h}])(\x_{h-1},a) \right|>{3\veps}}
		\leq\testbound,
	\end{align}
	where $M = \Mnum$ and $N_\test = \Ntestnum$. On the other hand, by \cref{lem:confidencesets}, we have 
	\begin{align}
		\forall h \in[H], \quad 	V^{\pib}_h\in \cVhat_h.
		\label{eq:learnlevel_inductive33}
	\end{align}
	Thus, substituting $V^{\pib}_h$ for $f_h$ in \eqref{eq:test_passed33} we get that for all $h\in[H+1]$. 
	\begin{align}
		\bbP^{\pibell}\brk*{\max_{a\in \cA}\left|(\cP_{h-1}[\Vhat_{h}]- \cP_{h-1}[V^{\pib}_h])(\x_{h-1},a)\right|   > {3\veps }
		}  \leq  \testbound.
	\end{align}
	This together with \cref{lem:tvdistance}, instantiated with $\mu[\cdot] = \bbP^{\pibell}[\cdot]$; $\tau = h-1$; $V_{\tau+1} = V^{\pib}_{h}$; and $\delta = \delta'$ (with $\delta'$ as in \cref{alg:learnlevel3}), translates to:
	\begin{align}
		\forall h\in[H], \quad 	\bbE^{\pibell}\brk*{ \tv{\pibell_h(\x_h)}{\pib_h(\x_h)}
		}&\leq \testbound + \delta',  \nn \\
		&= \testbound + \frac{\delta}{4 M^7 N_\test^2 H^8 |\cV|}, \nn \\
		&\leq \frac{\veps}{4H^3 \Cpush},
	\end{align}
	where the last step follows from the fact that $N_\test = \Ntestnum$ (with $M$ as in \cref{line:paramsVpiM}).

	\paragraph{Bounding the sample complexity}
	We now bound the number of episodes used by \cref{alg:learnlevel3} under the event $\wtilde\cE$. First, we fix $h\in[H]$, and focus on the number of episodes used within a to call $\learnlevel_h$, excluding any episodes used by any subsequent calls to $\learnlevel_\tau$ for $\tau >h$. We start by counting the number of episodes used to test the fit of the estimated value functions $\Vhat_{h+1:H}$. Starting from \cref{line:begin3}, there are for-loops over $(x_{h-1},a_{h-1})\in \cC_h$, $\ell=H, \dots, h+1$, and $n \in [N_\test]$ to collect partial episodes using the learned policy $\pihat$ in \cref{alg:learnlevel3}, where $N_\test = \Ntestnum$ and $M = \Mnum$. Note that executing $\pihat$ requires the local simulator and uses $N_\simu= 2 \log(4 M^7 N_\test^2 H^2 |\cV|/\delta)/\veps^2$ local simulator queries to output an action at each layer (since \cref{alg:learnlevel3} calls \cref{alg:Phat} with confidence level $\delta'=\deltaprime$). Also, note that whenever a test fails in \cref{line:test3} and the recursive $\learnlevel$ calls return, the for-loop in \cref{line:begin3} resumes. We also know (by \cref{lem:testfailures}) that the number of times the test fails in \cref{line:test3} is at most $M$. Thus, the number of times the for-loop in \cref{line:begin3} resumes is bounded by $H M$; here, $H$ accounts for possible test failures across all layers $\tau \in[h+1\ldotst H]$. Thus, the total sample complexity required to generate episodes between lines \cref{line:begin3} and \cref{line:draw} is bounded by 
	\begin{align}
		\text{\# episodes for roll-outs} \leq	\underbrace{MH}_{\text{\# of times \cref{line:begin3} resumes}}\cdot  \underbrace{MH^2 N_\test N_\simu}_{\text{Sample complexity  in case of no test failures}}. \label{eq:roll}
	\end{align}
	Note that the test in \cref{line:test3} also uses episodes because it calls the operator $\Phat$ for every $a\in \cA$. Thus, the number of episodes used for the test in \cref{line:test3} is bounded by 
	\begin{align}
		\text{\# episodes for the tests} \leq	\underbrace{MH}_{\text{\# of times \cref{line:begin3} resumes}}\cdot  \underbrace{M H A N_\test N_\simu}_{\text{Sample complexity for  \cref{line:test3}}}. \label{eq:test}
	\end{align}
	We now count the number of episodes used to re-fit the value function in \cref{line:forloop} and onwards. Note that starting from \cref{line:forloop}, there are for-loops over $(x_{h-1},a_{h-1})\in \cC_h$ and $i\in [N_\reg]$ to generate $A \cdot \Nest(|\cC_h|)\leq A \cdot \Nest(M)$ partial episodes using $\pihat$, where $\Nest(k)=2N_\reg^2 \log(8 A N_\reg H k^3/\delta)$ is defined as in \cref{alg:learnlevel3}. Since $\pihat$ uses the local simulator and requires $\Nest$ samples (see \cref{alg:Phat}) to output an action at each layer, the number of episodes used to refit the value function is bounded by 
	\begin{align}
		\text{\# episodes for $V$-refitting} \leq M N_\reg A \Nest(M) H N_\simu. \label{eq:refit}
	\end{align}
	Therefore, by \eqref{eq:roll}, \eqref{eq:test}, and \eqref{eq:refit}, the number of episodes used within a single call to $\learnlevel_h$ (not accounting for episodes used by recursive calls to $\learnlevel_\tau$, for $\tau >h$) is bounded by
	\begin{align}
		\text{\# episodes used locally within $\learnlevel_h$} & \leq  M^2 H (H + A) N_\test N_\simu  +M N_\reg A \Nest(M) H N_\simu. \label{eq:localtotal}
	\end{align}
	Finally, by \cref{lem:testfailures}, $\learnlevel_h$ may be called at most $M$ times throughout the execution of $\learnlevel_0$. Using this together with \eqref{eq:localtotal} and accounting for the number of episodes from all layers $h\in[H]$, we get that the total number of episodes is bounded by 
	\begin{align}
		M^3 H^2 (H + A) N_\test N_\simu   +M^2 H^2 N_\reg A \Nest(M) N_\simu.
	\end{align}  
	Substituting the expressions of $M$, $N_\test $, $\Nest$, $N_\simu$, and $N_\reg$ from \cref{alg:learnlevel3} and \cref{alg:Phat}, we obtain the desired number of episodes, which concludes the proof.
\end{proof}

\subsection{Proof of \creftitle{thm:vpiforward} (Guarantee of \rvflF)}
\label{proof:vpiforward}

\begin{proof}[\pfref{thm:vpiforward}]
Let $\Vhat_{1:H}$ be the value function estimates produced by $\learnlevel_0$ within \cref{alg:forward_vpi}, and let $\pihatb^{\learnlevel}_h(\cdot) \in \argmax_{a\in
\cA}
\Phat_{h,\veps_\learnlevel,\delta'}[\Vhat_{h+1}](\cdot,a)$, for all $h\in [H]$ with $\Vhat_{H+1}\equiv 0$ with $\veps_\learnlevel$ and $\delta'$ as in \cref{alg:forward_vpi}. Further, let and let $\pib_{1:H}\in \Pims$ be the stochastic policy defined recursively via
\begin{align}
	\forall x \in \cX,\ \ 	\pibb_\tau(x) \in \argmax_{a\in \cA} \left\{ \begin{array}{ll}  \bQhat_\tau(x,a),  & \text{if }  \| \bQhat_\tau(x,\cdot) -\cP_\tau [ V^{\pib}_{\tau+1}](x,\cdot)\|_\infty \leq 4\veps_\learnlevel, \\   
		\cP_\tau[V^{\pib}_{\tau+1}](x,a), & \text{otherwise}, \end{array}\right.  \label{eq:pitilde}
\end{align} 
for $\tau = H,\dots,1$, where $\bQhat_\tau(x,a) \coloneqq  \Phat_{\tau,\veps_\learnlevel, \delta'} [\Vhat_{\tau+1}](x,a)$. By \cref{thm:lbc3}, there is an event $\wtilde \cE$ of probability at least $1-\delta/2$ under which:
\begin{gather}
	\pib \in \Pi_{{4\veps_\learnlevel}}, \label{eq:inbenchee} 
	\shortintertext{and}
	\forall h \in[H], \quad 
	\bbE^{\pihat^{\learnlevel}}\brk*{ \tv{\pihat^{\learnlevel}_{h}(\x_h)}{ \pib_h(\x_{h})}
	} \leq   \frac{\veps_\learnlevel}{4 H^3 \Cpush} \leq  \frac{\veps}{4H^2},   \label{eq:secondevent}
\end{gather}
where the last inequality follows by the choice of $\veps_{\learnlevel}$ in \cref{alg:forward_vpi}.

For the rest of the proof, we condition on $\wtilde{\cE}$. By \eqref{eq:inbenchee}, \eqref{eq:secondevent}, and \cref{prop:forward} instantiated with $\veps_{\texttt{mis}}=0$ (due to all $\pi$-realizability), we have that there is an event $\wtilde\cE'$ of probability at least $1-\delta/2$ under which the policy $\pihat_{1:H}$ produced by \forward{} ensures that
	\begin{align}
		J(\pihat^{\learnlevel}_{1:H}) -  J(\pihat_{1:H}) \leq  \frac{\veps}{2}. \label{eq:forwardpart}
	\end{align}	
	Now, by \cref{lem:perform} (the performance difference lemma), we have for $\pib$ as in \eqref{eq:pitilde}:
	\begin{align}
		J(\pib)- J(\pihat^{\learnlevel}_{1:H}) & = \sum_{h=1}^H \E^{\pihat^{\learnlevel}}[Q^{\pib}_h(\x_h, \pibb_h(\x_h)) -  Q_h^{\pib}(\x_h, \bm{\pihat}^{\learnlevel}_h(\x_h))],\nn \\&  \leq H\sum_{h=1}^H \E^{\pihat^{\learnlevel}}[ \tv{\pihat^{\learnlevel}_{h}(\x_h)}{ \pib_h(\x_{h})}], 
	\end{align}
	and so by \eqref{eq:secondevent}, we have
	\begin{align}
		J(\pib)- J(\pihat^{\learnlevel}_{1:H})  & \leq \veps/4. \label{eq:perform}
	\end{align}
	Finally, since $\pib \in \Pi_{4\veps_\learnlevel}$ (see \eqref{eq:inbenchee}), we have by \cref{lem:policysubopt3},  
	\begin{align}
		J(\pistar) -J(\pib) 
		& \leq 12 H \veps_\learnlevel \leq \veps/4,
	\end{align}
	where the last inequality follows by the choice $\veps_\learnlevel$ in \cref{alg:forward_vpi}. Combining this with \eqref{eq:forwardpart} and \eqref{eq:perform}, we conclude that under $\wtilde\cE\cap \wtilde\cE'$:
	\begin{align}
		J(\pistar)  - J(\pihat_{1:H}) \leq \veps. 
	\end{align}
	By the union bound, we have $\P[\wtilde\cE\cap \wtilde\cE']\geq 1 - \delta$, and so the desired suboptimality guarantee in \eqref{eq:final} holds with probability at least $1-\delta$. 
	
	\paragraph{Bounding the sample complexity} The sample complexity is dominated by the call to $\learnlevel_0$ within \rvflF{} (\cref{alg:forward_vpi}). Since \rvflF{} calls $\learnlevel_0$ with $\veps = \veps_\learnlevel = \veps H^{-1}/48$, we conclude from \cref{thm:lbc3} that the total sample complexity is bounded by 
	\begin{align}
		\wtilde{O}\left(\Cpush^8 H^{23}A \cdot \veps^{-13}\right).
	\end{align}
\end{proof}

        	\section{Guarantee under $V^\star$-Realizability
          (Proof of \creftitle{thm:vpiforward_main}, \setupi)}
	\label{sec:vstar}

\ifdefined\vepsllnum
\else
\newcommand{\vepsllnum}{\veps H^{-1}/48}
\fi

\ifdefined\vepsll
\else
\newcommand{\vepsll}{\veps_\learnlevel}
\fi

\ifdefined\Mnum
\renewcommand{\Mnum}{\ceil{8  \veps^{-1} \Cpush H}}
\else
\newcommand{\Mnum}{\ceil{8  \veps^{-1} \Cpush H}}
\fi 

\ifdefined\Ntestnum
\renewcommand{\Ntestnum}{2^8  M^2 H \veps^{-1}\log(8 M^6 H^8 \veps^{-2} \delta^{-1})}
\else
\newcommand{\Ntestnum}{2^8  M^2 H \veps^{-1}\log(8 M^6 H^8 \veps^{-2} \delta^{-1})}
\fi

\ifdefined\Nregnum
\renewcommand{\Nregnum}{2^8 M^2  \veps^{-1}\log(8|\cV|^2 HM^2 \delta^{-1})}
\else
\newcommand{\Nregnum}{2^8 M^2 \veps^{-1}\log(8|\cV|^2 HM^2 \delta^{-1})}
\fi

\ifdefined\bbeta
\renewcommand{\bbeta}{\frac{9 M H^2\log(8M^2H|\cV|^2/\delta)}{N_\reg}  +  \frac{34 MH^3\log(8M^6 N^2_\test  H^8/\delta)}{N_\test} }
\else
\newcommand{\bbeta}{\frac{9 M H^2\log(8M^2H|\cV|^2/\delta)}{N_\reg}  +  \frac{34 MH^3\log(8M^6 N^2_\test  H^8/\delta)}{N_\test} }
\fi

\ifdefined\Nsimu
\renewcommand{\Nsimu}{2N_\reg^2 \log(8 A N_\reg H M^3/\delta)}
\else
\newcommand{\Nsimu}{2N_\reg^2 \log(8 A N_\reg H M^3/\delta)}
\fi

\ifdefined\deltaprime
\renewcommand{\deltaprime}{\delta/(8M^7N_\test^2 H^8|\cV|)}
\else
\newcommand{\deltaprime}{\delta/(8M^7N_\test^2 H^8|\cV|)}
\fi

\ifdefined\testbound
\renewcommand{\testbound}{\frac{4 \log(8M^6 N^2_\test  H^8/\delta)}{N_\test}}
\else
\newcommand{\testbound}{\frac{4 \log(8M^6 N^2_\test  H^8/\delta)}{N_\test}}
\fi

\ifdefined\deltap
\renewcommand{\deltap}{\frac{\delta}{|\cV|M H}}
\else
\newcommand{\deltap}{\frac{\delta}{|\cV|M H}}
\fi 

\ifdefined\pibell
\renewcommand{\pibell}{\pihat}
\else
\newcommand{\pibell}{\pihat}
\fi

\ifdefined\numepisodes
\renewcommand{\numepisodes}{\mathrm{poly}(S,A,...)}
\else
\newcommand{\numepisodes}{\mathrm{poly}(S,A,...)}
\fi

In this section, we prove \creftitle{thm:vpiforward_main} under
\setupi{} ($\Vstar/\pistar$-realizability (\cref{ass:real,ass:pireal})
and $\Delta$-gap (\cref{ass:gap})). We prove this result as a
consequence of the more general results (\cref{thm:vpiforward}) in
\cref{sec:vpi} by appealing to the relaxed $V^{\pi}$-realizability
condition in \cref{ass:relaxreal}.

\subsection{Analysis: Proof of \creftitle{thm:vpiforward_main} (\setupi)}
We begin by showing that \cref{ass:real} and \cref{ass:gap}
together imply that \cref{ass:relaxreal} holds for any $\veps_\rel \leq \Delta/2$; we prove this by showing that the benchmark policy class $\Pi_{\veps'}$ (\cref{sec:benchmark}) reduces to $\{\pi^\star\}$ when $\veps'  \leq \Delta/2$.
\begin{lemma}
	\label{lem:reduce}
	Assume that $\cV$ satisfies \cref{ass:real} ($V^{\star}$-realizability), and that \cref{ass:gap} (gap) holds with $\Delta>0$. Then, for all $\veps'\leq \Delta/2$, we have $\Pi_{\veps'} =\{\pi^\star\}$ and $\cV$ satisfies \cref{ass:relaxreal} with $\veps_\rel = \veps'$.
\end{lemma}
Informally \cref{lem:reduce}, whose proof is in \cref{proof:reduce},
states that under \cref{ass:real}, \cref{ass:gap} with $\Delta>0$, and
\cref{ass:pushforward} (pushforward coverability) with $\Cpush>0$, we
are essentially in the setting of \cref{thm:lbc3} (guarantee of
\learnlevel{} under relaxed $V^\pi$-realizability), as long as we
choose $\veps_\rel \leq \Delta/2$.  With this, we now state and prove a central guarantee for \learnlevel{} under $V^\star$-realizability with a gap.
\begin{lemma}[Intermediate guarantee for \learnlevel{} under \setupi]
	\label{thm:lbc}
	Let $\delta \in(0,1)$ be given, and suppose that:
	\begin{itemize} 
		\item \cref{ass:pushforward} (pushforward
                  coverability) holds with parameter $\Cpush>0$;
		\item \cref{ass:gap} (gap) holds with parameter $\Delta>0$;
		\item The function class $\cV$ satisfies \cref{ass:real} ($V^\star$-realizability).
	\end{itemize}
	Then, for any $f \in \cV$ and $\veps \in(0,\Delta/8)$, with
        probability at least $1-\delta$,
        $\learnlevel_{0}(f,\cV,\emptyset,\emptyset;\cV,\veps,\delta)$
        (\cref{alg:learnlevel3}) terminates and returns value
        functions $\Vhat_{1:H}$ that satisfy 
	\begin{align}
	\forall h \in [H],\quad 	\bbP^{\pibell}\brk*{\bm{\pibell}_{h}(\x_h)\neq  \pistar_h(\x_{h}) } \leq \frac{{\veps}}{4\Cpush H^3}  , \label{eq:returned}
	\end{align}
	where $\mb{\pibell}_h(x) \in \argmax_{a\in \cA}
        \Phat_{\tau,\veps,\delta'}[\Vhat_{h+1}](x,a)$, for all
        $h\in\brk{H}$, with $\delta'$ is defined as in \cref{alg:learnlevel3}.
\end{lemma}
\begin{proof}[\pfref{thm:lbc}]
  From \cref{lem:reduce}, we have that $\Pi_{4\veps}=\crl*{\pistar}$,
  and so \cref{thm:lbc3} implies that with probability at least $1-\delta$,
	\begin{align}
 \forall h \in[H], \quad \frac{{\veps}}{4\Cpush H^3} \geq 	\bbE^{\pibell}\brk*{\tv {\pibell_{h}(\x_h)}{\pistar_h(\x_{h}}} = 	\bbP^{\pibell}\brk*{\bm\pibell_{h}(\x_h)\neq  \pistar_h(\x_{h}) },
		\end{align}
	where the equality follows by the fact that $\pistar$ is deterministic.
\end{proof}

From here, \cref{thm:vpiforward_main} follows swiftly as a consequence.

\begin{proof}[Proof of \cref{thm:vpiforward_main} (\setupi)]
	Let $\Vhat_{1:H}$ be the value function estimates produced by $\learnlevel_0$ within \cref{alg:forward_vpi}, and let $\pihatb^{\learnlevel}_h(\cdot) \in \argmax_{a\in
	\cA}
	\Phat_{h,\veps_\learnlevel,\delta'}[\Vhat_{h+1}](\cdot,a)$, for all $h\in [H]$ with $\Vhat_{H+1}\equiv 0$ with $\veps_\learnlevel$ and $\delta'$ as in \cref{alg:forward_vpi}. By \cref{thm:lbc}, there is an event $\wtilde \cE$ of probability at least $1-\delta/2$ under which:
	\begin{gather}
		\bbP^{\pibell}\brk*{\bm{\pibell}_{h}(\x_h)\neq  \pistar_h(\x_{h}) }  \leq   \frac{\veps_\learnlevel}{4 H^3 \Cpush} \leq  \frac{\veps}{4H^2},   \label{eq:secondevent}
	\end{gather}
	where the last inequality follows by the choice of $\veps_{\learnlevel}$ in \cref{alg:forward_vpi}.
	
	For the rest of the proof, we condition on $\wtilde{\cE}$. By \eqref{eq:secondevent} and \cref{ass:pireal} ($\pistar$-realizability), the
	policy $\pihat^{\learnlevel}_{1:H}$ returned by $\mainalg_0$ satisfies
	the condition in \cref{prop:forward} with
	$\veps_{\texttt{mis}} = \veps/(4 \Cpush H^3)$. Thus, by
	\cref{prop:forward}, there is an event $\wtilde\cE'$ of probability at
	least $1-\delta/2$ under which the policies $\pihat_{1:H}$ produced by
	\rvflF{} satisfy
	\begin{align}
		J(\pihat^{\learnlevel}_{1:H}) -  J(\pihat_{1:H}) \leq  \frac{\veps}{H} +\frac{\veps}{2}\leq \frac{3\veps}{2}. \label{eq:forwardpart_star}
	\end{align}
	We now condition on $\wtilde{\cE} \cap \wtilde{\cE}'$. By \cref{lem:perform} (performance difference lemma), we have
\begin{align}
	J(\pistar)- J(\pihat^{\learnlevel}_{1:H}) & = \sum_{h=1}^H \E^{\pihat^{\learnlevel}}[Q^{\pistar}_h(\x_h, \pistar_h(\x_h)) -  Q_h^{\pistar}(\x_h, \pihatb^{\learnlevel}_h(\x_h))],\nn \\&  \leq H\sum_{h=1}^H \bbP^{\pibell}\brk*{\bm{\pibell}_{h}(\x_h)\neq  \pistar_h(\x_{h}) }, \nn \\
	& \leq \veps/(4H),
\end{align}
where the last inequality follows by \eqref{eq:secondevent}.

Finally, by the union bound, we have $\P[\wtilde\cE\cap \wtilde\cE']\geq 1 - \delta$, and so the desired suboptimality guarantee in \eqref{eq:final} holds with probability at least $1-\delta$.

\paragraph{Bounding the sample complexity} The sample complexity
is dominated by the call to $\learnlevel_0$ within \rvflF{}
(\cref{alg:forward_vpi}). Since \rvflF{} calls $\learnlevel_0$ with
$\veps = \veps_\learnlevel = \veps H^{-1}/48$, we conclude from \cref{thm:lbc3} that the total number of episodes is bounded by 
\begin{align}
\wtilde{O}\left(\Cpush^8 H^{23}A \cdot\veps^{-13}\right).
\end{align}
\end{proof}

\subsection{Proof of \creftitle{lem:reduce} (Relaxed
  $V^\pi$-Realizability under Gap)}
\label{proof:reduce}

\begin{proof}[\pfref{lem:reduce}]
	Fix $\veps' \in (0,1)$ and $\pi \in \Pi_{\veps'}$. Let $(\bQtilde_{h}(x,a))_{(h,x,a) \in [H]\times \cX\times \cA}$ be independent random variables such that 
	\begin{align}
		\forall h \in [H],  \quad \bm\pi_h(\cdot) \in \argmax_{a'\in \cA} \bQtilde_h(\cdot,a') \quad \text{and} \quad \| \bQtilde_h - Q^{\pi}_{h}\|_\infty \leq \veps', \text{ almost surely.}
	\end{align} 
	Such a collection of random state-action value functions
        $(\bQtilde_{h}(x,a))_{(h,x,a) \in [H]\times \cX\times \cA}$ is
        guaranteed to exist for $pi$ by the definition of $\Pi_{\veps'}$. We will show via backward induction over $\ell= H+1,\dots, 1$ that 
	\begin{align}
		\forall x \in \cX_\ell, \quad 	\bm\pi_\ell(x) = \pi^\star_\ell(x) \label{eq:toinducton}
	\end{align}
	almost surely,
	with the convention that $\pi_{H+1} \equiv \pi^\star_{H+1}\equiv \pi_\unif$. This convention makes the base case, $\ell=H+1$, hold trivially. 
	
	Now, we consider the general case. Fix $h \in[H]$ and suppose that \eqref{eq:toinducton} holds for all $\ell \in [h+1\ldotst H+1]$. We will show that it holds for $\ell=h$. 
	
	Thanks to the induction hypothesis, we have for all $x\in \cX_{h+1}$ and $a\in \cA$:
	\begin{gather}
		V^{\pi}_{h+1}(x)   = V^\star_{h+1}(x), 
		\intertext{and so}
		Q^{\pi}_h \equiv \cP_h[V^\pi_{h+1}] \equiv \cP_h[V^\star_{h+1}] = V^\star_h. \label{eq:key}
	\end{gather}
	Fix $x\in \cX$. We will show that $\pi_h(x)=\pi^\star_h(x)$ almost surely. Note that thanks to \eqref{eq:key}, the fact that $\|\bQtilde_h - \cT_{h}[Q^\pi_{h+1}]\|_\infty\leq \veps'$ almost surely, implies that 
	\begin{align}
		\|\bQtilde_h - Q^\star_{h}\|_\infty\leq \veps',
	\end{align}	
	almost surely. Using this, we have, almost surely
	\begin{align}
		Q^\star_h(x,\bm\pi_h(x))& \geq \bQtilde_h(x,\bm\pi_h(x))- \veps',\nn \\
		& \geq\bQtilde_h(x,\pistar_h(x)) - \veps', \nn \\
		& \geq  Q^\star_h(x,\pi_h^\star(x)) - 2\veps' = Q^\star_h(x,\pi_h^\star(x)) - \Delta . \label{eq:contra}
	\end{align}
	On the other hand, if $\bm\pi_h(x)\neq \pi_h^\star(x)$, then    
	\begin{align}
		Q^\star_h(x,\bm\pi_h(x)) < Q^\star_h(x,\pi_h^\star(x)) - \Delta,
	\end{align}
	which would contradict \eqref{eq:contra}. Thus, $\bm\pi_h(x)= \pi^\star_h(x)$, which concludes the induction and shows that $\pi \equiv \pi^\star$. We conclude that $\Pi_{\veps'}=\{\pi^\star\}$.
\end{proof}

	\section{Guarantee for Weakly Correlated ExBMDPs (Proof of \creftitle{thm:exbmdpforward_main})}
	\label{sec:exbmdp_app}
\ifdefined\Mnum
\renewcommand{\Mnum}{\ceil{8  \veps^{-2} \Ccor S A H}}
\else
\newcommand{\Mnum}{\ceil{8  \veps^{-2} \Ccor S A H}}
\fi 

\ifdefined\Ntestnum
\renewcommand{\Ntestnum}{2^8  M^2 H \veps^{-2}\log(8 M^6 H^8 \veps^{-2} \delta^{-1})}
\else
\newcommand{\Ntestnum}{2^8  M^2 H \veps^{-2}\log(8 M^6 H^8 \veps^{-2} \delta^{-1})}
\fi

\ifdefined\Nregnum
\renewcommand{\Nregnum}{2^8 M^2 \veps^{-2}\log(8|\cV| HM^2 \delta^{-1})}
\else
\newcommand{\Nregnum}{2^8 M^2 \veps^{-2}\log(8|\cV| HM^2 \delta^{-1})}
\fi

\ifdefined\bbeta
\renewcommand{\bbeta}{\frac{9 MH^2\log(8M^2H|\cV|/\delta)}{N_\reg}  +  \frac{34 MH^3\log(8M^6 N^2_\test  H^8/\delta)}{N_\test} }
\else
\newcommand{\bbeta}{\frac{9 MH^2\log(8M^2H|\cV|/\delta)}{N_\reg}  +  \frac{34 MH^3\log(8M^6 N^2_\test  H^8/\delta)}{N_\test} }
\fi

\ifdefined\Nsimu
\renewcommand{\Nsimu}{2N_\reg^2 \log(8  N_\reg H M^3/\delta)}
\else
\newcommand{\Nsimu}{2N_\reg^2 \log(8  N_\reg H M^3/\delta)}
\fi

\ifdefined\deltaprime
\renewcommand{\deltaprime}{\delta/(8M^7N_\test^2 H^8|\cV|)}
\else
\newcommand{\deltaprime}{\delta/(8M^7N_\test^2 H^8|\cV|)}
\fi

\ifdefined\bbetap
\renewcommand{\bbetap}{\frac{2H\log(4\abs{\cV}  H/\delta)}{N_\reg} + 8H^3 A |\cC_h| \cdot \frac{\log(4H|\cC_h|/\delta)}{N_\test}}
\else
\newcommand{\bbetap}{\frac{2H\log(4\abs{\cV}  H/\delta)}{N_\reg} + 8H^3 A |\cC_h| \cdot \frac{\log(4H|\cC_h|/\delta)}{N_\test}}
\fi

\ifdefined\vepsllnum
\else
\newcommand{\vepsllnum}{\veps H^{-1}/48}
\fi

In this section, we prove \cref{thm:exbmdpforward_main}, the main guarantee for \mainalge. First, in \cref{sec:analysis_exbmdp} we state a number of supporting technical lemmas and use them to prove \cref{thm:exbmdpforward_main}. The remainder of the section (\cref{proof:endobench} through \cref{proof:lbc2}) contains the proofs for the intermediate results.

\subsection{Analysis: Proof of \creftitle{thm:exbmdpforward_main}}
\label{sec:analysis_exbmdp}

Recall that the the $V^{\pi}$-realizability assumption required by \mainalg for \cref{thm:vpiforward_main} is not satisfied in ExBMDPs, as the value functions for policies that act on the exogenous noise variables cannot be realized as a function of the true decoder $\phistar$.
In \mainalge, we address this issue by applying the randomized rounding technique in \cref{line:draw2} to the learned value functions. The crux of the analysis will be to show that for an appropriate choice of the rounding parameters $\zeta_{1:H}$, the policies produced by \mainalge are endogenous in the sense that we can write $\pi(x)= \pi(\phistar(x))$ for all $x\in \cX$. This will allow us to leverage the decoder realizability (\cref{ass:phistar}), which implies that the function class $\cV = \cV_{1:H}$ given by
\begin{align}
	\cV_h\coloneqq
	\{x\mapsto f(\phi(x)) : f\in [0,H]^S, \phi\in \Phi\}, \label{eq:funclass}
	\end{align}
        satisfies $V^\pi$-realizability for all endogenous policies $\pi$.

In what follows, we first motivate the randomized rounding approach in \mainalge in detail and prove that it succeeds, then use this result to proceed with an analysis similar to that of \cref{thm:vpiforward_main} (\setupii), re-using many of the technical tools developed for \cref{thm:vpiforward_main}.
        
\subsubsection{Randomized Rounding for Endogeneity} 
 Naively, to ensure that the policies we execute are endogenous, it would seem that we require knowledge of the true decoder $\phistar$. Alas, knowing $\phistar$ trivializes the ExBMDP problem by reducing it to the tabular setting. To avoid requiring knowledge of $\phistar$, we apply a \emph{randomized rounding} to the policies learned by \mainalge{} to ensure their endogeneity.

Let $\veps>0$ be fixed going forward. Recall that compared to \mainalg, \mainalge{} (\cref{alg:learnlevel2}) takes an additional input $\zeta_{1:H}\subset (0,1/2)$ and executes the following coarsened policies:
\begin{align}
	\pihatb_h(\cdot) \in \argmax_{a\in\cA}
	\ceil{ \Phat_{h,\veps, \delta} [\Vhat_{h+1}](\cdot,a)/\veps+\zeta_h}.\label{eq:coarse}
\end{align}
The \emph{rounding parameters} $\zeta_{1:H}$, which can be thought of as an offset, are chosen randomly; this will be elucidated in the sequel.

Following a similar analysis to \cref{sec:vpi} (\setupii), we can associate a near-optimal benchmark policy $\pib\in\Pi_{2\veps}$ with $\pihat$ in order to emulate certain properties of the $\Delta$-gap assumption. In particular, generalizing the construction in \cref{eq:pib}, we define a near-optimal benchmark policy $\pib$ recursively via:
\colt{
	\begin{align}
	&	\forall x \in\cX, \nn \\ &	\pibb_\tau(x; \zeta_{1:H}, \veps, \delta) \in \argmax_{a\in \cA} \left\{ \begin{array}{ll}  \ceil{\bQhat_\tau(x,a)/\veps + \zeta_\tau},  & \text{if }  \| \bQhat_{\tau}(x,\cdot)- \cP_\tau[V^{\pib}_{\tau+1}](x,\cdot)\|_\infty \leq 4\veps^2, \\   
			\ceil{\cP_\tau[V^{\pib}_{\tau+1}](x,a)/\veps + \zeta_\tau}, & \text{otherwise}, \end{array}\right.
		\label{eq:pibar}
	\end{align} 
}
\arxiv{
\begin{align}
	\forall x \in\cX, \ 	\pibb_\tau(x; \zeta_{1:H}, \veps, \delta) \in \argmax_{a\in \cA} \left\{ \begin{array}{ll}  \ceil{\bQhat_\tau(x,a)/\veps + \zeta_\tau},  & \text{if }  \| \bQhat_{\tau}(x,\cdot)- \cP_\tau[V^{\pib}_{\tau+1}](x,\cdot)\|_\infty \leq 4\veps^2, \\   
		\ceil{\cP_\tau[V^{\pib}_{\tau+1}](x,a)/\veps + \zeta_\tau}, & \text{otherwise}, \end{array}\right.
	\label{eq:pibar}
\end{align} 
}
for $\tau = H,\dots,1$, where $\bQhat_\tau(\cdot,a) \coloneqq \Phat_{\tau,\veps, \delta} [\Vhat_{\tau+1}](\cdot,a)$.

Naively, to use the benchmark policy $\pib$ within the analysis based on relaxed $V^\pi$-realizability (\cref{ass:relaxreal}) in \cref{sec:vpi} , we would require the function class $\cV$ to realize $(V^{\pib}_h)$. However, as argued earlier, this is not feasible unless $\pib$ is an endogenous policy. Fortunately, it turns out that if $\zeta_{1:H}$ (the additional input to \mainalge) are drawn randomly from uniform distribution over $[0,1/2]$, then with constant probability, $\pib$ is indeed endogenous. What's more, under such an event, and for all possible choices of $(\Vhat_h)$ in \eqref{eq:pibar} uniformly, $\pib$ ``snaps'' onto the stochastic endogenous policy $\pibar(\cdot;\zeta_{1:H}, \veps)$ defined recursively as follows: 
\begin{align}
	\pibar_h(\cdot;\zeta_{1:H},\veps) \in \argmax_{a\in \cA} 
	\ceil{\cP_h[V^{\pibar}_{h+1}](x,a)/\veps + \zeta_h}, \label{eq:policies}
\end{align}
for $h = H,\dots, 1$. Informally, this happens because, as long as $\zeta_{1:H}\subset(0,1/2)$ avoid certain pathological locations in $\prn{0,1/2}$, the coarsened state-action value functions $\veps \cdot \ceil{\cP_\tau[V^{\pibar}_{\tau+1}](x,a)/\veps + \zeta_h}$ defining $\pibar$ exhibit a ``gap'' of order $\Theta(\veps^2)$ separating optimal actions from the rest. This ``snapping'' behavior is analogous to what happens in \setupi{} with $V^\star$-realizability and $\Delta$-gap, where $\Pi_{\veps}$ reduces to $\{\pistar\}$ for all $\veps<\Delta/2$ (see \cref{lem:reduce}). We formalize these claims in the next two lemmas. We start by showing that $\pibar$ is endogenous and that $\pibar \in \Pi_{2\veps}$. The proof is in \cref{proof:endobench}.
\begin{lemma}[Endogenous Benchmark policies]
	\label{lem:endobench}
	For any $\delta\in(0,1)$, $\veps \in(0,1/2)$, and $\zeta_{1:H}\subset (0,1/2)$, the stochastic policy $\pibar(\cdot; \zeta_{1:H},\veps)$ defined in \cref{eq:policies} is endogenous, and we have $\pibar(\cdot; \zeta_{1:H},\veps) \in \Pi_{2\veps}$. 
\end{lemma}
Next, we show that $\pib$ ``snaps'' onto $\pibar$ for the certain choices of $\zeta_{1:H}$. The proof is in \cref{proof:bench}.
\begin{lemma}[Snapping probability]
	\label{lem:bench}
	Let $\delta \in (0,1)$, $\veps \in(0,1/2)$ be given, and $\P^\zeta$ denote the probability law of $\bzeta_1, \dots, \bzeta_H\sim \unif([0,1/2])$. Then, there is an event $\cE_{\texttt{rand}}$ of probability at least $1 - 24 S A H \veps$ under $\bzeta_{1:H}\sim{}\P^\zeta$ such that for all $\wtilde{V} \in(\cX\times\brk{H}\to\brk{0,H})$ simultaneously,
	\begin{align}
		\forall h \in[H], \quad \pibb_h(\cdot;\wtilde{V},\bzeta_{1:H},\veps,\delta) = 	\pibar_h(\cdot; \bzeta_{1:H},\veps),
	\end{align}
	where $\pibb_h(\cdot;\wtilde{V},\bzeta_{1:H},\veps,\delta)$ is defined as in \eqref{eq:pibar} with $\Vhat = \wtilde{V}$, and $\pibar$ is defined as in \eqref{eq:policies}.
\end{lemma}
The lemma together, with \cref{lem:endobench}, implies that with constant probability under $\P^\zeta$, the benchmark policies $(\pib_h)$ used in the analysis of \mainalge{} are endogenous and satisfy $\pib \in \Pi_{2\veps}$.

\subsubsection{Pushforward Coverability}
In order to proceed with the analysis strategy in \cref{sec:vpi}, we need to verify that pushforward coverability is satisfied for ExBMDPs under the weak correlation assumption. We do so in the next lemma; see \cref{proof:pushforward} for a proof.
\begin{lemma}[Pushforward coverability]
	\label{lem:pushforward}
	A weakly correlated ExBMDP with constant $\Ccor$ (see \cref{ass:great}) satisfies $\Cpush$-pushforward coverability (\cref{ass:pushforward}) with $\Cpush = \Ccor \cdot S A$, where $S\in \mathbb{N}$ is the number of endogenous states. 
\end{lemma}
Equipped with the preceding lemmas, we proceed with an analysis similar to the approach for \cref{thm:vpiforward_main} (\setupii) in \cref{sec:vpi}. In what follows, we state a number of technical lemmas that apply the relevant results from \cref{sec:vpi} to the ExBMDP setting we consider here.

\subsubsection{Bounding the Number of Test Failures}
\label{sec:prelimsexbmdp}
Since the size of the core sets $\cC_{1:H}$ in $\learnlevel^\exo$ is directly proportional to the number of test failures, the next result, which bounds $|\cC_h|$ for all $h\in[H]$, allows us to show that $\learnlevel^{\exo}$ (\cref{alg:learnlevel2}) terminates in a polynomial number of iterations. 

\begin{lemma}[Bounding the number of test failures]
	\label{lem:testfailures_ex}
	Let $\delta,\veps\in(0,1)$ and $\zeta_{1:H}\in [0,1/2]$ be given, and suppose that \cref{ass:great} (weak correlation) holds with $\Ccor>0$. Let $f \in \cV$, be given, where $\cV$ is an arbitrary function class. Then, there is an event $\cE$ of probability at least $1-\delta$ under which the call to $\learnlevel^\exo_{0}(f,\cV^{H},\emptyset,\emptyset,0;\cV,\veps, \zeta_{1:H},\delta)$ (\cref{alg:learnlevel2}) terminates, and throughout the execution of $\learnlevel^\exo_{0}$, we have
	\begin{align}
		|\cC_h|\leq  \Mnum. \label{eq:one2}
	\end{align}
\end{lemma}
\begin{proof}[\pfref{lem:testfailures_ex}]
	The results follows from \cref{lem:pushforward} and \cref{lem:testfailures}.	
\end{proof}

\subsubsection{Value Function Regression Guarantee}
\label{sec:reg}
We next give a guarantee for the estimated value functions $\Vhat_{1:H}$ computed within $\learnlevel^\exo$ in \cref{line:updateQ2} of \cref{alg:learnlevel2}.
\begin{lemma}[Value function regression guarantee]
	\label{lem:confidencesets1_ex}
	Let $h\in [0 \ldotst H]$, $\delta, \veps\in (0,1)$, and $\zeta_{1:H}\in [0,1/2]$ be given, and consider a call to $\learnlevel^\exo_0$ in the setting of \cref{lem:testfailures_ex}. Further, let $\cV$ be defined as in \cref{eq:funclass}, and assume that $\Phi$ satisfies \cref{ass:phistar}. Then, for any \emph{endogenous} policy $\pi$ in $\Pim$, there is an event $\cE_h''$ of probability at least $1 -\delta/H$ under which for any $k\geq 1$, if 
	\begin{itemize}
		\item $\learnlevel^\exo_h$ gets called for the $k$th time during the execution of $\learnlevel^\exo_0$; and
		\item this $k$th call terminates and returns $(\Vhat_{h:H}, \cVhat_{h:H}, \cC_{h:H}, \cB_{h:H}, t_{h:H})$,
                \end{itemize}
                then {if $(\pibell_\tau)_{\tau\geq h}$ is the policy induced by $\Vhat_{h:H}$ and $N_\reg$ is set as in \cref{alg:learnlevel2}}, we have 
	\begin{align}
		& \sum_{(x_{h-1},a_{h-1})\in \cC_h}\frac{1}{N_\reg}\sum_{(x_h,-)\in \cD_h(x_{h-1}, a_{h-1})} \left( \Vhat_{h}(x_h)-V^{\pi}_{h}(x_{h})\right)^2\nn \\
		& \leq \frac{9 k H^2\log(8k^2H|\cV|/\delta)}{N_\reg}  +{8 H^2} \sum_{(x_{h-1},a_{h-1})\in \cC_h}\sum_{\tau=h}^H \E^{\pibell}\left[\tv{\pibell_\tau(\x_\tau)}{\pi_\tau(\x_\tau)}  \mid \x_{h-1}=x_{h-1},\a_{h-1}=a_{h-1}\right],
	\end{align}
	where the datasets $\{\cD_h(x, a): (x,a)\in \cC_h\}$ are as in the definition of $\cVhat_h$ in \eqref{eq:confidence2}.
\end{lemma}
\begin{proof}[\pfref{lem:confidencesets1_ex}]
  Since $\Phi$ satisfies \cref{ass:phistar}, the function class $\cV = \cV_{1:H}$ satisfies $V^\pi$-realizability for all endogenous policies $\pi$ (see \cref{lem:realex}). Thus, the proof of \cref{lem:confidencesets1_ex} follows from that of \cref{lem:confidencesets1} (see \cref{proof:confidencesets1}).
\end{proof}

\subsubsection{Confidence Sets}
\label{sec:confidencestes}
We now state a version of the confidence set validity lemma (\cref{lem:confidencesets}) that supports the ExBMDP setting. 
\begin{lemma}[Confidence sets]
	\label{lem:confidencesets_ex}
	Let $\veps\in(0,1/2)$ and $\zeta_{1:H}\subset [0,1/2]$ be given, and suppose that 
	\begin{itemize} 
		\item \cref{ass:great} holds with $\Ccor>0$;
		\item The decoder class $\Phi$ satisfies \cref{ass:phistar};
		\item $\zeta_{1:H} \in\cE_\texttt{rand}$, where $\cE_\texttt{rand}$ is the event in \cref{lem:bench}. 
	\end{itemize}
	Let $f \in \cV$ be arbitrary. There is an event $\cE'''$ of probability at least $1-3\delta$ under which a call to $\learnlevel^\exo_0(f,\cV,\emptyset,\emptyset,0;\cV,\veps,\zeta_{1:H},\delta)$ terminates and returns tuple $(\Vhat_{1:H}, \cVhat_{1:H}, \cC_{1:H}, \cB_{1:H}, t_{1:H})$ such that
	\begin{align}
		\forall h \in [H], \quad  V^{\pibar}_h \in \cVhat_h,
	\end{align}
	where $\pibar_{1:H}$ is the policy defined recursively via
	\begin{align}
		\pibar_\tau(x) \in \argmax_{a\in \cA}
		\ceil{\cP_\tau[V^{\pibar}_{\tau+1}](x,a)/\veps + \zeta_h },\quad \text{for } \tau = H,\dots,1.
		\label{eq:pibar_ex}
	\end{align} 
\end{lemma} 
While the proof of this lemma is very similar to that of \cref{lem:confidencesets}, we need a dedicated treatment to handle the rounding in $\learnlevel^\exo$. The fully proof of \cref{lem:confidence_ex} is in \cref{proof:confidencesets}.

\subsubsection{Main Guarantee for $\learnlevel^\exo$}
We now state the central technical guarantee for \mainalge, \cref{thm:lbc2}, which shows that the base invocation of the algorithm returns a set of value functions $\Vhat_{1:H}$ that induce a near-optimal policy $\pihat$, as long as the randomized rounding parameters $\zeta_{1:H}$ satisfy  $\zeta_{1:H} \in \cE_\texttt{rand}$, where $\cE_\texttt{rand}$ is the success event in \cref{lem:bench}. The proof of the theorem is in \cref{proof:lbc2}. 
\begin{lemma}[Main guarantee for \mainalge{}]
\label{thm:lbc2}
Let $\delta,\veps\in (0,1)$ and $\zeta_{1:H}\subset [0,\frac{1}{2}]$ be given, and suppose that
\begin{itemize} 
	\item \cref{ass:great} holds with $\Ccor>0$;
	\item The decoder class $\Phi$ satisfies \cref{ass:phistar};
	\item $\zeta_{1:H} \in \cE_\texttt{rand}$, where $\cE_\texttt{rand}$ is the event in \cref{lem:bench}. 
\end{itemize}
Then, for any $f \in \cV$, with probability at least $1-5\delta$, a call to $\learnlevel^\exo_{0}(f,\cV^{H},\emptyset,\emptyset,0;\cV,\veps,\zeta_{1:H},\delta)$ (\cref{alg:learnlevel2}) terminates and returns value functions $\Vhat_{1:H}$ such that 
\begin{align}
\forall h \in [H], \quad 	\bbP^{\pibell}\brk*{ \bm\pibell_{h}(\x_h) \neq  \pibar_h(\x_{h})
	} \leq  \frac{\veps^2}{4H^3 S A \Ccor}, \label{eq:returned22}
\end{align}
where $\bm\pibell_h(x) \in \argmax_{a\in \cA}  \ceil{\Phat_{h,\veps,\delta'}[\Vhat_{h+1}](x,a)/\veps + \zeta_h}$, for all $h\in\brk{H}$, $\pibar$ is defined as in \cref{eq:policies}, and $\delta'$ is defined as in \cref{alg:learnlevel2}. Furthermore, the number of episodes used by $\learnlevel^\exo_0$ is bounded by 
\begin{align}
\wtilde{O}	\left(\Ccor^8 S^8 H^{10}A^9\cdot{}\veps^{-26}\right).
\end{align}
\end{lemma}

\subsubsection{Concluding: Main Guarantee for \forwardexo}
\label{proof:exbmdpforward_main}

To conclude, we prove \cref{thm:exbmdpforward_main}, which shows that \forwardexo succeeds with high probability.
Recall that \forwardexo (i) invokes \mainalge multiple times for random samples $\zeta_{1:H}$ to ensure that the success event for \cref{thm:lbc2} occurs for at least one invocation, and (ii) extracts an executable policy using behavior cloning. Regarding the former point, note that the probability of the success event of \cref{lem:bench} can be boosted by sampling i.i.d.~$\bzeta_{1:H}\ind{1},\dots, \bzeta_{1:H}\ind{n}\sim \P^\zeta$ inputs to $\learnlevel^\exo$ for $n\geq 1$; as long as $n$ is polynomially large, with high probability at least one of the inputs $\bzeta_{1:H}\ind{1},\dots, \bzeta_{1:H}\ind{n}$ will satisfy the conclusion of \cref{lem:bench}. Thus, it suffices to pick the policy with the highest value function among the different calls to $\learnlevel^\exo$. Using this, we prove \cref{thm:exbmdpforward_main}.
~\\

\newcommand{\opt}{\texttt{opt}}
\newcommand{\eval}{\texttt{eval}}

\begin{proof}[\pfref{thm:exbmdpforward_main}]
	Recall that \cref{alg:forward_exbmdp} picks the final policy $\pihat\ind{i_\texttt{opt}}_{1:H}$ based on empirical value function estimates. In particular, for every $i\in[N_\boost]$ (with $N_\boost$ as in \cref{alg:forward_exbmdp}), the estimate $\widehat{J}(\pihat\ind{i}_{1:H})$ for $J(\pihat\ind{i}_{1:H})$ is computed using $N_\eval$ episodes. Thus, by Hoeffding's inequality and the union bound, we have that there is an event $\breve\cE$ of probability at least $1-\delta/4$ under which
	\begin{align}
		\forall i \in [N_\boost], \quad |J(\pihat\ind{i}_{1:H}) - \widehat{J}(\pihat\ind{i}_{1:H})| \leq \sqrt{2\log(2N_\boost/\delta)/N_\eval}.
	\end{align}
	Therefore, by definition of $i_\opt$ in \cref{alg:forward_exbmdp}, we have that under $\breve\cE$: 
	\begin{align}
		\forall i \in [N_\boost],\quad J(\pihat\ind{i}_{1:H}) & \leq J(\pihat\ind{i_\opt}_{1:H})  + \sqrt{2\log(2N_\boost/\delta)/N_\eval}, \nn \\ 
		& \leq  J(\pihat\ind{i_\opt}_{1:H}) + \veps/8,
		 \label{eq:frrr}
	\end{align}
	where the last inequality follows by the choice of $N_\eval$ in \cref{alg:forward_exbmdp}. On the other hand, by \cref{lem:bench}, there is an event $\cE^{\texttt{success}}$ of $\P^\zeta$-probability at least
        \[1-(24 S A H\veps)^{N_\boost}\geq 1- \delta/4 \quad \text{(by the choice of $N_\boost$ in \cref{alg:forward_exbmdp})} \]
	 under which there exists $j \in [N_\boost]$ such that $\bzeta\ind{j}_{1:H}\in \cE_\texttt{rand}$, where $\cE_\texttt{rand}$ is defined as in \cref{lem:bench}. In what follows, we condition on the event $\cE^{\texttt{success}}$ and let $j\in [N_\boost]$ be such that $\bzeta\ind{j}_{1:H}\in \cE_\texttt{rand}$. Further, we use $\pihat^{\learnlevel}_{1:H}$ to denote the policy returned by the instance of $\learnlevel^\exo$ that is used by \forwardexo to learn $\pihat\ind{j}_{1:H}$.
	
	By \cref{prop:forward} (instantiated with $\veps_\texttt{mis}=0$), there is an event $\wtilde\cE'$ of probability at least $1-\delta/4$ under which the policy $\pihat\ind{j}$ produced by \forward{} satisfies
	\begin{align}
	J(\pihat^{\learnlevel}_{1:H}) -  J(\pihat\ind{j}_{1:H}) \leq  \frac{\veps}{2}.
	\label{eq:forwardpart2}
	\end{align}	
By \cref{thm:lbc2} and the fact that $\bzeta\ind{j}_{1:H}\in \cE_\texttt{rand}$, there is an event $\wtilde \cE$ of probability at least $1-\delta/2$ under which: 
	\begin{align}
		\forall h \in[H], \quad 
		\forall h \in [H], \quad 	\bbP^{\pibell}\brk*{ \bm\pibell_{h}(\x_h) \neq  \pibar_h(\x_{h})
	}  \leq   \frac{\veps^2_\learnlevel}{4 H^3 S A\Ccor} \leq  \frac{\veps}{4H^2},   \label{eq:secondevent2}
	\end{align}
where $\pibar(\cdot) \coloneqq \pibar(\cdot; \bzeta_{1:H}\ind{j}, \veps_\learnlevel)$ which is defined in \eqref{eq:policies}. 

Moving forward, we condition on $\breve\cE\cap \wtilde{\cE} \cap \wtilde{\cE}'$. By \cref{lem:perform} (the performance difference lemma), we have
	\begin{align}
		J(\bar{\pi})- J(\pihat^{\learnlevel}_{1:H}) & = \sum_{h=1}^H \E^{\pihat^{\learnlevel}}[Q^{\pib}_h(\x_h, \bar\pi_h(\x_h)) -  Q_h^{\bar\pi}(\x_h, \pihatb^{\learnlevel}_h(\x_h))],\nn \\&  \leq H\sum_{h=1}^H \bbP^{\pibell}\brk*{ \bm\pibell_{h}(\x_h) \neq  \pibar_h(\x_{h})
		} ,\nn \\
		& \leq \veps/4, \label{eq:perform2}
	\end{align}
	where the last inequality follows by \eqref{eq:secondevent2}. Now, by \cref{lem:endobench}, we have $\pibar \in \Pi_{2 \veps_\learnlevel}$, and so by \cref{lem:policysubopt3},
	\begin{align}
	J(\pistar) - J(\pibar) \leq 6 H \veps_\learnlevel \leq \veps/8, \label{eq:less}
	\end{align}
	where the last inequality follows by the choice of $\veps_\learnlevel$ in \cref{alg:forward_exbmdp}. Combining \eqref{eq:less} with \eqref{eq:frrr}, \eqref{eq:forwardpart2}, and \eqref{eq:perform2}, we get that
	\begin{align}
		J(\pistar)- J(\pihat_{1:H})= J(\pistar)- J(\pihat\ind{i_\opt}_{1:H})  & \leq \veps. \label{eq:performeeee}
	\end{align}
	Finally, by the union bound, we have $\P[\breve\cE \cap  \wtilde\cE\cap \wtilde\cE']\geq 1 - \delta$, and so the desired suboptimality guarantee holds with probability at least $1-\delta$. 
	
	\paragraph{Bounding the sample complexity} The sample complexity is dominated by the calls to $\learnlevel^\exo_0$ within \forwardexo (\cref{alg:forward_exbmdp}). Since \forwardexo calls $\learnlevel^\exo_0$ with suboptimality parameter $\veps_\learnlevel = \veps H^{-1}/48$, we get by \cref{thm:lbc2} that the total sample
        complexity is bounded by 
	\begin{align}
	\wtilde{O}	\left(\Cexo^8 S^8H^{36}A^9\cdot \veps^{-26} \right).
	\end{align}
\end{proof}

\subsection{Proof of \creftitle{lem:endobench} (Endogenous Benchmark Policies)}
\label{proof:endobench}
\begin{proof}[\pfref{lem:endobench}]
	Fix $\delta \in (0,1)$, $\veps \in (0,1/2)$, and $\zeta_{1:H}' \subset [0,1/2]$. We show via backward induction over $\ell = H+1, \dots, 1$ that $\pibar_{\tau}(\cdot; \zeta'_{1:H},\veps)$ is endogenous for all $\tau \in[\ell \ldotst H+1]$, with the convention that $\pibar_{H+1} =\piunif$. The base case holds trivially by convention.
	
	Fix $h \in [H]$ and suppose that the induction hypothesis holds for all $\ell \in [h+1\ldotst H+1]$. We show that it holds for $\ell=h$. First, by the induction hypothesis, $\pibar_{\ell}(\cdot; \zeta'_{1:H},\veps)$ is endogenous for all $\ell \in [h+1\ldotst H]$. Thus, there exists a function $f_{h+1}: \cS \rightarrow [0, H-h]$ such that 
	\begin{align}
		V^{\pibar}_{h+1}(x') = f_{h+1}(\phistar(x')), \quad \forall x' \in \cX.
	\end{align}
	Therefore, we have for all $(x,a)\in \cX \times \cA$:
	\begin{align}
		\cP_h[V^{\pibar}_{h+1}](x,a) & = r_h(x,a) +\E[f_{h+1}(\phistar(\x_{h+1})) \mid  \x_h = x,\a_h = a], \nn \\
		& =  r_h(x,a)  + \E[f_{h+1}(\bs_{h+1}) \mid  \x_h = x,\a_h = a], \nn \\
		& =  r_h(x,a)  + \E[f_{h+1}(\bs_{h+1}) \mid  \bs_h = \phistar(x),\a_h = a], \label{eq:endogenous}
	\end{align}
	where the last equality follows by the ExBMDP transition structure. \cref{eq:endogenous} together with the fact that the rewards are endogenous (by assumption) implies that there exists $g_h:\cS\times \cA\rightarrow [0,H-h+1]$ such that
	\begin{align}
		\forall (x,a)\in \cX \times \cA,\quad 	\cP_h[V^{\pibar}_{h+1}](x,a) = g_{h}(\phistar(x),a),
	\end{align}
	which in turn implies that $ x\mapsto \ceil{\cP_h[V^{\pibar}_{h+1}](x,a)/\veps + \zeta_h'}$ is only a function of $x$ through $\phistar(x)$ for all $a\in \cA$. Thus, $\pibar_h$ is an endogenous policy and the induction is completed.

        For the second claim, observe that for the functions $\wtilde{Q}_{1}, \cdots, \wtilde{Q}_H \in [0,H]^{\cX\times \cA}$ defined as
	\begin{gather}
		\forall h \in[H], \forall (x,a)\in \cX \times \cA, \quad 	\wtilde{Q}_{h}(x,a)  = \veps \cdot \ceil{\cP_h[V^{\pibar}_{h+1}](x,a)/\veps + \zeta_h'},
		\intertext{we have}
		\forall h \in[H], \quad \pibar_h(\cdot; \zeta_{1:H}',\veps) \in \argmax_{a\in \cA}\wtilde{Q}_{h}(\cdot,a) \quad \text{and}\quad   \|\wtilde{Q}_h- Q^{\pibar}_h\|_{\infty} \leq 2 \veps,
	\end{gather}
	which implies that $\pibar(\cdot;\zeta_{1:H}',\veps)\in \Pi_{2\veps}$.
\end{proof}

\subsection{Proof of \creftitle{lem:bench} (Snapping Probability)}
\label{proof:bench}
\begin{proof}[\pfref{lem:bench}]
	Fix $\veps \in (0,1)$ and $\delta \in (0,1/2)$. For $\tau\leq\ell\in [H]$, let $\P_{\tau:\ell}^{\zeta}$ denote the probability law of $\bzeta_\tau, \dots, \bzeta_\ell$. We also use the shorthand $\P_{\tau}^\zeta$ for $\P_{\tau:\tau}^\zeta$, for all $\tau \in[H]$. We show via backward induction over $\ell = H+1, \dots, 1$ that there exists an event $\cE_\ell$ of $\P^{\zeta}_{\ell:H}$-probability at least $1- 24 S A (H-\ell+1)  \veps$ under which for all $\wtilde{V} \in (\cX \times [H] \rightarrow [0,H])$: 
	\begin{align}
		\forall \tau \in [\ell \ldotst H] ,\quad  \pibb_\tau(\cdot; \wtilde{V}, \bzeta_{1:H},\veps, \delta) = \pibar_\tau(\cdot; \bzeta_{1:H},\veps),
	\end{align}
	with the convention that $\pibar_{H+1}\equiv \pib_{H+1}\equiv \pi_\unif$. We then set $\cE_\texttt{rand} = \cE_1$.
	
	The base case follows trivially by convention. 
	
	We now proceed with the inductive step. Fix $h\in[H]$ and suppose that the induction hypothesis holds for all $\ell \in[h+1\ldotst H]$. We show that it holds for $\ell=h$. Throughout, we condition on $\cE_{h+1}$. By definition of $\cE_{h+1}$, we have for all $\wtilde{V} \in (\cX \times [H]\rightarrow [0,H])$:
	\begin{align}
		\forall \ell  \in [h+1 \ldotst H] ,\quad  \pibar_\ell(\cdot; \bzeta_{1:H},\veps)= \pibb_\ell(\cdot; \wtilde{V},\bzeta_{1:H},\veps, \delta).
	\end{align}
	This implies that for all $\wtilde{V} \in (\cX \times [H]\rightarrow [0,H])$:
	\colt{
	\begin{align}
	&	\forall x \in \cX, \nn\ \\ 	 & \pibb_h(x;\wtilde{V}, \bzeta_{1:H},\veps, \delta) \in \argmax_{a\in \cA} \left\{ \begin{array}{ll}  \ceil{\bQhat_h(x,a)/\veps + \bzeta_h},  & \text{if }  \|\bQhat_h(x,\cdot) -\cP_h [V^{\pibar}_{h+1}](x,\cdot)\|_\infty \leq 4\veps^2, \\   
			\ceil{\cP_h[V^{\pibar}_{h+1}](x,a)/\veps + \bzeta_h}, & \text{otherwise}, \end{array}\right.
	\end{align}
}
	\arxiv{
	\begin{align}
		\forall x \in \cX,\ \	\pibb_h(x; \wtilde{V},\bzeta_{1:H},\veps, \delta) \in \argmax_{a\in \cA} \left\{ \begin{array}{ll}  \ceil{\bQhat_h(x,a)/\veps + \bzeta_h},  & \text{if }  \|\bQhat_h(x,\cdot) -\cP_h [V^{\pibar}_{h+1}](x,\cdot)\|_\infty \leq 4\veps^2, \\   
			\ceil{\cP_h[V^{\pibar}_{h+1}](x,a)/\veps + \bzeta_h}, & \text{otherwise}, \end{array}\right.
	\end{align}
}
	by the definition of $\pib_h$ in \eqref{eq:pibar}, where $\bQhat_h(\cdot,a) \coloneqq \Phat_{h,\veps, \delta} [\wtilde{V}_{h+1}](\cdot,a)$. From this, we see that to prove $\pibb_h(\cdot; \wtilde{V}, \bzeta_{1:H},\veps, \delta) = \pibar_h(\cdot; \bzeta_{1:H},\veps)$ for all $\wtilde{V}$, it suffices to show that for all $x\in \cX$ and $\wtilde{V}$,
	\colt{
	\begin{align}
		\argmax_{a\in \cA}  \ceil{\bQhat_h(x,a)/\veps + \bzeta_h}  &= 	\argmax_{a\in \cA}  \ceil{\cP_h[V^{\pibar}_{h+1}](x,a)/\veps + \bzeta_h},  \nn \\ \shortintertext{whenever}     | \bQhat_h(x,a)-\cP_h [V^{\pibar}_{h+1}](x,a)| & \leq 4\veps^2.
	\end{align}
	A sufficient condition for this to hold (for all $\wtilde{V}$ simultaneously) is:
	\begin{align}
	& 	\forall 	x\in\cX, \forall a\in\cA, \forall \delta \in [-4\veps^2,4\veps^2],\nn \\  &   \ceil*{(\cP_h [V^{\pibar}_{h+1}](x,a)+\delta)\cdot \veps^{-1} +\bzeta_h} =   \ceil*{\cP_h[V^{\pibar}_{h+1}](x,a) \cdot \veps^{-1}+\bzeta_h}, \label{eq:target}
	\end{align}
}
	\arxiv{
	\begin{align}
          \argmax_{a\in \cA}  \ceil{\bQhat_h(x,a)/\veps + \bzeta_h} = 	\argmax_{a\in \cA}  \ceil{\cP_h[V^{\pibar}_{h+1}](x,a)/\veps + \bzeta_h}, &  ~~\text{whenever} \ \  | \bQhat_h(x,a)-\cP_h [V^{\pibar}_{h+1}](x,a)| \leq 4\veps^2.
	\end{align}
	Observe that a sufficient condition for this to hold is that
	\begin{align}
			\forall 	x\in\cX, \forall a\in\cA, \forall \delta \in [-4\veps^2,4\veps^2],\quad    \ceil*{(\cP_h [V^{\pibar}_{h+1}](x,a)+\delta)\cdot \veps^{-1} +\bzeta_h} =   \ceil*{\cP_h[V^{\pibar}_{h+1}](x,a) \cdot \veps^{-1}+\bzeta_h}, \label{eq:target}
	\end{align}
}
	where $\delta$ represents all the possible values that the difference $\bQhat_h(x,a)-\cP_h [V^{\pibar}_{h+1}](x,a)$ is allowed to take. By \cref{lem:endobench}, we know that $\pibar$ is endogenous, and so there exists a function $g_h: \cS\times \cA \rightarrow [0,H-h+1]$ such that 
	\begin{align}
		\forall x\in \cX,a\in\cA, \quad	\cP_h [V^{\pibar}_{h+1}](x,a) = g_h(\phistar(x),a).
	\end{align}
        Toward proving \cref{eq:target}, observe that for any $(s,a)\in \cS\in \cA$, if $\bzeta_h$ is such that 
	\begin{equation}
		\begin{aligned}
			g_h(s,a)/\veps  + \bzeta_h + 4\veps  &\leq  	\ceil{g_h(s,a)/\veps+\bzeta_h}, \\
			\text{and}\qquad 		g_h(s,a)/\veps  + \bzeta_h - 4\veps & >   	\ceil{g_h(s,a)/\veps+\bzeta_h} - 1, 
		\end{aligned}
		\label{eq:condition}
	\end{equation}
	then, for all $\delta \in[-4\veps^2,4\veps^2]$ and all $x\in \cX$ such that $\phistar(x)=s$, we have 
	\colt{
	\begin{align}
		\ceil{(\cP_h [V^{\pibar}_{h+1}](x,a)+\delta) /\veps+\bzeta_h} =	\ceil{(g_h(s,a)+\delta)/\veps+\bzeta_h}  &=\ceil{g_h(s,a)/\veps+\bzeta_h}, \nn \\ &  = 	\ceil{\cP_h [V^{\pibar}_{h+1}](x,a) /\veps+\bzeta_h}. \label{eq:thisonee}
	\end{align}
}
	\arxiv{
	\begin{align}
		\ceil{(\cP_h [V^{\pibar}_{h+1}](x,a)+\delta) /\veps+\bzeta_h} =	\ceil{(g_h(s,a)+\delta)/\veps+\bzeta_h}  =\ceil{g_h(s,a)/\veps+\bzeta_h} = 	\ceil{\cP_h [V^{\pibar}_{h+1}](x,a) /\veps+\bzeta_h}. \label{eq:thisonee}
	\end{align}
}
	Therefore, if we let $\cE_h(s,a)$ denote the event in \eqref{eq:condition}, then under $\bigcap_{(s,a)\in \cS\times \cA}\cE_h(s,a)$, the desired condition in \eqref{eq:target} holds. At this point, setting $\cE_h = (\bigcap_{(s,a)\in \cS\times \cA} \cE_h(s,a)) \cap \cE_{h+1}$ would complete the induction as long as $\P^\zeta_{h:H}[\cE_h] \geq 1 - 24 SA(H-h+1)\veps$. We now show that this is indeed the case by bounding the probability of the event $\bigcap_{(s,a)\in \cS\times \cA}\cE_h(s,a)$. By the union bound, we have 
	\begin{align}
		\P^\zeta_{h:H}\left[\bigcap_{(s,a)\in \cS\times \cA}\cE_h(s,a) \mid \cE_{h+1}\right] \geq 1 - \sum_{(s,a)\in \cS\times \cA}  \P^\zeta_{h:H}\left[\cE_h(s,a)^c \mid \cE_{h+1} \right], \label{eq:theunionbound}
	\end{align}
	where $\cE_h(s,a)^c$ denotes the complement of $\cE_h(s,a)$. We now bound the probability \[\P^\zeta_{h:H}\left[\cE_h(s,a)^c \mid \cE_{h+1} \right].\] Fix $(s,a)\in \cS\times \cA$. We have that $\bzeta_h \in \cE(s,a)^c$ if and only if 
	\begin{equation}
		\begin{aligned}
			g_h(s,a)/\veps  + \bzeta_h + 4\veps  &> 	\ceil{g_h(s,a)/\veps+\bzeta_h}, \\
			\text{or} \qquad 
			g_h(s,a)/\veps  + \bzeta_h - 4\veps  &\leq   	\ceil{g_h(s,a)/\veps+\bzeta_h} - 1. 
		\end{aligned}
		\label{eq:conditioned}
	\end{equation}
	Now, since $\bzeta_h \in[0,1/2]$, \cref{lem:onlyif} (instantiated with $(x,\zeta, \nu) =(g_h(s,a)/\veps,\bzeta_h,4\veps)$) implies that  \eqref{eq:conditioned} holds only if 
	\begin{align}
		\ceil{g_h(s,a)/\veps} -4 \veps  \leq   g_h(s,a)/\veps  + \bzeta_h    \leq   	\ceil{g_h(s,a)/\veps}  + 4\veps \quad   
		\text{or}  \quad   0  \leq  \bzeta_h  \leq  4 \veps. 
		\label{eq:conditioned2}
	\end{align}
	Further, note that since $\bzeta_h$ is uniformly distributed over $[0,1/2]$, the $\P^\zeta_h$-probability of the event in \eqref{eq:conditioned2} is at most the sum of the lengths of the intervals \[[\ceil{g_h(s,a)/\veps}- g_h(s,a)/\veps-4\veps, \ceil{g_h(s,a)/\veps}- g_h(s,a)/\veps+4\veps] \quad \text{and} \quad  [0,4\veps],\] multiplied by 2, which is equal to $24\veps$. Therefore, we have 
	\begin{align}
	& 	\P^\zeta_{h:H}\left[\cE_h(s,a)^c \mid \cE_{h+1} \right] \nn \\  & \leq \P^\zeta_h\left[\ceil{g_h(s,a)/\veps} - 4\veps  \leq   g_h(s,a)/\veps  + \bzeta_h    \leq   	\ceil{g_h(s,a)/\veps}  + 4\veps \;\;
                                                                                \text{or} \;\;  0  \leq  \bzeta_h  \leq   4\veps\right] \leq 24 \veps. 
	\end{align}
	Combining this with \eqref{eq:theunionbound}, we obtain
	\begin{align}
		\P^\zeta_{h:H}\left[\bigcap_{(s,a)\in \cS\times \cA}\cE_h(s,a) \mid \cE_{h+1}\right]  \geq 1 - \sum_{(s,a)\in \cS\times \cA}  \P^\zeta_{h:H}\left[\cE_h(s,a)^c \mid \cE_{h+1} \right]\geq 1 - 24 S A \veps.  
	\end{align}
	Thus, by setting $\cE_h = (\bigcap_{(s,a)\in \cS\times \cA} \cE_h(s,a)) \cap \cE_{h+1}$, we get that 
	\begin{align}
		\P^{\zeta}_{h:H}[\cE_h] \geq \P^{\zeta}_{h+1:H}[\cE_{h+1}] \cdot  \P^{\zeta}_{h:H}[\cE_h\mid \cE_{h+1}]  & \geq (1-24 S A (H-h)\veps) (1-24 S A \veps), \nn \\
		& \geq 1 - 24 SA(H-h+1)\veps,
	\end{align} 
	which completes the induction.
\end{proof}

\subsection{Proof of \creftitle{lem:pushforward} (Coverability in Weakly Correlated ExBMDP)}
\label{proof:pushforward}

\begin{proof}[\pfref{lem:pushforward}]
	Fix $h\in[2\ldotst H]$ and define the measure $\mu$ as
	\begin{align}
		\mu(x) \coloneqq \sum_{\xi' \in \Xi} q(x' \mid (\phistar_h(x'), \xi')) \cdot \P[\bxi_h = \xi'] \cdot \P[\bs_h = \phistar_h(x')\mid \bs_{h-1}=\phistar_{h-1}(x),\a_{h-1}=a],
	\end{align}
	for all $h \in[H]$ and $x \in   \cX$. We show that $\mu$ satisfies \cref{ass:pushforward} with $\Cpush = \Ccor \cdot S A$. First, note that $\mu$ is indeed a probability measure over $\cX$. Fix $(x,a,x')\in \cX\times \cA \times \cX$. We have
	\begin{align}
		& \P[\x_h = x' \mid \x_{h-1}=x,\a_{h-1}=a] \nn \\
		&=  \frac{\P[\x_h = x', \x_{h-1}=x\mid \a_{h-1}=a]}{\P[\x_{h-1}=x\mid \a_{h-1}=a]}, \nn \\
		&=  \frac{\P[\x_h = x', \x_{h-1}=x\mid \a_{h-1}=a]}{\P[\x_{h-1}=x]}, \nn \\
		&=  \frac{\sum_{\xi,\xi'\in \Xi} q(x' \mid (\phistar_h(x'), \xi'))\cdot q(x \mid (\phistar_{h-1}(x), \xi))\cdot \P[\bs_h = \phistar_h(x'), \bxi_{h}=\xi',\bs_{h-1}=\phistar_{h-1}(x), \bxi_{h-1}=\xi \mid \a_{h-1}=a]}{\sum_{\xi\in \Xi} q(x \mid (\phistar_{h-1}(x), \xi))\cdot \P[\bs_{h-1}=\phistar_{h-1}(x), \bxi_{h-1}=\xi]}, \nn \\
		\intertext{and so by the ExBMDP structure:}
		& = \frac{\sum_{\xi, \xi' \in \Xi} q(x' \mid (\phistar_h(x'), \xi'))\cdot q(x \mid (\phistar_{h-1}(x), \xi)) \cdot \P[\bxi_h = \xi',  \bxi_{h-1} = \xi]  \cdot \P[\bs_h = \phistar_h(x')\mid \bs_{h-1}=\phistar_{h-1}(x),\a_{h-1}=a]}{\sum_{\xi \in \Xi} q(x \mid (\phistar_h(x), \xi)) \cdot \P[\bxi_{h-1} = \xi]}, \nn \\
		\intertext{and by \cref{ass:great}}
		& \leq \Ccor \frac{\sum_{\xi, \xi' \in \Xi} q(x' \mid (\phistar_h(x'), \xi'))\cdot q(x \mid (\phistar_{h-1}(x), \xi)) \cdot \P[\bxi_h = \xi']\cdot \P[\bxi_{h-1} = \xi]  \cdot \P[\bs_h = \phistar_h(x')\mid \bs_{h-1}=\phistar_{h-1}(x),\a_{h-1}=a]}{\sum_{\xi \in \Xi} q(x \mid (\phistar_h(x), \xi)) \cdot \P[\bxi_{h-1} = \xi]}, \nn \\
		& = \Ccor \sum_{\xi' \in \Xi} q(x' \mid (\phistar_h(x'), \xi')) \cdot \P[\bxi_h = \xi']  \cdot \P[\bs_h = \phistar_h(x')\mid \bs_{h-1}=\phistar_{h-1}(x),\a_{h-1}=a], \nn \\
		& = \Ccor S A \cdot \mu(x'),
	\end{align}
 This completes the proof. 
\end{proof}

\subsection{Proof of \creftitle{lem:confidencesets_ex} (Confidence Sets)}
\label{sec:confidencesets}
To prove \cref{lem:confidencesets_ex}, we need the following consequence of tests in \cref{line:test2} passing for all $\ell \in [h+1\ldotst H]$. 
\begin{lemma}[Consequence of passed tests]
	\label{lem:confidence_ex}
	Let $h\in [0 \ldotst H]$, $\veps>0$, and $\zeta_{1:H}\in[0,1/2]$ be given and consider a call to $\learnlevel^\exo_0$ (\cref{alg:learnlevel2}) in the setting of  \cref{lem:testfailures_ex}. Further, let $\cE$ be the event of \cref{lem:testfailures_ex}. There exists an event $\cE_h'$ of probability at least $1- \delta/H$ such that under $\cE \cap \cE'_h$, if a call to $\learnlevel^\exo_h$ during the execution of $\learnlevel^\exo_0$ terminates and returns $(\Vhat_{h:H}, \cVhat_{h:H}, \cC_{h:H}, \cB_{h:H}, t_{h:H})$, then for any $(x_{h-1},a_{h-1}) \in \cC_h$ and $\ell\in [h+1 \ldotst H+1]$:  
	\begin{align}
		\label{eq:test_passed_lem_ex}
		&	\bbP^{\pibell}\brk*{\sup_{f\in \cVhat _\ell} \max_{a\in \cA} \left| (\cP_{\ell-1}[\Vhat_{\ell}]-\cP_{\ell-1}[f_\ell])(\x_{\ell-1},a) \right|  > 3\veps^2
			\mid \x_{h-1}= x_{h-1},\a_{h-1}=  a_{h-1}}  \leq \testbound, 
	\end{align}where $(\pibell_\tau)_{\tau\geq h} \subset \Pims$, $M$, and $N_\test$ are as in $\learnlevel^\exo_h$ (\cref{alg:learnlevel2}).
\end{lemma} 
\begin{proof}[\pfref{lem:confidence_ex}]
	This is just a restatement of \cref{lem:confidence}, and the proof is exactly the same as the latter.
\end{proof}

We will also use \cref{lem:recurse2}; even though this result is stated in section for the $V^\pi$-realizable setting, it is also applicable to the ExBMDP variant of \learnlevel{} as it merely says something about the order in which the $(\learnlevel^\exo_h)$ instances are called. With this, we now prove \cref{lem:confidencesets_ex}.

\begin{proof}[Proof of \cref{lem:confidencesets_ex}]
	The proof is very similar to that of \cref{lem:confidence}, with differences to account for the ``coarsening'' of the learned and benchmark policies.

	We prove the desired result for $\cE'''\coloneqq \cE \cap \cE'_{1} \cap \cE_1'' \cap \dots \cap \cE_H' \cap \cE_H''$, where $\cE$, $(\cE_h')$, and $(\cE_h'')$ are the events in \cref{lem:testfailures_ex}, \cref{lem:confidence_ex}, and \cref{lem:confidencesets1_ex}, respectively. Throughout, we condition on $\cE'''$. First, note that by \cref{lem:testfailures_ex}, $\learnlevel^\exo_0$ terminates. Let $(\Vhat_{1:H},\cVhat _{1:H},
	\cC_{1:H},\cB_{1:H}, t_{1:H})$ be its returned tuple.
	
	We show via backward induction over $\ell=H+1,\dots, 1$, that
	\begin{align}
		V^{\pib}_\ell \in \cVhat_\ell, \label{eq:newthirt_ex}
	\end{align} 
	where $\pib_{1:H}$ is the stochastic policy defined recursively via
		\colt{
		\begin{align}
		& 	\forall x \in\cX, \nn \\
		& 	\pibb_\tau(x; \zeta_{1:H}, \veps, \delta) \in \argmax_{a\in \cA} \left\{ \begin{array}{ll}  \ceil{\bQhat_\tau(x,a)/\veps + \zeta_\tau},  & \text{if }  \| \bQhat_{\tau}(x,\cdot)- \cP_\tau[V^{\pib}_{\tau+1}](x,\cdot)\|_\infty \leq 4\veps^2, \\   
				\ceil{\cP_\tau[V^{\pib}_{\tau+1}](x,a)/\veps + \zeta_\tau}, & \text{otherwise}, \end{array}\right.
		\end{align} 
	}
	\arxiv{
	\begin{align}
		\forall x \in\cX, \ 	\pibb_\tau(x; \zeta_{1:H}, \veps, \delta) \in \argmax_{a\in \cA} \left\{ \begin{array}{ll}  \ceil{\bQhat_\tau(x,a)/\veps + \zeta_\tau},  & \text{if }  \| \bQhat_{\tau}(x,\cdot)- \cP_\tau[V^{\pib}_{\tau+1}](x,\cdot)\|_\infty \leq 4\veps^2, \\   
			\ceil{\cP_\tau[V^{\pib}_{\tau+1}](x,a)/\veps + \zeta_\tau}, & \text{otherwise}, \end{array}\right.
	\end{align} 
}
	for $\tau = H,\dots,1$, where $\bQhat_\tau(\cdot,a) \coloneqq \Phat_{\tau,\veps, \delta} [\Vhat_{\tau+1}](\cdot,a)$. Note that since $\zeta_{1:H} \in \cE_\texttt{rand}$ (for $\cE_{\texttt{rand}}$ is defined in \cref{lem:bench}), we have $\pib \equiv \pibar$, where $\pibar$ is as in \eqref{eq:pibar_ex}. Thus, instantiating the induction hypothesis with $\ell=h$ and using the definition of the confidence sets $(\cVhat_{\ell})$ in \eqref{eq:confidence2} together with $V^{\pib}_{h} \equiv V_h^{\pibar}$ (since $\pib \equiv \pibar$) implies the desired result. 
	
	\paragraph{Base case [$\ell=H+1$]} Holds trivially since $V^{\pi}_{H+1}\equiv 0$ for any $\pi \in \Pims$ by convention. 
	
	\paragraph{General case [$\ell\leq H$]} Fix $h\in[H]$ and suppose that \eqref{eq:newthirt_ex} holds for all $\ell\in[h+1\ldotst H+1]$. We show that this remains true for $\ell=h$. First, note that if $\learnlevel^\exo_h$ is never called during the execution of $\learnlevel^\exo_0$, then $\cVhat_h = \cV_h$, and so \eqref{eq:newthirt_ex} holds for $\ell=h$, since $\pib=\pibar$ is endogenous under $\zeta_{1:H}\in\cE_\texttt{rand}$, where $\cE_\texttt{rand}$ is the event in \cref{lem:bench}.
	
	Now, suppose that $\learnlevel^\exo_h$ is called at least once, and let $(\Vhat^+_{h:H},\cVhat^+_{h:H},
	\cC^+_{h:H},\cB^+_{h:H},t^+_{h:H})$ be the output of the last call to $\learnlevel^\exo_h$ throughout the execution of $\learnlevel^\exo_0$. Next, we show that \begin{align}(\Vhat^+_{h:H},\cVhat^+_{h:H},
		\cC^+_{h:H}) = (\Vhat_{h:H},\cVhat_{h:H},
		\cC_{h:H}).  \label{eq:same_ex}
	\end{align}
	The for-loop in \cref{line:forloop2} ensures that no instance of $(\learnlevel^\exo_\tau)_{\tau>h}$ can be called after the last call to $\learnlevel^\exo_h$ (see \cref{lem:recurse2}). Thus, the estimated value functions, confidence sets, and core sets for layers $h+1, \dots, H$ remain unchanged after the last call to $\learnlevel^\exo_h$; that is, \eqref{eq:same_ex} holds. Thus, by \cref{lem:confidence_ex}, and since we are conditioning on $\cE_{h+1:H}'$, we have that for all $(x_{h-1},a_{h-1}) \in \cC_h$ and $\ell \in [h+1 \ldotst H+1]$:
	\begin{align}
		\bbP^{\pibell}\brk*{\sup_{f\in \cVhat _\ell} \max_{a\in \cA} \left|( \cP_{\ell-1}[\Vhat_{\ell}]- \cP_{\ell-1}[f_\ell])(\x_{\ell-1},a) \right|  > 3\veps^2 \mid \x_{h-1}= x_{h-1},\a_{h-1}= a_{h-1}}  \leq  \testbound. \label{eq:whole_ex}
	\end{align}
	Now, by the induction hypothesis, we have $V^{\pib}_\ell\in \cVhat_\ell$, and so substituting $V^{\pib}_\ell$ for $f_\ell$ in \eqref{eq:whole_ex}, we get that for all $(x_{h-1},a_{h-1}) \in \cC_h$ and $\ell \in [h+1 \ldotst H+1]$:
	\begin{align}
		\bbP^{\pibell}\brk*{ \max_{a\in \cA} \left| (\cP_{\ell-1}[\Vhat_{\ell}]-  \cP_{\ell-1}[V^{\pib}_\ell])(\x_{\ell-1},a) \right|  > 3\veps^2 \mid \x_{h-1}= x_{h-1},\a_{h-1}=  a_{h-1}}  \leq  \testbound.
	\end{align}
	Therefore, by \cref{lem:tvdistance_ex} (instantiated with $\mu[\cdot] = \bbP^{\pibell}[\cdot \mid \x_{h-1}= x_{h-1},\a_{h-1}=a_{h-1}]$, $\tau = \ell-1$, $\veps'=\veps^2$, and $V_{\tau+1} = V^{\pib}_{\ell}$), we have that for all $(x_{h-1},a_{h-1}) \in \cC_h$ and $\ell \in [h+1 \ldotst H+1]$: 
	\begin{align}
		\En^{\pibell}\brk*{ \tv{\pibell_{\ell-1}(\x_{\ell-1})}{\pib_{\ell-1}(\x_{\ell-1})}  \mid \x_{h-1}=x_{h-1},\a_{h-1}=a_{h-1}}\leq \testbound + \delta'. \label{eq:above_ex}
	\end{align}
	Now, since $\pibar(\cdot; \zeta_{1:H},\veps)$ is endogenous and $\pib \equiv \pibar$ (thanks to $\zeta_{1:H} \in \cE_\texttt{rand}$), \cref{lem:confidencesets1_ex} (applied with $\pi=\pibar$) and the conditioning on $\cE_{h+1:H}''$ imply that:
	\begin{align}
		& \sum_{(x_{h-1},a_{h-1})\in \cC_h}\frac{1}{N_\reg}\sum_{(x_h,-)\in \cD_h(x_{h-1}, a_{h-1})} \left( \Vhat_{h}(x_h)-V^{\pib}_{h}(x_{h})\right)^2\nn \\
		& \leq \frac{9 kH^2\log(8k^2H|\cV|/\delta)}{N_\reg}  +8 H^2  \sum_{(x_{h-1},a_{h-1})\in \cC_h}\sum_{\tau=h}^H \E^{\pibell}\left[\tv{\pibell_\tau(\x_\tau)}{\pib_\tau(\x_\tau)}  \mid \x_{h-1}= x_{h-1},\a_{h-1}= a_{h-1}\right], \label{eq:happ_ex}
	\end{align}
	where the datasets $\{\cD_h(x, a): (x,a)\in \cC_h\}$ are as in the definition of $\cVhat_h$ in \eqref{eq:confidence2}. Combining \eqref{eq:happ_ex} with \eqref{eq:above_ex}, we conclude that 
	\begin{align}
		& \sum_{(x_{h-1},a_{h-1})\in \cC_h}\frac{1}{N_\reg}\sum_{(x_h,-)\in \cD_h(x_{h-1}, a_{h-1})} \left( \Vhat_{h}(x_h)-V^{\pib}_{h}(x_{h})\right)^2\nn\\
		& \leq \frac{9 M H^2\log(8M^2H|\cV|/\delta)}{N_\reg}  + 8 M H^3\cdot \testbound  + 8M H^3\delta', \nn \\
		& = \frac{9M  H^2\log(8M^2H|\cV|/\delta)}{N_\reg}  + 8 M H^3\cdot \testbound  + 8M H^3  \frac{\delta}{4M^7N_\test^2 H^8|\cV|},\nn \\
		& \leq   \veps_\reg^2, \label{eq:los_ex}
	\end{align}
        where we have used that $\abs{\cC_h}\leq{}M$.
	By the definition of $\cVhat_h$ in \eqref{eq:confidence2}, \eqref{eq:los_ex} implies that $V^{\pib}_h \in \cVhat_h$, which completes the induction.
\end{proof}

\subsection{Proof of \creftitle{thm:lbc2} (Main Guarantee of $\learnlevel^\exo$)}
\label{proof:lbc2}
\begin{proof}[\pfref{thm:lbc2}] We condition on the event $\wtilde{\cE}\coloneqq \cE \cap \cE'''\cap \cE'_1\cap \dots \cap \cE_H'$, where $\cE$, $\cE'''$, and $(\cE_h')$ are the events in \cref{lem:testfailures_ex}, \cref{lem:confidencesets_ex}, and \cref{lem:confidence_ex}, respectively. Note that by the union bound, we have $\P[\wtilde{\cE}] \geq  1 - 5 \delta$. By \cref{lem:confidencesets_ex}, we have that 
	\begin{align}
		\label{eq:test_passed33_ex}
		\forall h \in[H],\quad 	\bbP^{\pibell}\brk*{\sup_{f\in \cVhat _h} \max_{a\in \cA}\left|(\cP_{h-1}[\Vhat_{h}]-\cP_{h-1}[f_{h}])(\x_{h-1},a) \right|>{3\veps^2}}
		\leq\testbound,
	\end{align}
	where $M = \Mnum$ and $N_\test= \Ntestnum$. On the other hand, by \cref{lem:confidencesets_ex}, we have 
	\begin{align}
		\forall h \in[H], \quad 	V^{\pib}_h\in \cVhat_h.
		\label{eq:learnlevel_inductive33_ex}
	\end{align}
	Thus, substituting $V^{\pib}_h$ for $f_h$ in \eqref{eq:test_passed33_ex} we get that for all $h\in[H+1]$.
	\begin{align}
		\bbP^{\pibell}\brk*{\max_{a\in \cA}\left|(\cP_{h-1}[\Vhat_{h}]- \cP_{h-1}[V^{\pib}_h])(\x_{h-1},a)\right|   > {3\veps^2}
		}  \leq  \testbound.
	\end{align}
	This together with \cref{lem:tvdistance_ex}, instantiated with $\mu[\cdot] = \bbP^{\pibell}[\cdot]$; $\tau = h-1$; $V_{\tau+1} = V^{\pib}_{h}$; and $\delta = \delta'$ (with $\delta'$ as in \cref{alg:learnlevel3}), translates to:
	\begin{align}
		\forall h\in[H], \quad 	\bbE^{\pibell}\brk*{ \tv{\pib_h(\x_h)}{\pibell_h(\x_h)}
		}&\leq \testbound + \delta',  \nn \\
		&= \testbound + \frac{\delta}{4 M^7 N_\test^2 H^8 |\cV|}, \nn \\
		&\leq \frac{\veps^2}{4H^3 S A \Ccor}, \label{eq:thisoe}
	\end{align}
		where the last step follows from the fact that $N_\test = \Ntestnum$ (with $M$ as in \cref{line:paramsExBMDPM}). Now, since $\zeta_{1:H} \in \cE_\texttt{rand}$ (by assumption), \cref{lem:bench} implies that $\pib \equiv \pibar$, where the latter is the deterministic policy defined in \eqref{eq:policies}. Thus, by \eqref{eq:thisoe}, we have 
		\begin{align}
		\forall h \in [H], \quad \P^{\pihat}[\pibar_h(\x_h)\neq \pihatb_h(\x_h)] = \bbE^{\pibell}\brk*{ \tv{\pib_h(\x_h)}{\pibell_h(\x_h)}} \leq \frac{\veps^2}{4 H^3 S A \Ccor},
		\end{align}
		where the first equality follows by the fact that $\P[\pibar_h(x)\neq \pihatb_h(x)] =  \tv{\pib_h(x)}{\pibell_h(x)}$, for all $x\in \cX$, since $\pibar_h$ is deterministic.
	
	\paragraph{Bounding the sample complexity}
	We now bound the number of episodes used by \cref{alg:learnlevel2} under $\wtilde\cE$. First, we fix $h\in[H]$, and focus on the number of episodes used within a to call $\learnlevel^\exo_h$; excluding any episodes used by any recursive calls to $\learnlevel^\exo_\tau$ for $\tau >h$. We start by counting the number of episodes used to test the fit of the estimated value functions $\Vhat_{h+1:H}$. Starting from \cref{line:begin2}, there are for-loops over $(x_{h-1},a_{h-1})\in \cC_h$, $\ell=H, \dots, h+1$, and $n \in [N_\test]$ to collected partial episodes using the learned policy $\pihat$ in \cref{alg:learnlevel2}, where $N_\test = \Ntestnum$ and $M = \Mnum$. Note that $\pihat$ uses the local simulator and requires $N_\simu= 2 \log(4 M^7 N_\test^2 H^2 |\cV|/\delta)/\veps^2$ samples to output an action at each layer (since \cref{alg:learnlevel2} calls \cref{alg:Phat} with confidence level $\delta'=\deltaprime$). Also, note that whenever a test fails in \cref{line:test2}, the for-loop in \cref{line:begin3} resumes. We also know (by \cref{lem:testfailures_ex}) that the number of times the test fails in \cref{line:test2} is at most $M$. Thus, the number of times the for-loop in \cref{line:begin2} resumes is bounded by $H M$; here, $H$ accounts for test failures for all layers $\tau \in[h+1\ldotst H]$. Thus, the number of episodes required to between lines \cref{line:begin2} and \cref{line:draw2} is bounded by 
	\begin{align}
		\text{\# episodes for roll-outs} \leq	\underbrace{MH}_{\text{\# of times \cref{line:begin2} resumes}}\cdot  \underbrace{MH^2 N_\test N_\simu}_{\text{Number of episodes in case of no test failures}}. \label{eq:roll2}
	\end{align}
	Note that the test in \cref{line:test2} also uses local simulator access because it calls the operator $\Phat$ for every $a\in \cA$. Thus, the number of episodes used for the test in \cref{line:test2} is bounded by 
	\begin{align}
		\text{\# episodes for the tests} \leq	\underbrace{MH}_{\text{\# of times \cref{line:begin2} resumes}}\cdot  \underbrace{M H A N_\test N_\simu}_{\text{Number of episodes used in \cref{line:test2}}}. \label{eq:test2}
	\end{align}
	We now count the number of episodes used to re-fit the value function; \cref{line:forloop2} onwards. Note that starting from \cref{line:forloop2}, there are for-loops over $(x_{h-1},a_{h-1})\in \cC_h$ and $i\in [N_\reg]$ to generate $A \cdot \Nest(|\cC_h|)\leq A \cdot \Nest(M)$ partial episodes using $\pihat$, where $\Nest(k)=2N_\reg^2 \log(8 A N_\reg H k^3/\delta)$ is as in \cref{alg:learnlevel2}. And, since $\pihat$ uses local simulator access and requires $\Nest$ samples (see \cref{alg:Phat}) to output an action at each layer, the number of episodes used to refit the value function is bounded by 
	\begin{align}
		\text{\# episodes for $V$-refitting} \leq M N_\reg A \Nest(M) H N_\simu. \label{eq:refit2}
	\end{align}
	Therefore, by \eqref{eq:roll2}, \eqref{eq:test2}, and \eqref{eq:refit2}, the number of episodes used within a single call to $\learnlevel^\exo_h$ (not accounting for episodes used by recursive calls to $\learnlevel^\exo_\tau$, for $\tau >h$) is bounded by
	\begin{align}
		\text{\# episodes used locally within $\learnlevel^\exo_h$} & \leq  M^2 H (H + A) N_\test N_\simu  +M N_\reg A \Nest(M) H N_\simu. \label{eq:localtotal2}
	\end{align}
	Finally, by \cref{lem:testfailures_ex}, $\learnlevel^\exo_h$ may be called at most $M$ times throughout the execution of $\learnlevel^\exo_0$. Using this together with \eqref{eq:localtotal2} and accounting for the number of episodes from all layers $h\in[H]$, we get that the total number of episodes is bounded by 
	\begin{align}
		M^3 H^2 (H + A) N_\test N_\simu   +M^2 H^2 N_\reg A \Nest(M) N_\simu.
	\end{align}  
	Substituting the expressions of $M$, $N_\test $, $\Nest$, $N_\simu$, and $N_\reg$ from \cref{alg:learnlevel2} and \cref{alg:Phat}, we obtain the desired number of episodes, which concludes the proof.
\end{proof}

\section{Additional Technical Lemmas}

\begin{lemma}
	\label{lem:tvdistance}
	Let $\tau\in[H]$ and $\veps, \delta, \nu \in(0,1)$ be given. Consider two value functions $V_{\tau+1}, \Vhat_{\tau+1}\in [0, H]$ and a measure $\mu \in \Delta(\cX)$ such that 
	\begin{align}
		\bbP_{\x_\tau\sim\mu}\brk*{ \mathbb{I}\left\{ \max_{a\in \cA} \left| (\cP_{\tau}[\Vhat_{\tau+1}]- \cP_{\tau}[V_{\tau+1}])(\x_{\tau},a) \right|  > 3\veps      \right\}  }  \leq\nu.\label{eq:event}
	\end{align}
	Further, for $x\in \cX$, let $\pihatb_\tau(x) \in
	\argmax_{a\in \cA} \bQhat_\tau(x,a) \coloneqq
	\Phat_{\tau,\veps, \delta}[\Vhat_{\tau+1}](x,a)$ and
	inductively define a randomized policy $\pib$ via
	\begin{align}
		\pibb_{\tau}(x) \in \argmax_{a\in \cA} \left\{ \begin{array}{ll}  \bQhat_\tau(x,a) ,  & \text{if }  \| \bQhat_\tau(x,\cdot) -\cP_\tau [ V_{\tau+1}](x,\cdot)\|_\infty \leq 4\veps, \\   
			\cP_\tau[V_{\tau+1}](x,a), & \text{otherwise}. \end{array}\right. 
	\end{align}
	Then, we have 
	\begin{align}
		\E_{\x_\tau\sim\mu}[\tv{\pihat_\tau(\x_\tau)}{\pib_\tau(\x_\tau)}] \leq \nu +\delta.\label{eq:nu}
	\end{align}
\end{lemma}
\begin{proof}[\pfref{lem:tvdistance}]
	In this proof, we let $\P_\mu$ denote the probability law of $\x_\tau$ and $\P_{\cP}$ denote the probability law of $\Phat_{\tau,\veps, \delta}$. Denote by $\cE$ the $\P_\mu$-measurable event that $ \max_{a\in \cA}
	\left|(\cP_{\tau}[\Vhat_{\tau+1}]-
	\cP_{\tau}[V_{\tau+1}])(\x_{\tau},a) \right|  \leq 3\veps$. Fix $x\in \cE$, and let
	$\cE_x$ be the $\P_{\cP}$-measurable event that $\max_{a\in \cA} |\Phat_{\tau,\veps,
		\delta} [\Vhat_{\tau+1}](x,a)-
	\cP_\tau[V_{\tau+1}](x,a)|\leq 4\veps$. 
	From the definition of $\pib_\tau$, we have that 
	\begin{align}
		&	\tv{\pihat_\tau(x)}{\pib_\tau(x)} \nn \\ &  = \frac{1}{2} \sum_{a\in \cA}\left| \P_{\cP}\left[\pihatb_\tau(x)=a\right] - \P_{\cP}\left[\pibb_\tau(x)=a\right]  \right|, \nn \\
		&  =\frac{1}{2} \sum_{a\in \cA}\left| \P_{\cP}[\cE_x]  \P_{\cP}\left[\pihatb_\tau(x)=a\mid \cE_x \right] +  \P_{\cP}[\cE_x^c]  \P_{\cP}\left[\pihatb_\tau(x)=a\mid \cE_x^c \right]  - \P_{\cP}[\cE_x] \P_{\cP}\left[\pibb_\tau(x)=a\mid \cE_x\right] - \P_{\cP}[\cE_x^c] \P_{\cP}\left[\pibb_\tau(x)=a\mid \cE_x^c\right] \right|,\nn \\
		&  \leq  \frac{1}{2}\sum_{a\in \cA} \P_{\cP}[\cE_x] \cdot \left| \P_{\cP}\left[\pihatb_\tau(x)=a\mid \cE_x \right] - \P_{\cP}\left[\pibb_\tau(x)=a\mid \cE_x\right]  \right|\nn \\ & \quad   +\sum_{a\in \cA}\P_{\cP}[\cE_x^c]\cdot \left| \P_{\cP}\left[\pihatb_\tau(x)=a\mid \cE_x^c \right] - \P_{\cP}\left[\pibb_\tau(x)=a\mid \cE_x^c\right] \right|, \quad \text{(Jensen's inequality)} \nn \\
		\intertext{and since $\P_{\cP}\left[\pihatb_\tau(x)=a\mid \cE_x \right]=\P_{\cP}\left[\pibb_\tau(x)=a\mid \cE_x \right]$  $\forall a\in \cA$, we have that}
		&  =\frac{1}{2}  \sum_{a\in \cA}\P_{\cP}[\cE_x^c]\cdot \left| \P_{\cP}\left[\pihatb_\tau(x)=a\mid \cE_x^c \right] - \P_{\cP}\left[\pibb_\tau(x)=a\mid \cE_x^c\right] \right|,, \nn \\
		&  \leq  \P_{\cP}\left[\cE_x^c\right], \nn \\
		&  = \P_{\cP}\left[ \max_{a\in \cA} |\Phat_{\tau,\veps, \delta} [\Vhat_{\tau+1}](x,a)- \cP_\tau[V_{\tau+1}](x,a)|  > 4 \veps\right], \nn \\
		& \leq  \P_{\cP}\left[ \max_{a\in \cA} |\Phat_{\tau,\veps, \delta} [\Vhat_{\tau+1}](x,a)- \cP_\tau[\Vhat_{\tau+1}](x,a)|  >  \veps\right], \quad \text{(see below)} \label{eq:thebelow} \\
		& \leq \delta, \label{eq:delta}
	\end{align}
	where \eqref{eq:thebelow} follows from $x\in \cE$ and the last inequality follows from \cref{lem:phat}. Therefore, we have  
	\begin{align}
		\E_{\mu}[\tv{\pihat_\tau(\x_\tau)}{\pib_\tau(\x_\tau)}] & \leq \bbP_\mu[\cE] \cdot \E_{\mu}[\tv{\pihat_\tau(\x_\tau)}{\pib_\tau(\x_\tau)} \mid \cE] + \bbP_\mu[\cE^c],\nn \\
		& \leq   \delta +\nu,
	\end{align}
	where the first inequality follows by the fact that the total variation distance is bounded by 1, and the last inequality follows by \eqref{eq:event} and \eqref{eq:delta}.
\end{proof}
\begin{lemma}
	\label{lem:tvdistance_ex}
	Let $\tau\in[H]$ and $\veps', \delta,\nu\in(0,1)$, and $\zeta_{1:H}\in [0,1/2]$ be given. Further, consider two value functions $V_{\tau+1},\Vhat_{\tau+1}\in[0,H]$ and measure $\mu\in \Delta(\cX)$ such that 
	\begin{align}
		\bbP_{\x_\tau\sim\mu}\brk*{\mathbb{I}\left\{\max_{a\in \cA} \left| (\cP_{\tau}[\Vhat_{\tau+1}]-  \cP_{\tau}[V_{\tau+1}])(\x_{\tau},a) \right|>3\veps'\right\}}
		\leq\nu.
		\label{eq:event_ex}
	\end{align}
	Further, for $x\in \cX_\tau$, let $\pihatb_\tau(x) \in
	\argmax_{a\in \cA} \ceil{\bQhat_\tau(x,a)/\veps' +\zeta_\tau}$,
	where $\bQhat_\tau(x,a)\coloneqq  \Phat_{\tau,\veps',
		\delta}[\Vhat_{\tau+1}](x,a)$, and inductively define
	\begin{align}
		\pibb_{\tau}(x) \in \argmax_{a\in \cA} \left\{ \begin{array}{ll} \ceil{\bQhat_\tau(x,a)/\veps' + \zeta_\tau},  & \text{if }  \| \bQhat_\tau(x,\cdot) -\cP_\tau [ V_{\tau+1}](x,\cdot)\|_\infty \leq 4\veps', \\   
			\ceil{\cP_\tau[V_{\tau+1}](x,a)/\veps' + \zeta_\tau}, & \text{otherwise}. \end{array}\right. 
	\end{align}
	Then, we have 
	\begin{align}
		\E_{\x_\tau\sim\mu}[\tv{\pihat_\tau(\x_\tau)}{\pib_\tau(\x_\tau)}] \leq \nu +\delta.\label{eq:nu_ex}
	\end{align}
\end{lemma}
\begin{proof}[\pfref{lem:tvdistance_ex}]
	In this proof, we let $\P_\mu$ denote the probability law of $\x_\tau$ and $\P_{\cP}$ denote the probability law of $\Phat_{\tau,\veps, \delta}$.
	Denote by $\cE$ be the $\P_\mu$-measurable event that $\max_{a\in \cA} \left|
	(\cP_{\tau}[\Vhat_{\tau+1}]-
	\cP_{\tau}[V_{\tau+1}])(\x_{\tau},a) \right|  \leq 3\veps'$. Fix $x\in \cE$, and let $\cE_x$ be the $\P_{\cP}$-measurable event that
	$\max_{a\in \cA} |\Phat_{\tau,\veps', \delta}
	[\Vhat_{\tau+1}](x,a)- \cP_\tau[V_{\tau+1}](x,a)|\leq
	4\veps'$. From the definition of $\pib_\tau$, we have that
	\begin{align}
		&	\tv{\pihat_\tau(x)}{\pib_\tau(x)} \nn \\ &  = \frac{1}{2} \sum_{a\in \cA}\left| \P_{\cP}\left[\pihatb_\tau(x)=a\right] - \P_{\cP}\left[\pibb_\tau(x)=a\right] \right|, \nn \\
		&  =\frac{1}{2} \sum_{a\in \cA}\left| \P_{\cP}[\cE_x]  \P_{\cP}\left[\pihatb_\tau(x)=a\mid \cE_x \right] +  \P_{\cP}[\cE_x^c]  \P_{\cP}\left[\pihatb_\tau(x)=a\mid \cE_x^c \right]  - \P_{\cP}[\cE_x] \P_{\cP}\left[\pibb_\tau(x)=a\mid \cE_x\right] - \P_{\cP}[\cE_x^c] \P_{\cP}\left[\pibb_\tau(x)=a\mid \cE_x^c\right] \right|,\nn \\
		&  \leq  \frac{1}{2}\sum_{a\in \cA} \P_{\cP}[\cE_x] \cdot \left| \P_{\cP}\left[\pihatb_\tau(x)=a\mid \cE_x \right] - \P_{\cP}\left[\pibb_\tau(x)=a\mid \cE_x\right]  \right|\nn \\ & \quad   +\sum_{a\in \cA}\P_{\cP}[\cE_x^c]\cdot \left| \P_{\cP}\left[\pihatb_\tau(x)=a\mid \cE_x^c \right] - \P_{\cP}\left[\pibb_\tau(x)=a\mid \cE_x^c\right] \right|, \quad \text{(Jensen's inequality)} \nn \\
		\intertext{and since $\P_{\cP}\left[\pihatb_\tau(x)=a\mid \cE_x \right]=\P_{\cP}\left[\pibb_\tau(x)=a\mid \cE_x \right]$ $\forall a\in \cA$, }
		&  =\frac{1}{2}  \sum_{a\in \cA}\P_{\cP}[\cE_x^c]\cdot \left| \P_{\cP}\left[\pihatb_\tau(x)=a\mid \cE_x^c \right] - \P_{\cP}\left[\pibb_\tau(x)=a\mid \cE_x^c\right] \right|, \nn \\
		&  \leq  \P_{\cP}\left[\cE_x^c\right], \nn \\
		&  = \P_{\cP}\left[ \max_{a\in \cA} |\Phat_{\tau,\veps', \delta} [\Vhat_{\tau+1}](x,a)- \cP_\tau[V_{\tau+1}](x,a)|  > 4 \veps'\right], \nn \\
		& \leq  \P_{\cP}\left[ \max_{a\in \cA} |\Phat_{\tau,\veps', \delta} [\Vhat_{\tau+1}](x,a)- \cP_\tau[\Vhat_{\tau+1}](x,a)|  >  \veps'\right], \quad \text{(see below)} \label{eq:thebelow_ex} \\
		& \leq \delta, \label{eq:delta_ex}
	\end{align}
	where \eqref{eq:thebelow_ex} follows from $x\in \cE$ and the last inequality follows from \cref{lem:phat}. Therefore, we have  
	\begin{align}
		\En_{\x_\tau\sim\mu}[\tv{\pihat_\tau(\x_\tau)}{\pib_\tau(\x_\tau)}] & \leq \bbP_\mu[\cE] \cdot \E_{\x_\tau\sim\mu}[\tv{\pihat_\tau(\x_\tau)}{\pib_\tau(\x_\tau)} \mid \cE] + \bbP_\mu[\cE^c],\nn \\
		& \leq   \delta +\nu,
	\end{align}
	where the first inequality follows by the fact that the total variation is bounded by 1, and the last inequality follows by \eqref{eq:event_ex} and \eqref{eq:delta_ex}.
\end{proof}

\begin{lemma}
	\label{lem:onlyif}
	Let $x\in \reals$ and $\nu\in(0,1/2)$ be given. Further, let $\zeta \in (0,1/2)$. Then, 
	\begin{gather}
		x + \zeta +\nu > \ceil{x+\zeta} \quad \text{or} 	 \quad x + \zeta -\nu \leq  \ceil{x+\zeta}-1,
		\shortintertext{only if}
		\ceil{x}-\nu \le x + \zeta \leq \ceil{x} + \nu \quad \text{or} \quad \zeta\leq \nu. 
	\end{gather}
\end{lemma}
\begin{proof}[\pfref{lem:onlyif}]
	To prove the claim, it suffices to show the following items:
	\begin{enumerate}
		\item $x + \zeta +\nu > \ceil{x+\zeta}$ only if $\ceil{x} \geq x + \zeta > \ceil{x} - \nu$; and 
		\item $x + \zeta -\nu \leq  \ceil{x+\zeta}-1$ only if $\ceil{x} < x + \zeta \leq \ceil{x}+\nu$ or $\zeta\leq \nu$.
	\end{enumerate} 	
	We start by showing the first item. We proceed by showing the contrapositive; that is, we will show that if $x+\zeta \leq \ceil{x}-\nu$ or $x+\zeta> \ceil{x}$, then $x + \zeta +\nu \leq \ceil{x+\zeta}$. Suppose that $x+\zeta \leq \ceil{x}-\nu$. This, together with the fact that $\zeta \geq 0$, implies that 
	\begin{align}
		\ceil{x+\zeta} = \ceil{x} \geq x +\zeta +\nu. 
	\end{align}
	Now, suppose that $x+\zeta> \ceil{x}$. Then, we have 
	\begin{align}
		\ceil{x+\zeta} \geq \ceil{x}+1 \geq \ceil{x} +\zeta +\nu  \geq x + \zeta +\nu, 
	\end{align}
	where the penultimate inequality follows by $\zeta,\nu \in (0,1/2)$.
	
	We now prove the second claim. Again, we proceed by showing the contrapositive; that is, we will show that if $\crl{\text{$\ceil{x}+ \nu <  x + \zeta$ or $\ceil{x}\geq x+\zeta$}}$ and $\zeta> \nu$, then $x + \zeta -\nu >  \ceil{x+\zeta}-1$.
	
	Suppose that $\ceil{x}+ \nu < x + \zeta$ and $\zeta> \nu$. The first inequality together with $\nu\geq 0$ implies that $\ceil{x + \zeta} > \ceil{x}$. On the other hand, since $\zeta\leq 1/2$, we have $\ceil{x+\zeta} \leq \ceil{x}+1$, and so  
	\begin{align}
		\ceil{x+\zeta} -1 = \ceil{x} <  x + \zeta- \nu,
	\end{align}
	where the last inequality follows by the current assumption that $\ceil{x}+ \nu < x + \zeta$. 
	
	Now, suppose that $\ceil{x}\geq x+\zeta$ and that $\zeta >\nu$. Then, we have 
	\begin{align}
		\ceil{x+\zeta} \leq \ceil{x} \leq x +1 < x + \zeta - \nu +1,\label{eq:rearrange}
	\end{align}
	where the last inequality follows by $\zeta>\nu$. Rearranging \eqref{eq:rearrange} completes the proof. 
\end{proof} 
                        \section{\forward{} Algorithm and Analysis}		
                        \label{sec:forward}

In this section, we give a self-contained presentation and analysis
for the standard \emph{behavior cloning} algorithm for imitation learning
(e.g., \citet{ross2010efficient}), displayed in
\cref{alg:forward_generic}. Given access to trajectories from an expert policy
$\pihat_{1:H}$ (which may be non-executable in the sense of
\cref{def:executable}) the algorithm learns an executable policy
$\pi^\exe$ with similar performance. We use this scheme within \rvflF{} and \forwardexo{}.

\begin{algorithm}[htp]
	\caption{\forward{}: Imitation Learning Algorithm.}
	\label{alg:forward_generic}
	\begin{algorithmic}[1]
          \State {\bfseries input:} Policy class $\Pi\subseteq \Pim$, expert
          policy $\pihat_{1:H}$, suboptimality $\veps \in(0,1)$, and confidence $\delta \in(0,1)$.
		\State Set $N_\imit = 16 H^2 \log (|\Pi|/\delta)/\veps$. \label{line:paramsImit}
		\State Set $\cD \gets \emptyset$.
		\For{$i = 1,\dots, N_\imit$}
		\State Generate trajectory $\bm{\tau}=
                ((\x_1,\a_1),\dots, (\x_H, \a_H))\sim \P^{\pihat}$.
		\State Update $\cD \gets \cD \cup \{\bm{\tau} \}$.
		\EndFor
		\State Compute $\pi^\exe \in \argmin_{\pi\in \Pi}\sum_{((x_1,a_1),\dots, (x_H,a_H))\in \cD} \sum_{h\in[H]} \mathbb{I}\{ \bm{\pi}_h(x_h) \neq a_h\}$.
		\State Return $\pi^\exe$.
	\end{algorithmic}
\end{algorithm}

\begin{proposition}
	\label{prop:forward}
Let $\veps, \delta\in(0,1)$ be given and let $\Pi \subseteq \Pim$ and $\pihat_{1:H}$ be an expert policy such that 
\begin{align}
\inf_{\pi \in \Pi}\sum_{h=1}^H 	\P^{\pihat}[\pihatb_h(\x_h)\neq \bm{\pi}_h(\x_h)]\leq \veps_{\texttt{mis}}. \label{eq:mispecification}
	\end{align}
Then, the policy $\pi^\exe_{1:H} = \forward(\Pi, \veps, \pihat_{1:H}, \delta)$ returned by \cref{alg:forward_generic} satisfies, with probability at least $1-\delta$,
\begin{align}
J(\pihat)  - J({\pi^\exe})  \leq 4H \veps_{\texttt{mis}}  + \veps/2.
\end{align}
\end{proposition}
\begin{proof}[\pfref{prop:forward}]
First, by the performance difference lemma, we have 
\begin{align}
	\E[	V^{\pihat}_1(\x_1)]  - \E[V^{\pi^\exe}_1(\x_1)]  & = \sum_{h=1}^H \E^{\pihat}[Q^{\pi^\exe}_h(\x_h, \pihatb_h(\x_h)) -  Q_h^{\pi^\exe}(\x_h, \bm{\pi}^\exe_h(\x_h))],\nn \\
	& \leq {H}\sum_{h=1}^H \P^{\pihat}[ \pihatb_h(\x_h) \neq \bm{\pi}^\exe_h(\x_h)]. \label{eq:thirr}
\end{align}	
We now bound the probability terms on the right-hand side. Fix $h\in[H]$ and let $\cD$ be the dataset in \cref{alg:forward_generic}, which consists of $N_\imit$ i.i.d.~trajectories $((\x_1,\a_1),\dots,(\x_H,\a_H))$ generated by rolling with $\pihat_{1:H}$. By \cref{lem:corbern} (with i.i.d.~data, $B=
H$, and $\cQ = \Pi$), we have that, with probability at least $1-\delta$,
\begin{align}
\forall \pi \in \Pi, \ \  \sum_{((x_1,a_H),\dots, (x_H,a_H) )\in \cD} \sum_{h\in[H]} \mathbb{I}\{\bm{\pi}_h(x_h) \neq \pihatb(x_h)\} & \leq 2   \sum_{h\in[H]}  \P^{\pihat}[\bm{\pi}_h(\x_h) \neq \pihatb_h(\x_h)]\nn \\
& \quad + \frac{2H \log(2 |\Pi|/\delta)}{N_\imit},  \label{eq:firss}
\end{align}
and
\begin{align}
\forall \pi \in \Pi, \ \   \sum_{h\in[H]}  \P^{\pihat}[\bm{\pi}_h(\x_h) \neq \pihatb_h(\x_h)]  & \leq 2 \sum_{((x_1,a_H),\dots, (x_H,a_H) )\in \cD} \sum_{h\in[H]} \mathbb{I}\{\bm{\pi}_h(x_h) \neq \pihatb(x_h)\} \nn \\
& \quad + \frac{4H \log(2 |\Pi|/\delta)}{N_\imit}. \label{eq:seconn}
	\end{align}
Taking the infimum over $\pi$ on both sides of \eqref{eq:firss} and using the definition of $\pi^\exe_h$ in \cref{alg:forward_generic} gives:
\begin{align}
 \sum_{((x_1,a_H),\dots, (x_H,a_H) )\in \cD} \sum_{h\in[H]} \mathbb{I}\{\bm{\pi}^\exe_h(x_h) \neq \pihatb(x_h)\} & \leq 2 \inf_{\pi \in \Pi}  \sum_{h\in[H]}  \P^{\pihat}[\bm{\pi}_h(\x_h) \neq \pihatb_h(\x_h)]\nn \\
 & \quad + \frac{2H \log(2 |\Pi|/\delta)}{N_\imit}, \nn \\
 & \leq 2\veps_{\texttt{mis}} + \frac{2 H \log(2 |\Pi|/\delta)}{N_\imit},
\end{align}
where the last inequality follows from \eqref{eq:mispecification}. Using this together with \eqref{eq:seconn}, instantiated with $\pi \equiv \pi^\exe$, we get that with probability at least $1-\delta$:
\begin{align}
	\sum_{h\in[H]}  \P^{\pihat}[\bm{\pi}^\exe_h(\x_h) \neq \pihatb_h(\x_h)]  \leq 4 \veps_{\texttt{mis}} +  \frac{8H \log(2 |\Pi|/\delta)}{N_\imit}.
\end{align}
Plugging this into \eqref{eq:thirr}, we get that with probability at least $1-\delta$:
\begin{align}
	\E[	V^{\pihat}_1(\x_1)]  - \E[V^{\pi^\exe}_1(\x_1)] \leq 4H \veps_{\texttt{mis}} +  \frac{8H^2 \log(2 |\Pi|/\delta)}{N_\imit}\leq 4H \veps_{\texttt{mis}}  + \veps/2,
\end{align}
where the last inequality follows by the fact that $N_\imit = 16H^2 \log(2|\Pi|/\delta)/\veps$. This completes the proof.
\end{proof}

\end{document}